\documentclass{article}

\usepackage{times}

\usepackage[utf8]{inputenc}
\usepackage{amsmath}
\usepackage{amsthm}
\usepackage{amssymb}
\usepackage{amsfonts}
\usepackage{bbm}
\usepackage{graphicx}
\usepackage{natbib}
\usepackage{algorithm}
\usepackage{algorithmic}
\usepackage{array}
\usepackage{fullpage}
\usepackage{xspace}
\usepackage{subfigure}
\usepackage{dsfont}
\usepackage{hyperref}
\usepackage{multirow}
\newtheorem{theorem}{Theorem}
\newtheorem{lemma}{Lemma}
\newtheorem{proposition}{Proposition}
\newtheorem{corollary}{Corollary}

\newenvironment{proof-of-corollary}[1][{}]{\noindent{\bf
    Proof of Corollary {#1}}
  \hspace*{1em}}{\qed\smallskip\\}
\newenvironment{proof-of-proposition}[1][{}]{\noindent{\bf
    Proof of Proposition {#1}}
  \hspace*{1em}}{\qed\smallskip\\}
\newenvironment{proof-of-theorem}[1][{}]{\noindent{\bf
    Proof of Theorem {#1}}
  \hspace*{1em}}{\qed\smallskip\\}
\newenvironment{proof-of-lemma}[1][{}]{\noindent{\bf Proof of Lemma
    {#1}}
  \hspace*{1em}}{\qed\smallskip\\}

\makeatletter
\newcommand{\pushright}[1]{\ifmeasuring@#1\else\omit\hfill$\displaystyle#1$\fi\ignorespaces}
\newcommand{\pushleft}[1]{\ifmeasuring@#1\else\omit$\displaystyle#1$\hfill\fi\ignorespaces}
\makeatother

\def\bx{\mathbf x}

\def\blambda{\boldsymbol{\lambda}}
\def\bzero{\mathbf 0}
\def\bbmone{\mathds{1}}
\def\bq{\mathbf q}
\def\dev{\mbox{dev}}

\def\order{\ensuremath{\mathcal{O}}}
\renewcommand{\H}{\ensuremath{\mathcal{H}}}
\newcommand{\ztil}{\ensuremath{\tilde{Z}}}
\newcommand{\err}{\ensuremath{\mathrm{err}}}
\newcommand{\reg}{\ensuremath{\mathrm{reg}}}
\newcommand{\erri}[1]{\ensuremath{\err_{#1}}}
\newcommand{\erravg}[1]{\ensuremath{\overline{\err}_{#1}}}
\newcommand{\regi}[1]{\ensuremath{\reg_{#1}}}

\newcommand{\regbias}[1]{\ensuremath{\widetilde{\regi{#1}}}}
\newcommand{\regbiasi}[1]{\ensuremath{\reg^{\ddag}_{#1}}}
\def\dis{\ensuremath{\mathrm{DIS}}}
\def\prob{\ensuremath{\mathrm{Pr}}}
\def\query{\ensuremath{\mathrm{query}}}
\def\error{\ensuremath{\mathrm{error}}}
\def\auc{\ensuremath{\mathrm{AUC}}}
\def\aucg{\ensuremath{\mathrm{AUC\mbox{-}GAIN}}}
\def\median{\ensuremath{\mathrm{median}}}
\def\mean{\ensuremath{\mathrm{mean}}}

\newcommand{\alg}{\textsc{AC}\xspace}
\newcommand{\alglong}{\textsc{Active Cover}\xspace}
\newcommand{\cover}{\textsc{oac}\xspace}
\newcommand{\coverlong}{\textsc{Online Active Cover}\xspace}
\newcommand{\iwal}{\textsc{iwal}$_1$\xspace}
\newcommand{\iwalzero}{\textsc{iwal}$_0$\xspace}
\newcommand{\iwalora}{\textsc{ora-iwal}$_1$\xspace}
\newcommand{\iwalorazero}{\textsc{ora-iwal}$_0$\xspace}

\newcommand{\algrand}{\textsc{passive}\xspace}
\newcommand{\E}{\ensuremath{\mathbb{E}}}
\newcommand\optprob{\textsc{op}}
\newcommand{\pmin}[1][\epoch]{\ensuremath{P_{\min,{#1}}}}
\newcommand{\disind}[1]{\ensuremath{\mathcal{I}_{{#1}}^{\epoch}}}
\newcommand{\minp}{\ensuremath{\mu}}
\newcommand{\event}{\ensuremath{\mathcal{E}}}
\newcommand{\term}{\ensuremath{\mathcal{T}}}
\newcommand{\regconst}{\ensuremath{\eta}}

\newcommand{\slackconst}{\ensuremath{\xi}}
\newcommand{\sign}{\ensuremath{\mathrm{sign}}}
\newcommand{\coversize}{\ensuremath{l}}

\def\bc{\mathbf c}

\newcommand{\epoch}{\ensuremath{m}}
\newcommand{\epochi}[1][i]{\ensuremath{\epoch(#1)}}
\newcommand{\epend}[1][\epoch]{\ensuremath{\tau_{#1}}}
\newcommand{\epmax}{\ensuremath{M}}

\newcommand\parens[1]{(#1)}

\newcommand\braces[1]{\{#1\}}
\newcommand\brackets[1]{[#1]}

\newcommand\ind[1]{\ensuremath{\mathds{1}(#1)}}
\newcommand\dotp[1]{\langle #1 \rangle}

\newcommand{\otil}{\ensuremath{\tilde{\mathcal{O}}}}

\newcommand\Parens[1]{\left(#1\right)}

\newcommand\Braces[1]{\left\{#1\right\}}
\newcommand\Brackets[1]{\left[#1\right]}
\newcommand\Abs[1]{\left|#1\right|}

\def\ddefloop#1{\ifx\ddefloop#1\else\ddef{#1}\expandafter\ddefloop\fi}
\def\ddef#1{\expandafter\def\csname bb#1\endcsname{\ensuremath{\mathbb{#1}}}}
\ddefloop ABCDEFGHIJKLMNOPQRSTUVWXYZ\ddefloop
\def\ddef#1{\expandafter\def\csname c#1\endcsname{\ensuremath{\mathcal{#1}}}}
\ddefloop ABCDEFGHIJKLMNOPQRSTUVWXYZ\ddefloop

\newcommand\samperr{{\ensuremath{\varepsilon}}}
\newcommand{\dualobj}{{\cal D}}
\newcommand{\Psoln}{{P^*}}
\newcommand{\Psolnhatapx}{{\widehat{P}_\samperr^*}}
\newcommand{\lone}[1]{\ensuremath{\|{#1}\|_1}}

\newcommand\dataspace{\ensuremath{\mathcal{X}}}
\newcommand\dist{\ensuremath{\mathbb{P}}}
\newcommand\distx{\ensuremath{\dist_{\dataspace}}}
\newcommand{\Lset}{\ensuremath{\Lambda_{\varepsilon}}}
\newcommand{\Lsetbig}{\ensuremath{\Lambda_{\samperr/2}}}

\newcommand{\tsybamult}{\ensuremath{\zeta}}
\newcommand{\tsybaexp}{\ensuremath{\omega}}
\newcommand{\tsybatol}{\ensuremath{\varepsilon_0}}

\newcommand{\unlab}{\ensuremath{u}}
\DeclareMathOperator*{\argmin}{arg\,min}

\usepackage[textsize=tiny]{todonotes}

\newcolumntype{C}[1]{>{\centering\let\newline\\\arraybackslash\hspace{0pt}}m{#1}}

\begin{document}

\begin{center}
{\LARGE{{\bf{Efficient and Parsimonious Agnostic Active Learning}}}}

\vspace*{.2in}

\begin{tabular}{ccc}
  Tzu-Kuo Huang$^\dagger$
  &
  Alekh Agarwal$^\dagger$
  &
  Daniel J. Hsu$^\ddagger$\\
  \texttt{tkhuang@microsoft.com} &
  \texttt{alekha@microsoft.com} &
  \texttt{djhsu@cs.columbia.edu}
  \vspace*{.1in}\\
  \multicolumn{3}{c}{
  John Langford$^\dagger$
  \qquad \qquad \quad
  Robert E. Schapire$^\dagger$}\\
    \multicolumn{3}{c}{
  \texttt{jcl@microsoft.com} \qquad
  \texttt{schapire@microsoft.com}}
\end{tabular}

\vspace*{.2in}

\begin{tabular}{cc}
  Microsoft Research$^\dagger$ & Department of Computer Science$^\ddagger$\\
  New York, NY & Columbia University, New York, NY  
\end{tabular}
\end{center}

\begin{abstract}
We develop a new active learning algorithm for the streaming setting
satisfying three important properties: 1) It provably works for any
classifier representation and classification problem including those
with severe noise. 2) It is efficiently implementable with an ERM
oracle.  3) It is more aggressive than all previous approaches
satisfying 1 and 2. To do this we create an algorithm based on a newly
defined optimization problem and analyze it. We also conduct the first
experimental analysis of all efficient agnostic active learning
algorithms, evaluating their strengths and weaknesses in different
settings.
\end{abstract}

\section{Introduction}
\label{sec:intro}
How can you best learn a classifier given a label budget?  

Active learning approaches are known to yield exponential improvements
over supervised learning under strong assumptions~\citep{CohnAL94}.
Under much weaker assumptions, streaming-based agnostic active
learning~\citep{balcan2006agnostic,BeygelDL09,BeygelHLZ10,DasguptaHM07,zhang2014beyond}
is particularly appealing since it is known to work for \emph{any}
classifier representation and \emph{any} label noise distribution with
an i.i.d.~data source.\footnote{See the monograph
  of~\citet{Hanneke2014} for an overview of the existing literature,
  including alternative settings where additional assumptions are
  placed on the data source (e.g., \emph{separability}) as is common
  in other
  works~\citep{dasgupta2005coarse,balcan2007margin,balcan2013active}.%
} Here, a learning algorithm decides for each unlabeled example in
sequence whether or not to request a label, never revisiting this
decision.  Restated then: What is the best possible active learning
algorithm which works for \emph{any} classifier representation,
\emph{any} label noise distribution, and is computationally tractable?

Computational tractability is a critical concern, because most known
algorithms for this
setting~\citep[e.g.,][]{balcan2006agnostic,koltchinskii2010rademacher,zhang2014beyond}
require explicit enumeration of classifiers, implying
exponentially-worse computational complexity compared to typical
supervised learning algorithms.  Active learning algorithms based on
empirical risk minimization (ERM)
oracles~\citep{BeygelDL09,BeygelHLZ10,hsu2010algorithms} can overcome
this intractability by using passive classification algorithms as the
oracle to achieve a computationally acceptable solution.

Achieving generality, robustness, and acceptable computation has a
cost. For the above
methods~\citep{BeygelDL09,BeygelHLZ10,hsu2010algorithms}, a label is
requested on nearly \emph{every} unlabeled example where two
empirically good classifiers disagree.  This results in a poor label
complexity, well short of information-theoretic
limits~\citep{Castro2008} even for general robust
solutions~\citep{zhang2014beyond}.  Until now.

In Section~\ref{sec:algorithm}, we design a new algorithm \alglong
(\alg) for constructing query probability functions that minimize the
probability of querying inside the \emph{disagreement region}---the
set of points where good classifiers disagree---and never query
otherwise.  This requires a new algorithm that maintains a
parsimonious cover of the set of empirically good classifiers.  The
cover is a result of solving an optimization problem (in
Section~\ref{sec:optimization}) specifying the properties of a
desirable query probability function.  The cover size provides a
practical knob between computation and label complexity, as
demonstrated by the complexity analysis we present in
Section~\ref{sec:optimization}.

In Section~\ref{sec:results}, we provide our main results which
demonstrate that \alg effectively maintains a set of good classifiers,
achieves good generalization error, and has a label complexity bound
tighter than previous approaches. The label complexity bound depends
on the disagreement coefficient~\citep{Hanneke09}, which does not
completely capture the advantage of the algorithm. In
Section~\ref{sec:example}, we provide an example of a hard active
learning problem where \alg is substantially superior to previous
tractable approaches. Together, these results show that \alg is better
and sometimes substantially better in theory. The key aspects in the
proof of our generalization results are presented in
Section~\ref{proofs-regret}, with more technical details and label
complexity analysis presented in the appendix.

\emph{Do agnostic active learning algorithms work in practice?}  No
previous works have addressed this question empirically. Doing so is
important because analysis cannot reveal the degree to which existing
classification algorithms effectively provide an ERM oracle.  We
conduct an extensive study in Section~\ref{sec:experiments} by
simulating the interaction of the active learning algorithm with a
streaming supervised dataset. Results on a wide array of datasets 
show that agnostic active learning typically outperforms passive learning, 
and the magnitude of improvement depends on how carefully the active learning hyper-parameters 
are chosen. 

\section{Preliminaries}

Let $\H \subseteq \{\pm1\}^{\dataspace}$ be a set of binary
classifiers, which we assume is finite for simplicity.\footnote{%
  The assumption that $\H$ is finite can be relaxed to VC-classes
  using standard arguments.%
}
Let $\E_X[\cdot]$ denote expectation with respect to $X \sim \distx$,
the marginal of $\dist$ over $\dataspace$.
The \emph{expected error} of a classifier $h \in \H$ is \mbox{$\err(h) :=
  \Pr_{(X,Y) \sim \dist}(h(X) \neq Y)$}, and the error minimizer is
denoted by $h^* := \argmin_{h \in \H} \err(h)$.
The \emph{(importance weighted) empirical error} of $h \in \H$ on a
multiset $S$ of importance weighted and labeled examples drawn from
$\dataspace \times \{\pm1\} \times \bbR_+$ is \mbox{$\err(h,S) := \sum_{(x,y,w) \in S} w \cdot
\ind{h(x) \neq y} / |S|$}.
The \emph{disagreement region} for a subset of classifiers $A \subseteq
\H$ is
  $\dis(A)
  :=
  \braces{
    x \in \dataspace \mid
    \exists h,h' \in A \ \text{such that} \ h(x) \neq h'(x)
  }
  $.
The \emph{regret} of a classifier $h \in \H$ relative to another $h'
\in \H$ is $\reg(h,h') := \err(h) - \err(h')$, and the analogous
empirical regret on $S$ is $\reg(h,h',S) := \err(h,S) - \err(h',S)$.
When the second classifier $h'$ in (empirical) regret is omitted, it
is taken to be the (empirical) error minimizer in $\H$.

A streaming-based active learner receives i.i.d.~labeled examples
$(X_1,Y_1), (X_2,Y_2), \dotsc$ from $\dist$ one at a time; each label
$Y_i$ is hidden unless the learner decides on the spot to query it.
The goal is to produce a classifier $h \in \H$ with low error
$\err(h)$, while querying as few labels as possible.

In the IWAL framework~\citep{BeygelDL09}, a decision whether or not to
query a label is made \emph{randomly}: the learner picks a probability
$p \in [0,1]$, and queries the label with that probability. Whenever
$p>0$, an unbiased error estimate can be produced using inverse
probability weighting~\citep{horvitz1952generalization}.
Specifically, for any classifier $h$, an unbiased estimator $E$ of
$\err(h)$ based on $(X,Y) \sim \dist$ and $p$ is as follows: if $Y$ is
queried, then $E = \ind{h(X) \neq Y} / p$; else, $E = 0$.  It is easy
to check that $\E(E) = \err(h)$.
Thus, when the label is queried, we produce the importance weighted
labeled example $(X,Y,1/p)$.\footnote{%
  If the label is not queried, we produce an ignored example of weight
  zero; its only purpose is to maintain the correct count of querying
  opportunities.
  This ensures that $1/|S|$ is the correct normalization in
  $\err(h,S)$.%
  \label{footnote:noquery}%
}

\section{Algorithm}
\label{sec:algorithm}

\begin{algorithm}[t!]
  \caption{\alglong (\alg)}
  \label{alg:main}	
  \begin{algorithmic}[1]
    \renewcommand{\algorithmicrequire}{\textbf{input:}}
    \REQUIRE Constants $c_1, c_2, c_3$, confidence $\delta$, error
    radius $\gamma$, parameters $\alpha, \beta, \slackconst$ for
    $(\optprob)$, epoch schedule $0 = \epend[0] < 3 = \epend[1] <
    \epend[2] < \epend[3] < \ldots < \epend[\epmax]$ satisfying
    $\epend[\epoch+1] \leq 2\epend[\epoch]$ for $\epoch\geq1$.
    \renewcommand{\algorithmicrequire}{\textbf{initialize:}}
    \REQUIRE epoch $m = 0$, $\tilde{Z}_0 := \emptyset$, $\Delta_0 := c_1
    \sqrt{\epsilon_1} + c_2 \epsilon_1 \log 3$, where 
    \begin{equation*}
      \epsilon_\epoch \;:=\; \frac{32(\log(|\mathcal{H}|/\delta) + \log \epend)}{\epend}.
    \end{equation*}
    \FOR{$i = 4,\ldots,n,$} \IF{$i = \epend+1$} \STATE Set
    $\ztil_\epoch = \ztil_{\epoch-1} \cup S$, and $S = \emptyset$.
    \STATE Let
    \begin{eqnarray}
      h_{\epoch+1} &:=& \arg\min_{h \in \mathcal{H}} \; \mbox{err}(h,
      \tilde{Z}_{\epoch}),\label{alg:main:erm_update}\\ 
      \Delta_{\epoch} &:=& c_1 \sqrt{\epsilon_{\epoch}
        \err(h_{\epoch+1},\tilde{Z}_{\epoch})} + c_2 \epsilon_{\epoch} \log
      \epend, \label{eq:Delta}\\ 
      A_{\epoch+1} &:=& \{h \mid \mbox{err}(h,\tilde{Z}_{\epoch})  -
      \mbox{err}(h_{\epoch+1},\tilde{Z}_{\epoch}) \leq \gamma
      \Delta_{\epoch}\}. \label{eq:A} 
    \end{eqnarray}
    \STATE Compute the solution $P_{\epoch+1}(\cdot)$ to the
    optimization  \label{alg:main:P_update}
    problem \eqref{eq:queryp}.
    \STATE $\epoch := \epoch + 1$.
    \ENDIF
    \STATE Receive unlabeled data point $X_i$.
    \IF{$X_i \in D_{\epoch} := \mbox{DIS}(A_{\epoch}$),} \label{alg:main:query}
      \STATE Draw $Q_i \sim \mbox{Bernoulli}(P_\epoch(X_i))$. 
      \STATE Update the set of examples:\footnotemark
      \begin{eqnarray*}
        S :=& 
        \begin{cases}
          S \cup \{(X_i,Y_i, 1/P_\epoch(X_i))\}, & Q_i = 1\\
          S \cup \{X_i, 1, 0\},  & \mbox{otherwise}.
        \end{cases}
      \end{eqnarray*}
      \ELSE
      \STATE
      \centerline{$S \;\;:=\;\; S \cup \{(X_i, h_\epoch(X_i),1)\}.$}
      \ENDIF \label{alg:main:query_end}
      \ENDFOR
      \STATE $h_{\epmax+1} := \arg \min_{h \in \mathcal{H}} \; \mbox{err}(h, \tilde{Z}_{\epmax})$.
  \end{algorithmic}
\end{algorithm}
\footnotetext{%
  See Footnote~\ref{footnote:noquery}. Adding an example of importance
  weight zero simply increments $|S|$ without updating other state of
  the algorithm, hence the label used does not matter.%
}

Our new algorithm, shown in Algorithm~\ref{alg:main}, breaks the
example stream into epochs. 
The algorithm admits any epoch schedule so
long as the epoch lengths satisfy $\epend[\epoch-1] \leq 2\epend$. For
technical reasons, we always query the first 3 labels to kick-start
the algorithm. 
At the start of epoch~$\epoch$, \alg\ computes a
\emph{query probability function} $P_\epoch \colon \cX \to [0,1]$
which will be used for sampling the data points to query during the
epoch. This is done by maintaining a few objects of interest during
each epoch:

\begin{enumerate}
\item In step~\ref{alg:main:erm_update}, we compute the best
  classifier on the sample $\ztil_\epoch$ that we have collected so
  far. Note that the sample consists of the queried, true labels on some
  examples, while predicted labels for the others.
\item A radius $\Delta_\epoch$
is computed in step~\ref{eq:Delta} 
based on the desired level of concentration 
we want the various empirical quantities to satisfy.
\item The set $A_{\epoch+1}$ in step~\ref{eq:A} consists of all the
  hypotheses which are good according to our sample $\ztil_\epoch$,
  with the notion of good being measured as empirical regret being at
  most $\Delta_\epoch$.
\end{enumerate}
Within the epoch, $P_\epoch$ determines the probability of querying an
example in the disagreement region for this set $A_{\epoch}$ of ``good''
classifiers; examples outside this region are not queried but given
labels predicted by $h_\epoch$. Consequently, the sample is not
unbiased unlike some of the predecessors of our work. The various
constants in Algorithm~\ref{alg:main} must satisfy:

\begin{align}
  \alpha &\geq 1, ~~\regconst \geq 864,~~ \slackconst \leq \frac{1}{8
    n\epsilon_{\epmax}\log n}, ~~ \beta^2 \leq \frac{\eta}{864 \gamma
    n \epsilon_{\epmax} \log n}, ~~ \gamma \geq
  \regconst/4,\nonumber\\
  &\qquad c_1 \geq 2 \alpha \sqrt{6}, ~~ c_2
  \geq \regconst c_1^2 / 4, ~~ c_3 \geq 1.
  \label{eqn:constants}
\end{align}
%

\paragraph{Epoch Schedules:} The algorithm as stated takes an arbitrary
epoch schedule subject to \mbox{$\epend[\epoch] < \epend[\epoch+1]
  \leq 2\epend[\epoch]$}. Two natural extremes are unit-length epochs,
$\epend[\epoch] = \epoch$, and doubling epochs, $\epend[\epoch+1] =
2\epend[\epoch]$. The main difference comes in the number of times
$(\optprob)$ is solved, which is a substantial computational
consideration. Unless otherwise stated, we assume the doubling epoch
schedule so that the query probability and ERM classifier are
recomputed only $\order(\log n)$ times.

\paragraph{Optimization problem (\optprob) to obtain $P_\epoch$:}
\alg\ computes $P_\epoch$ as the solution to the optimization problem 
(\optprob). In essence, the problem encodes the properties of a query
probability function that are essential to ensure good generalization,
while maintaining a low label complexity. As we will discuss later,
some of the previous works can be seen as specific ways of
constructing feasible solutions to this optimization problem. The
objective function of $(\optprob)$ encourages small query
probabilities in order to minimize the label complexity. It might
appear odd that we do not use the more obvious choice for objective
which would be $\E_X[P(X)]$, however our choice simultaneously
encourages low query probabilities and also provides a barrier for the
constraint $P(X) \leq 1$--an important algorithmic aspect as we will
discuss in Section~\ref{sec:optimization}.

The constraints~\eqref{eq:queryp} in (\optprob) bound the variance in
our importance-weighted regret estimates for every $h \in \H$. This is
key to ensuring good generalization as we will later use Bernstein-style
bounds which rely on our random variables having a small variance. Let
us examine these constraints in more detail. The LHS of the
constraints measures the variance in our empirical regret estimates
for $h$, measured only on the examples in the disagreement region
$D_\epoch$. This is because the importance weights in the form of
$1/P_\epoch(X)$ are only applied to these examples; outside this
region we use the predicted labels with an importance weight of 1. The
RHS of the constraint consists of three terms. The first term ensures
the feasibility of the problem, as $P(X) \equiv 1/(2\alpha^2)$ for $X \in
D_\epoch$ will always satisfy the constraints. The second empirical
regret term makes the constraints easy to satisfy for bad
hypotheses--this is crucial to rule out large label complexities in
case there are bad hypotheses that disagree very often with
$h_\epoch$. A benefit of this is easily seen when $-h_\epoch \in \H$,
which might have a terrible regret, but would force a near-constant query
probability on the disagreement region if $\beta = 0$. Finally, the
third term will be on the same order as the second one for hypotheses
in $A_\epoch$, and is only included to capture the allowed level of
slack in our constraints which will be exploited for the efficient
implementation in Section~\ref{sec:optimization}.

Of course, variance alone is not adequate to ensure concentration, and
we also require the random variables of interest to be appropriately
bounded. This is ensured through the
constraints~\eqref{eq:minbndcons}, which impose a minimum query
probability on the disagreement region. Outside the disagreement
region, we use the predicted label with an importance weight of 1, so that our estimates
will always be bounded (albeit biased) in this region. Note that this
optimization problem is written with respect to the marginal
distribution of the data points $\dist_X$, meaning that we might have
infinite number of the latter constraints. In Section~\ref{sec:optimization},
we describe how to solve this optimization problem efficiently, and
using access to only unlabeled examples drawn from $\dist_X$.

Finally we verify that the choices for $P_\epoch$ according to some of
the previous methods are indeed feasible in (\optprob). This is most
easily seen for Oracular CAL~\citep{hsu2010algorithms} which queries
with probability 1 if $X \in D_\epoch$ and 0 otherwise. Since $\alpha
\geq 1$~\eqref{eqn:constants} in the variance
constraints~\eqref{eq:queryp}, the choice $P(X) \equiv 1$ for $X \in
D_\epoch$ is feasible for (\optprob), and consequently Oracular CAL
always queries more often than the optimal distribution $P_\epoch$ at
each epoch. A similar argument can also be made for the IWAL
method~\citep{BeygelHLZ10}, which also queries in the disagreement
region with probability 1, and hence suffers from the same
sub-optimality compared to our choice.

\begin{figure}[t]
\begin{center} \fbox{\begin{minipage}{0.95\columnwidth} 
\centerline{%
  {\bf Optimization Problem} (\optprob) to compute $P_\epoch$%
}

\begin{eqnarray}	
  \min_{P} &&  \E_X \left[ \frac{1}
                                         {1 - P(X)}
                           \right]
\nonumber \\ 
  \mbox{s.t.} && \forall h \in \H\;\; \E_X \left[
    \frac{\bbmone(h(x) \neq h_\epoch(x) \wedge x
  \in D_\epoch)}{P(X)}\right] \leq b_\epoch(h),
\label{eq:queryp} \\ 
  && \forall x \in \dataspace \;\; 0 \leq P(x) \leq 1, \quad \mbox{and} \quad \forall x \in D_\epoch \;\; P(x) \geq P_{\min,\epoch}
\label{eq:minbndcons}
\end{eqnarray}

\begin{align}
  \mbox{where}\quad \disind{h}( X) &= \bbmone(h(x) \ne h_{\epoch}(x)
  \wedge x \in D_\epoch),\nonumber\\
  b_\epoch(h) &= 2 \alpha^2  \E_X[\disind{h}( X)] + 2
  \beta^2 \gamma \reg(h,h_\epoch,\tilde{Z}_{\epoch-1})\epend[\epoch-1] 
  \Delta_{\epoch-1} + \slackconst \epend[\epoch-1]\Delta_{\epoch-1}^2
  ,~~\mbox{and}\nonumber \\       
  P_{\min,\epoch} &=
  \min\left(\frac{c_3}{\sqrt{\frac{\epend[\epoch-1]\err(h_{\epoch},\tilde{Z}_{\epoch-1})}{n\epsilon_\epmax}}
    + \log\epend[\epoch-1]},\frac{1}{2}\right). 
\label{eq:pmin}
\end{align}
\end{minipage}}
\end{center}
\end{figure}

\section{Generalization and Label Complexity}
\label{sec:results}
We now present guarantees on the generalization error and label
complexity of Algorithm~\ref{alg:main} assuming a solver for
$(\optprob)$, which we provide in the next section.

\subsection{Generalization guarantees}
\label{sec:regret-result}

Our first theorem provides a bound on generalization error.  Define
\begin{align*}
  \erravg{\epoch}(h) &:= \frac{1}{\epend}
  \sum_{j=1}^\epoch(\epend[j] - \epend[j-1])
  \E_{(X,Y)\sim\dist}[\bbmone(h(X) \neq Y \wedge X \in
    D_j)], \\ 
  \Delta^*_0 &:= \Delta_0 ~~\mbox{and}~~ \Delta_{\epoch}^* :=
  c_1 \sqrt{\epsilon_\epoch \erravg{\epoch}(h^*) } + c_2
  \epsilon_\epoch \log \epend ~~\mbox{for}~~ \epoch \geq 1.
\end{align*}
Essentially $\Delta_{\epoch}^*$ is a population counterpart of the
quantity $\Delta_{\epoch}$ used in Algorithm~\ref{alg:main}, and
crucially relies on $\erravg{\epoch}(h^*)$, the true error of $h^*$
restricted to the disagreement region instead of the empirical error
of the ERM at epoch $\epoch$. This quantity captures the inherent
noisiness of the problem, and modulates the transition between
$\order(1/\sqrt{n})$ to $\order(1/n)$ type error bounds as we see
next.
\begin{theorem}
  Pick any $0 < \delta < 1/e$ such that $|\H|/\delta >
  \sqrt{192}$. Then recalling that $h^* = \argmin_{h \in \H} \err(h)$,
  we have for all epochs $\epoch = 1,2,\ldots, \epmax$, with
  probability at least $1 - \delta$
  \begin{align}
    \reg(h,h^*) & \ \leq\ 16 \gamma \Delta_\epoch^* \quad \mbox{for all}~ h
    \in A_{\epoch+1}, \quad \mbox{and}\label{eqn:regret}\\
    \reg(h^*, h_{\epoch+1}, \ztil_\epoch) & \ \leq\
    \regconst \Delta_\epoch / 4 . \label{eq:h_best}
  \end{align}
  \label{thm:regret}
\end{theorem}
The theorem is proved in Section~\ref{sec:proof-thm-regret}, using the
overall analysis framework described in Section~\ref{proofs-regret}. 

Since we use $\gamma \geq \regconst/4$, the bound \eqref{eq:h_best}
implies that $h^* \in A_\epoch$ for all epochs $\epoch$. This also
maintains that all the predicted labels used by our algorithm are
identical to those of $h^*$, since no disagreement amongst classifiers
in $A_\epoch$ was observed on those examples. 
This observation will be
critical to our proofs, where we will exploit the fact that using
labels predicted by $h^*$ instead of observed labels on certain
examples only introduces a bias in favor of $h^*$, thereby ensuring
that we never mistakenly drop the optimal classifier from our version
space $A_\epoch$. 

The bound~\eqref{eqn:regret} shows that every hypothesis in
$A_{\epoch+1}$ has a small regret to $h^*$. 
Since the ERM classifier
$h_{\epoch+1}$ is always in $A_{\epoch+1}$, this yields our main
generalization error bound on the classifier $h_{\epend+1}$ output by
Algorithm~\ref{alg:main}. Additionally, it also clarifies the
definition of the sets $A_\epoch$ as the set of good classifiers: these
are classifiers which have small population regret relative to $h^*$
indeed. In the worst case, if $\erravg{\epoch}(h^*)$ is a constant,
then the overall regret bound is $\order(1/\sqrt{n})$. The actual
rates implied by the theorem, however depend on the properties of the
distribution and below we illustrate this with two corollaries. We
start with a simple specialization to the realizable setting.

\begin{corollary}[Realizable case]
  Under the conditions of Theorem~\ref{thm:regret}, suppose further
  that $\err(h^*) = 0$. Then $\Delta_\epoch = \Delta^*_\epoch = c_2
  \epend \log\epend$ and hence $\reg(h, h^*) \leq
  16c_2\epend\log\epend$ for all hypotheses $h \in A_{\epoch+1}$.
  \label{cor:regret-realizable}
\end{corollary}
In words, the corollary demonstrates a $\otil(1/n)$ rate after seeing
$n$ unlabeled examples in the realizable setting. Of course the use of
$\erravg{\epoch}(h^*)$ in defining $\Delta^*_\epoch$ allows us to
retain the fast rates even when $h^*$ makes some errors but they do
not fall in the disagreement region of good classifiers. One intuitive
condition that controls the errors within the disagreement region is
the low-noise condition of~\citet{Tsybakov2004}, which
asserts that there exist constants $\tsybamult > 0$ and $0 < \tsybaexp
\leq 1$ such that

\begin{equation}
  \prob(h(X) \ne h^*(X)) \leq \tsybamult \cdot (\err(h) -
  \err(h^*))^\tsybaexp, \quad \forall h \in \H~~\mbox{such
    that}~~\err(h) - \err(h^*) \leq \tsybatol.
  \label{eqn:tsybakov}
\end{equation}
Under this assumption, the extreme $\tsybaexp = 0$ corresponds to the
worst-case setting while $\tsybaexp = 1$ corresponds to $h^*$ having a
zero error on disagreement set of the classifiers with regret at most
$\tsybatol$. Under this assumption, we get the following corollary of
Theorem~\ref{thm:regret}.

\begin{corollary}[Tsybakov noise]
  Under conditions of Theorem~\ref{thm:regret}, suppose further that
  Tsybakov's low-noise condition~\eqref{eqn:tsybakov} is satisfied
  with some parameters $\tsybamult, \tsybaexp$, and $\tsybatol =
  1$. Then after $\epoch$ epochs, we have \mbox{$\reg(h, h^*) =
    \otil\left(\epend^{-\frac{1}{2-\tsybaexp}}
    \log(|\H|/\delta)\right)$.}
  \label{cor:regret-tsybakov}
\end{corollary}
The proof of this result is deferred to Appendix~\ref{sec:tsybakov}.
It is worth noting that the rates obtained here are known to be
unimprovable for even passive learning under the Tsybakov noise
condition~\citep{Castro2008}.\footnote{$\tsybaexp$ in our statement of
  the low-noise condition~\eqref{eqn:tsybakov} corresponds to
  $1/\kappa$ in the results of~\citet{Castro2008}.} Consequently,
there is no loss of statistical efficiency in using our active
learning approach. The result is easily extended for other values of
$\tsybatol$ by using the worst-case bound until the first epoch
$\epoch_0$ when $16\gamma\Delta^*_{\epoch_0}$ drops below $\tsybatol$
and then apply our analysis above from $\epoch_0$ onwards. We leave
this development to the reader.

\subsection{Label complexity}
\label{sec:label-result}

Generalization alone does not convey the entire quality of an active
learning algorithm, since a trivial algorithm queries always with
probability 1, thereby matching the generalization guarantees of
passive learning. In this section, we show that our algorithm can
achieve the aforementioned generalization guarantees, despite having a
small label complexity in favorable situations. We begin with a
worst-case result in the agnostic setting, and then describe a
specific example which demonstrates some key differences of our
approach from its predecessors.

\subsubsection{Disagreement-based label complexity bounds}

In order to quantify the extent of gains over passive learning, we
measure the hardness of our problem using the disagreement
coefficient~\citep{Hanneke2014}, which is defined as

\begin{equation}
\theta = \theta(h^*) := \sup_{r > 0} \; \frac{\distx\braces{ x \mid
    \exists h \in \H \,\text{s.t.}\, h^*(x) \neq h(x) ,\,
    \distx\braces{ x' \mid h(x') \neq h^*(x') } \leq r }}{r}.
\label{eqn:dis-coeff}
\end{equation}
Intuitively, given a set of classifiers $\H$ and a data distribution
$\dist$, an active learning problem is easy if good classifiers
disagree on only a small fraction of the examples, so that the active
learning algorithm can increasingly restrict attention only to this
set. With this definition, we have the following result for the label
complexity of Algorithm~\ref{alg:main}. 

\begin{theorem}
  Under conditions of Theorem~\ref{thm:regret}, with probability at
  least $1-\delta$, the number
  of label queries made by Algorithm~\ref{alg:main} after $n$ examples
  over $\epmax$ epochs is at most
  \begin{equation*}
    4\theta \erravg{\epmax}(h^*) n + \theta \cdot \otil
    (\sqrt{n\overline{\err}_{\epmax}(h^*) \log(|\mathcal{H}|/\delta)}
    + \log(|\mathcal{H}|/\delta)) + 4 \log(8(\log n) / \delta).
   \end{equation*}
\label{thm:label}
\end{theorem}
The proof is in Appendix \ref{appendix:proof_label_complexity}.
The dominant first term of the label complexity bound is linear in the
number of unlabeled examples, but can be quite small if $\theta$ is
small, or if $\erravg{\epmax}(h^*) \approx 0$---it is indeed 0 in the
realizable setting. We illustrate this aspect of the theorem with a
corollary for the realizable setting.

\begin{corollary}[Realizable case]
  Under the conditions of Theorem~\ref{thm:label}, suppose further
  that $\err(h^*) = 0$. Then the expected number of label queries made
  by Algorithm~\ref{alg:main} is at most $\theta
  \otil(\log(|\H|/\delta))$. 
  \label{cor:label-realizable}
\end{corollary}
In words, we attain a logarithmic label complexity in the realizable
setting, so long as the disagreement coefficient is bounded. We
contrast this with the label complexity of IWAL~\citep{BeygelHLZ10},
which grows as $\theta\sqrt{n}$ independent of $\err(h^*)$. This leads
to an exponential difference in the label complexities of the two
methods in low-noise problems. A much closer comparison is with
respect to the Oracular CAL algorithm~\citep{hsu2010algorithms}, which
does have a dependence on $\sqrt{n\err(h^*)}$ in the second term, but
has a worse dependence on the disagreement coefficient $\theta$.

Just like Corollary~\ref{cor:regret-tsybakov}, we can also obtain
improved bounds on label complexity under the Tsybakov noise
condition. 

\begin{corollary}[Tsybakov noise]
  Under conditions of Theorem~\ref{thm:label}, suppose further that
  the disagreement coefficient $\theta$ is bounded and Tsybakov's
  low-noise condition~\eqref{eqn:tsybakov} is satisfied with some
  parameters $\tsybamult,\tsybaexp$, and $\tsybatol = 1$. Then after
  $\epoch$ epochs, the expected number of label queries made by
  Algorithm~\ref{alg:main} is at most $ \otil\left(\epend^{\frac{2(1 -
      \tsybaexp)}{2-\tsybaexp}} \log(|\H|/\delta)\right)$.
  \label{cor:label-tsybakov}
\end{corollary}
The proof of this result is deferred to Appendix~\ref{sec:tsybakov}.
The label complexity obtained above is indeed optimal in terms of the
dependence on $n$, the number of unlabeled examples, matching known
information-theoretic rates of~\citet{Castro2008} when the
disagreement coefficient $\theta$ is bounded. This can be seen since
the regret from Corollary~\ref{cor:regret-tsybakov} falls as a
function of the number of queries at a rate of
$\otil(q_\epoch^{-\frac{1}{2(1-\tsybaexp)}} \log(|\H|/\delta))$ after
$\epoch$ epochs, where $q_\epoch$ is the number of label queries. This
is indeed optimal according to the lower bounds of~\citet{Castro2008},
after recalling that $\tsybaexp = 1/\kappa$ in their results. Once
again, the corollary highlights our improvements on top of IWAL, which
does not attain this optimal label complexity.

These results, while strong, still do not completely capture the
performance of our method. Indeed the proofs of these results are
entirely based on the fact that we do not query outside the
disagreement region, a property shared by the previous Oracular CAL
algorithm~\citep{hsu2010algorithms}. Indeed we only improve upon that
result as we use more refined error bounds to define the disagreement
region. However, such analysis completely ignores the fact that we
construct a rather non-trivial query probability function on the
disagreement region, as opposed to using any constant probability of
querying over this entire region. This gives our algorithm the ability
to query much more rarely even over the disagreement region, if the
queries do not provide much information regarding the optimal
hypothesis $h^*$. The next section illustrates an example where this
gain can be quantified.

\subsubsection{Improved label complexity for a hard problem instance}
\label{sec:example}

We now present an example where the label complexity of
Algorithm~\ref{alg:main} is significantly smaller than both IWAL and
Oracular CAL by virtue of rarely querying in the disagreement
region. The example considers a distribution and a classifier space
with the following structure: (i) for most examples a single good
classifier predicts differently from the remaining classifiers (ii) on
a few examples half the classifiers predict one way and half the
other.  In the first case, little advantage is gained from a label
because it provides evidence against only a single classifier.
\alglong queries over the disagreement region with a probability close
to $P_{\min}$ in case (i) and probability $1$ in case (ii), while
others query with probability $\Omega(1)$ everywhere implying
$\order(\sqrt{n})$ times more queries.

Concretely, we consider the following binary classification problem.
Let $\mathcal{H}$ denote the finite classifier space (defined later),
and distinguish some $h^* \in \mathcal{H}$.  Let $U\{-1,1\}$ denote
the uniform distribution on $\{-1,1\}$.  The data distribution
$\mathcal{D}(\mathcal{X},\mathcal{Y})$ and the classifiers are defined
jointly:
\begin{itemize}
\item With probability $\epsilon$, 
\begin{eqnarray*}
&&y = h^*(x), \quad h(x) \sim U\{-1,1\}, \ \forall h \neq h^*.
\end{eqnarray*}
\item With probability $1-\epsilon$,
\begin{eqnarray*}
	&& y \sim U\{-1,1\}, \quad h^*(x) \sim U\{-1,1\}, \\ 
&& h_r(x) = -h^*(x) \mbox{ for some } h_r \mbox{ drawn uniformly at random from } \mathcal{H} \setminus h^*, \\ 
&& h(x) = h^*(x) \; \forall h \neq h^* \wedge h \neq h_r.
\end{eqnarray*}
\end{itemize}
Indeed, $h^*$ is the best classifier because $\mbox{err}(h^*) =
\epsilon \cdot 0 + (1 - \epsilon) (1/2) = (1 - \epsilon) / 2$, while
$\mbox{err}(h) = 1/2 \; \forall h \neq h^*$. This problem is hard
because only a small fraction of examples contain information about
$h^*$.  Ideally we want to focus label queries on those informative
examples while skipping the uninformative ones. However, algorithms
like IWAL, or more generally, active learning algorithms that
determine label query probabilities based on error differences between
a pair of classifiers, query frequently on the uninformative
examples. Let $u(h,h') := \mathbbm{1}(h(x) \neq y) - \mathbbm{1}(h'(x)
\neq y)$ denote the error difference between two different classifiers
$h$ and $h'$. Let $C$ be a random variable such that $C=1$ for the
$\epsilon$ case and $C=0$ for the $1-\epsilon$ case. Then it is easy
to see that
\begin{eqnarray*}
\E[u(h,h') \mid C = 1] &=&
\begin{cases}
	0, & h \neq h^*, h' \neq h^*,\\
	-1/2, & h =h^*, h' \neq h^*,\\
	1/2, & h \neq h^*, h' = h^*,
 \end{cases} \\
 \E[u(h,h') \mid C = 0] &=& 0, \; \forall h \neq h'.
\end{eqnarray*}
Therefore, IWAL queries all the time on uninformative examples ($C = 0$).

Now let us consider the label complexity of Algorithm~\ref{alg:main}
on this problem. Let us focus on the query probability inside the $1 -
\epsilon$ region, and fix it to some constant $p$. Let us also allow a
query probability of 1 on the $\epsilon$ region. Then the left hand
side in the constraint~(\ref{eq:queryp}) for any classifier $h$ is at
most $\epsilon + P(h(X) \ne h_m(X))/p \leq \epsilon + 2/(p(|\H|-1))$,
since $h$ and $h_m$ disagree only on those points in the $1-\epsilon$
region where one of them is picked as the disagreeing classifier $h_r$
in the random draw. On the other hand, the RHS of the constraints is
at least $\slackconst \epend[\epoch-1]\Delta_{\epoch-1}^2 \geq
\slackconst \err(h_m, \ztil_{m-1})$, which is at least $\slackconst/4$
as long as $\epsilon$ is small enough and $\epend[\epoch]$ is large
enough for empirical error to be close to true error. Consequently,
assuming that $\epsilon \leq \slackconst/8$, we find that any $p \geq
16/(\slackconst(|\H| - 1))$ satisfies the constraints. Of course we
also have that $p \geq \pmin$, which is $\order(1/\sqrt{\epend})$ in
this case since $\erravg{m}(h^*)$ is a constant. Consequently, for
$|\H|$ large enough $p = \pmin$ is feasible and hence optimal for the
population $(\optprob)$. Since we find an approximately optimal
solution based on Theorem~\ref{thm:optprob-unlabeled}, the label
complexity at epoch $m$ is $\order(1/\sqrt{\epend})$. Summing things
up, it can then be checked easily that we make $\order(\sqrt{n})$
queries over $n$ examples, a factor of $\sqrt{n}$ smaller than
baselines such as IWAL and Oracular CAL on this example.

\section{Efficient implementation}
\label{sec:optimization}
In Algorithm \ref{alg:main}, 
the computation of $h_\epoch$ is an ERM operation, which can be
performed efficiently whenever an efficient passive learner is
available.  However, several other hurdles remain.  Testing for $x \in
D_\epoch$ in the algorithm, as well as finding a solution to
$(\optprob)$ are considerably more challenging. The epoch schedule
helps, but $(\optprob)$ is still solved $\order(\log n)$ times,
necessitating an extremely efficient solver.

Starting with the first issue, we follow~\citet{DasguptaHM07} who
cleverly observed that $x\in D_\epoch$ can be efficiently determined
using a single call to an ERM oracle.  Specifically, to apply their
method, we use the oracle to find\footnote{We only have access to an
  \emph{unconstrained} oracle.  But that is adequate to solve with one
  constraint.  See Appendix F of~\citep{KLonline} for details.} $h' =
\arg\min \braces{ \err(h, \ztil_{\epoch-1}) \mid h \in \H, h(x) \ne
  h_\epoch(x) }$.  It can then be argued that $x \in D_\epoch =
\dis(A_\epoch)$ if and only if the easily-measured regret of $h'$
(that is, $\reg(h', h_\epoch, \ztil_{\epoch-1})$) is at most $\gamma
\Delta_{\epoch-1}$.

Solving $(\optprob)$ efficiently is a much bigger challenge because,
as an optimization problem, it is enormous: There is one variable
$P(x)$ for every point $x \in \dataspace$, one
constraint~\eqref{eq:queryp} for each classifier $h$ and bound
constraints~\eqref{eq:minbndcons} on $P(x)$ for every $x$. This leads
to infinitely many variables and constraints, with an ERM oracle being
the only computational primitive available. Another difficulty is that
$(\optprob)$ is defined in terms of the true expectation with respect
to the example distribution $\distx$, which is unavailable.

In the following we first demonstrate how to efficiently solve
$(\optprob)$ assuming access to the true expectation $\E_X[\cdot]$,
and then discuss a relaxation that uses expectation over samples. For
the ease of exposition, we recall the shorthand $\disind{h}(x) =
\bbmone(h(x) \neq h_\epoch(x) \wedge x \in D_\epoch)$ from earlier.

\subsection{Solving (\optprob) with the true expectation}
The main challenge here is that the optimization variable $P(x)$ is of
infinite dimension. We deal with this difficulty using Lagrange
duality, which leads to a dual representation of $P(x)$ in terms of a
set of classifiers found through successive calls to an ERM oracle. As
will become clear shortly, each of these classifiers corresponds to
the most violated variance constraint \eqref{eq:queryp} under some
intermediate query probability function. Thus at a high level, our
strategy is to expand the set of classifiers for representing $P(x)$
until the amount of constraint violation gets reduced to an acceptable
level.

We start by eliminating the bound constraints using barrier functions.
Notice that the objective $\E_X[1/(1-P(x))]$ is already a barrier at
$P(x) = 1$. To enforce the lower bound~\eqref{eq:minbndcons}, we
modify the objective to
\begin{equation} \label{eq:mod-obj}
  \E_X \left[ \frac{1}{1 -
      P(X)}\right] + \minp^2 \E_X\left[ \frac{\bbmone(X \in
    D_\epoch)}{P(X)}\right],
\end{equation}
where $\minp$ is a parameter chosen momentarily to ensure $P(x) \geq
\pmin$ for all $x \in D_\epoch$.  Thus, the modified goal is to
minimize \eqref{eq:mod-obj} over non-negative $P$ subject only to
\eqref{eq:queryp}.

\begin{algorithm}[t]
  \caption{Coordinate ascent algorithm to solve $(\optprob)$}
  \label{alg:coord}
  \begin{algorithmic}[1]
    \renewcommand{\algorithmicrequire}{\textbf{input}}
    \REQUIRE Accuracy parameter $\varepsilon>0$.
    \textbf{initialize} $\blambda \leftarrow \bzero$. 
    \LOOP
    \STATE Rescale:  $\blambda \leftarrow s \cdot \blambda$ where
                     $s = \arg \max_{s\in [0,1]} \dualobj(s \cdot \blambda)$.
    \STATE Find $\displaystyle\bar{h} = \arg\max_{h \in \H}
      \E_X\left[\frac{\disind{h}(X)}{P_{\blambda}(X)} \right] -
      b_\epoch(h)$.
    \label{step:violated}

    
    \IF{$\E_X\left[\frac{\disind{\bar{h}}(X)}{P_{\blambda}(X)} \right] -
      b_\epoch(\bar{h}) \leq \varepsilon$}
    \RETURN $\blambda$
    \ELSE
    \STATE Update $\lambda_{\bar{h}}$ as
    $\displaystyle
       \lambda_{\bar{h}} \leftarrow \lambda_{\bar{h}} +
       2\frac{\E_X[\disind{\bar{h}}(X)/P_{\blambda}(X)] -
         b_\epoch(\bar{h})}{\E_X[\disind{\bar{h}}(X)/q_{\blambda}(X)^3]}
         $.
    \label{step:coord-update}
    \ENDIF
    \ENDLOOP
  \end{algorithmic}
\end{algorithm}
We solve the problem in the dual where we have a large but finite
number of optimization variables, and efficiently maximize the dual
using coordinate ascent with access to an ERM oracle over $\H$.  Let
$\lambda_h \geq 0$ denote the Lagrange multiplier for the
constraint~\eqref{eq:queryp} for classifier $h$.  Then for any
$\blambda$, we can minimize the Lagrangian
\begin{equation}
\mathcal{L}(P,\blambda) := 
  \E_X \left[ \frac{1}{1 -
      P(X)}\right] + \minp^2 \E_X\left[ \frac{\bbmone(X \in
    D_\epoch)}{P(X)}\right] - \sum_{h \in \H} \lambda_h \left(b_{\epoch}(h) - \E_X\left[ \frac{\bbmone(h(X) \neq h_\epoch(X) \wedge X \in D_\epoch)}{P(X)}\right] \right)
\label{eq:Lagrangian}
\end{equation}
over each primal variable
$P(x) \in [0,1]$ yielding the solution.
\begin{align}
  P_{\blambda}(x) = \frac{\bbmone(x \in D_\epoch)q_{\blambda}(x)}{1 +
    q_{\blambda}(x)}, ~~ \mbox{where} ~~ q_{\blambda}(x) =
  \sqrt{\minp^2 + \sum_{h \in \H} \lambda_h \disind{h}(x)}.
  \label{eqn:primal-dual}
\end{align}
To see this, pick any $\widetilde{P}$ satisfying $\widetilde{P}(x) \in [0,1]$ for all $x \in \dataspace$
and consider the difference in the Lagrangians evaluated at $\widetilde{P}$ and $P_{\blambda}$:
\begin{eqnarray*}
\mathcal{L}(\widetilde{P},\blambda) - \mathcal{L}(P_{\blambda},\blambda)
&=&
\E_X\left[ \bbmone(X \notin D_\epoch) \left(\frac{1}{1 - \widetilde{P}(X)}-1 \right)\right]  \nonumber \\ 
&& +\E_X\left[ \bbmone(X \in D_\epoch) \left( \frac{1}{1-\widetilde{P}(X)} + \frac{\minp^2 + \sum_{h \in \H} \lambda_h \disind{h}(x)}{\widetilde{P}(X)} - (1 + q_{\blambda}(X))^2\right)\right].
\end{eqnarray*}
The first term is non-negative because $\widetilde{P}(x) \in [0,1]$.
For the second term, notice that 
\begin{equation*}
P_{\blambda}(x) = \arg \min_{0 \leq v \leq 1} \quad \bbmone(x \in D_\epoch) \left( \frac{1}{1-v} + \frac{\minp^2 + \sum_{h \in \H}\lambda_h \disind{h}(x)}{v} \right)
\end{equation*} 
and that the minimum function value is exactly $\bbmone(x \in D_\epoch)(1 + q_{\blambda}(x))^2$.
Hence the second term is also non-negative.
 
Clearly, $\minp/(1 + \minp) \leq P_{\blambda}(x) \leq 1$
for all $x \in D_\epoch$, so all the bound
constraints~\eqref{eq:minbndcons} in $(\optprob)$ are satisfied if we
choose $\minp = 2\pmin$. Plugging the solution $P_{\blambda}$ into the
Lagrangian, we obtain the dual problem of maximizing the dual
objective
\begin{equation}  \label{eqn:dual-obj}
  \dualobj(\blambda) = \E_X\left[\bbmone(X \in D_\epoch) (1 +
    q_{\blambda}(X))^2 \right] - \sum_{h \in \H} \lambda_h b_\epoch(h)
  + C_0
\end{equation}
over $\blambda\geq 0$. The constant $C_0$ is equal to
$1-\prob(D_\epoch)$ where $\prob(D_\epoch)=\prob(X \in D_\epoch)$.  An
algorithm to approximately solve this problem is presented in
Algorithm~\ref{alg:coord}.  The algorithm takes a parameter
$\varepsilon>0$ specifying the degree to which all of the constraints
\eqref{eq:queryp} are to be approximated.  Since $\dualobj$ is
concave, the rescaling step can be solved using a straightforward
numerical line search.  The main implementation challenge is in
finding the most violated constraint (Step~\ref{step:violated}).
Fortunately, this step can be reduced to a single call to an ERM
oracle.  To see this, note that the constraint violation on classifier
$h$ can be written as
\begin{align*}
  \E_X\left[ \frac{\disind{h}(X)}{P(X)}\right] - b_\epoch(h)
  &= \E_X\left[\bbmone(X \in D_\epoch)\left(\frac{1}{P(X)} -
    2\alpha^2\right) \bbmone(h(X) \ne h_\epoch(X)) \right]\\ & -
  2\beta^2 \gamma\epend[\epoch-1] \Delta_{\epoch-1}(\err(h,
  \ztil_{\epoch-1}) - \err(h_\epoch, \ztil_{\epoch-1})) - \slackconst
  \epend[\epoch-1]\Delta_{\epoch-1}^2.  
\end{align*}
The first term of the right-hand expression is the risk (classification error) of $h$ in predicting samples
labeled according to $h_\epoch$ with importance weights of $1/P(x) -
2\alpha^2$ if $x \in D_\epoch$ and 0 otherwise; note that these
weights may be positive or negative. 
The second term is simply the scaled risk
of $h$ with respect to the actual labels. 
The last two terms do not depend
on $h$. Thus, given access to $\distx$ (or samples approximating
it, discussed shortly), the most violated constraint can be found by
solving an ERM problem defined on the labeled samples in $\ztil_{\epoch-1}$
and 
samples drawn from $\distx$ labeled by $h_\epoch$, with appropriate importance weights detailed in Appendix \ref{appendix:most_violated_erm}.


When all primal constraints are approximately satisfied, the algorithm
stops. Consequently, we can execute each step of
Algorithm~\ref{alg:coord} with one call to an appropriately defined
ERM oracle, and approximate primal feasibility is guaranteed when the
algorithm stops.  More specifically, we can prove the following
guarantee on the convergence of the algorithm.

\begin{theorem}
\label{thm:opt_converge}	
When run on the $\epoch$-th epoch, Algorithm~\ref{alg:coord} 
has the following guarantees.
\begin{enumerate}
\item It halts in at most $\frac{\prob(D_\epoch)}{8 \pmin^3 \varepsilon^2}$ iterations.
\item The solution $\hat{\blambda} \geq \bzero$ it 
outputs has bounded $\ell_1$ norm: $\lone{\hat{\blambda}}\leq
\prob(D_\epoch) / \varepsilon$. 
\item The query probability function $P_{\hat{\blambda}}$ satisfies: 
\begin{itemize}
\item The variance constraints
\eqref{eq:queryp} up to an additive factor of $\varepsilon$, i.e.,
\begin{equation*}
  \forall h \in \H\;\; \E_X \left[
    \frac{\bbmone(h(x) \neq h_\epoch(x) \wedge x
  \in D_\epoch)}{P_{\hat{\blambda}}(X)}\right] \leq b_\epoch(h) + \varepsilon,
\end{equation*} 
\item 
The simple bound constraints
\eqref{eq:minbndcons} exactly, 
\item Approximate primal optimality:
\begin{equation}  \label{eq:prim-obj-bnd}
	\E_X\left[ \frac{1} {1 - P_{\hat{\blambda}}(X)} \right] \leq
        f^* + 4 \pmin
        \prob(D_\epoch),
\end{equation}
where $f^*$ denotes the optimal value of $(\optprob)$, i.e, 
\begin{equation}
\begin{split}
 f^* \;:= \;& \inf_P \; \E_X \left[ \frac{1}{1-P(X)} \right] \\ 
\mbox{s.t.} \quad& P \mbox{ satisfying } \eqref{eq:queryp} \mbox{ and } \eqref{eq:minbndcons}
\end{split}
\label{eq:op_value}
\end{equation}
\end{itemize}
\end{enumerate}
\end{theorem}
That is, we find a solution with small constraint violation to ensure
generalization, and a small objective value to be label efficient. 
If
$\varepsilon$ is set to $\slackconst
\epend[\epoch-1]\Delta_{\epoch-1}^2$, an amount of constraint
violation tolerable in our analysis, the number of iterations in
Theorem \ref{thm:opt_converge} varies between
$\order(\epend[\epoch-1]^{3/2})$ and $\order(\epend[\epoch-1]^2)$ as
the $\err(h_\epoch,\ztil_{\epoch-1})$ varies between a constant and
$\order(1/\epend[\epoch-1])$. The theorem is proved in
Appendix~\ref{sec:opt-pop}.

\subsection{Solving (\optprob) with expectation over samples}
So far we considered solving $(\optprob)$ defined on the unlabeled
data distribution $\distx$, which is not available in practice.
A simple and natural substitute for $\distx$ is an i.i.d. sample drawn from it. 
Here we show that solving a properly-defined sample variant of (\optprob) leads to 
a solution to the original $(\optprob)$ with similar guarantees 
as in Theorem \ref{thm:opt_converge}.

More specifically, we define the following sample variant of $(\optprob)$.
Let $S$ be a large sample drawn i.i.d. from $\distx$, 
and $(\optprob_S)$ be the
same as $(\optprob)$ except with all population expectations
replaced by empirical expectations taken with
respect to $S$. Now for any $\samperr\geq0$, define
$(\optprob_{S,\samperr})$ to be the same as $(\optprob_S)$ except 
that the variance constraints~\eqref{eq:queryp} are relaxed by an additive slack of $\samperr$.

Every time \alglong needs to solve $(\optprob)$ (Step \ref{alg:main:P_update} of Algorithm \ref{alg:main}), 
it draws a fresh unlabeled i.i.d. sample $S$ of size $\unlab$ from $\distx$, which can be done easily  
in a streaming setting by collecting the next $u$ examples. It then 
applies Algorithm~\ref{alg:coord} to solve $(\optprob_{S,\samperr})$
with accuracy parameter $\samperr$. Note that this is different from solving $(\optprob_S)$  
with accuracy parameter $2\varepsilon$. 
We establish the following convergence guarantees.
\begin{theorem}
  \label{thm:optprob-unlabeled}
  Let $S$ be an i.i.d.~sample of size $\unlab$ from $\distx$.  
When run on the $\epoch$-th epoch for solving $(\optprob_{S,\samperr})$ with accuracy parameter
  $\samperr$, Algorithm~\ref{alg:coord} satisfies the following. 
\begin{enumerate}
\item It halts in at most $\frac{\widehat{\prob}(D_\epoch)}{8 \pmin^3 \varepsilon^2}$ iterations, where $\widehat{\prob}(D_\epoch) := \sum_{X \in S} \bbmone(X \in D_\epoch) / \unlab$.
\item The solution $\hat{\blambda} \geq \bzero$ it 
outputs has bounded $\ell_1$ norm: $\lone{\hat{\blambda}}\leq
\widehat{\prob}(D_\epoch) / \varepsilon$. 
\item If $\unlab \geq \order ((1/(\pmin\samperr)^4 +
  \alpha^4/\samperr^2) \log(|\H|/\delta))$, then with probability
  $\geq~1-\delta$, the query probability function $P_{\hat{\blambda}}$ satisfies: 
\begin{itemize}
\item All constraints of $(\optprob)$ except with an
  additive slack of $2.5\samperr$ in the variance constraints~\eqref{eq:queryp}, 
\item Approximate primal optimality:
  \begin{equation*}
    \E_X \Brackets{ \frac{1}{1 - P_{\hat{\blambda}}(X)} } \ \leq
    \ 
    f^* + 8 \pmin
    \prob(D_\epoch) + \parens{ 2 + 4 \pmin} \samperr, 
  \end{equation*}
where $f^*$ is the optimal value of $(\optprob)$ defined in \eqref{eq:op_value}.
\end{itemize}
\end{enumerate}
\end{theorem}
The proof is in Appendix \ref{sec:opt-samp}. Intuitively, the optimal
solution $\Psoln$ to $(\optprob)$ is also feasible in $(\optprob_{S,
  \samperr})$ since satisfying the population constraints leads to
approximate satisfaction of sample constraints. Since our solution
$P_{\hat{\blambda}}$ is approximately optimal for $(\optprob_{S,
  \varepsilon})$ (this is essentially due to
Theorem~\ref{thm:opt_converge}), this means that the sample objective
at $P_{\hat{\blambda}}$ is not much larger than $\Psoln$. We now use a
concentration argument to show that this guarantee holds also for the
population objective with slightly worse constants. The approximate
constraint satisfaction in $(\optprob)$ follows by a similar
concentration argument. Our proofs use standard concentration
inequalities along with Rademacher complexity to provide uniform
guarantees for all vectors $\blambda$ with bounded $\ell_1$ norm.

The first two statements, finite convergence and boundedness of
$\lone{\hat{\blambda}}$, are identical to Theorem
\ref{thm:opt_converge} except $\prob(D_\epoch)$ is replaced by
$\widehat{\prob}(D_\epoch)$.  When $\varepsilon$ is set properly, i.e,
to be $\slackconst^2 \epend[\epoch-1]\Delta_{\epoch-1}^2$, the number
of unlabeled examples $\unlab$ in the third statement varies between
$\order(\epend[\epoch-1]^{2})$ and $\order(\epend[\epoch-1]^4)$ as the
$\err(h_\epoch,\ztil_{\epoch-1})$ varies between a constant and
$\order(1/\epend[\epoch-1])$. The third statement shows that with
enough unlabeled examples, we can get a query probability function
almost as good as the solution to the population problem $(\optprob)$.

\section{Experiments with Agnostic Active Learning}
\label{sec:experiments}
\begin{algorithm}
\caption{\coverlong}	
\label{alg:cover}
\begin{algorithmic}[1]
\renewcommand{\algorithmicrequire}{\textbf{input:}}
\REQUIRE cover size $\coversize$, parameters $c_0, \alpha$ and $\beta_{scale}$. 
\STATE Initialize online importance weighted minimization oracles $\{O_t\}_{t=0}^\coversize$, each controlling a classifier and some associated weights $\{(h_{t}, \lambda_{t},\nu_{t}, \omega_{t})\}_{t=1}^\coversize$ with all  
weights initialized to 0.
\STATE For the first three examples $\{X_i\}_{i=1}^3$, query the labels $\{Y_i\}_{i=1}^3$. 
\STATE Let $h := O_0(\{(X_i,Y_i,1)\}_{i=1}^2)$.
\STATE Get error estimate $e_2$ from $O_0$ and compute $P_{\min,3}$. 
\STATE Let $(X, Y^*, \tilde{Y}, W) := (X_3, Y_3, h(X_3), 1)$. Set $\beta := (\sqrt{\alpha/c_0})/ \beta_{scale}$. 
\FOR{$i = 4,\ldots,n,$}
\STATE Update the ERM, the error estimate and the threshold \label{alg:cover:erm_update} 
\begin{eqnarray*}
	h &:=& O_0((X,Y^*,W)),\\
	e_{i-1} &:=& \frac{(i-2) e_{i-2} + \bbmone(\tilde{Y} \neq Y^*)W}{i-1},\\
        \widehat{\Delta}_{i-1} &:=& \sqrt{c_0e_{i-1}/(i-1)} + \max(2\alpha,4)c_0 \log(i-1)/(i-1).
\end{eqnarray*}
\FOR{$t = 1,\ldots,\coversize$} \label{alg:cover:update_begin}
\STATE Compute $p_t := q_t / (1 + q_t)$, where $q_t := \sqrt{(2P_{\min,i-1})^2 + \sum_{t' < t} \lambda_t \bbmone(h_t(X) \neq \tilde{Y}))}$.
\STATE Set up the cost of predicting $y \in \{1,-1\}$, the target label and the importance weight:
\begin{eqnarray}
\bc_y &:=& 2 \beta^2 (i-2)\widehat{\Delta}_{i-2} \bbmone(y \neq Y^*)W + \left(2 \alpha^2 - \frac{1}{p_t}\right) \bbmone(X \in D_{i-1} \wedge y \neq \tilde{Y}),
\label{eq:cover_cost} \\
Y_{t} &:=& \arg \min_y \bc_y, \nonumber \\
W_{t} &:=& |\bc_{1} - \bc_{-1}|.\nonumber
\end{eqnarray}
\STATE Update the $t$-th classifier in the cover and its associated weights: 
\begin{eqnarray}
	h_{t} &:=& O_t((X,Y_{t},W_{t})),\nonumber \\
	\nu_{t} &:=& \max\left( \nu_{t} + 2 \Big( \bc_{\tilde{Y}} - \bc_{h_{t}(X)}\Big) , 0 \right), \label{eq:cover:update_num}\\
	\omega_{t}&:=& \omega_{t} + \bbmone(h_{t}(X) \neq \widetilde{Y} \wedge X \in D_{i-1}) / q_t^3, \label{eq:cover:update_den}\\
	\lambda_{t} &:=& \frac{\nu_{t}}{\omega_{t}} \bbmone\big((\nu_{t},\omega_{t}) \neq (0,0)\big). \label{eq:cover:update_lambda}
\end{eqnarray}
\ENDFOR \label{alg:cover:update_end}
\STATE Receive new data point $X_i$ and let $\tilde{Y} := h(X_i)$. \label{alg:cover:query} 
\STATE Compute $P_{\min,i} := \min\big( (\sqrt{(i-1)e_{i-1}}+\log(i-1))^{-1}, 1/2\big)$.
\IF{$X_i \in D_i:= \mbox{DIS}(A_i)$,} \label{alg:cover:dis_test} 
\STATE Compute $P_i := q /(1 + q)$, where $q := \sqrt{(2P_{\min,i})^2 + \sum_{t = 1}^l \lambda_t \bbmone(h_t(X_i) \neq \tilde{Y}))}$. 
\STATE Draw $Q \sim \mbox{Bernoulli}(P)$.  
\IF{$Q = 1$}
\STATE Query $Y_i$ and set $(X, Y^*,W) := (X_i, Y_i, 1/P_i)$.
\ELSE
\STATE Set $(X, Y^*,W) := (X_i, 1,0)$.
\ENDIF
\ELSE
\STATE
Set $(X, Y^*,W) := (X_i, h(X_i), 1)$.
\ENDIF \label{alg:cover:query_end}
\ENDFOR
\end{algorithmic}
\end{algorithm}
While \alg is efficient in the number of ERM oracle calls, it needs to
store all past examples, resulting in large space complexity. As
Theorem \ref{thm:opt_converge} suggests, the query probability
function \eqref{eqn:primal-dual} may need as many as
$\order(\epend[i]^2)$ classifiers, further increasing storage
demand.  In Section \ref{sec:online_cover} we discuss 
a scalable online approximation to \alglong, \coverlong (\cover), 
which we implemented and tested empirically with the setup in Section \ref{sec:exp_setup}. 
Experimental results and discussions are in Section \ref{sec:exp_results}. 

\subsection{Online Active Cover (\cover)}
\label{sec:online_cover}
Algorithm \ref{alg:cover} gives the online approximation that we implemented, 
which uses an epoch schedule of $\epend[i] = i$, assigning every new example to a new epoch.

To explain the connections between Algorithms \ref{alg:main} (\alg) and \ref{alg:cover} (\cover), we start with
the update of the ERM classifier and thresholds, corresponding to Step \ref{alg:main:erm_update} of \alg 
and Step \ref{alg:cover:erm_update}  of \cover. 
Instead of batch ERM oracles, \cover invokes online importance weighted ERM oracles that are stateful 
and process examples in a streaming fashion without the need to store them.
The specific importance weighted oracle we use is a reduction to online importance-weighted logistic regression \citep{KLonline} implemented in Vowpal Wabbit (VW).
$Y^*$ denotes the actual label that is used to update the ERM classifier and, depending on the query decision (Steps \ref{alg:main:query} to \ref{alg:main:query_end}), 
can be a queried label, a predicted label by the previous ERM classifier, or a dummy label of 1 associated with an importance weight of zero.
The error variable $e_{i-1}$ keeps track of the \emph{progressive validation loss}, which is a better estimate of the true classification error
than the training error \citep{blum1999beating,cesa2004generalization}. 

Instead of computing the query probability function by solving a batch optimization problem as in Step \ref{alg:main:P_update} of \alg,  
\cover maintains a fixed number $\coversize$ of classifiers that are intended to be a cover of the set of good classifiers. 
On every new example, this cover undergoes a sequence of online, importance weighted 
updates (Steps \ref{alg:cover:update_begin} to \ref{alg:cover:update_end} of \cover),
which are meant to approximate the coordinate ascent steps in Algorithm \ref{alg:coord}.
The importance structure \eqref{eq:cover_cost} is derived from \eqref{eq:cost},  
accounting for the fact that 
the algorithm simply uses the incoming stream of examples to estimate
$\bbE_X[\cdot]$ rather than a separate unlabeled sample.
The same approximation is also present in the updates \eqref{eq:cover:update_num} and  \eqref{eq:cover:update_den}, 
which are online estimates of the numerator and the denominator of the additive coordinate update
in Step \ref{step:coord-update} of Algorithm \ref{alg:coord}. Because \eqref{eq:cover:update_num}
is an online estimate, we need to explicitly enforce non-negativity. Note that \eqref{eq:cover:update_num}
has the following straightforward interpretation: if the prediction of $h_t$, the $t$-th classifier in the cover, 
is the same as that of the ERM, the weight associated with $h_t$ will not change. Otherwise, 
the weight of $h_t$ increases/decreases when its prediction has a smaller/larger cost than the prediction of the ERM.

To further clarify the effect of \eqref{eq:cover_cost}, we perform the following case analysis:
\begin{itemize}
\item If $X_{i-1} \notin D_{i-1}$, then for all $t \in \{1,\ldots,l\}$, 
\begin{equation*}
(\bc_{\tilde{Y}},\bc_{-\tilde{Y}}) \;=\; (0, 2 \beta^2 (i-2) \widehat{\Delta}_{i-2}),
\end{equation*}
so $Y_t = \tilde{Y}$. This means that all the classifiers in the cover are trained with the predicted label when the example is outside of the disagreement region.
\item Otherwise, the costs for the $t$-th classifier in the cover are: 
\begin{eqnarray*}
(\bc_{\tilde{Y}},\bc_{-\tilde{Y}}) &=& 
\begin{cases}
(0,2\alpha^2 - 1 / p_t), & Q = 0, \mbox{i.e., the true label was not queried},\\
(0, 2\alpha^2 - 1 / p_t + 2 \beta^2 (i-2) \widehat{\Delta}_{i-2}/P_{i-1}), & Q = 1, \tilde{Y} = Y_{i-1},\\
(2 \beta^2 (i-2) \widehat{\Delta}_{i-2}/P_{i-1}, 2 \alpha^2 - 1/p_t), & Q = 1, \tilde{Y} \neq Y_{i-1}. 
\end{cases}
\end{eqnarray*}
In the first case, if  $p_t > 1/(2 \alpha^2)$, i.e., 
the query probability based on the previous $t-1$ classifiers in the cover is large enough, then $\bc_{-\tilde{Y}} > 0$ and the $t$-th classifier will be trained to agree with the predicted label. 
Otherwise, the $t$-th classifier will be trained to disagree with the predicted label, thereby increasing 
the query probability. In the second case, the true label $Y_{i-1}$ was queried and found to be the same as the predicted label, so unless 
$p_t$ is very small, the $t$-th classifier will not be trained to disagree with the ERM $h$. In the third case, the cost associated with the predicted 
label $\tilde{Y}$ is always positive, so the true label $Y_{i-1}$ will be preferred unless $p_t$ or $\alpha$ is fairly large.
\end{itemize}

Finally, Steps \ref{alg:main:query} to \ref{alg:main:query_end} of \alg 
and Steps \ref{alg:cover:query} to \ref{alg:cover:query_end} of \cover
perform the querying of labels. 
As pointed out in Section \ref{sec:optimization}, the test in Step \ref{alg:cover:dis_test} of \cover
is done via an {\em online technique} detailed in Appendix F of \citet{KLonline}.

\subsection{Experiment Setting}
\label{sec:exp_setup}
We conduct an empirical comparison of $\cover$ with the following active learning algorithms.
\begin{itemize}
\item \iwalzero: Algorithm 1 of \citet{BeygelHLZ10}, which performs importance-weighted 
sampling of labels and maintains an unbiased estimate of classification error. On every new example, it queries the true label with probability 1 if  
the error difference $G_k$ (Step 2 in Algorithm 1 of \citet{BeygelHLZ10}) is smaller than the threshold 
\begin{equation}
\sqrt{\frac{C_0 \log k}{k-1}} + \frac{C_0 \log k}{k-1},
\label{eq:iwal0_threshold}
\end{equation}
where $C_0$ is a hyper-parameter. Otherwise, the query probability is a decreasing function of $G_k$.

\item \iwal: A slight modification of \iwalzero that uses 
a more aggressive, error-dependent threshold:
\begin{equation}
\sqrt{\frac{C_0 \log k}{k-1} e_{k-1}} + \frac{C_0 \log k}{k-1},
\label{eq:iwal_threshold}
\end{equation}
where $e_{k-1}$ is the importance-weighted error estimate after the algorithm processes $k-1$ examples.

\item \iwalorazero: An Oracular-CAL \citep{hsu2010algorithms} style variant of \iwalzero that queries the label 
of a new example with probability 1 if the error difference $G_k$ (see \iwalzero above) is smaller than the threshold
\eqref{eq:iwal0_threshold}. Otherwise, it uses the predicted label by the current ERM classifier.
\item \iwalora: An Oracular-CAL \citep{hsu2010algorithms} style variant of \iwal
that resembles \iwalorazero except that it uses the error-dependent threshold \eqref{eq:iwal_threshold}.
Note that the error estimate $e_{k-1}$ now uses 
both the queried labels and predicted labels, and is no longer unbiased. 
We remark that a theoretical analysis of this algorithm has recently
been given by~\citet{zhang2015oracular}. In fact, it is almost identical to an Oracular-CAL \citep{hsu2010algorithms} style variant of Algorithm \ref{alg:cover} 
that uses a query probability $P_i$ of 1 whenever the disagreement test in Step \ref{alg:cover:dis_test}
of Algorithm \ref{alg:cover} returns true, except that its threshold \eqref{eq:iwal_threshold} is slightly different from 
the one used by Algorithm \ref{alg:cover} (Step \ref{alg:cover:erm_update}).
\item \algrand: Passive learning using all the labels of incoming examples up to some label budget.
\end{itemize}

We implemented these algorithms
in Vowpal Wabbit\footnote{\url{http://hunch.net/~vw/}.} (VW), 
a fast learning system using online convex optimization, which 
fits nicely with the streaming active learning setting.
We performed experiments on 22 binary classification datasets with
varying sizes ($10^3$ to $10^6$) and diverse feature
characteristics. Details about the datasets are in Appendix
\ref{appendix:datasets}. Our goals are:
\begin{enumerate}
\item Investigating the maximal test error improvement per label query achievable by different algorithms;
\item Comparing different algorithms when each uses the best fixed hyper-parameter setting.
\end{enumerate}
We thus consider the following experiment setting. To simulate the streaming setting, 
we randomly permuted the datasets, ran the active learning algorithms 
through the first 80$\%$ of data, and evaluated the learned classifiers 
on the remaining 20$\%$. We repeated this process 9 times to reduce variance 
due to random permutation. 
For each active learning algorithm, we obtain the test error rates of classifiers trained at doubling numbers of label queries 
starting from 10 to 10240. Formally, let $\error_{a,p}(d,j,q)$ denote the test error of the classifier returned by algorithm $a$ 
using hyper-parameter setting $p$ on the $j$-th permutation of dataset $d$ under a label budget of $10\cdot 2^{(q-1)}$, $1 \leq q \leq 11$, 
and $\query_{a,p}(d,j,q)$ denote the actual number of label queries made.  
Note that under the same label budget, \cover and the Oracular-CAL variants may use more example-label pairs for learning 
than \iwalzero and \iwal because the former algorithms use predicted labels.
Also note that $\query_{a,p}(d,j,q) < 10\cdot 2^{(q-1)}$ when algorithm $a$ reaches the end of the training data before 
hitting the $q$-th label budget.
To evaluate the overall performance of an algorithm, we consider the area under its curve of test error against 
$\log$ number of label queries:
\begin{equation}
\auc_{a,p}(d,j) = \frac{1}{2} \sum_{q = 1}^{10} \Big(\error_{a,p}(d,j,q+1) +  \error_{a,p}(d,j,q)\Big) \cdot \left(\log_2 \frac{\query_{a,p}(d,j,q+1)}{\query_{a,p}(d,j,q)}\right).
\end{equation}
A good active learning algorithm has a small value of \auc, which indicates that the test error decreases quickly as the number of label queries increases.
We use a logarithmic scale for the number of label queries to focus on the performance under few label queries where active learning 
is the most relevant. More details about hyper-parameters are in
Appendix \ref{appendix:hyper-parameters}.

For the first goal, we compare the performances of different algorithms optimized on a per dataset basis. 
More specifically, we 
measure of the performance of algorithm $a$ by the following aggregated metric:
\begin{equation}
\displaystyle \aucg^*(a) := \underset{d}{\mean} \max_{p} \underset{1 \leq j \leq 9}{\median} \left\{ \frac{\auc_{base}(d,j) - \auc_{a,p}(d,j)}{\auc_{base}(d,j)} \right\},
\label{eq:auc-gain-per-data}
\end{equation}
where $\auc_{base}$ denotes the \auc \; of \algrand using a default hyper-parameter setting, corresponding to a learning rate of 0.4 (see Appendix \ref{appendix:hyper-parameters} for more details).
In this metric, we first take the median of the relative test error improvements over the \algrand baseline, which gives 
a representative performance among the 9 random permutations, and then take the maximum of the medians over hyper-parameters, and finally 
average over datasets. This metric shows the maximal gain each algorithm achieves with the best hyper-parameter setting for
each dataset. 

In practice it is difficult to select active learning hyper-parameters on a per-dataset
basis because labeled validation data are not available. With
a variety of classification datasets, a reasonable alternative might be to
look for the single hyper-parameter setting that performs the best on
average across datasets, thereby reducing over-fitting to any
individual dataset, and compare different algorithms under such fixed
parameter settings. We thus consider the following metric: 
\begin{equation}
\displaystyle \aucg(a) := \max_{p} \underset{d}{\mean} \; \underset{1 \leq j \leq 9}{\median} \left\{ \frac{\auc_{base}(d,j) - \auc_{a,p}(d,j)}{\auc_{base}(d,j)} \right\},
\label{eq:auc-gain}
\end{equation}
which first averages the median improvements over datasets and then maximizes over hyper-parameter settings.

\subsection{Results and Discussions}
\label{sec:exp_results}
Table \ref{tbl:summary} gives a
summary of the performances of different algorithms, measured 
by the two metrics $\aucg^*$ \eqref{eq:auc-gain-per-data} and $\aucg$ \eqref{eq:auc-gain}.
When using hyper-parameters optimized on a per-dataset basis (top row in Table \ref{tbl:summary}), \cover achieves the largest improvement over the \algrand baseline, 
with \iwalorazero achieving almost the same improvement and 
other active learning algorithms improving slightly less.
When using the best fixed hyper-parameter setting across all datasets (bottom row in Table \ref{tbl:summary}), all active learning algorithms
achieve less improvement compared with \algrand, which achieves a $7$\% improvement with the best fixed learning rate.
\iwalorazero performs the best, achieving a 9\% improvement, while \iwalzero and \iwalora 
achieve more than 8\%. Both \iwal and \cover achieve around $7.5$\% improvements, slightly better than 
\algrand. This suggests that careful tuning of hyper-parameters is critical for \cover and an important direction for future work.

\begin{table}[t]
\centering
\caption{Summary of performance metrics}
\setlength{\tabcolsep}{1pt}
\begin{tabular}{l||C{50pt}|C{50pt}|C{50pt}|C{50pt}|C{50pt}|C{50pt}|C{50pt}}
  & \cover & \iwalzero & \iwal & \iwalorazero & \iwalora & \algrand \\ \hline \hline
$\aucg^*$ &\textbf{0.1611} & 0.1466 & 0.1552 & 0.1586 & 0.1549  & 0.0950\\ \hline
$\aucg$   & 0.0722 & 0.0863 & 0.0755 & \textbf{0.0945} & 0.0807 & 0.0718
\end{tabular}
\label{tbl:summary}
\end{table}
\begin{figure}[t]
\centering
\subfigure[Median over permutations]{\label{fig:rel-err-gain-median}\includegraphics[width=0.33\textwidth,height=0.25\textwidth]{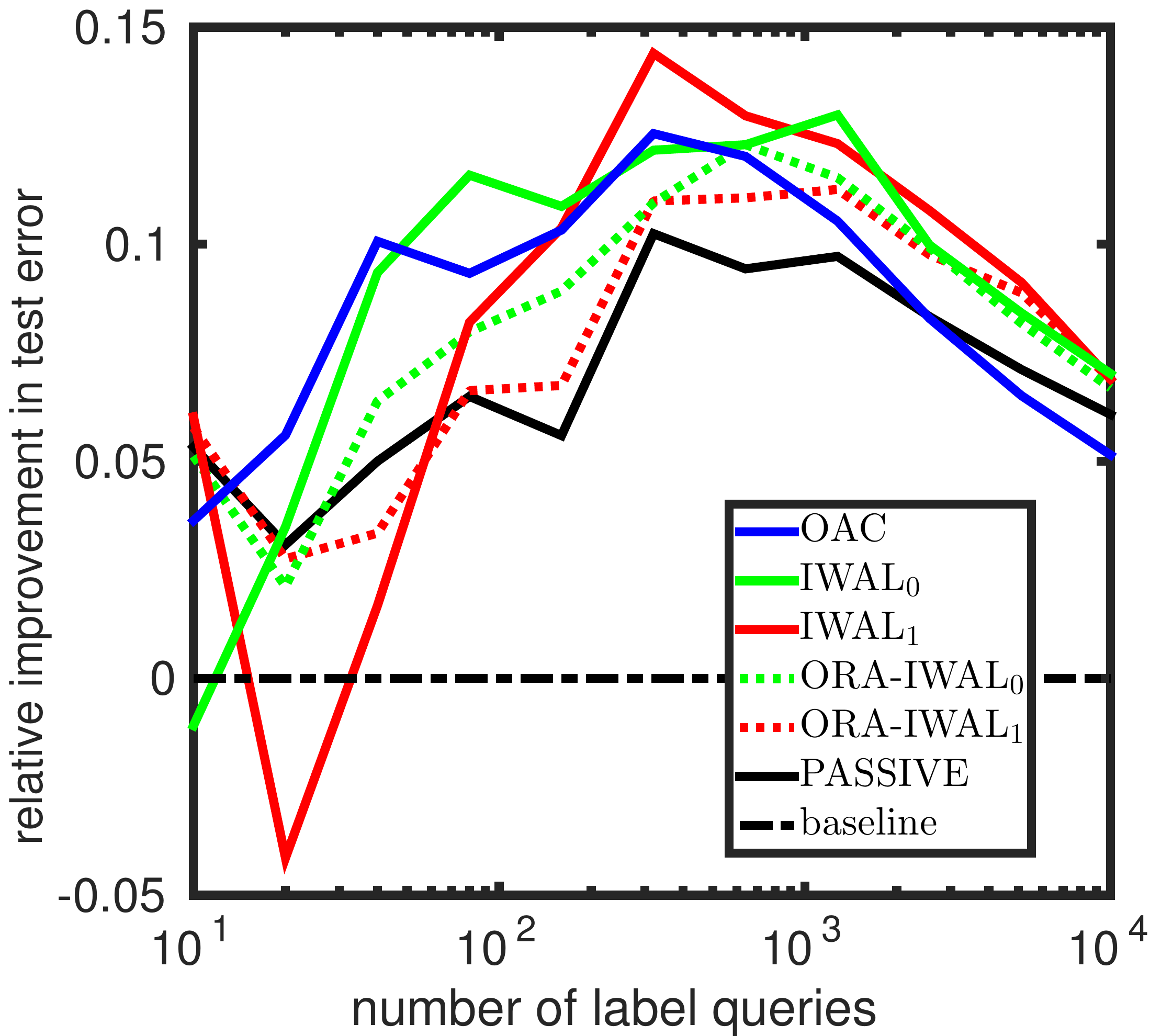}}  
\subfigure[The first permutation]{\label{fig:rel-err-gain-single}\includegraphics[width=0.33\textwidth,height=0.25\textwidth]{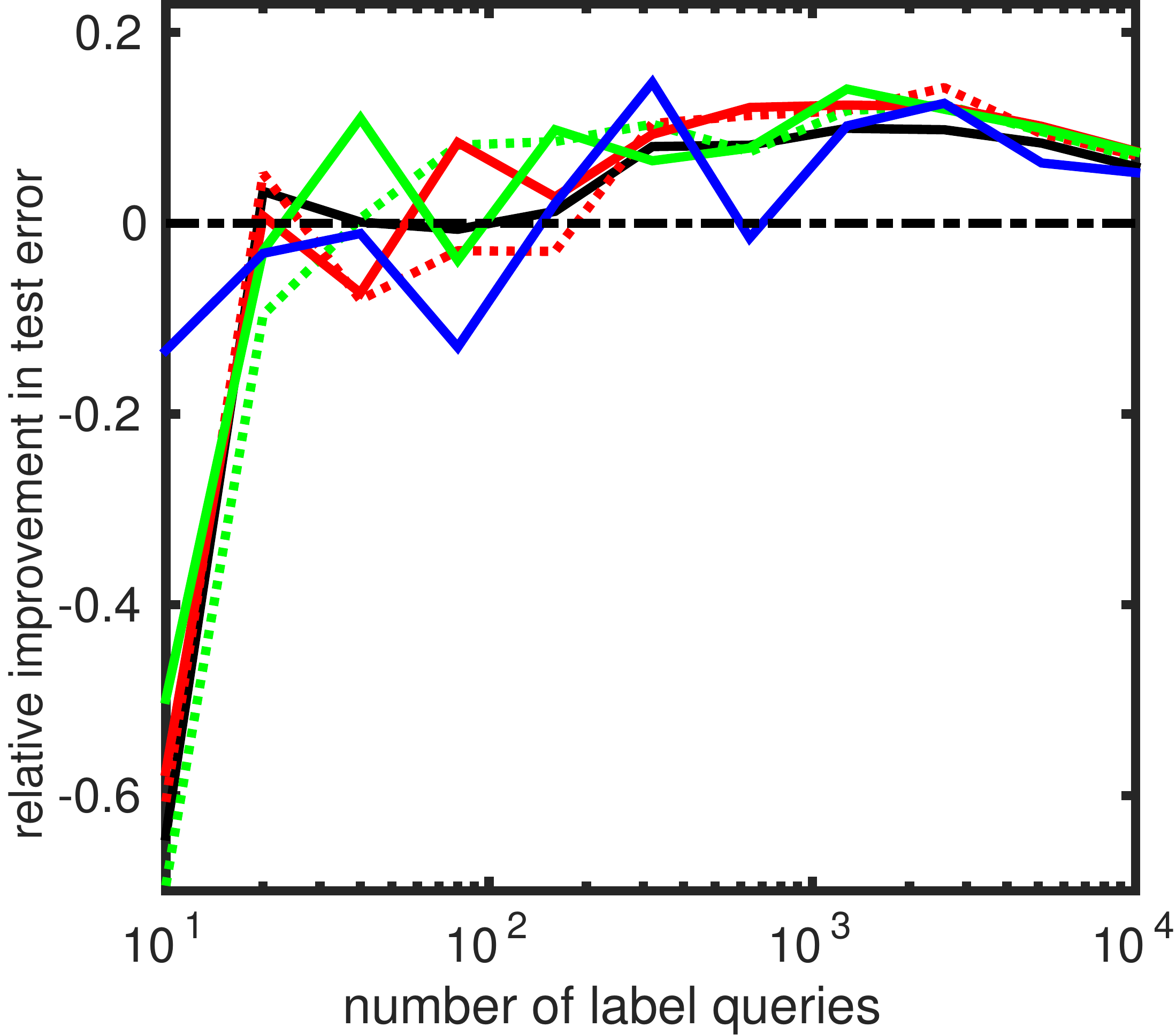}} 
\subfigure[Three quartiles over permutations]{\label{fig:rel-err-gain-quartiles}\includegraphics[width=0.33\textwidth,height=0.25\textwidth]{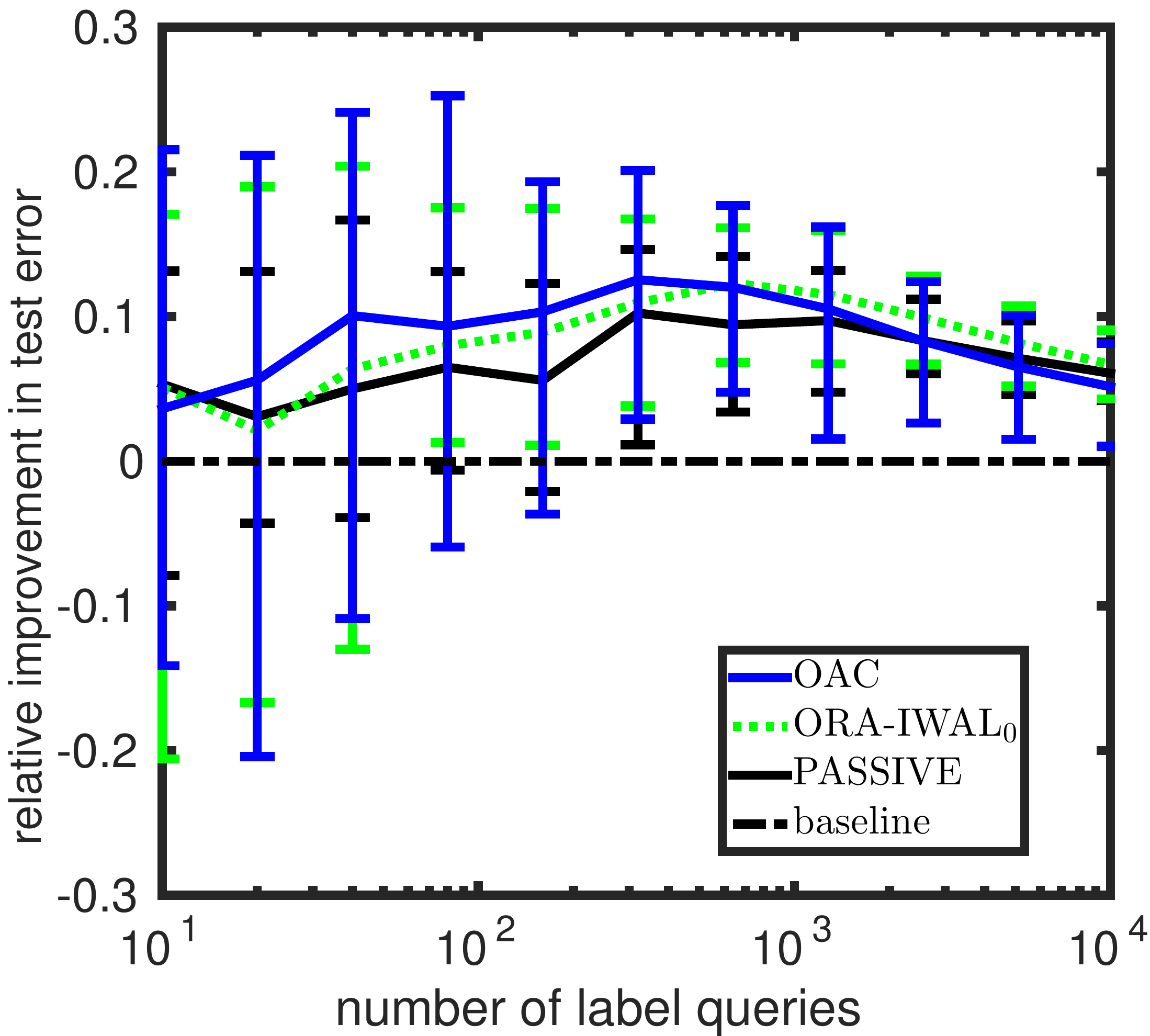}} 
\caption{Relative improvement in test error v.s. number of label queries under the best fixed hyper-parameter setting across datasets. Results are averaged over all datasets.}
\label{fig:rel-err-gain}
\end{figure}
To describe the behaviors of different algorithms in more details, we plot the relative improvement in test error
against number of label queries.
In Figure \ref{fig:rel-err-gain-median}, for each algorithm $a$ we identify the best fixed hyper-parameter setting
\begin{equation}
p^* := \arg \max_{p} \underset{d}{\mean} \; \underset{1 \leq j \leq 9}{\median} \left\{ \frac{\auc_{base}(d,j) - \auc_{a,p}(d,j)}{\auc_{base}(d,j)} \right\},
\end{equation} 
and plot the relative test error improvement by $a$ using $p^*$ averaged across all datasets at the 11 label budgets:
\begin{equation}
\left \{\left(10 \cdot 2^{(q-1)}, \underset{d}{\mean} \; \underset{1 \leq j \leq 9} \median \left\{\frac{\error_{base}(d,j,q) - \error_{a,p^*}(d,j,q)}{\error_{base}(d,j,5)}\right\}\right)\right\}_{q=1}^{11}.
\label{eq:rel-err-gain}
\end{equation}
The two IWAL algorithms start off badly at small numbers of label queries, but 
outperform other algorithms after 100-or-so label queries.
\cover performs better than the two Oracular-CAL algorithms until a few hundred label queries, 
but becomes worse afterwards. 

To give a sense of the variation due to random permutation, we plot in Figure \ref{fig:rel-err-gain-single} average results on the first permutation 
of each dataset, i.e., instead of taking the median in \eqref{eq:rel-err-gain}, we simply took results from the first permutation. 
Figures \ref{fig:rel-err-gain} and \ref{fig:rel-err-gain-single} suggest that variation due to permuting the data is quite large,
especially for the two Oracular-CAL algorithms and \iwal. Figure \ref{fig:rel-err-gain-quartiles} gives another view that shows variation for \cover, \iwalorazero, and \algrand: 
in addition to the median improvement, we also plot error bars corresponding to the first and the third quartiles of the relative improvement over random permutations, i.e., \eqref{eq:rel-err-gain} with median replaced by the two quartiles, respectively.

\begin{figure}[t]
\centering
\subfigure[Median over permutations]{\label{fig:rel-err-gain-per-data-median}\includegraphics[width=0.33\textwidth,height=0.25\textwidth]{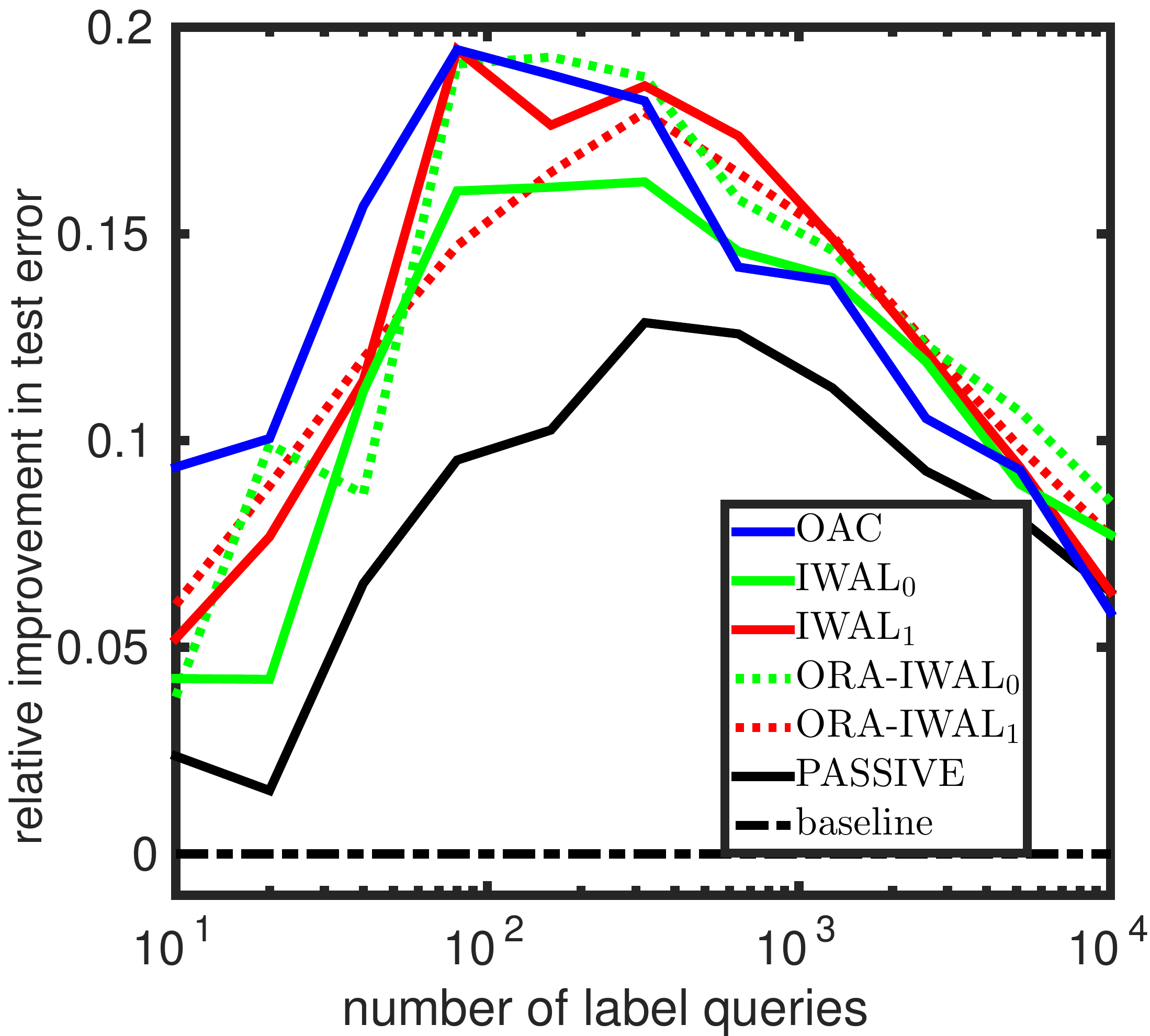}}  
\subfigure[The first permutation]{\label{fig:rel-err-gain-per-data-single}\includegraphics[width=0.33\textwidth,height=0.25\textwidth]{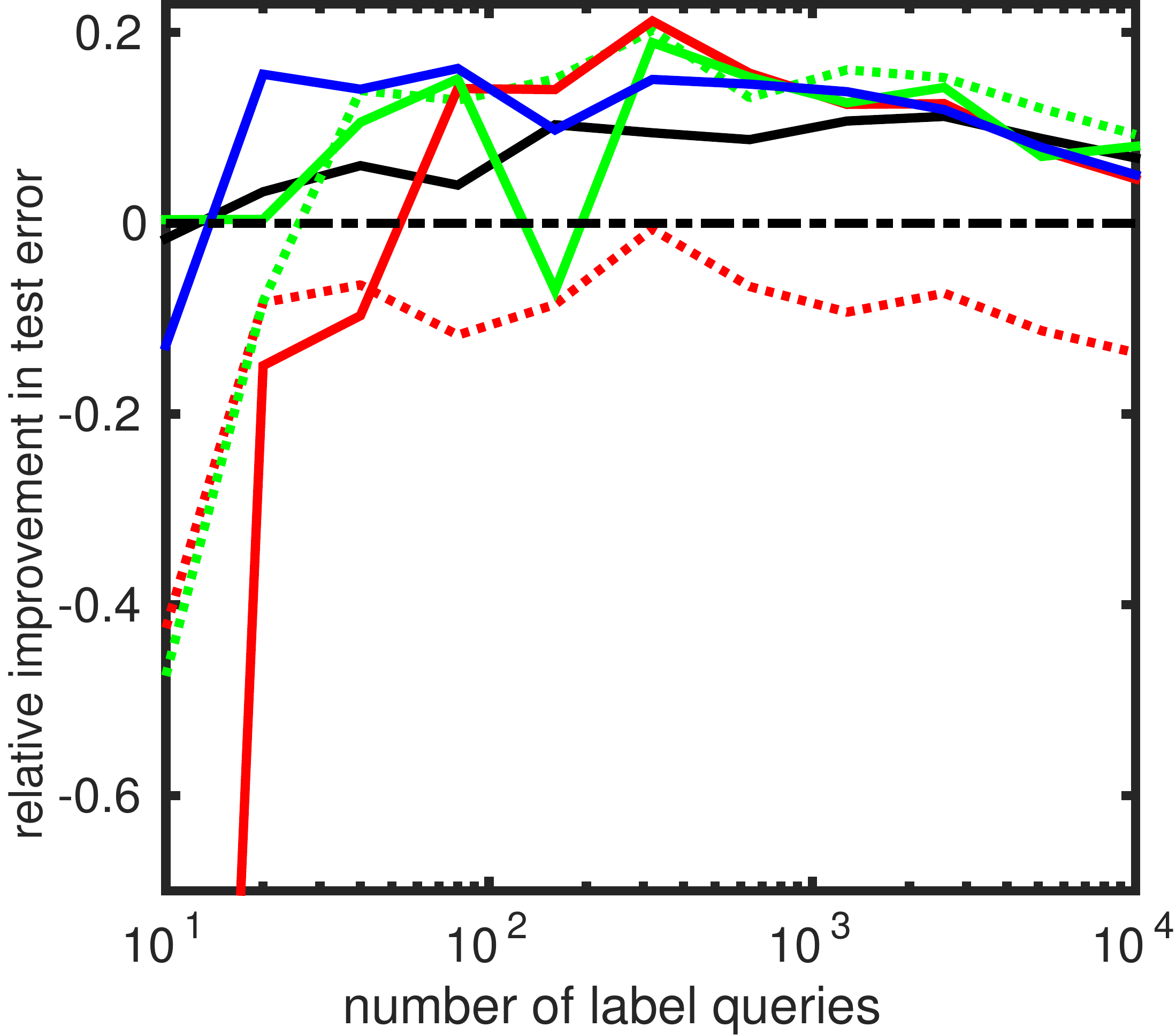}} 
\subfigure[Three quartiles over permutations]{\label{fig:rel-err-gain-per-data-quartiles}\includegraphics[width=0.33\textwidth,height=0.25\textwidth]{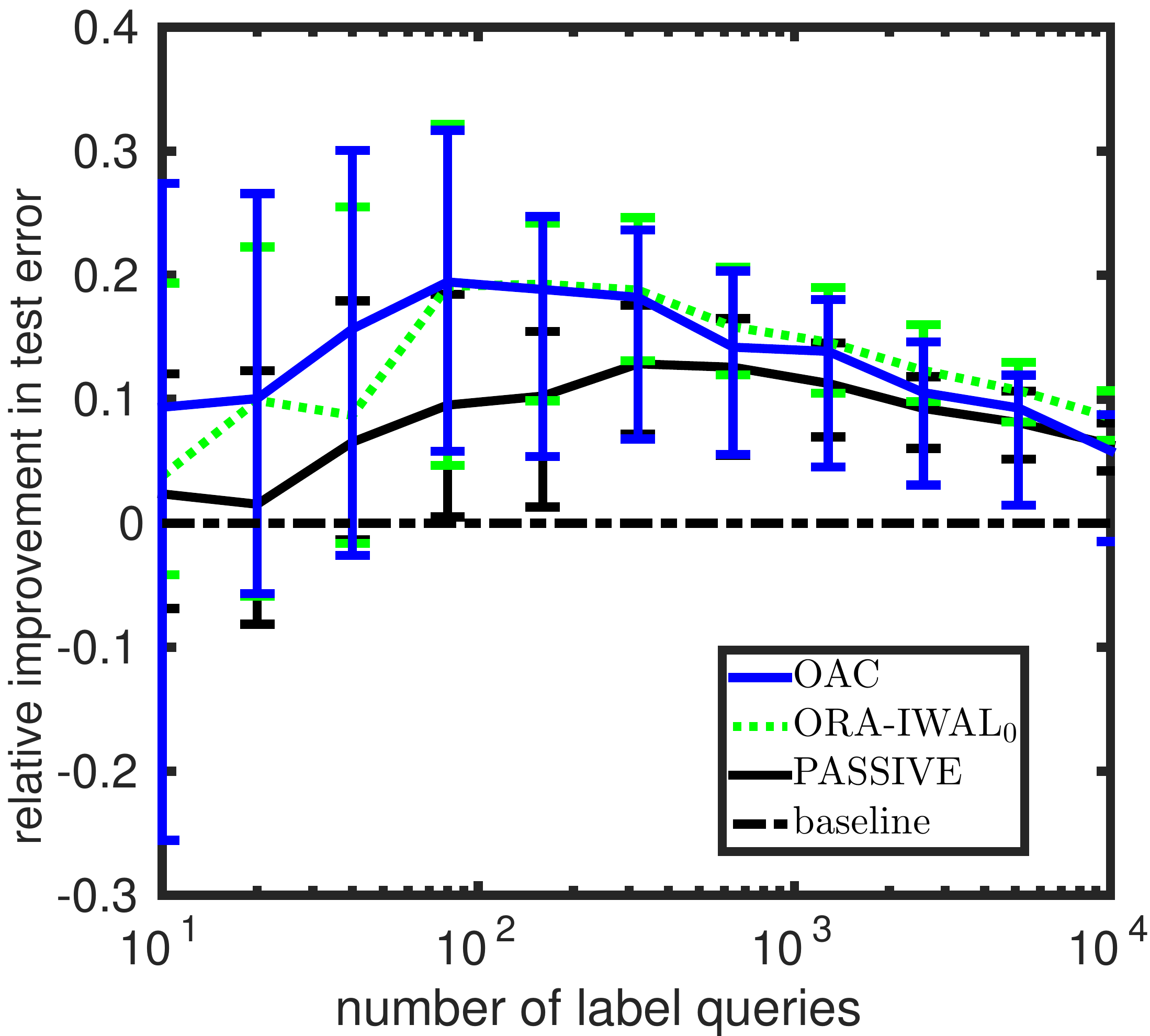}} 
\caption{Relative improvement in test error v.s. number of label queries under the hyper-parameter settings optimized on a per dataset basis. Results are averaged over all datasets.}
\label{fig:rel-err-gain-per-data}
\end{figure}
In Figures \ref{fig:rel-err-gain-per-data-median} to \ref{fig:rel-err-gain-per-data-quartiles}, 
we plot results obtained by each algorithm $a$ using the best hyper-parameter setting for each dataset $d$:
\begin{equation}
p^*_d := \arg \max_{p} \; \underset{1 \leq j \leq 9}{\median} \left\{ \frac{\auc_{base}(d,j) - \auc_{a,p}(d,j)}{\auc_{base}(d,j)} \right\}.
\end{equation} 
As expected, all algorithms perform better by using the best hyper-parameter setting for each dataset. 
Note that \cover performs the best at small numbers of label queries, but 
after a few hundred label queries all active learning algorithms perform quite similarly. 
%
\begin{figure}[t]
\centering
\includegraphics[width=0.45\textwidth,height=0.34\textwidth]{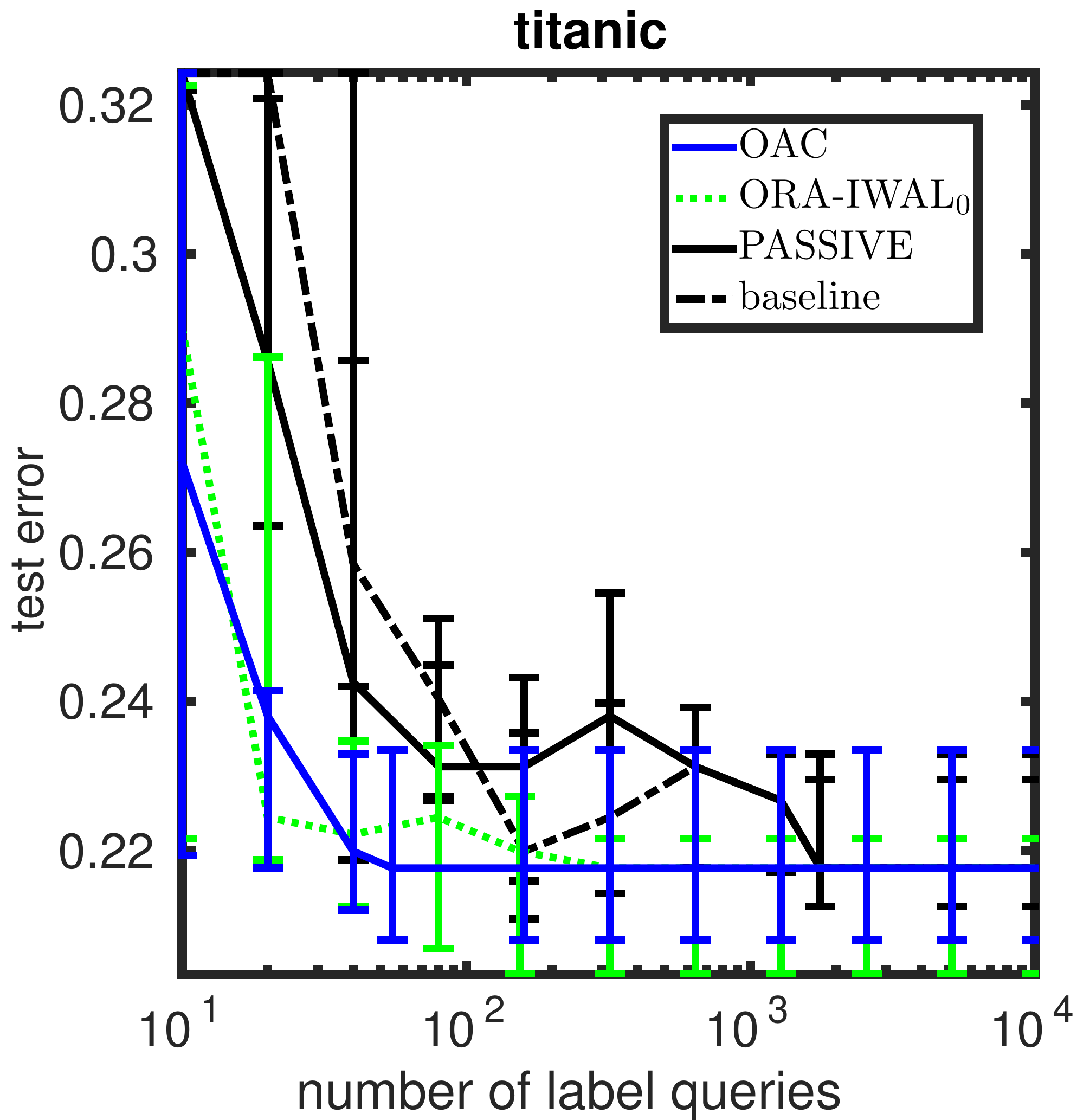} \qquad
\includegraphics[width=0.45\textwidth,height=0.34\textwidth]{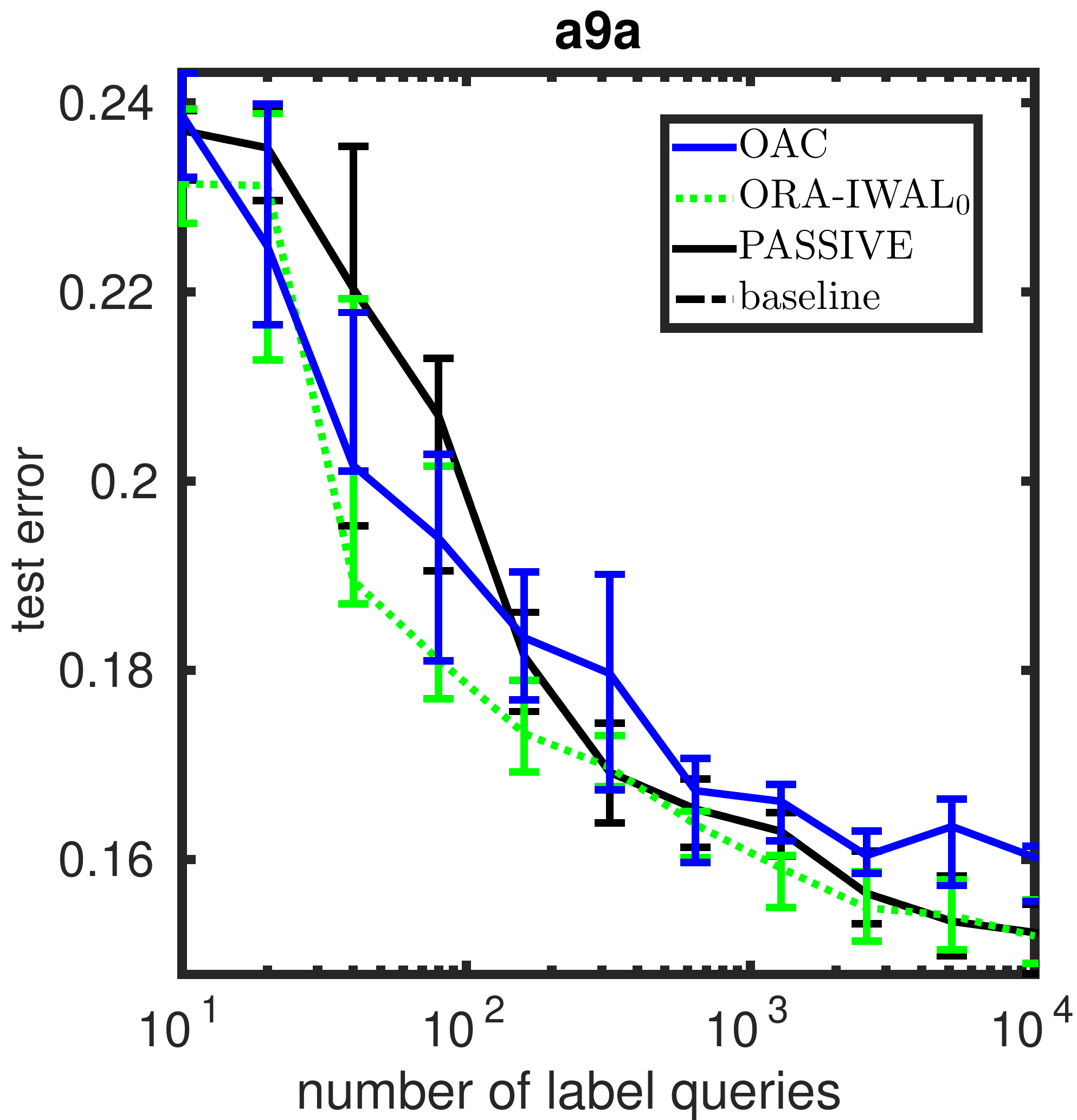}
\caption{Test error under the best hyper-parameter setting for each dataset v.s. number of label queries}
\label{fig:test-err-examples}
\end{figure}

Finally in Figure \ref{fig:test-err-examples}, we show the test error rates obtained by \cover, \iwalorazero, and \algrand 
against number of label queries for 2 of the 22 datasets, using the best hyper-parameter setting for each dataset. 
Results for all datasets and all algorithms are in Appendix \ref{appendix:results}.

In sum, when using the best fixed hyper-parameter setting, \iwalorazero outperforms other active learning algorithms. 
When using the best hyper-parameter setting tuned for each dataset, \cover and \iwalorazero 
perform equally well and better than other algorithms. 

\section{Analysis of generalization ability}
\label{proofs-regret}

In this section we present the main framework and analysis for the
results on the generalization properties of the \alglong
algorithm. Our analysis is broken up into several steps. We start by
setting up some additional notation for the proofs. Our analysis
relies on two deviation bounds for the empirical regret and the
empirical error of the ERM classifier. These are obtained by
appropriately applying Freedman-style concentration bounds for
martingales. Both these bounds depend on the variance and range of our
error and regret estimates for all classifiers $h \in \H$, and these
quantities are controlled using the constraints~\eqref{eq:queryp}
and~(\ref{eq:minbndcons}) in the definition of the optimization
problem (\optprob). Since our data consists of examples from different
epochs, which use different query probabilities $P_\epoch$, the above
steps with appropriate manipulations yield bounds for the epoch
$\epoch$, in terms of various quantities involving the previous
epochs. Theorem~\ref{thm:regret} and its corollaries are then obtained
by setting up appropriate inductive claims. We make this intuition
precise in the following sections.

\subsection{Framework for generalization analysis}
\label{sec:regret-prelim}

Before we can prove our main results, we recall some notations and
introduce a few additional ones. We also prove some technical
lemmas in this section which are used to prove our main results.

Recall the notation $\mbox{reg}(h,h') := \mbox{err}(h) -
\mbox{err}(h'), h^* \in \arg \min_{h \in \mathcal{H}} \mbox{err}(h)$,
$\reg(h) := \reg(h,h^*)$. Let $Z_\epoch$ denote the set of
importance-weighted examples in $\ztil_\epoch$, and the corresponding
empirical error is denoted as:

\begin{equation}
\err(h,Z_\epoch) := \frac{1}{\epend} \sum_{j=1}^\epoch
\sum_{i=\epend[j-1]+1}^{\epend[j]} \Big(\frac{ Q_i
  \bbmone(h(X_i) \neq Y_i \wedge X_i \in D_j)}{P_j(X_i)}
\Big).
\label{eqn:labeled}
\end{equation}
Taking expectations, we define the following quantities with respect
to the sequence of regions $\{D_\epoch\}$:

\begin{eqnarray}
  \err_\epoch(h) &:=& \E_{X,Y}[\mathbbm{1}(h(X) \neq Y \wedge X \in
    D_\epoch)], \label{eq:err_epoch}\\ 
  \erravg{\epoch}(h) &:=& \frac{1}{\epend}
  \sum_{j=1}^\epoch(\epend[j] - \epend[j-1]) \err_j(h). \nonumber
\end{eqnarray}
Intuitively, $\err_\epoch$ captures the population error of $h$,
restricted to only the examples in the disagreement region. This is
also the expectation of the sample error restricted to the
importance-weighted examples in epoch $\epoch$. Averaging these
quantities, we obtain $\erravg{\epoch}$ which is the expectation of
the sample error over $Z_\epoch$. Centering around the corresponding
errors of $h^*$, we obtain the following regret terms:

\begin{eqnarray}
  \reg_\epoch(h) &:=& \err_\epoch(h) - \err_\epoch(h^*), \nonumber\\
  \overline{\reg}_\epoch(h) &:=& \frac{1}{\epend}
  \sum_{j=1}^\epoch (\epend[j] - \epend[j-1]) \reg_j(h). \nonumber
\end{eqnarray}

While the above quantities only concern the importance-weighted
examples, it is also useful to measure error and regret terms over the
entire biased sample. We define the empirical error and regret on
$\tilde{Z}_\epoch$ as follows:
 
\begin{align*}
\err(h,\tilde{Z}_\epoch)
& :=\frac{1}{\epend} \sum_{j=1}^\epoch
\sum_{i=\epend[j-1]+1}^{\epend[j]} \Big(
\bbmone(h(X_i) \neq h_j(X_i) \wedge X_i \notin D_j)  + \frac{ Q_i
  \bbmone(h(X_i) \neq Y_i \wedge X_i \in D_j)}{P_j(X_i)}
\Big),\\  
\reg(h,h',\tilde{Z}_\epoch) & := \err(h,\tilde{Z}_\epoch) -
\err(h',\tilde{Z}_\epoch), 
\end{align*}
and the associated expected regret:

\begin{eqnarray}
  \regbiasi{\epoch}(h,h') & := & \mathbbm{E}_{X}[(\bbmone(h(X) \neq h_\epoch(X)) -
    \bbmone(h'(X) \neq h_\epoch(X)))\bbmone(X \notin D_\epoch)]
  +  \nonumber\\
  && \mathbbm{E}_{X,Y}[(\bbmone(h(X) \neq Y) - \bbmone(h'(X) \neq
    Y))\bbmone(X \in D_\epoch)], \label{eq:regbiasi}\\ 
  \regbias{\epoch}(h,h') & := &  \frac{1}{\epend}\sum_{j=1}^\epoch
  (\epend[j] - \epend[j-1]) \regbiasi{j}(h,h').
\label{eqn:regbias}
\end{eqnarray}
The
quantity $\regbias{\epoch}(h,h')$ will play quite a central role in
our analysis as it is the expectation of the empirical regret of $h$
relative to $h'$ on our biased sample $\ztil_\epoch$. We also recall
the earlier notations

\begin{eqnarray*}
  \Delta_{\epoch} &:=& c_1 \sqrt{\epsilon_{\epoch}
    \err(h_{\epoch+1},\tilde{Z}_{\epoch})} + c_2 \epsilon_{\epoch} \log
  \epend, \nonumber\\ 
  A_{\epoch+1} &:=& \{h \in \cH \mid \mbox{err}(h,\tilde{Z}_{\epoch})  -
  \mbox{err}(h_{\epoch+1},\tilde{Z}_{\epoch}) \leq \gamma
  \Delta_{\epoch}\}, \quad \mbox{and} \nonumber\\
    \Delta_{\epoch}^* &:=& 
  \begin{cases}
    \left( c_1 \sqrt{\epsilon_\epoch
      \overline{\err}_\epoch(h^*)} + c_2 \epsilon_\epoch
    \log \epend \right), & \epoch \geq 1. \nonumber\\ 
    \Delta_0, & m = 0.
  \end{cases}
\end{eqnarray*}

Unless stated otherwise, we adopt the convention that in the quantities
defined above, summations from $1$ to $m$
take the value of zero when $m = 0$.
We use the shorthand $\epochi$ to denote the epoch
containing example $i$.
We also sometimes use the shorthand $\reg(h,\ztil_{\epoch}) :=
\reg(h,h_{\epoch+1},\ztil_\epoch)$, $\regbiasi{\epoch}(h) :=
\regbiasi{\epoch}(h,h^*)$, and $\regbias{\epoch}(h) :=
\regbias{\epoch}(h,h^*)$.

With the notations in place, we start with an extremely important
lemma, which shows that the biased sample $\ztil$ which we create
introduces a bias in the favor of good hypotheses, overly penalizing
the bad hypotheses while favorably evaluating the optimal $h^*$. 

\begin{lemma}[Favorable Bias]
\label{lemma:favorable_bias}
$\forall \epoch \geq 1, \forall \bar{h} \in A_\epoch, \forall h \in
\H$, the following holds:
\begin{equation*}
  \regbiasi{\epoch}(h,\bar{h}) \geq \reg(h,\bar{h}).
\end{equation*}
\end{lemma}

The next key ingredient for our proofs is a deviation bound, which
will be appropriately used to control the deviation of the empirical
regret and error terms.

\begin{lemma}[Deviation Bounds]
  \label{lemma:reg_err_concen}	
  Pick $0 < \delta < 1/e$ such that $|\mathcal{H}| / \delta > \sqrt{192}$. 
  With probability at least $1 - \delta$ the following holds.
  For all $(h,h') \in \mathcal{H}^2$ and $\forall \epoch \geq 1$,
  
  \begin{align}
    &|\regbias{\epoch}(h,h') - \reg(h,h',\tilde{Z}_{\epoch})| \nonumber\\
    &\qquad \leq \sqrt{\frac{\epsilon_{\epoch}}{\epend} \sum_{i=1}^\epoch
      (\epend[i] - \epend[i-1]) \E_X\left[ \left(\bbmone(X
        \notin D_i) + \frac{\mathbbm{1}(X \in D_i)}{P_i(X)}\right)
        \bbmone(h(X) \neq h'(X))\right]} \nonumber\\
    &\qquad \qquad +\frac{\epsilon_{\epoch}}{P_{\min,\epoch}}, 
    \label{eqn:dev-bias} \\
    &|\err(h,Z_\epoch) - \erravg{\epoch}(h)| \nonumber\\
    &\qquad \leq \sqrt{\frac{\epsilon_\epoch}{\epend} \sum_{i=1}^\epoch
      (\epend[i] - \epend[i-1]) \E_{X,Y}\left[
        \frac{\mathbbm{1}(X \in D_i \wedge h(X) \neq Y)}{P_i(X)}\right]}
  + \frac{\epsilon_\epoch}{P_{\min,\epoch}}, 
  \label{eq:dev-label}
  \end{align}
  where 
  \begin{equation*}
    \epsilon_\epoch \;:=\; 32 \left(\frac{\log(|\mathcal{H}|/\delta) +
      \log \epend}{\epend} \right).
  \end{equation*}
\end{lemma}

The lemma is obtained by applying a form of Freedman's inequality
presented in Appendix~\ref{sec:deviation}. Intuitively, the deviations
are small so long as the average importance weights over the
disagreement region and the minimum query probability over the
disagreement region are well-behaved. This lemma also highlights why
$\regbias{\epoch}$ is a very natural quantity for our analysis, since
the empirical regret on our biased sample $\ztil$ concentrates around
it.

To keep the handling of probabilities simple, we assume for the bulk
of this section that the conclusions of
Lemma~\ref{lemma:reg_err_concen} hold deterministically. The failure
probability is handled once at the end to establish our main
results. Let $\event$ denote the event that the assertions of
Lemma~\ref{lemma:reg_err_concen} hold deterministically, and we know
that $\Pr(\event^C) \leq \delta$. Based on the above lemma, we obtain
the following propositions for the concentration of empirical regret
and error terms.

\begin{proposition}[Regret concentration]
  Fix an epoch $\epoch \geq 1$. Suppose the event $\event$ holds and
  assume that $h^* \in A_j$ for all epochs $j \leq \epoch$.
  \begin{align*}
    &|\reg(h, h^*, \ztil_\epoch) - \regbias{\epoch}(h, h^*)| \\
    &\qquad \leq
    \frac{1}{4} \regbias{\epoch}(h) + 2\alpha
    \sqrt{\frac{\epsilon_\epoch}{\epend} \sum_{i=1}^\epoch (\epend[i]
      - \epend[i-1]) \regi{i}(h_i)} + 2\alpha
    \sqrt{3\erravg{\epoch}(h^*) \epsilon_\epoch}\\ 
    &\qquad \qquad + \beta \sqrt{2\gamma
      \epsilon_\epoch\Delta_\epoch \sum_{i=1}^\epoch (\epend[i] -
      \epend[i-1]) (\reg(h, \ztil_{i-1}) + \reg(h^*, \ztil_{i-1}))} +
    4\Delta_\epoch
  \end{align*}
  \label{prop:dev-regret-bias}
\end{proposition}

We need an analogous result for the empirical error of the ERM at each
epoch. 

\begin{proposition}[Error concentration]
  Fix an epoch $\epoch \geq 1$. Suppose the event $\event$ holds and
  assume that $h^* \in A_j$ for all epochs $j \leq \epoch$.
  \begin{align*}
    |\erravg{\epoch}(h^*) - \err(h_{\epoch+1}, \ztil_\epoch)| \leq
    \frac{\erravg{\epoch}(h^*)}{2} + \frac{3\Delta_\epoch}{2} +
    \reg(h^*, h_{\epoch+1}, \ztil_{\epoch}).
  \end{align*}
  \label{prop:err-emp-best}
\end{proposition}

We now present the proofs of our main results based on these
propositions.

\subsection{Proofs of main results}
\label{sec:proofs-regret-main}

We prove a more general version of the
theorem. Theorem~\ref{thm:regret} and its corollaries follow as
consequences of this more general result.

\begin{theorem}
  For all epochs $\epoch = 1, 2, \ldots, \epmax$ and all $h \in \H$,
  the following holds with probability at least $1 - \delta$:
  \begin{eqnarray}
    |\reg(h, h^*, \ztil_\epoch) - \regbias{\epoch}(h, h^*)| &\leq&
    \frac{1}{2} \regbias{\epoch}(h, h^*) +
    \frac{\regconst}{4}\Delta_\epoch,
    \label{eqn:main-regret}\\
    \reg(h^*, h_{\epoch+1}, \ztil_\epoch) &\leq& \frac{\regconst
      \Delta_\epoch}{4}  
      \quad \text{and} \quad h^* \in A_i, 
    \label{eq:h_best-main}\\
|\erravg{\epoch}(h^*) - \err(h_{\epoch+1}, \ztil_\epoch)| &\leq&
\frac{\erravg{\epoch}(h^*)}{2} + \frac{\regconst}{2}\Delta_\epoch.
\label{eq:err_dev} 
  \end{eqnarray}
  \label{thm:main-regret}
\end{theorem}

The theorem is proved inductively. We first give the proof outline for
this theorem, and then show how Theorem~\ref{thm:regret} and its
corollaries follow.

\subsubsection{Proof of Theorem~\ref{thm:main-regret}}

The theorem is proved via induction. Let us start with the base
case for $\epoch=1$. Clearly, $A_1 = \cH \ni h^*$, and
\begin{align*}
  |\reg(h, h^*, \ztil_1) - \regbias{1}(h,h^*)| \leq 1 \leq \regconst
  \Delta_1 / 4,  
\end{align*}
since $P_{\min,1} = 1$. The conclusions for the second and third
statements follow similarly. This establishes the base case. Let us
now assume that the hypothesis holds for $i = 1,2, \ldots, \epoch-1$
and we establish it for the epoch $i = \epoch$. We start from the
conclusion of Proposition~\ref{prop:dev-regret-bias}, which yields
\begin{align*}
 &|\reg(h, h^*, \ztil_\epoch) - \regbias{\epoch}(h, h^*)| \\
&\leq
  \frac{1}{4} \regbias{\epoch}(h) + \underbrace{2\alpha
    \sqrt{\frac{\epsilon_\epoch}{\epend} \sum_{i=1}^\epoch (\epend[i]
      - \epend[i-1]) \regi{i}(h_i)}}_{\term_1} + \underbrace{2\alpha
      \sqrt{3\erravg{\epoch}(h^*) \epsilon_\epoch}}_{\term_2}\\ 
  & \qquad \qquad+ \underbrace{\beta
    \sqrt{2\gamma \epsilon_\epoch\Delta_\epoch \sum_{i=1}^\epoch
      (\epend[i] - \epend[i-1]) (\reg(h, \ztil_{i-1}) + \reg(h^*,
      \ztil_{i-1}))}}_{\term_3} + 4\Delta_\epoch
\end{align*}
We now control $\term_1$, $\term_2$ and $\term_3$ in the sum using our
inductive hypothesis and the propositions in a series of lemmas. To
state the lemmas cleanly, let $\event_\epoch$ refer to the event where
the bounds~\eqref{eqn:main-regret}-\eqref{eq:err_dev} hold at epoch
$\epoch$. Then we have the following lemmas. The first lemma gives a
bound on $\term_1$. 

\begin{lemma}
  Suppose that the event $\event$ holds and that the events $\event_i$
  hold for all epochs $i = 1,2,\ldots, \epoch-1$. Then we have

  \begin{align*}
    2\alpha \sqrt{\frac{\epsilon_\epoch}{\epend} \sum_{i=1}^\epoch
      (\epend[i] - \epend[i-1]) \regi{i}(h_i)} &\leq
    \frac{\regconst\Delta_\epoch}{12} + 24 \alpha^2 \epsilon_\epoch
    \log \epend.
  \end{align*}
  \label{lemma:regret-t1}
\end{lemma}

Intuitively, the lemma holds since Lemma~\ref{lemma:favorable_bias}
allows us to bound $\regi{i}(h_i)$ with $\regbias{i-1}(h_i)$. The
latter is then controlled using the event $\event_i$. Some algebraic
manipulations then yield the lemma, with a detailed proofs in
Appendix~\ref{sec:regret-lemmas}. We next present a lemma that helps
us control $\term_2$. 

\begin{lemma}
  Suppose that the event $\event$ holds and that the events $\event_i$
  hold for all epochs $i = 1,2,\ldots, \epoch-1$. Then we have

  \begin{align*}
    2\alpha \sqrt{3\erravg{\epoch}(h^*) \epsilon_\epoch} &\leq 2\alpha
    \sqrt{6 \epsilon_\epoch \err(h_{\epoch+1}, \ztil_\epoch)} +
    \Delta_\epoch + \frac{1}{4} \reg(h^*, h_{\epoch+1}, \ztil_\epoch)
    + 33\alpha^2 \epsilon_\epoch.
  \end{align*}
  \label{lemma:regret-t2}
\end{lemma}

The lemma follows more or less directly from
Proposition~\ref{prop:err-emp-best} combined with some
algebra. Finally, we present a lemma to bound $\term_3$. 

\begin{lemma}
  Suppose that the event $\event$ holds and that the events $\event_i$
  hold for all epochs $i = 1,2,\ldots, \epoch-1$. Then we have

  \begin{align*}
    \beta \sqrt{2\gamma \epsilon_\epoch\Delta_\epoch \sum_{i=1}^\epoch
      (\epend[i] - \epend[i-1]) (\reg(h, \ztil_{i-1}) + \reg(h^*,
      \ztil_{i-1}))} &\leq \frac{1}{4} \regbias{\epoch}(h,h^*) +
    \frac{7 \regconst \Delta_\epoch }{72} .
  \end{align*}
  \label{lemma:regret-t3}
\end{lemma}

The $\reg(h^*, h_i, \ztil_{i-1})$ terms in the lemma are bounded
directly due to the event $\event_i$. For the second term, we observe
that the empirical regret of $h$ relative to $h_i$ is not too
different from the empirical regret to $h^*$ (since $h^*$ has a small
empirical regret by $\event_i$). Furthermore, the empirical regret to
$h^*$ is close to $\regbias{i-1}(h, h^*)$ by the event
$\event_i$. These observations, along with some technical
manipulations yield the lemma.

Given these lemmas, we can now prove the theorem in a relatively
straightforward manner. Given our inductive hypothesis, the events
$\event_i$ indeed hold for all epochs $i = 1,2,\ldots, \epoch-1$ which
allows us to invoke the lemmas. Substituting the above bounds on
$\term_1$ from Lemma~\ref{lemma:regret-t1}, $\term_2$ from
Lemma~\ref{lemma:regret-t2} and $\term_3$ from~\ref{lemma:regret-t3}
into Proposition~\ref{prop:dev-regret-bias} yields

\begin{align*}
  &|\reg(h, h^*, \ztil_\epoch) - \regbias{\epoch}(h, h^*)|\\
  &\leq \frac{1}{4} \regbias{\epoch}(h) +
  \frac{\regconst\Delta_\epoch}{12} + 24 \alpha^2 \epsilon_\epoch \log
  \epend + 2\alpha\sqrt{6\epsilon_\epoch \err(h_{\epoch+1},
    \ztil_\epoch)} + \Delta_\epoch \\&\qquad \qquad + \frac{1}{4} \reg(h^*, h_{\epoch+1},
  \ztil_\epoch) + 33\alpha^2 \epsilon_\epoch + \frac{1}{4}
  \regbias{\epoch}(h, h^*) + \frac{7 \regconst \Delta_\epoch}{72} + 
  4\Delta_\epoch\\  
  &\leq \frac{1}{2}\regbias{\epoch}(h, h^*) +
  57\alpha^2\epsilon_\epoch\log \epend + 
  \frac{13 \regconst}{72} \Delta_\epoch +
  2\alpha\sqrt{6\epsilon_\epoch \err(h_{\epoch+1}, 
    \ztil_\epoch)} + 5\Delta_\epoch \\&\qquad \qquad + \frac{1}{4} \reg(h^*,
  h_{\epoch+1}, \ztil_\epoch) 
\end{align*}
Further recalling that $c_1 \geq 2\alpha \sqrt{6}$ and $c_2 \geq
57\alpha^2$ by our assumptions on constants, we obtain

\begin{align}
  |\reg(h, h^*, \ztil_\epoch) - \regbias{\epoch}(h,h^*)| \leq
  \frac{1}{2}\regbias{\epoch}(h, h^*) + \frac{13 \regconst}{72}
  \Delta_\epoch + 6\Delta_\epoch + \frac{1}{4} \reg(h^*, h_{\epoch+1},
  \ztil_\epoch).
  \label{eqn:regret-dev-almost}
\end{align}
To complete the proof of the bound~(\ref{eqn:main-regret}), we now
substitute $h = h_{\epoch+1}$ in the above bound, which yields

\begin{align*}
  \frac{1}{2}\regbias{\epoch}(h_{\epoch+1}, h^*) - \frac{5}{4}\reg(h,
  h^*, \ztil_\epoch) \leq \frac{13 \regconst}{72} \Delta_\epoch +
  6\Delta_\epoch.
\end{align*}
Since $h^* \in A_i$ for all epochs $i \leq \epoch$, we have
$\regbias{\epoch}(h, h^*) \geq \reg(h, h^*) \geq 0$ for all
classifiers $h \in \H$. Consequently, we see that

\begin{align}
  \reg(h^*, h_{\epoch+1}, \ztil_m) = -\reg(h_{\epoch+1}, h^*,
  \ztil_\epoch) \leq \frac{52 \regconst}{360} \Delta_\epoch +
  \frac{24}{5}\Delta_\epoch \leq \frac{\regconst}{4} \Delta_\epoch,
  \label{eqn:emp-reg-opt}
\end{align}
where the last inequality uses the condition $38\regconst \geq
1728$. We can now substitute this back into our earlier
bound~\eqref{eqn:regret-dev-almost} and obtain

\begin{align*}
  &|\reg(h, h^*, \ztil_\epoch) - \regbias{\epoch}(h,h^*)| \\
  &\qquad \qquad \leq
  \frac{1}{2}\regbias{\epoch}(h, h^*) + \frac{13 \regconst}{72}
  \Delta_\epoch + 6\Delta_\epoch + \frac{\regconst}{16} \Delta_\epoch
  \leq \frac{1}{2}\regbias{\epoch}(h, h^*) + \frac{\regconst}{4}
  \Delta_\epoch, 
\end{align*}
where we use the condition $\regconst/144 \geq 6$. This completes the
proof of the first part of our inductive claim.

For the second part, this is almost a by product of the first part
through Equation~\eqref{eqn:emp-reg-opt}. Recalling that $\gamma \geq
\regconst/4$ by assumption, this ensures that $h^* \in A_{\epoch+1}$.

We next establish the third part of the claim. This is obtained by
combining our bound~\eqref{eqn:emp-reg-opt} with
Proposition~\ref{prop:err-emp-best}. We have

\begin{align*}
  |\erravg{\epoch}(h^*) - \err(h_{\epoch+1}, \ztil_\epoch)| &\leq
  \frac{\erravg{\epoch}(h^*)}{2} + \frac{3\Delta_\epoch}{2} +
  \reg(h^*, h_{\epoch+1}, \ztil_{\epoch})\\
  &\leq \frac{\erravg{\epoch}(h^*)}{2} + \frac{3\Delta_\epoch}{2} +
  \frac{\regconst \Delta_\epoch}{4} \\
  &\leq \frac{\erravg{\epoch}(h^*)}{2} + \frac{\regconst
    \Delta_\epoch}{2}, 
\end{align*}
since $\regconst \geq 6$. This completes the third part.

Finally, note that our analysis has been conditioned on the event
$\event$ so far. By Lemma~\ref{lemma:reg_err_concen}, $\Pr(\event^C)
\leq \delta$, which completes the proof of the theorem.

We now provide a proof for Theorem~\ref{thm:regret}.

\subsubsection{Proof of Theorem~\ref{thm:regret}}
\label{sec:proof-thm-regret}

We only prove the first part of the theorem. The second part is simply
a restatement of the inequality~\eqref{eq:h_best-main} in
Theorem~\ref{thm:main-regret}. The first part is essentially a
restatement of~(\ref{eqn:main-regret}) in
Theorem~\ref{thm:main-regret}, except the bound uses $\Delta_\epoch^*$
instead of $\Delta_\epoch$. In order to prove the theorem, pick any
epoch $\epoch \leq \epmax$ and $h \in A_{\epoch+1}$. Because $h^* \in
A_j, 1 \leq j \leq \epoch+1$, we have by Lemma
\ref{lemma:favorable_bias} that
\begin{equation*}
  \reg(h) \leq \regbias{\epoch}(h,h^*).
\end{equation*}
It then suffices to bound $\regbias{\epoch}(h,h^*)$. By the deviation
bound \eqref{eqn:main-regret}, we have
\begin{align*}
  \regbias{\epoch}(h, h^*) &\leq \reg(h, h^*, \ztil_\epoch) +
  \frac{1}{2}\regbias{\epoch}(h, h^*) +
  \frac{\regconst}{4}\Delta_\epoch\\
  &\leq \reg(h, h_{\epoch+1}, \ztil_\epoch) +
  \frac{1}{2}\regbias{\epoch}(h, h^*) +
  \frac{\regconst}{4}\Delta_\epoch \\ 
  &\leq \frac{1}{2}\regbias{\epoch}(h, h^*) + \left(\gamma +
  \frac{\regconst}{4}\right)\Delta_\epoch. 
\end{align*}
Rearranging terms leads to 
\begin{equation*}
  \regbias{\epoch}(h,h^*) \leq 4 \gamma \Delta_\epoch
\end{equation*}
because $ \gamma \geq \regconst / 4$. Now we show that $\Delta_\epoch
\leq 4 \Delta_\epoch^*$, which leads to the desired result.  It is
trivially true for $\epoch = 1$ because $\Delta_1^* = \Delta_1$. For
$\epoch \geq 2$, by the deviation bound on the empirical error
\eqref{eq:err_dev} we have
\begin{eqnarray*}
  \Delta_\epoch &\leq& c_1 \sqrt{\epsilon_\epoch \left(\frac{3}{2}
    \erravg{\epoch}(h^*) + \frac{\regconst}{2} \Delta_\epoch\right)} +
  c_2 \epsilon_\epoch \log \epend\\ 
  &\leq& 2 c_1 \sqrt{\epsilon_\epoch \erravg{\epoch}(h^*)} +
  \sqrt{\frac{c_1^2 \epsilon_\epoch \regconst}{2}\Delta_\epoch} + c_2
  \epsilon_\epoch \log \epend\\  
  &\leq& 2 c_1 \sqrt{\epsilon_\epoch \erravg{\epoch}(h^*)} +
  \frac{c_1^2 \epsilon_\epoch \regconst}{4} + \frac{\Delta_\epoch}{2}
  + c_2 \epsilon_\epoch \log \epend\\ 
  &\leq& 2 \Delta_\epoch^* + \frac{\Delta_\epoch}{2},
\end{eqnarray*}
where the last inequality uses our choice of constants $c_1^2
\regconst / 4 \leq c_2$. Rearranging terms completes the proof.

\section{Conclusion}

In this paper, we proposed a new algorithm for agnostic active
learning in a streaming setting. The algorithm has strong theoretical
guarantees, maintaining good generalization properties while attaining
a low label complexity in favorable settings. Specifically, we show
that the algorithm has an optimal performance in a disagreement-based
analysis of label complexity, as well in special cases such as
realizable problems and under Tsybakov's low-noise
condition. Additionally, we present an interesting example that
highlights the structural difference between our algorithm and some
predecessors in terms of label complexities. Indeed a key improvement
of our algorithm is that we do not always need to query over the
entire disagreement region--a limitation of most computationally
efficient predecessors. This is achieved through a careful
construction of an optimization problem defining good query
probability functions, which relies on using refined data-dependent
error estimates.

We complement our theoretical analysis with an extensive empirical
evaluation of several approaches across a suite of 22 datasets. The
experiments show both the pros and cons of our proposed method, which
performs well when hyperparameter tuning is allowed, but suffers from
lack of robustness when we fix these hyperparameters across
datasets. Such a comprehensive empirical evaluation on a range of
diverse datasets has not been previously done for agnostic active
learning algorithms before to our knowledge, and is a key contribution
of this work.

We believe that our work naturally leads to several interesting
directions for future research. As the example in
Section~\ref{sec:example} reveals, the worst-case label complexity
analysis in Theorem~\ref{thm:label} is rather pessimistic. It would be
interesting to obtain sharper characterization of the label
complexity, by exploiting the structure of the query probability
function over the disagreement region. This would likely involve
understanding more fine-grained properties that make a problem easy or
hard for active learning beyond the disagreement coefficient, and such
a development might also lead to better algorithms. A limitation of
the current theory is the somewhat poor dependence in
Theorem~\ref{thm:optprob-unlabeled} on the number of unlabeled
examples needed to solve the optimization problem. Ideally, we would
like to be able to use $\order(\epend)$ unlabeled examples to solve
(\optprob) at epoch $\epoch$, and improving this dependence is perhaps
the most important direction for future work. Finally, while \alg is
extremely attractive from a theoretical standpoint, a direct
implementation still seems somewhat impractical. Obtaining theory for
an algorithm even closer to the practical variant \cover would be an
important step in bringing the theory and implementation closer.

\section*{Acknowledgements}
The authors would like to thank Kamalika Chaudhuri for helpful
initial discussions. 

\bibliographystyle{plainnat} \bibliography{bib}

\begin{thebibliography}{24}
\providecommand{\natexlab}[1]{#1}
\providecommand{\url}[1]{\texttt{#1}}
\expandafter\ifx\csname urlstyle\endcsname\relax
  \providecommand{\doi}[1]{doi: #1}\else
  \providecommand{\doi}{doi: \begingroup \urlstyle{rm}\Url}\fi

\bibitem[Balcan and Long(2013)]{balcan2013active}
Maria-Florina Balcan and Phil Long.
\newblock Active and passive learning of linear separators under log-concave
  distributions.
\newblock In \emph{Conference on Learning Theory}, pages 288--316, 2013.

\bibitem[Balcan et~al.(2006)Balcan, Beygelzimer, and
  Langford]{balcan2006agnostic}
Maria-Florina Balcan, Alina Beygelzimer, and John Langford.
\newblock Agnostic active learning.
\newblock In \emph{Proceedings of the 23rd international conference on Machine
  learning}, pages 65--72. ACM, 2006.

\bibitem[Balcan et~al.(2007)Balcan, Broder, and Zhang]{balcan2007margin}
Maria-Florina Balcan, Andrei Broder, and Tong Zhang.
\newblock Margin based active learning.
\newblock In \emph{Proceedings of the 20th annual conference on Learning
  theory}, pages 35--50. Springer-Verlag, 2007.

\bibitem[Bartlett and Mendelson(2002)]{Bartlett02}
P.~Bartlett and S.~Mendelson.
\newblock {G}aussian and {R}ademacher complexities: {R}isk bounds and
  structural results.
\newblock \emph{Journal of Machine Learning Research}, 3:\penalty0 463--482,
  2002.

\bibitem[Beygelzimer et~al.(2009)Beygelzimer, Dasgupta, and
  Langford]{BeygelDL09}
A.~Beygelzimer, S.~Dasgupta, and J.~Langford.
\newblock Importance weighted active learning.
\newblock In \emph{ICML}, 2009.

\bibitem[Beygelzimer et~al.(2010)Beygelzimer, Hsu, Langford, and
  Zhang]{BeygelHLZ10}
A.~Beygelzimer, D.~Hsu, J.~Langford, and T.~Zhang.
\newblock Agnostic active learning without constraints.
\newblock In \emph{NIPS}, 2010.

\bibitem[Blum et~al.(1999)Blum, Kalai, and Langford]{blum1999beating}
Avrim Blum, Adam Kalai, and John Langford.
\newblock Beating the hold-out: Bounds for k-fold and progressive
  cross-validation.
\newblock In \emph{Proceedings of the twelfth annual conference on
  Computational learning theory}, pages 203--208. ACM, 1999.

\bibitem[Castro and Nowak(2008)]{Castro2008}
R.M. Castro and R.D. Nowak.
\newblock Minimax bounds for active learning.
\newblock \emph{Information Theory, IEEE Transactions on}, 54\penalty0
  (5):\penalty0 2339 --2353, 2008.

\bibitem[Cesa-Bianchi et~al.(2004)Cesa-Bianchi, Conconi, and
  Gentile]{cesa2004generalization}
Nicolo Cesa-Bianchi, Alex Conconi, and Claudio Gentile.
\newblock On the generalization ability of on-line learning algorithms.
\newblock \emph{Information Theory, IEEE Transactions on}, 50\penalty0
  (9):\penalty0 2050--2057, 2004.

\bibitem[Cohn et~al.(1994)Cohn, Atlas, and Ladner]{CohnAL94}
D.~Cohn, L.~Atlas, and R.~Ladner.
\newblock Improving generalization with active learning.
\newblock \emph{Machine Learning}, 15:\penalty0 201--221, 1994.

\bibitem[Dasgupta(2005)]{dasgupta2005coarse}
S.~Dasgupta.
\newblock Coarse sample complexity bounds for active learning.
\newblock In \emph{Advances in Neural Information Processing Systems 18}, 2005.

\bibitem[Dasgupta et~al.(2007)Dasgupta, Hsu, and Monteleoni]{DasguptaHM07}
S.~Dasgupta, D.~Hsu, and C.~Monteleoni.
\newblock A general agnostic active learning algorithm.
\newblock In \emph{NIPS}, 2007.

\bibitem[Freedman(1975)]{Freedman75}
D.~A. Freedman.
\newblock On tail probabilities for martingales.
\newblock \emph{The Annals of Probability}, 3\penalty0 (1):\penalty0 100--118,
  February 1975.

\bibitem[Hanneke(2009)]{Hanneke09}
S.~Hanneke.
\newblock \emph{Theoretical Foundations of Active Learning}.
\newblock PhD thesis, Carnegie Mellon University, 2009.

\bibitem[Hanneke(2014)]{Hanneke2014}
Steve Hanneke.
\newblock Theory of disagreement-based active learning.
\newblock \emph{Foundations and Trends in Machine Learning}, 7\penalty0
  (2-3):\penalty0 131--309, 2014.

\bibitem[Horvitz and Thompson(1952)]{horvitz1952generalization}
D.~G. Horvitz and D.~J. Thompson.
\newblock A generalization of sampling without replacement from a finite
  universe.
\newblock \emph{J. Amer. Statist. Assoc.}, 47:\penalty0 663--685, 1952.
\newblock ISSN 0162-1459.

\bibitem[Hsu(2010)]{hsu2010algorithms}
Daniel~J. Hsu.
\newblock \emph{Algorithms for Active Learning}.
\newblock PhD thesis, University of California at San Diego, 2010.

\bibitem[Kakade and Tewari(2009)]{KakadeTe09}
S.~M. Kakade and A.~Tewari.
\newblock On the generalization ability of online strongly convex programming
  algorithms.
\newblock In \emph{Advances in Neural Information Processing Systems 21}, 2009.

\bibitem[Kakade et~al.(2009)Kakade, Sridharan, and Tewari]{KakadeSrTe2009}
Sham~M Kakade, Karthik Sridharan, and Ambuj Tewari.
\newblock On the complexity of linear prediction: Risk bounds, margin bounds,
  and regularization.
\newblock In \emph{Advances in neural information processing systems}, pages
  793--800, 2009.

\bibitem[Karampatziakis and Langford(2011)]{KLonline}
Nikos Karampatziakis and John Langford.
\newblock Online importance weight aware updates.
\newblock In \emph{{UAI} 2011, Proceedings of the Twenty-Seventh Conference on
  Uncertainty in Artificial Intelligence, Barcelona, Spain, July 14-17, 2011},
  pages 392--399, 2011.

\bibitem[Koltchinskii(2010)]{koltchinskii2010rademacher}
Vladimir Koltchinskii.
\newblock Rademacher complexities and bounding the excess risk in active
  learning.
\newblock \emph{J. Mach. Learn. Res.}, 11:\penalty0 2457--2485, December 2010.

\bibitem[Tsybakov(2004)]{Tsybakov2004}
A.~B. Tsybakov.
\newblock Optimal aggregation of classifiers in statistical learning.
\newblock \emph{Ann. Statist.}, 32:\penalty0 135--166, 2004.

\bibitem[Zhang(2015)]{zhang2015oracular}
Chicheng Zhang.
\newblock A simplified treatment of oracular {CAL}.
\newblock Personal communication, 2015.

\bibitem[Zhang and Chaudhuri(2014)]{zhang2014beyond}
Chicheng Zhang and Kamalika Chaudhuri.
\newblock Beyond disagreement-based agnostic active learning.
\newblock In \emph{Advances in Neural Information Processing Systems}, pages
  442--450, 2014.

\end{thebibliography}

\newpage

\appendix

\section{Deviation bound}
\label{sec:deviation}
We use an adaptation of Freedman's inequality~\citep{Freedman75}
as the main concentration tool.

\begin{lemma}
\label{lemma:dev_bound}
Let $X_1, X_2, \ldots, X_n$ be a martingale difference sequence
adapted to the filtration $\mathcal{F}_i$.  Suppose there exists a
function $b_n$ of $X_1, \ldots, X_n$ that satisfies
\begin{eqnarray*}
  &&\forall 1 \leq i \leq n, \quad |X_i| \leq b_n, \\
  && 1 \leq b_n \leq b_{\max}, 
\end{eqnarray*}
where $b_{\max}$ is a non-random quantity that may depend on $n$.
Define
\begin{eqnarray*}
  S_n &:=& \sum_{i=1}^n X_i, \\
  V_n &:=& \sum_{i=1}^n \mathbbm{E}[X_i^2 \mid \mathcal{F}_{i-1}].
\end{eqnarray*}
Pick any $0 < \delta < 1 / e^2$ and $n \geq 3$. We have 
\begin{equation*}
  \mbox{Pr}\left(S_n \geq 2 \sqrt{V_n \log(1 /\delta)} + 3 b_n
  \log(1/\delta) \right) \leq 4 \sqrt{\delta} (2 + \log_2 b_{\max})
  \log n.
\end{equation*}
\end{lemma}
\begin{proof}
  Define $r_j := 2^{j}$ for $-1 \leq j \leq m := \lceil \log_2 b_{\max} \rceil$. Then we have
  \begin{eqnarray}
    &&	\mbox{Pr}\left(S_n \geq 2 \sqrt{V_n \log(1 /\delta)} +  3 b_n \log(1/\delta) \right) \nonumber \\
    &=&	\sum_{j=0}^m \mbox{Pr}\left(S_n \geq 2 \sqrt{V_n \log(1 /\delta)} +  3 b_n \log(1/\delta) \wedge r_{j-1} < b_n \leq r_j \right) \nonumber \\
    &\leq&\sum_{j=0}^m \mbox{Pr}\left(S_n \geq 2 \sqrt{V_n \log(1 /\delta)} +  3 r_{j-1} \log(1/\delta) \wedge  b_n \leq r_j \right) \nonumber \\
    &\leq& \sum_{j=0}^m \mbox{Pr}\left(S_n \geq 2 \sqrt{V_n \frac{\log(1 /\delta)}{2}} + 3 r_j \frac{\log(1/\delta)}{2} \wedge  b_n \leq r_j \right) \nonumber \\
    &\leq& \sum_{j=0}^m 4 (\log n) \sqrt{\delta} \label{eq:shams_bound}\\
    &\leq& 4 \sqrt{\delta} (2 + \log_2 b_{\max} ) \log n, \nonumber 
  \end{eqnarray} 
  where \eqref{eq:shams_bound} is a direct consequence of Lemma 3 of \cite{KakadeTe09}.
  \citet{KakadeTe09} and the others result from simple
  algebra.
\end{proof}


\section{Auxiliary results for Theorem~\ref{thm:regret}}
\label{sec:regret}
 Before presenting our regret analysis, we
first establish several useful results.

\begin{lemma}
  \label{lemma:delta_pmin_monotone}
  The threshold defined in \eqref{eq:Delta} and the minimum
  probability $P_{\min,\epoch}$ defined in \eqref{eq:pmin} satisfy the
  following for all $ \epoch \geq 1$,
  \begin{eqnarray}
    \epend[\epoch-1]\Delta_{\epoch-1} &\leq& \epend[\epoch]
    \Delta_\epoch, \label{eq:delta_increase}\\ 
    P_{\min,\epoch} &\geq& P_{\min,\epoch+1}, \label{eq:pmin_decrease} \\
    \frac{\epsilon_\epoch}{P_{\min,\epoch}} &\leq&
    \Delta_\epoch. \label{eq:eps_pmin_delta} 
  \end{eqnarray}
\end{lemma}

\begin{proof}
  
  Notice that
  
  \begin{eqnarray}
    \epend[\epoch-1]\epsilon_{\epoch-1} &=& 32
    (\log(|\mathcal{H}|/\delta) + \log \epend[\epoch-1]) \nonumber \\ 
    &\leq& 32 (\log(|\mathcal{H}|/\delta) + \log \epend) \nonumber \\
    &=& \epend \epsilon_\epoch.
  \end{eqnarray}
  We first prove \eqref{eq:delta_increase}.  It holds trivially for
  $m = 1$.  For $m \geq 2$ we have
  \begin{eqnarray*}
    && \epend[\epoch-1] \Delta_{\epoch-1} \\
    &=& c_1 \sqrt{\epend[\epoch-1]^2
      \epsilon_{\epoch-1}\err(h_\epoch,\tilde{Z}_{\epoch-1})} + c_2
    \epend[\epoch-1]\epsilon_{\epoch-1}\log \epend[\epoch-1] \\ 
    &\leq& c_1 \sqrt{(\epend[\epoch-1]\epsilon_{\epoch-1})
      \epend[\epoch-1]\err(h_{\epoch+1},\tilde{Z}_{\epoch-1})} + c_2
    \epend[\epoch-1]\epsilon_{\epoch-1}\log \epend[\epoch-1]\\
    &\leq& c_1 \sqrt{( \epend
      \epsilon_{\epoch})\epend\err(h_{\epoch+1},\tilde{Z}_\epoch)} +
    c_2 \epend \epsilon_{\epoch}\log \epend\\ 	
    &=& \epend \Delta_\epoch,
  \end{eqnarray*}
  where the first inequality is by the fact that $h_\epoch$
  minimizes the empirical error on $\tilde{Z}_{\epoch-1}$ and the
  second inequality is by $\epend[\epoch-1]\epsilon_{\epoch-1} \leq
  \epend \epsilon_\epoch$.  Then for \eqref{eq:pmin_decrease}, it is
  easy to see
  \begin{eqnarray*}
    &&\sqrt{\frac{\epend[\epoch-1]\err(h_\epoch,\tilde{Z}_{\epoch-1})}{n
        \epsilon_\epmax}} + \log \epend[\epoch-1]\\
    &\leq&\sqrt{\frac{\epend[\epoch-1]\err(h_{\epoch+1},\tilde{Z}_{\epoch-1})}{n
        \epsilon_\epmax}} + \log \epend[\epoch-1]\\	
    &\leq&\sqrt{\frac{\epend \err(h_{\epoch+1},\tilde{Z}_\epoch)}{n
        \epsilon_\epmax}} + \log \epend, 
  \end{eqnarray*}
  for $\epoch \geq 1$, implying $P_{\min,\epoch} \geq P_{\min,\epoch+1}$.
  Finally to prove \eqref{eq:eps_pmin_delta}, we have that
  
  \begin{eqnarray*}
    \frac{\epsilon_\epoch}{P_{\min,\epoch}} 
    &\leq&\frac{\epsilon_\epoch}{P_{\min,\epoch+1}} \\
    &=& \max\left( \frac{\sqrt{\epend \epsilon_\epoch^2
        \err(h_{\epoch+1},\tilde{Z}_{\epoch}) / (n \epsilon_\epmax)} +
      \epsilon_\epoch \log \epend}{c_3}, 2\epsilon_\epoch \right)\\ 
    &\leq& \max\left( \frac{\sqrt{\epsilon_\epoch
        \err(h_{\epoch+1},\tilde{Z}_{\epoch})} + \epsilon_\epoch \log
      \epend}{c_3}, 2 \epsilon_\epoch \right) \\ 
    &\leq& \Delta_\epoch,
  \end{eqnarray*}
  where the second inequality is by $\epend \epsilon_\epoch \leq n
  \epsilon_\epmax$, and the third inequality is by our choices of
  $c_1, c_2$ and $c_3$. 
\end{proof}

We also need a lemma regarding the epoch schedule.
\begin{lemma}
  Let $\epend[\epoch-1] < \epend \leq 2\epend[\epoch-1]$ for all
  $\epoch > 1$. Then we have for all $\epoch \geq 1$,
\begin{eqnarray*}
  \sum_{i=1}^\epoch \frac{\epend[i+1] - \epend[i]}{\epend[i]} &\leq& 4\log \epend[\epoch+1],\\
  \sum_{i=1}^\epoch (\epend[i] - \epend[i-1]) \Delta_{i-1} &\leq& 4
  \epend[\epoch] \Delta_{\epoch} \log \epend[\epoch]. 
\end{eqnarray*}
\label{lemma:epend-sum}
\end{lemma}

\begin{proof}
  Note that we can rewrite the summation in question as
  \begin{align*}
    \sum_{i=1}^\epoch \frac{\epend[i+1] - \epend[i]}{\epend[i]} &=
    \sum_{i=1}^\epoch \sum_{j=\epend[i]+1}^{\epend[i+1]}
    \frac{1}{\epend[i]}\\
    &\leq \sum_{i=1}^\epoch \sum_{j=\epend[i]+1}^{\epend[i+1]}
    \frac{2}{\epend[i+1]},
  \end{align*}
  where the second inequality uses our assumption on epoch
  lengths. The summation can then be further bounded as
  \begin{align}
    \sum_{i=1}^\epoch \frac{\epend[i+1] - \epend[i]}{\epend[i]} &\leq
    \sum_{i=1}^\epoch \sum_{j=\epend[i]+1}^{\epend[i+1]}
    \frac{2}{j} \leq \sum_{i=1}^{\epend[\epoch+1]} \frac{2}{i} \nonumber \\
    & \leq 2(1 + \log\epend[\epoch+1]) \label{eq:bound_on_epochs}\\
    &\leq 4 \log \epend[\epoch+1] \nonumber,
  \end{align}
  where the third inequality is by the bound $\sum_{i=1}^n 1/i \leq 1 + \log n$, and the final inequality is 
  by $1 \leq \log \epend, m \geq 1$. To prove the second bound in the lemma, we write
  \begin{eqnarray*}
	\sum_{i=1}^\epoch (\epend[i] - \epend[i-1]) \Delta_{i-1} 
	&=& \epend[1] \Delta_0 + \sum_{i=1}^{\epoch-1} (\epend[i+1] - \epend[i]) \Delta_{i} \\
 &=& \epend[1] \Delta_0 + \sum_{i=1}^{\epoch-1}\frac{\epend[i+1] - \epend[i]}{\epend[i]} \epend[i] \Delta_i\\
 &\leq& \epend[1] \Delta_0 + (2 + 2 \log \epend) \epend[\epoch] \Delta_{\epoch}\\
 &\leq& (2 \log \epend[1] - 2) \epend[1] \Delta_1 + (2 + 2 \log \epend) \epend \Delta_{\epoch}\\
 &\leq& (2 \log \epend - 2) \epend \Delta_{\epoch} + (2 + 2 \log \epend) \epend \Delta_{\epoch}\\	
	&=& 	4 \epend \Delta_{\epoch} \log \epend,
  \end{eqnarray*}
  where the first inequality is by \eqref{eq:bound_on_epochs} and $\epend[i] \Delta_i \leq \epend \Delta_{\epoch}$ (Lemma \ref{lemma:delta_pmin_monotone}), 
  the second inequality is by our choice of $\Delta_0$ and the fact that $\epend[1] \Delta_1 \leq 1$, and 
  the third inequality again uses $\epend[i] \Delta_i \leq \epend \Delta_{\epoch}$.
\end{proof}


\section{Proofs omitted from Section~\ref{sec:proofs-regret-main}} 
\label{sec:regret-lemmas}

We now provide the proofs of the lemmas and propositions from
Section~\ref{sec:proofs-regret-main} that were used in proving
Theorem~\ref{thm:regret}. We start with proofs of
Lemmas~\ref{lemma:favorable_bias} and~\ref{lemma:reg_err_concen}. 

\begin{proof-of-lemma}[\ref{lemma:favorable_bias}]

  Pick any $\epoch \geq 1, h \in \H$ and $\bar{h} \in A_\epoch$.  Note
  that the definitions of $\regbiasi{\epoch}(h,\bar{h})$ and
  $\reg(h,\bar{h})$ only differ on $X \notin D_\epoch :=
  \dis(A_\epoch)$, and $\forall X \notin D_\epoch, \; \bar{h}(X) =
  h_\epoch(X)$.  We thus have
  \begin{align*}
    & \regbiasi{\epoch}(h,\bar{h})  - \reg(h,\bar{h}) \nonumber \\
    =\;& \E_{X,Y} \bigg[ \bbmone(X \notin D_\epoch) \Big(
      \big(\bbmone(h(X) \neq h_\epoch(X)) - \bbmone(\bar{h}(X) \neq
      h_\epoch(X))\big) \\
      & \qquad -  \big(\bbmone(h(X) \neq Y) -
      \bbmone(\bar{h}(X) \neq Y) \big)\Big) \bigg] \nonumber \\ 
    =\;& \E_{X,Y} [\bbmone(X \notin D_\epoch) \big(\bbmone(h(X) \neq
      h_\epoch(X)) - (\bbmone(h(X) \neq Y) - \bbmone(h_\epoch(X) \neq
      Y))\big)].  
  \end{align*}
  The desired result then follows from the inequality that 

  \begin{equation*}
    \bbmone(h(X) \neq Y) - \bbmone(h_\epoch(X) \neq Y) \leq
    \bbmone(h(X) \neq h_\epoch(X)).  
  \end{equation*}
\end{proof-of-lemma}

\begin{proof-of-lemma}[~\ref{lemma:reg_err_concen}]

  Our proof strategy is to apply Lemma \ref{lemma:dev_bound} to
  establish concentration of properly defined martingale difference
  sequences for fixed classifiers $h,h'$ and some epoch $\epoch$, and
  then use a union bound to get the desired statement.  First we look
  at the concentration of the empirical regret on $\ztil_\epoch$. To
  avoid clutter, we overload our notation so that $D_i = D_{\epochi}$,
  $h_i = h_{\epochi}$ and $P_i = P_{\epochi}$ when $i$ is the index of
  an example rather than a round. 

  For any pair of classifiers $h$ and $h'$, we define the random
  variables for the instantaneous regrets:

  \begin{eqnarray*}
    \tilde{R}_i &:=& \bbmone(X_i \notin D_i)(\bbmone(h(X_i) \neq
    h_i(X_i)) - \bbmone(h'(X_i) \neq h_i(X_i))) + \\ 
    &&\bbmone(X_i \in D_i)(\bbmone(h(X_i) \neq Y_i) - \bbmone(h'(X_i)
    \neq Y_i)) Q_i / P_i(X_i) 
  \end{eqnarray*}
  and the associated $\sigma$-fields $\mathcal{F}_i :=
  \sigma(\{X_j,Y_j,Q_j\}_{j=1}^i)$. We have that $\tilde{R}_i$ is
  measurable with respect to $\mathcal{F}_{i}$. Therefore $\tilde{R}_i
  - \mathbbm{E}[\tilde{R}_i \mid \mathcal{F}_{i-1}]$ forms a
  martingale difference sequence adapted to the filtrations $F_i, i
  \geq 1$, and

  \begin{equation*}
    \mathbbm{E}[\tilde{R}_i \mid \mathcal{F}_{i-1}] \;=\; \regbiasi{\epochi}(h,h')
  \end{equation*}
  according to \eqref{eq:regbiasi} and the fact that $X_i, Y_i, Q_i$
  are independent from the past. To use Lemma \ref{lemma:dev_bound},
  we first identify an upper bound on elements in the sequence:
  
  \begin{eqnarray}
    |\tilde{R}_i -\mathbbm{E}[\tilde{R}_i \mid \mathcal{F}_{i-1}] |
    &=& |\tilde{R}_i - \regbiasi{\epochi}(h,h')| \leq \max(\tilde{R}_i,
    \regbiasi{\epochi}(h,h')) \nonumber \\ 
    &\leq& \frac{1}{P_{\min,\epochi}}  \leq \frac{1}{P_{\min,\epoch}}, 
    \label{eq:b_n}
  \end{eqnarray}
  for all $i$ such that $\epochi \leq \epoch$, where the last
  inequality is by Lemma \ref{lemma:delta_pmin_monotone}. The
  definition of $P_{\min,\epoch}$ implies that 

  \begin{equation}
    \frac{1}{P_{\min,\epoch}} \leq \max( \sqrt{\epend[\epoch-1] / (n
      \epsilon_\epmax)} + \log \epend[\epoch-1], 2) \leq 2
    \sqrt{\epend[\epoch-1]+1} 
    \label{eq:b_max}
  \end{equation}
  because $n \epsilon_\epmax \geq 1$.
  Then we consider the conditional second moment. Using the fact that
  \begin{equation}
    (\bbmone(h(X_i) \neq Y_i) - \bbmone(h'(X_i) \neq Y_i))^2 \; \leq \;
    \bbmone(h(X_i) \neq h'(X_i)),
    \label{eq:inst_reg_bound}
  \end{equation}
  we get 
  \begin{eqnarray}
    && \mathbbm{E}[(\tilde{R}_i - \mathbbm{E}[\tilde{R}_i \mid
        \mathcal{F}_{i-1}])^2 \mid \mathcal{F}_{i-1}] \nonumber \\ 
    &=&\mathbbm{E}[(\tilde{R}_i - \regbiasi{\epochi}(h,h'))^2 \mid
      \mathcal{F}_{i-1}] \;\leq\; \mathbbm{E}[\tilde{R}_i^2 \mid
      \mathcal{F}_{i-1}] \nonumber\\ 
    &\leq&\mathbbm{E}\left[\left( \bbmone(X_i \notin D_i) +
      \frac{\bbmone(X_i \in D_i) Q_i}{P_i(X_i)}\right)^2\bbmone(h(X_i)
      \neq h'(X_i))  \mid \mathcal{F}_{i-1}\right] \nonumber\\ 
    &=&\mathbbm{E}\left[\left( \bbmone(X_i \notin D_i) +
      \frac{\bbmone(X_i \in D_i) Q_i}{P_i(X_i)^2}\right)
      \bbmone(h(X_i) \neq h'(X_i)) \mid \mathcal{F}_{i-1}\right] \nonumber\\ 
    &=&\mathbbm{E}\left[\left( \bbmone(X_i \notin D_i) +
      \frac{\bbmone(X_i \in D_i)}{P_i(X_i)}\right) \bbmone(h(X_i) \neq
      h'(X_i)) \mid \mathcal{F}_{i-1}\right] \nonumber\\ 
    &=&\mathbbm{E}_X\left[\left( \bbmone(X \notin D_i) +
      \frac{\bbmone(X \in D_i)}{P_i(X)}\right) \bbmone(h(X) \neq
      h'(X))\right] \nonumber \\
    &=& \mathbbm{E}_X\left[\left( \bbmone(X \notin D_{\epochi}) +
      \frac{\bbmone(X \in D_{\epochi})}{P_{\epochi}(X)}\right)
      \bbmone(h(X) \neq h'(X))\right]    
    \label{eq:var_i}
  \end{eqnarray}
  where the last two equalities are from the fact that $X_i$ is
  independent from the past and replacing our overloaded notation
  respectively. Lemma \ref{lemma:dev_bound} with \eqref{eq:b_n},
  \eqref{eq:b_max}, and \eqref{eq:var_i} then implies for any $0 <
  \delta_\epoch < 1/e^2$ and $\epoch \geq 1$, the following holds with
  probability at most $8 \sqrt{\delta_\epoch} (2 + \log_2 (2
  \sqrt{\epend[\epoch-1]+1})) \log \epend$:
  \begin{align}
    &|\reg(h,h',\ztil_\epoch) - \regbias{\epoch}(h,h')| \nonumber \\
    &\geq \sqrt{\frac{4
    \log(1/\delta_\epoch)}{\epend^2}\sum_{i=1}^\epoch(\epend[i] -
      \epend[i-1]) \mathbbm{E}_X\left[\left(
    \bbmone(X \notin D_{i}) +  \frac{\bbmone(X \in
      D_{i})}{P_{i}(X)}\right) 
    \bbmone(h(X) \neq h'(X))\right] } \nonumber \\  
    &\qquad +  \frac{4  \log(1/\delta_\epoch)}{n
      P_{\min,\epoch}}.  \label{eq:large_dev} 
  \end{align}

  Then we consider the concentration of the empirical error on the
  importance-weighted examples.  Define the random examples for the
  empirical errors:
  \begin{equation*}
    E_i := \frac{Q_i\bbmone(h(X_i) \neq Y_i \wedge X_i \in D_i) }{P_i(X_i)}
  \end{equation*}
  and the associated $\sigma$-fields $\mathcal{F}_i :=
  \sigma(\{X_j,Y_j,Q_j\}_{j=1}^i)$.  By the same analysis of the
  sequence of instantaneous regrets, we have $E_i - \E[E_i \mid
    \mathcal{F}_{i-1}]$ is a martingale difference sequence adapted to
  the filtrations $\mathcal{F}_i, i \geq 1$, with the following
  properties:

  \begin{eqnarray*}
    \E[E_i \mid \mathcal{F}_{i-1}] &=&
    \E[\mathbbm{1}(X_i \in D_i \wedge h(X_i) \neq Y_i) \mid
      \mathcal{F}_{i-1}] = \err_{\epochi}(h), \\ 
    |E_i -\E[E_i \mid \mathcal{F}_{i-1}]| &\leq&
    \frac{1}{P_{\min,\epochi}} \leq \frac{1}{P_{\min,\epoch}} \leq 2
    \sqrt{\epend[\epoch-1]+1}, 
  \end{eqnarray*}
  for all $i$ such that $\epochi \leq \epoch$. Furthermore, 

  \begin{eqnarray*}
    \E[(E_i -\E[E_i \mid \mathcal{F}_{i-1}])^2 \mid
      \mathcal{F}_{i-1}] &\leq& \E\left[\frac{\bbmone(X_i \in
        D_i \wedge h(X_i) \neq Y_i)}{P_i(X_i)} \;\biggr|\;
      \mathcal{F}_{i-1}\right]\\ 
    &=& \E_{X,Y}\left[\frac{\bbmone(X \in D_i \wedge h(X) \neq
        Y)}{P_i(X)}\right]. 
  \end{eqnarray*}
  With these properties, Lemma \ref{lemma:dev_bound} then implies 
  for any $0 < \delta_\epoch < 1/e^2$ and $\epoch \geq 1$, 
  the following holds with probability at most $8 \sqrt{\delta_\epoch}
  (2 + \log_2 (2 \sqrt{\epend[\epoch-1]+1})) \log \epend$:  

  \begin{eqnarray}
    |\err(h, Z_\epoch) - \overline{\err}_\epoch(h)| 
    &\geq& \sqrt{\frac{4
        \log(1/\delta_\epoch)}{\epend^2}\sum_{i=1}^\epoch (\epend[i] -
      \epend[i-1]) \mathbbm{E}_{X,Y}\left[\frac{\bbmone(X
          \in D_{i} \wedge h(X) \neq Y)}{P_{i}(X)}\right]
    } \nonumber \\   
    &&+  \frac{4  \log(1/\delta_\epoch)}{n P_{\min,\epoch}}.  \label{eq:large_dev_err}
  \end{eqnarray}
  
  Setting 
  \begin{equation*}
    \delta_\epoch = \left( \frac{\delta}{192 |\mathcal{H}|^2 \epend^2
      (\log \epend)^2} \right)^2
  \end{equation*}
  ensures that the probability of the union of the bad events
  \eqref{eq:large_dev}, and \eqref{eq:large_dev_err} over all pairs of
  classifiers $h,h'$ and $\epoch \geq 1$ is bounded by $\delta > 0$.
  Choosing $\delta \leq |\mathcal{H}| / \sqrt{192}$, we have
  \begin{eqnarray*}
    \log(1 / \delta_\epoch) &=& 2 \log\left( \frac{192 |\mathcal{H}|^2
      \epend^2 (\log \epend)^2}{\delta}\right) \\ 
    &\leq& 2(2 \log(|\mathcal{H}| / \delta) + 4 \log \epend + \log 192)  \\
    &\leq& 8(\log(|\mathcal{H}|/\delta) + \log \epend),
  \end{eqnarray*}
  leading to the desired statement.   
\end{proof-of-lemma}

We then provide the proofs of Propositions~\ref{prop:dev-regret-bias}
and~\ref{prop:err-emp-best}. 

\begin{proof-of-proposition}[\ref{prop:dev-regret-bias}]
  By the inequality~\eqref{eqn:dev-bias} of
  Lemma~\ref{lemma:reg_err_concen}, we have 

  \begin{align}
    &|\reg(h, h^*, \ztil_\epoch) - \regbias{\epoch}(h,h^*)| \nonumber\\
    &\leq \sqrt{\underbrace{\frac{\epsilon_{\epoch}}{\epend} \sum_{i=1}^\epoch
      (\epend[i] - \epend[i-1]) \E_X\left[ \left(\bbmone(X
        \notin D_i) + \frac{\mathbbm{1}(X \in D_i)}{P_i(X)}\right)
        \bbmone(h(X) \neq h^*(X))\right]}_{\dev_\epoch(h)}} +
    \frac{\epsilon_{\epoch}}{P_{\min,\epoch}}
    \label{eqn:dev-bias-main}
  \end{align}

We now control the term $\dev_\epoch(h)$ in order to establish the
proposition. We have

\begin{align*}
  & \frac{\epend}{\epsilon_\epoch}\dev_\epoch(h) \\
  &= \sum_{i=1}^\epoch
  (\epend[i] - \epend[i-1]) \E_X\left[
	\left(\frac{\bbmone(X \in D_i)}{P_i(X)} + \bbmone(X \notin D_i)\right)
    \bbmone(h(X) \ne h^*(X))\right]\\
  &\leq \sum_{i=1}^\epoch (\epend[i] - \epend[i-1]) \E_X\bigg[\frac{\bbmone(X \in
      D_i)}{P_i(X)}\Big(\bbmone(h(X) \ne h_i(X)) + \bbmone(h^*(X) \ne
    h_i(X)) \Big) \\
    & \qquad \qquad \qquad \qquad \qquad+ \bbmone(X \notin D_i) \bbmone(h(X) \ne
    h^*(X))\bigg]\\
  &\leq \sum_{i=1}^\epoch (\epend[i] - \epend[i-1]) \E_X\bigg[
    2\alpha^2\bbmone(X \in D_i)\Big(\bbmone(h(X) \ne h_i(X)) + 
    \bbmone(h^*(X) \ne h_i(X))\Big) \\ 
     & \qquad \qquad \qquad \qquad + 2\beta^2\gamma\epend[i-1]\Delta_{i-1}(\reg(h, \ztil_{i-1}) +
    \reg(h^*, \ztil_{i-1}))+
     2\slackconst \epend[i-1]\Delta_{i-1}^2 \\
     & \qquad \qquad \qquad \qquad + \bbmone(h(X) \ne h^*(X) \wedge X \notin D_i) \bigg],  
\end{align*}
where the second inequality uses our variance constraints in defining
the distribution $P_i$ for classifiers $h$ and $h^*$. Note that
\begin{align*}\bbmone(h(X) \ne h^*(X)) & \leq \bbmone(h(X) \ne Y) + \bbmone(h^*(X)
\ne Y)\\
& = (\bbmone(h(X) \ne Y) - \bbmone(h^*(X) \ne Y)) +
2\bbmone(h^*(X) \ne Y),
\end{align*}
so that the final inequality can be rewritten as
\begin{align*}
 & \frac{\epend}{\epsilon_\epoch}\dev_\epoch(h)\\
 &\leq \sum_{i=1}^\epoch
  (\epend[i] - \epend[i-1]) \bigg[ 2\alpha^2 (\regi{i}(h) + 2
    \regi{i}(h_i)) + 12\alpha^2\erri{i}(h^*) + 2\beta^2\gamma
    \epend[i-1]\Delta_{i-1}(\reg(h, \ztil_{i-1}) 
    \\ &\qquad \qquad+ \reg(h^*, \ztil_{i-1})) +2\slackconst \epend[i-1]\Delta_{i-1}^2 +
    \E_X[\bbmone(h(X) \ne h^*(X) \wedge X \notin D_i)]\bigg].
\end{align*}
With the assumptions $\alpha \geq 1$ and $h^* \in A_i$ for all epochs
$i \leq \epoch$, the first term $\regi{i}(h)$ can be combined with the
last disagreement term and bounded by $2\alpha^2
\regbiasi{i}(h)$. Further noting that $\epend[i-1]\Delta_{i-1} \leq
\epend \Delta_\epoch$ by Lemma \ref{lemma:delta_pmin_monotone}, we can
further simplify the inequality to
\begin{align*}
  \frac{\epend}{\epsilon_\epoch}\dev_\epoch(h) &\leq 2 \alpha^2
  \sum_{i=1}^\epoch (\epend[i] - \epend[i-1]) \regbiasi{i}(h) +
  4\alpha^2 \sum_{i=1}^\epoch (\epend[i] - \epend[i-1]) \regi{i}(h_i)
  + 12 \epend\alpha^2 \erravg{\epoch}(h^*) \\
  &+\qquad 
  2\beta^2\gamma \epend \Delta_{\epoch} \sum_{i=1}^\epoch (\epend[i] -
  \epend[i-1]) (\reg(h, \ztil_{i-1}) \\
  &\qquad+ \reg(h^*, \ztil_{i-1}))    
  + 2\slackconst \sum_{i=1}^\epoch (\epend[i] -
  \epend[i-1])\epend[i-1]\Delta_{i-1}^2. 
\end{align*}
The first summand is simply $2\alpha^2 \epend\regbias{m}(h)$ by definition.
The final summand above can be bounded using
Lemmas~\ref{lemma:delta_pmin_monotone} and \ref{lemma:epend-sum} since
\begin{align*}
  \sum_{i=1}^\epoch (\epend[i] - \epend[i-1])
  \epend[i-1]\Delta_{i-1}^2 &= \sum_{i=1}^{\epoch-1} (\epend[i+1] -
  \epend[i]) \epend[i]\Delta_i^2 \leq \epend \Delta_{\epoch}\sum_{i=1}^{\epoch-1}
  (\epend[i+1] - \epend[i]) \Delta_i \\  
  &\leq 4\epend^2\Delta_\epoch^2 
  \log \epend.
\end{align*}
Substituting the above inequalities back, we obtain
\begin{align*}
  \frac{\epend}{\epsilon_\epoch}\dev_\epoch(h) &\leq 2 \alpha^2
  \epend\regbias{\epoch}(h) + 4\alpha^2 \sum_{i=1}^\epoch (\epend[i] -
  \epend[i-1]) \regi{i}(h_i) + 12 \epend\alpha^2 \erravg{\epoch}(h^*)
  \\&+ 2\beta^2\gamma \epend \Delta_{\epoch} \sum_{i=1}^\epoch
  (\epend[i] - \epend[i-1]) (\reg(h, \ztil_{i-1}) + \reg(h^*,
  \ztil_{i-1})) + 8\slackconst \epend^2\Delta_\epoch^2 
  \log \epend.
\end{align*}
Since $\sqrt{a+b} \leq \sqrt{a} + \sqrt{b}$, we can further bound
\begin{align*}
  \sqrt{\dev_\epoch(h)} &\leq \sqrt{2\alpha^2
      \epsilon_\epoch\regbias{\epoch}(h)} + 2\alpha
  \sqrt{\frac{\epsilon_\epoch}{\epend} \sum_{i=1}^\epoch (\epend[i] -
    \epend[i-1]) \regi{i}(h_i)} + 2\alpha
    \sqrt{3\erravg{\epoch}(h^*)\epsilon_{\epoch}}\\ 
  &+ \beta \sqrt{2\gamma \epsilon_\epoch\Delta_\epoch \sum_{i=1}^\epoch
  (\epend[i] - \epend[i-1]) (\reg(h, \ztil_{i-1}) + \reg(h^*,
  \ztil_{i-1}))}\\
    &+ 2\Delta_\epoch\sqrt{2\slackconst
    \epend\epsilon_\epoch \log \epend}. 
\end{align*}
Substituting this inequality back into our deviation
bound~\eqref{eqn:dev-bias-main}, we obtain
\begin{align*}
	&  |\reg(h, h^*, \ztil_\epoch) - \regbias{\epoch}(h, h^*)| \\
	&\leq
  \frac{\epsilon_\epoch}{P_{\min, \epoch}} + \sqrt{2\alpha^2
    \epsilon_\epoch\regbias{\epoch}(h)} + 2\alpha
  \sqrt{\frac{\epsilon_\epoch}{\epend} \sum_{i=1}^\epoch (\epend[i] -
    \epend[i-1]) \regi{i}(h_i)} + 2\alpha
  \sqrt{3\erravg{\epoch}(h^*)\epsilon_\epoch}\\ &+ \beta \sqrt{2\gamma
    \epsilon_\epoch\Delta_\epoch \sum_{i=1}^\epoch (\epend[i] -
    \epend[i-1]) (\reg(h, \ztil_{i-1}) + \reg(h^*, \ztil_{i-1}))} +
  2\Delta_\epoch\sqrt{2\slackconst \epend\epsilon_\epoch\log \epend}.
\end{align*}
We can further use Cauchy-Schwarz inequality to obtain the bound
\begin{align*}
  &|\reg(h, h^*, \ztil_\epoch) - \regbias{\epoch}(h, h^*)| \\&\leq
  \frac{1}{4} \regbias{\epoch}(h) + 2\alpha^2 \epsilon_\epoch + 2\alpha
  \sqrt{\frac{\epsilon_\epoch}{\epend} \sum_{i=1}^\epoch (\epend[i] -
    \epend[i-1]) \regi{i}(h_i)} + 2\alpha
  \sqrt{3\erravg{\epoch}(h^*)\epsilon_\epoch}\\ &+ \beta \sqrt{2\gamma
    \epsilon_\epoch\Delta_\epoch \sum_{i=1}^\epoch (\epend[i] -
    \epend[i-1]) (\reg(h, \ztil_{i-1}) + \reg(h^*, \ztil_{i-1}))} +
  2\Delta_\epoch\sqrt{2\slackconst \epend\epsilon_\epoch\log \epend} \\
&+ \frac{\epsilon_\epoch}{P_{\min, \epoch}} \\
  &\leq
  \frac{1}{4} \regbias{\epoch}(h) + 2\alpha^2 \epsilon_\epoch + 2\alpha
  \sqrt{\frac{\epsilon_\epoch}{\epend} \sum_{i=1}^\epoch (\epend[i] -
    \epend[i-1]) \regi{i}(h_i)} + 2\alpha
  \sqrt{3\erravg{\epoch}(h^*)\epsilon_\epoch}\\ &+ \beta \sqrt{2\gamma
    \epsilon_\epoch\Delta_\epoch \sum_{i=1}^\epoch (\epend[i] -
    \epend[i-1]) (\reg(h, \ztil_{i-1}) + \reg(h^*, \ztil_{i-1}))} +
  \Delta_\epoch + \frac{\epsilon_\epoch}{P_{\min, \epoch}} \\
  &\leq
  \frac{1}{4} \regbias{\epoch}(h) + 2\alpha
  \sqrt{\frac{\epsilon_\epoch}{\epend} \sum_{i=1}^\epoch (\epend[i] - 
    \epend[i-1]) \regi{i}(h_i)} + 2\alpha
  \sqrt{3\erravg{\epoch}(h^*)\epsilon_\epoch}\\ &+ \beta \sqrt{2\gamma 
    \epsilon_\epoch\Delta_\epoch \sum_{i=1}^\epoch (\epend[i] -
    \epend[i-1]) (\reg(h, \ztil_{i-1}) + \reg(h^*, \ztil_{i-1}))} +
  4\Delta_\epoch 
\end{align*}
where the last two inequalities use our assumptions on $\slackconst$
and $\alpha$ respectively.
\end{proof-of-proposition}

\begin{proof-of-proposition}[\ref{prop:err-emp-best}]
  We start by observing that 
  \begin{align*}
    |\erravg{\epoch}(h^*) - \err(h_{\epoch+1}, \ztil_\epoch)| &\leq |\erravg{\epoch}(h^*) -
    \err(h^*, \ztil_\epoch)| + \reg(h^*, h_{\epoch+1}, \ztil_\epoch). 
  \end{align*}
  Since $h^* \in A_i$ for all epochs $i \leq \epoch$, we know that
  $h^*$ agrees with all the predicted labels. Consequently, $\err(h^*,
  \ztil_\epoch) = \err(h^*, Z_\epoch)$, where we recall that
  $Z_\epoch$ is the set of all examples where we queried labels up to
  epoch $\epoch$. This allows us to rewrite

  \begin{align*}
    |\erravg{\epoch}(h^*) - \err(h^*, \ztil_\epoch)| &=
    |\erravg{\epoch}(h^*) - \err(h^*,Z_\epoch)|.
  \end{align*}
  Under the event $\event$, the above deviation is bounded, according
  to Lemma \ref{lemma:reg_err_concen}, by
  \begin{equation*}
    \sqrt{\frac{\epsilon_\epoch}{\epend} \sum_{i=1}^\epoch (\epend[i]
      - \epend[i-1])\E_{X,Y} \frac{\bbmone(h^*(X) \ne Y, X \in
        D_i)}{P_i(X)}} + \frac{\epsilon_\epoch}{P_{\min, \epoch}}
    \;\leq\; \sqrt{\epsilon_\epoch
      \frac{\erravg{\epoch}(h^*)}{P_{\min, \epoch}}} +
    \frac{\epsilon_\epoch}{P_{\min, \epoch}},
  \end{equation*}
where the inequality uses the bound $P_i(X) \geq P_{\min, i}$ for all
$X \in D_i$ and $P_{\min, i} \geq P_{\min, \epoch}$ for all epochs $i
\leq \epoch$ by Lemma~\ref{lemma:delta_pmin_monotone}. A further
application of Cauchy-Schwarz inequality yields the bound

\begin{align*}
  |\erravg{\epoch}(h^*) - \err(h^*, \ztil_\epoch)| &\leq
  \frac{\erravg{\epoch}(h^*)}{2} + \frac{3\epsilon_\epoch}{2P_{\min, \epoch}} \\
  &\leq \frac{\erravg{\epoch}(h^*)}{2} + \frac{3\Delta_\epoch}{2}. 
\end{align*}
Combining the bounds yields
\begin{eqnarray*}
  |\erravg{\epoch}(h^*) - \err(h_{\epoch+1}, \ztil_\epoch)| &\leq&
  \frac{\erravg{\epoch}(h^*)}{2} + \frac{3\Delta_\epoch}{2} +
  \reg(h^*, h_{\epoch+1}, \ztil_\epoch), 
\end{eqnarray*}
which completes the proof of the proposition.
\end{proof-of-proposition}

Finally, we prove Lemmas~\ref{lemma:regret-t1}
to~\ref{lemma:regret-t3} used in the proof of
Theorem~\ref{thm:regret}. 

\begin{proof-of-lemma}[\ref{lemma:regret-t1}]
  We first bound the $\regi{i}(h_i)$ terms. For $i = 1$, we have
  
  \begin{equation*}
    \regi{1}(h_1) = \reg(h_1) \leq 1 \leq \frac{\regconst \Delta_0}{ 2}
  \end{equation*}
  by $P_{\min,1} = 1$ and our choices of $\regconst$ and $\Delta_0$.
  For $2 \leq i < \epoch$,
  we have
  \begin{align*}
    \regi{i}(h_i) &= \E_{X,Y}\left[\bbmone(h_i(X) \ne Y, X \in D_i) -
      \bbmone(h^*(X) \ne Y, X \in D_i) \right] = \reg(h_i) \leq
    \regbias{i-1}(h_i, h^*),
  \end{align*}
  where the second equality uses the fact that $h^* \in A_i$ for all $i
  \leq \epoch$ by inductive hypothesis \eqref{eq:h_best} and the
  inequality    uses Lemma~\ref{lemma:favorable_bias}. Consequently, 
  we can bound $\regi{i-1}(h_i)$ using the event $\event_i$, since
  $\reg(h_i, h^*, \ztil_{i-1}) = 0$.  The event $\event_i$ now further 
  implies that
  
  \begin{align*}
    \regi{i}(h_i) \leq \regbias{i-1}(h_i, h^*) \leq 2\reg(h_i, h^*,
    \ztil_{i-1}) + \frac{\regconst \Delta_{i-1}}{2} \leq \frac{\regconst
      \Delta_{i-1}}{2}.
  \end{align*}
  Using this, we can simplify $\term_1$ as
  \begin{align}
    \term_1 &= 2 \alpha
    \sqrt{\frac{\epsilon_\epoch}{\epend}\sum_{i=1}^\epoch (\epend[i] -
      \epend[i-1]) \regi{i}(h_i)} \leq 2 \alpha
    \sqrt{\frac{\epsilon_\epoch}{\epend}\sum_{i=1}^\epoch (\epend[i] -
      \epend[i-1]) \frac{\regconst \Delta_{i-1}}{2}}\\
    &\leq 2 \alpha \sqrt{2 \regconst\epsilon_\epoch
      \Delta_\epoch \log \epend} \nonumber \\ 
    &\leq \frac{\regconst\Delta_\epoch}{12} + 24 \alpha^2
  \epsilon_\epoch \log \epend. 
  \label{eqn:term1-regret}
  \end{align}
  here the second inequality is by Lemma~\ref{lemma:epend-sum} and
  the third inequality is by Cauchy-Schwarz.
\end{proof-of-lemma}

\begin{proof-of-lemma}[\ref{lemma:regret-t2}]
  We first invoke Proposition~\ref{prop:err-emp-best}, whose
  assumptions now hold due to the claim $h^* \in A_i$ in $\event_i$
  for all $i \leq \epoch$, and obtain
  \begin{align*}
    \erravg{\epoch}(h^*) &= 2\err(h_{\epoch+1}, \ztil_\epoch) +
    3\Delta_\epoch + 2\reg(h^*, h_{\epoch+1}, \ztil_\epoch). 
  \end{align*}
  The above inequality allows us to simplify $\term_2$ as
  \begin{align}
    \term_2 &= 2\alpha\sqrt{3\epsilon_\epoch\erravg{\epoch}(h^*)} \leq 2\alpha
    \sqrt{3\epsilon_\epoch\left( 2\err(h_{\epoch+1}, \ztil_\epoch) +
      3\Delta_\epoch + 2\reg(h^*, h_{\epoch+1}, \ztil_\epoch) \right)} \nonumber\\
    &\leq 2\alpha \sqrt{6 \epsilon_\epoch \err(h_{\epoch+1},
      \ztil_\epoch)} + 2\alpha \sqrt{9\epsilon_\epoch \Delta_\epoch} +
    2\alpha\sqrt{6\epsilon_\epoch\reg(h^*, h_{\epoch+1},
      \ztil_\epoch)}\nonumber \\ 
    &\leq 2\alpha \sqrt{6 \epsilon_\epoch \err(h_{\epoch+1}, \ztil_\epoch)} +
    \Delta_\epoch + \frac{1}{4} \reg(h^*, h_{\epoch+1},
    \ztil_\epoch) + 33\alpha^2 \epsilon_\epoch,
    \label{eqn:term2-regret}
  \end{align}
  where the last inequality uses the Cauchy-Schwarz inequality.  
\end{proof-of-lemma}

\begin{proof-of-lemma}[\ref{lemma:regret-t3}]
  
  Observe that the event $\event_i$ gives a direct bound of $\regconst
  \Delta_{i-1}/4$ on the $\reg(h^*, h_i, \ztil_{i-1})$ terms. For the
  other term, recall by the same event that for all $h \in \H$ and for
  all $i = 1,2\ldots, \epoch-1$,

  \begin{equation*}
    \reg(h, h^*, \ztil_{i}) \leq \frac{3}{2} \regbias{i}(h, h^*) +
    \frac{\eta}{4} \Delta_i. 
  \end{equation*}
  Combining with the empirical regret bound for $h^*$, this implies that 
  
  \begin{equation*}
    \reg(h, \ztil_{i}) \leq \frac{3}{2} \regbias{i}(h, h^*) +
    \frac{\eta}{2} \Delta_i. 
  \end{equation*}
  Consequently we have the bound
  \begin{align*}
    \term_3^2 \leq \beta^2 \gamma \Delta_\epoch \epsilon_\epoch
    \sum_{i=1}^\epoch (\epend[i] - \epend[i-1]) \left(3\regbias{i-1}(h,
    h^*) + \frac{3\regconst}{2} \Delta_{i-1} \right)
  \end{align*}
  To simplify further, note that by the definition of $\regbias{i}(h,
  h^*)$ and our earlier definition of $\regbiasi{i}(h, h^*)$, we have
  \begin{eqnarray*}
    \sum_{i=1}^\epoch (\epend[i] - \epend[i-1]) \regbias{i-1}(h, h^*) &=&
    \sum_{i=1}^{\epoch-1} \frac{\epend[i+1] - \epend[i]}{\epend[i]}
    \sum_{j=1}^{i} (\epend[j] - \epend[j-1])\regbiasi{j}(h, h^*) \\
    &=& 
    \sum_{j=1}^{\epoch-1} (\epend[j] - \epend[j-1])\regbiasi{j}(h, h^*)
    \sum_{i= j}^{\epoch-1}  \frac{\epend[i+1] - \epend[i]}{\epend[i]}\\
    &\leq& 4\log \epend \sum_{j=1}^{\epoch-1} (\epend[j] -
    \epend[j-1])\regbiasi{j}(h, h^*)\\
    &\leq& 4\epend\log \epend\,\regbias{\epoch}(h, h^*),
  \end{eqnarray*}
  where the first equality uses our convention $\regbias{0}(h,h^*) = 0$ 
  and proper index shifting, and the first inequality uses Lemma
  \ref{lemma:epend-sum}. We also have
  \begin{align*}
    \sum_{i=1}^\epoch (\epend[i] - \epend[i-1]) \Delta_{i-1} &\leq
    4\epend\Delta_\epoch \log \epend.
  \end{align*}
  by Lemma \ref{lemma:epend-sum}.
  Consequently, we can rewrite
  \begin{align*}
    \term_3^2 &\leq \beta^2 \gamma \Delta_\epoch \epsilon_\epoch
    \left(12\epend\log \epend\, \regbias{\epoch}(h, h^*) +
    6\epend\regconst\log \epend \Delta_\epoch \right) \nonumber\\ 
    &= \beta^2 \gamma \epend \epsilon_\epoch \log \epend \Delta_\epoch
    \left(12\regbias{\epoch}(h, h^*) +  6\regconst \Delta_\epoch \right)
    \\ 
    &\leq \frac{\regconst \Delta_\epoch \regbias{\epoch}(h,h^*)}{72} +
    \frac{\regconst^2 \Delta_\epoch^2}{144}, 
  \end{align*}
  where the last inequality is by our choice of $\beta$ such that
  $\beta^2 \gamma n \epsilon_n \log n \leq \regconst / 864$. 
  Taking square roots, we obtain
  \begin{align}
    \term_3 &\leq \sqrt{\frac{\regconst \Delta_\epoch
        \regbias{\epoch}(h,h^*)}{72} + \frac{\regconst^2
        \Delta_\epoch^2}{144} }  \nonumber\\
    &\leq \frac{1}{4} \regbias{\epoch}(h,h^*) + \frac{7 \regconst
      \Delta_\epoch }{72} 
  \label{eqn:term3-regret}
  \end{align}
\end{proof-of-lemma}

\section{Label Complexity}
\label{appendix:proof_label_complexity}
Here we prove Theorem \ref{thm:label}.
We start with the following simple bound on the total number of label queries:
\begin{equation}
\sum_{i=1}^n Q_i \; \leq \; \max\left(3, \sum_{i=1}^n \bbmone(X_i \in D_{\epochi})\right)
\end{equation}
by the fact that Algorithm \ref{alg:main} queries only the labels of points in the disagreement region.
The random variable $\bbmone(X_i \in D_{\epochi})$ is measurable with respect to
the $\sigma$-field $\mathcal{F}_i := \sigma(\{X_j,Y_j,Q_j\}_{j=1}^i)$, so $$R_i := \bbmone(X_i \in D_{\epochi}) - \E_i[\bbmone(X_i \in D_{\epochi})]$$
forms a martingale difference sequence adapted to the filtrations $\mathcal{F}_i, i \geq 1$, where $\E_i[\cdot] := \E[\cdot \mid \mathcal{F}_{i-1}]$.
Moreover, we have $|R_i| \leq 1$ and $$\E_i[R_i^2] \leq \E_i[\bbmone(X_i \in D_{\epochi})].$$ 
Applying Lemma 3 of \citet{KakadeTe09} with the above bounds and Cauchy-Schwarz, we get that with probability at least $1 - \delta$,
\begin{equation}
\forall n \geq 3, \quad \sum_{i=1}^n \bbmone(X_i \in D_{\epochi}) \leq 2 \sum_{i=1}^n \E_i[\bbmone(X_i \in D_{\epochi})] + 4\log(4 (\log n) / \delta).
\label{eq:q_high_prob}
\end{equation}
We next bound the sum of the conditional expectations. Pick some $i$ and consider 
the case $X_i \in D_{\epochi}$. Let $m := \epochi$ for the ease of notation. 
Define
  \begin{equation*}
    \bar{h} := \begin{cases}
      h_\epoch, & h_\epoch(X_i) \neq h^*(X_i),\\
      h', & h'(X_i) \neq h^*(X_i),
    \end{cases}
  \end{equation*}
  where 
\begin{eqnarray}
h_m &:=& \arg \min_{h \in \H} \err(h,\ztil_{\epoch-1}), \\
h' &:=& \arg \min_{h \in \mathcal{H} \wedge h(X_i) \neq
    h_\epoch(X_i)} \err(h,\tilde{Z}_{\epoch-1}). 
\end{eqnarray}
Because $X_i \in D_{\epoch} := 
  \mbox{DIS}(A_\epoch)$, we have $h' \in A_\epoch$, implying
  $\bar{h} \in A_\epoch$.  Conditioned on the high probability event in Theorem \ref{thm:main-regret}, 
we have $h^* \in A_\epoch$ and hence
\begin{eqnarray*}
\prob_{X}( \bar{h}(X) \neq h^*(X)) &=&
\prob_{X}( \bar{h}(X) \neq h^*(X) \wedge X \in D_\epoch) \\
&\leq& \reg_\epoch(\bar{h}) + 2 \err_\epoch(h^*) \\ 
&\leq& 16 \gamma \Delta_{\epoch-1}^* + 2 \err_\epoch (h^*),  
\end{eqnarray*}
where the last inequality is by Theorem~\ref{thm:regret}.
This implies that 
\begin{equation*}
  X_i \in \mbox{DIS}(\{ h \mid \prob_X(h(X) \neq h^*(X)) \leq 16 \gamma
  \Delta_{\epoch-1}^* + 2 \err_\epoch(h^*)\}).  
\end{equation*}
We thus have
\begin{eqnarray}
\E_i[\mathbbm{1}(X_i \in \mbox{DIS}(A_\epoch))] 
  &\leq& \E_i[\mathbbm{1}(X_i \in \mbox{DIS}(\{ h \mid
    \prob_X(h(X) \neq h^*(X)) \leq 16 \gamma \Delta_{\epoch-1}^* + 2
    \err_\epoch(h^*)\}))] \nonumber \\ 
  &\leq& \theta (16 \gamma \Delta_{\epoch-1}^* + 2 \err_\epoch(h^*)), \label{eq:e_bound}
\end{eqnarray}
where the last inequality uses the definition 
of the disagreement coefficient
\begin{equation}
  \theta(h^*) := \sup_{r > 0} \; \frac{\prob_X(\{X \mid \exists h
    \in \mathcal{H} \mbox{ s.t. } \prob_X(h(X) \neq h^*(X)) \leq
    r, h^*(X) \neq h(X)\})}{r}
    .
    \notag
\end{equation}
Summing \eqref{eq:e_bound} over $i \in \{1,\ldots, n\}$ and noting that the high probability event in Theorem \ref{thm:main-regret} holds 
over all epochs, we get that with probability at least $1 - \delta$, 
\begin{eqnarray*}
\forall n \geq 3, \quad  \sum_{i=1}^n \E_i[\mathbbm{1}(X_i \in D_{\epochi})]  &\leq&
  3 + \sum_{j=2}^\epmax (\epend[j] - \epend[j-1]) \theta (16 \gamma 
  \Delta_{j-1}^* + 2 \err_j(h^*))\\ 
 &\leq& 3 + 2 n \theta \erravg{\epmax}(h^*) + 16\gamma \theta
  \sum_{j=2}^\epmax (\epend[j] - \epend[j-1])  \Delta_{j-1}^* \\ 
 &=& 3 + 2 n \theta \erravg{\epmax}(h^*) + 16 \gamma \theta \sum_{j=2}^\epmax
  \frac{(\epend[j] - \epend[j-1])}{\epend[j-1]} \epend[j-1]
  \Delta_{j-1}^*. 
\end{eqnarray*}
A similar argument as Lemma~\ref{lemma:delta_pmin_monotone} shows that
$\epend[j]\Delta_j^*$ is increasing in $j$, so we have by a further
invocation of Lemma~\ref{lemma:epend-sum}
\begin{eqnarray*}
\sum_{i=1}^n \E_i[\mathbbm{1}(X_i \in D_{\epochi})] 
&\leq&	3 + 2 n \theta \erravg{\epmax}(h^*) + 128 \gamma \theta
(n-1)\Delta_{\epmax-1}^* \log(n-1) \\ 
&=& 3+ 2 n\theta \erravg{\epmax}(h^*)  \nonumber \\ 
 && + \theta \order
\left(\sqrt{n\overline{\err}_{\epmax}(h^*)\left(\log\big(\frac{|\mathcal{H}|}{\delta}\big)\log^2 
  n + \log^3 n \right)}  +
\log\big(\frac{|\mathcal{H}|}{\delta}\big)\log^2 n + \log^3 n
\right). \nonumber 	 
\end{eqnarray*}
Combining this and \eqref{eq:q_high_prob} via a union bound leads to the desired result.
\section{Proofs for Tsybakov's low-noise condition}
\label{sec:tsybakov}

We begin with a lemma that captures the behavior of the
$\Delta^*_\epoch$ terms, $\erravg{\epoch}(h^*)$ and the probability of
disagreement region under the Tsybakov noise
condition~\eqref{eqn:tsybakov}. The proofs of
Corollaries~\ref{cor:regret-tsybakov} and~\ref{cor:label-tsybakov} are
immediate given the lemma.

\begin{lemma}
  Under the conditions of Theorem~\ref{thm:regret}, suppose further
  that the low-noise condition~\eqref{eqn:tsybakov} holds. Then we
  have for all epochs $\epoch = 1,2,\ldots, \epmax$
  \begin{equation}
    \err_\epoch(h^*) \leq c\epsilon_\epoch \log\epend\,
    \epend^{\frac{2(1 - \tsybaexp)}{2 - \tsybaexp}}, \quad \mbox{and}
    \quad \erravg{\epoch}(h^*) \leq 5c\epsilon_\epoch\log^2\epend\,
    \epend^{\frac{2(1-\tsybaexp)}{2 - \tsybaexp}}.
    \label{eqn:tsybakov-ind}
  \end{equation}
  \label{lemma:tsybakov}
\end{lemma}

\begin{proof}
  We will establish the lemma inductively. We make the following
  inductive hypothesis. There exists a constant $c > 0$ (dependent on
  the distributional parameters) such that for all epochs $j \geq 1$,
  the bounds~\eqref{eqn:tsybakov-ind} in the statement of the Lemma
  hold. The base case for $j = 1$ trivially follows since $\err_1(h^*)
  = \erravg{1}(h^*) = \err(h^*) \leq 1 \leq c\epsilon_1
  \log\epend[1]\, \epend[1]^{\frac{2(1 - \tsybaexp)}{2 - \tsybaexp}}$,
  which is clearly true for an appropriately large value of
  $c$. Suppose now that the claim is true for epochs $j = 1,2,\ldots,
  \epoch-1$. We will establish the claim at epoch $\epoch$. To see
  this, first note that we have

  \begin{align*}
    \err_\epoch(h^*) &= \prob(\bbmone(h^*(X) \ne Y, X \in D_\epoch))
    \leq \prob(X \in D_\epoch).
  \end{align*}
  Under the noise condition, we can further upper bound the
  probability of the disagreement region, since by
  Theorem~\ref{thm:regret} we obtain

  \begin{align*}
    \prob(X \in D_\epoch) &= \prob(X \in \dis(A_\epoch)) \leq \prob
    \left(X \in \dis(\{h \in \H~:~\reg(h) \leq
    16\gamma\Delta^*_{\epoch-1})\right)\\
    &\leq \prob\left( X \in \dis(h \in \H~:~\prob(h(X) \ne h^*(X))
    \leq \tsybamult\, (16\gamma\Delta^*_{\epoch-1})^{\tsybaexp})
    \right), 
  \end{align*}
  where the first inequality follows from Theorem~\ref{thm:regret} and
  the second one is a consequence of Tsybakov's noise
  condition~\eqref{eqn:tsybakov}. Recalling the definition of
  disagreement coefficient~\eqref{eqn:dis-coeff}, this can be further
  upper bounded by

  \begin{equation}
    \prob(X \in D_\epoch) \leq \theta \tsybamult\, (16\gamma
    \Delta^*_{\epoch-1})^\tsybaexp.
    \label{eqn:tsyba-disagree}
  \end{equation}
  Hence, we have obtained the bound

  \begin{align*}
    \err_\epoch(h^*) &\leq \theta \tsybamult\, (16\gamma
    \Delta^*_{\epoch-1})^\tsybaexp.
  \end{align*}
  Note that \mbox{$\Delta^*_{\epoch-1} = c_1 \sqrt{\epsilon_{\epoch-1}
      \erravg{\epoch-1}(h^*)} +
    c_2\epsilon_{\epoch-1}\log\epend[\epoch-1]$}. Our inductive
  hypothesis~\eqref{eqn:tsybakov-ind} allows us to upper bound the
  $\erravg{\epoch-1}$ in this expression for $\Delta^*_{\epoch-1}$ and
  hence we obtain

  \begin{align*}
    \Delta^*_{\epoch-1} &\leq c_1
    \sqrt{\epsilon_{\epoch-1}\,5c\epsilon_{\epoch-1}\log^2\epend[\epoch-1]\,\epend[\epoch-1]^{\frac{2(1  
          - \tsybaexp)}{2 -\tsybaexp}}} + c_2
    \epsilon_{\epoch-1}\log\epend[\epoch-1]\\ 
    &\leq c_1
    \epsilon_{\epoch-1}\log\epend\,\epend^{\frac{1-\tsybaexp}{2-\tsybaexp}}\sqrt{5c}
    + c_2\epsilon_{\epoch-1}\log\epend[\epoch-1]\\
    &\leq \frac{\epsilon_{\epoch}\epend}{\epend[\epoch-1]}
    \log\epend\, \left( c_1\sqrt{5c}
    \epend^{\frac{1-\tsybaexp}{2-\tsybaexp}} + c_2 \right)\\
    &\leq 2\epsilon_\epoch \log\epend\,\left( c_1\sqrt{5c}
    \epend^{\frac{1-\tsybaexp}{2-\tsybaexp}} + c_2 \right).
  \end{align*}
  Since $\epend \geq 3$ and $0 < \tsybaexp \leq 1$, we can further
  write

  \begin{equation}
    \Delta^*_{\epoch-1} \leq 2\epsilon_\epoch \log\epend\,
    \epend^{\frac{1-\tsybaexp}{2-\tsybaexp}}\left( c_1\sqrt{5c} + c_2
    \right).
    \label{eqn:delta-tsybakov}
  \end{equation}
  Substituting this inequality in our earlier bound on
  $\err_\epoch(h^*)$ yields

  \begin{align*}
    \err_\epoch(h^*) &\leq \theta \tsybamult\, \left(32\gamma
    \epsilon_\epoch \log\epend\,
    \epend^{\frac{1-\tsybaexp}{2-\tsybaexp}}\left( c_1\sqrt{5c} + c_2
    \right)\right)^\tsybaexp.
  \end{align*}
  Since $\epsilon_\epoch\epend\,\log\epend \geq 1$ and $0 < \tsybaexp
  \leq 1$, we can further bound

  \begin{align*}
    \err_\epoch(h^*) &\leq \theta \tsybamult \epsilon_\epoch
    \epend\,\log\epend\, \left(32\gamma \,
    \epend^{\frac{-1}{2-\tsybaexp}}\left( c_1\sqrt{5c} + c_2
    \right)\right)^\tsybaexp\\
    &= \theta \tsybamult \epsilon_\epoch
    \epend\,\log\epend\, \left(32\gamma \,
    \left( c_1\sqrt{5c} + c_2 \right)\right)^\tsybaexp
    \epend^{\frac{-\tsybaexp}{2-\tsybaexp}} \\
    &= \theta \tsybamult \epsilon_\epoch
    \epend^{\frac{2(1-\tsybaexp)}{2-\tsybaexp}}\,\log\epend\,
    \left(32\gamma \, \left( c_1\sqrt{5c} + c_2
    \right)\right)^\tsybaexp\\
    &\leq c\epsilon_\epoch\log\epend\,
    \epend^{\frac{2(1-\tsybaexp)}{2-\tsybaexp}}. 
  \end{align*}

  Here the last bound follows for any choice of $c$ such that

  \begin{align*}
    c \geq \theta \tsybamult\, \left(32\gamma \, \left( c_1\sqrt{5c}
    + c_2 \right)\right)^\tsybaexp.
  \end{align*}
  The above inequality has a solution since the LHS is smaller than
  the RHS at $c = 0$, while for $c$ large enough, the LHS grows
  linearly in $c$, while the RHS grows as $c^{\tsybaexp/2}$, and hence
  is asymptotically smaller than the LHS.

  We now verify the second part of our induction hypothesis for epoch
  $\epoch$. Note that we have
  
  \begin{align*}
    \erravg{\epoch}(h^*) &= \frac{1}{\epend[\epoch]}\,
    \sum_{j=1}^{\epoch} (\epend[j] - \epend[j-1]) \err_j(h^*)\\
    &\leq \frac{1}{\epend[\epoch]}\, \sum_{j=1}^{\epoch}
    (\epend[j] - \epend[j-1])\, c\epsilon_j \log\epend[j]\,
    \epend[j]^{\frac{2(1 - \tsybaexp)}{2 - \tsybaexp}}\\
    &= \frac{1}{\epend[\epoch]}\, \sum_{j=1}^{\epoch}
    \frac{(\epend[j] - \epend[j-1])}{\epend[j]}\, c\epsilon_j\epend[j]
    \log\epend[j]\, \epend[j]^{\frac{2(1 - \tsybaexp)}{2 - \tsybaexp}}.
  \end{align*}
  We now observe that $\epend[j]$ is clearly increasing in $j$, and so
  is $\epend[j]\epsilon_j$ by definition. Consequently, we can further
  upper bound this inequality by

  \begin{align*}
    \erravg{\epoch}(h^*) &\leq \frac{1}{\epend[\epoch]}\,
    \epsilon_{\epoch}\epend[\epoch] \log\epend[\epoch]\,
    \sum_{j=1}^{\epoch} \frac{(\epend[j] -
      \epend[j-1])}{\epend[j]}\, c \epend[j]^{\frac{2(1 -
        \tsybaexp)}{2 - \tsybaexp}} \\
    &\stackrel{(a)}{\leq} c\epsilon_{\epoch}\log\epend[\epoch]\,\epend^{\frac{2(1 - 
        \tsybaexp)}{2 - \tsybaexp}}\,\left(1 +  
    \sum_{j=2}^{\epoch} \frac{(\epend[j] -
        \epend[j-1])}{\epend[j]}\right)\\
    &= c\epsilon_{\epoch}\log\epend[\epoch]\,\epend^{\frac{2(1 - 
        \tsybaexp)}{2 - \tsybaexp}}\,\left(1 +  
    \sum_{j=1}^{\epoch-1} \frac{(\epend[j+1] -
        \epend[j])}{\epend[j+1]}\right)\\
    &\leq c\epsilon_{\epoch}\log\epend[\epoch]\,\epend^{\frac{2(1 - 
        \tsybaexp)}{2 - \tsybaexp}} \,\left(1 +  
    \sum_{j=1}^{\epoch-1} \frac{(\epend[j+1] -
        \epend[j])}{\epend[j]}\,\right)\\
  \end{align*}
  where the inequality $(a)$ holds since $\epend[j]$ is increasing in
  $j$ and $\tsybaexp \in (0,1]$ so that the exponent on $\epend[j]$ is
    non-negative, and the final inequality follows since $\epend[j]
    \leq \epend[j+1]$. Invoking Lemma~\ref{lemma:epend-sum}, we obtain

  \begin{align*}
    \erravg{\epoch}(h^*) &\leq
    \epsilon_{\epoch}\log\epend[\epoch]\, (1 + 4c\log \epend)\, 
    \epend^{\frac{2(1 - \tsybaexp)}{2 - \tsybaexp}}\\
    &\leq 5c\epsilon_{\epoch}\log^2\epend\,
    \epend^{\frac{2(1 - \tsybaexp)}{2 - \tsybaexp}},
  \end{align*}
  where we used the fact that $1 \leq \log\epend$. Therefore, we have
    established the second part of the inductive claim, finishing the
    proof of the lemma.
\end{proof}

Using the lemma, we now prove the corollaries.

\begin{proof-of-corollary}[\ref{cor:regret-tsybakov}]
  Based on the proof of Lemma~\ref{lemma:tsybakov}, we see that
  $\Delta^*_\epoch$ satisfies the
  bound~\eqref{eqn:delta-tsybakov}. Plugging this into the statement
  of Theorem~\ref{thm:regret} immediately yields the lemma.
\end{proof-of-corollary}

\begin{proof-of-corollary}[\ref{cor:label-tsybakov}]
  Based on the proof of Lemma~\ref{lemma:tsybakov}, we see that the
  probability of the disagreement region follows the
  bound~\eqref{eqn:tsyba-disagree}. Substituting the
  bound~\eqref{eqn:delta-tsybakov} yields the stated result.
\end{proof-of-corollary}

\section{Analysis of the Optimization Algorithm}
\label{sec:opt}

We begin by showing how to find the most violated constraint (Step \ref{step:violated})
by calling an importance-weighted ERM oracle. Then we  
prove Theorem~\ref{thm:opt_converge}, 
followed by the framework and proof for
Theorem~\ref{thm:optprob-unlabeled}.

\subsection{Finding the Most Violated Constraint}
\label{appendix:most_violated_erm}
Recall our earlier notation $\disind{h}(x) = \bbmone(h(x) \neq
h_\epoch(x) \wedge x \in D_\epoch)$. Consider solving (\optprob) using
an unlabeled sample $S$ of size $\unlab$.  Note that
Step~\ref{step:violated} is equivalent to
\begin{eqnarray}
	&  \arg \min_{h \in \H}&  b_\epoch(h) - \widehat\E_X \left[ \frac{\disind{h}(X)}{P_{\blambda}(X)}\right] \\
	=& \arg \min_{h \in \H}&  2 \gamma \beta^2 (\epend-1)
  \Delta_{\epoch-1} \err(h,\ztil_{\epoch-1}) + \widehat\E_X \left[ \left(2 \alpha^2 - \frac{1}{P_{\blambda}(X)}\right) \disind{h}(X)\right] \nonumber \\
	=& \arg \min_{h \in \H}&  2 \gamma \beta^2 (\epend-1) \Delta_{\epoch-1} \err(h,\ztil_{\epoch-1}) \nonumber \\
	 &&    + \widehat\E_X \left[ \left(2 \alpha^2 - \frac{1}{P_{\blambda}(X)}\right) \disind{h}(X) + \max\left( \frac{1}{P_{\blambda}(X)} - 2 \alpha^2,0\right)\bbmone(X \in D_\epoch) \right] \nonumber\\
	= &\arg \min_{h \in \H}& 2 \gamma \beta^2 (\epend-1) \Delta_{\epoch-1} \err(h,\ztil_{\epoch-1}) \nonumber \\ 
	  &&  + \widehat\E_X \left[ \max\left(2 \alpha^2 - \frac{1}{P_{\blambda}(X)},0\right)\bbmone(X \in D_\epoch) \bbmone(h(X) \neq h_\epoch(X)) \right] \nonumber \\
	  && + \widehat\E_X \left[ \max\left(\frac{1}{P_{\blambda}(X)} - 2 \alpha^2,0\right)\bbmone(X \in D_\epoch) \bbmone(h(X) \neq -h_\epoch(X)) \right] \nonumber \\ 
	= &\arg \min_{h \in \H}& 2 \gamma \beta^2 (\epend-1) \Delta_{\epoch-1} \err(h,\ztil_{\epoch-1}) \nonumber \\ 
	  &&  + \widehat\E_X \left[ |s_{\blambda}(X)| \bbmone(X \in D_\epoch) \bbmone(h(X) \neq \sign(s_{\blambda}(X))h_\epoch(X)) \right],  \label{eq:batch_cost_sensitive} \nonumber
\end{eqnarray}
where $s_{\blambda}(X) := 2 \alpha^2 - 1 / P_{\blambda}(X)$. In the above derivation, the second equality 
is by the fact that the extra term added to the objective is independent of $h$ and hence does not change the minimizer. 
The third equality uses a case analysis on the sign of $s_{\blambda}(X)$ and the identity $1 - \bbmone(h(X) \neq h_\epoch(X)) = \bbmone(h(X) \neq -h_\epoch(X))$.
The last expression suggests that an importance-weighted error minimization oracle can find the desired classifier on examples $\{(X,Y^*,W)\}$ 
with labels and importance weights defined as:
\begin{eqnarray*}
Y^* &:=& \arg \min_{Y} c(X,Y), \\ 
	W &:=& |c(X,1) - c(X,-1)|,
\end{eqnarray*}
where
\begin{align}
c(X,Y) := 
\begin{cases}
	2 \gamma \beta^2 \Delta_{\epoch-1}\left( \frac{\bbmone(X_i \in D_{\epochi} \wedge Y \neq Y_i)Q_i}{P_{\epochi}(X_i)} + \bbmone(X_i \notin D_{\epochi} \wedge Y \neq h_{\epochi}(X_i))\right), & X = X_i \in \ztil_{\epoch-1},\\ 
	\frac{1}{\unlab}|s_{\blambda}(X)| \bbmone(X \in D_\epoch) \bbmone(Y \neq \sign(s_{\blambda}(X))h_\epoch(X)), & X \in S.
\end{cases}	
	\label{eq:cost}
\end{align}

\subsection{Proof of Theorem~\ref{thm:opt_converge}}
\label{sec:opt-pop}

Where clear from context, we drop the subscript $\epoch$.

We first show that each coordinate ascent step causes sufficient
increase in the dual objective. Pick any $h$ and $\blambda$. Let
$\blambda'$ be identical to $\blambda$ except that $\lambda'_h =
\lambda_h + \delta$ for some $\delta>0$. Then the increase in the dual
objective $\dualobj$ can be computed directly:

\begin{eqnarray}
  \lefteqn{\dualobj(\blambda') - \dualobj(\blambda)}
  \nonumber \\
  &=& \delta \E_X [\disind{h}(X)] +
  2\E_X[\bbmone(X \in D_\epoch)(\sqrt{\bq_{\blambda}(X)^2 +
      \delta \disind{h}(X)}  - \bq_{\blambda}(X))] - \delta b(h)
  \nonumber \\ 
  &\geq&\delta \E_X [\disind{h}(X)]+ 2 \E_X \left[
    \bq_{\blambda}(X) \left( \frac{\delta \disind{h}(X)}{2
      \bq_{\blambda}(X)^2} - \frac{\delta^2 \disind{h}(X)^2}{8
      \bq_{\blambda}(X)^4}\right) \right]  - \delta b(h) 
\label{eq:sqrt-approx} 
\\
&=& \delta \E_X \left[ \left( 1 +
  \frac{1}{\bq_{\blambda}(X)}\right)\disind{h}(X) - b(h) \right] -
\delta^2 \E\left[ \frac{\disind{h}(X)^2}{4
    \bq_{\blambda}(X)^3}\right] 
\nonumber
\\ 
&=&\delta \left(\E_X \left[
  \frac{\disind{h}(X)}{P_{\blambda}(X)} \right] - b(h)\right) 
 - \frac{\delta^2}{4} \; \E\left[
   \frac{\disind{h}(X)^2}{\bq_{\blambda}(X)^3}\right]. 
\label{eq:obj_inc}
\end{eqnarray}
The inequality \eqref{eq:sqrt-approx} uses the fact that $\sqrt{1+z}
\geq 1 + z /2 - z^2 / 8$ for all $z\geq 0$ (provable, for instance,
using Taylor's theorem).  The lower bound \eqref{eq:obj_inc} on the
increase in the objective value is maximized exactly at
\begin{equation}
  \delta = 2 \frac{\E [ \disind{h}(X) / P_{\blambda}(X) - b(h)
  ]}{\E_X[\disind{h}(X)^2 / \bq_{\blambda}(X)^3]}, 
\end{equation}
as in Step~\eqref{step:coord-update}.  Plugging into \eqref{eq:obj_inc}, it
follows that if $h$ is chosen on some iteration of
Algorithm~\ref{alg:coord} prior to halting then the dual objective
$\dualobj$ increases by at least
\begin{equation}  \label{eq:coor-up-val}
  \frac{\E_X[\disind{h}(X) / P_{\blambda}(X) -
      b(h)]^2}{\E_X[\disind{h}(X)^2 / \bq_{\blambda}(X)^3]} \geq
  \varepsilon^2 \minp^3
\end{equation}
since $\bq_{\blambda}(x) \geq \minp$, and since 
$
\E_X[\disind{h}(X) / P_{\blambda}(X) - b(h)] \geq \varepsilon
$.

The initial dual objective is $\dualobj(\bzero) = (1 + \minp)^2
\prob(D_\epoch)$.  Further, by duality and the fact that $P(X) = 1/2$
is a feasible solution to the primal problem, we have
$\dualobj(\blambda) \leq 2(1 + \minp^2)\prob(D_\epoch)$.  And of
course, rescaling can never cause the dual objective to decrease.
Combining, it follows that the coordinate ascent algorithm halts in at
most $ {\prob(D_\epoch)(2(1+\minp^2) - (1+\minp)^2)}/{(\varepsilon^2
  \minp^3)} \leq {\prob(D_\epoch)}/{(\varepsilon^2 \minp^3)} $ rounds
proving the bound given in the theorem.

By this same reasoning, the left hand side of \eqref{eq:coor-up-val}
is equal to $ \delta \cdot \E_X[\disind{h}(X) / P_{\blambda}(X)
  - b(h)] $, which is at least $\delta \varepsilon$.  That is, the
change on each round in the dual objective $\dualobj$ is at least
$\varepsilon$ times the change in one of the coordinates $\lambda_h$.
Furthermore, the rescaling step can never cause the weights
$\lambda_h$ to increase.  Therefore, $\varepsilon
\lone{\hat{\blambda}}$ is upper bounded by the total change in the
dual objective, which we bounded above.  This proves the bound on
$\lone{\hat{\blambda}}$ given in the theorem.

To see \eqref{eq:prim-obj-bnd}, consider first the function $g(s) =
\dualobj(s \cdot \blambda)$ for $\blambda$ as in the algorithm {\em
  after\/} the rescaling step has been executed.  At this point, it
is necessarily the case that $s=1$ maximizes $g$ over $s\in
[0,1]$ (since $\blambda$ has already been rescaled).  This implies
that $g'(1)\geq 0$ where $g'$ is the derivative of $g$; that is,
\begin{equation}  \label{eq:dif-rescale-bnd}
  0 \leq g'(1) = \E\left[ \frac{\sum_h \lambda_h
      \disind{h}(X)}{ P_{s\cdot \blambda}(X)} \right] - \sum_h
  \lambda_h b(h).
\end{equation}

Now let $F(P)$ denote the modified primal objective function in
\eqref{eq:mod-obj} and let $F^*$ denote the optimal objective value
\begin{equation}
\begin{split}
 F^* \;:= \;& \inf_P \; \E_X \left[ \frac{1}{1-P(X)} \right] + \minp^2 \E_X \left[ \frac{\bbmone(X \in D_\epoch)}{P(X)}\right]\\ 
\mbox{s.t.} \quad& P \mbox{ satisfying } \eqref{eq:queryp} \mbox{ and } \forall x \in \dataspace \; 0 \leq P(x) \leq 1.
\end{split}
\label{eq:op_value_modified}
\end{equation}
Then we have
\begin{eqnarray}
  F(P_{\hat{\blambda}}) &\leq& F(P_{\hat{\blambda}})
  + \sum_h \hat{\lambda}_h \left( \E_X\left[\frac{\disind{h}(X)}{P_{\hat{\blambda}}(X)}\right]
  - b(h) \right)
  \label{eq:eqn-a1}
  \\
  &=& \inf_{0 \leq P(x) \leq 1} \mathcal{L}(P,\hat{\blambda})
  \label{eq:eqn-a2}
  \\
  &\leq& \sup_{\blambda \geq 0} \inf_{0 \leq P(x) \leq 1} \mathcal{L}(P,\blambda)
  \nonumber \\
  &\leq& F^*.
  \label{eq:eqn-a3}
\end{eqnarray}
Here, \eqref{eq:eqn-a1} follows from \eqref{eq:dif-rescale-bnd};
\eqref{eq:eqn-a2} by the definition of $P_{{\blambda}}(X)$ as the
minimizer of the Lagrangian.
To establish \eqref{eq:eqn-a3}, first notice that the following holds for all feasible $\widetilde{P}$ and all non-negative $\blambda$:
\begin{equation*}
\inf_{0 \leq P(x) \leq 1} \mathcal{L}(P,\blambda) \leq \mathcal{L}(\widetilde{P},\blambda) \leq F(\widetilde{P})
\end{equation*}
by the definition of the Lagrangian \eqref{eq:Lagrangian}. This implies 
\begin{equation*}
\inf_{0 \leq P(x) \leq 1} \mathcal{L}(P,\blambda) \leq F^*
\end{equation*}
for all non-negative $\blambda$, leading to \eqref{eq:eqn-a3}.
Then we have 
\[
\E\left[ \frac{1}{1 - P_{\hat{\blambda}}(X)}\right] \leq
F(P_{\hat{\blambda}}) \leq F^* \leq
f^* + \minp
\mbox{Pr}(D_\epoch).
\]

\subsection{Proof of Theorem~\ref{thm:optprob-unlabeled}}
\label{sec:opt-samp}

For $\varepsilon>0$, define $\Lset := \braces{ \blambda \in
  \mathbb{R}^{\H} : \blambda \geq \bzero,\, \lone{\blambda} \leq
  1/\varepsilon }$. We begin with a simple lemma.

\begin{lemma}
  \label{lem:uniform}
  Suppose $\phi \colon \bbR \times \dataspace \to \bbR$ be $L$-Lipschitz with
  respect to its first argument, and $\phi\parens{\sum_{h \in \H}
    \lambda_h \disind{h}(x) ,\, x} \leq R$ for all $\blambda \in
  \Lset$ and $x \in \dataspace$.  Let $\widehat\E_X[\cdot]$ denote the
  empirical expectation with respect to an i.i.d.~sample from
  $\distx$.
  For any $\delta \in (0,1)$, with
  probability at least $1-\delta$, every $\blambda \in \Lset$
  satisfies
  \begin{align*}
  & \Abs{ \widehat\E_X\Brackets{ \phi\Parens{\sum_{h \in \H} \lambda_h
        \disind{h}(X) ,\, X} } - \E_X\Brackets{ \phi\Parens{\sum_{h
          \in \H} \lambda_h \disind{h}(X) ,\, X} } } \\
  &\leq \frac{2L}{\varepsilon} \cdot \sqrt{\frac{2\ln|\H|}{\unlab}} + R
  \cdot \sqrt{\frac{\ln(1/\delta)}{\unlab}} .
  \end{align*}
\end{lemma}
\begin{proof}
  Let $\bx \in \{0,1\}^{\H}$ denote the vector with $x_h = \ind{h(x)
    \neq h_\epoch(x)}$, and define the linear function class
  \[
    \cF := \Braces{ x \mapsto \dotp{\blambda,\bx} : \blambda \in \Lset
    } .
  \]
  By a simple variant of the argument by~\citet{Bartlett02}, with probability at least $1-\delta$,
  \[
    \Abs{ \widehat\E_X\Brackets{ \phi\Parens{\sum_{h \in \H}
          \lambda_h \disind{h}(X) ,\, X} } - \E_X\Brackets{
        \phi\Parens{\sum_{h \in \H} \lambda_h \disind{h}(X) ,\, X} } }
    \ \leq \ 2L \cdot \cR_\unlab(\cF) + R \cdot
    \sqrt{\frac{\ln(1/\delta)}{\unlab}}
  \]
  for all $\blambda \in \Lset$, where $\cR_\unlab(\cF)$ is the expected
  Rademacher average for the linear function class $\cF$ for an
  i.i.d.~sample of size $n$.  By \citet{KakadeSrTe2009}, this
  Rademacher complexity satisfies
  \[
  \cR_\unlab(\cF) \ \leq
  \ \frac1{\varepsilon}\sqrt{\frac{2\ln|\H|}{\unlab}} .
  \]
  This completes the proof.
\end{proof}

\begin{lemma}
  \label{lem:optprob-uniform}
  Pick any $\delta \in (0,1)$.  Let $\widehat\E_X[\cdot]$ denote the
  empirical expectation with respect to an i.i.d.~sample from
  $\distx$.  With probability at least
  $1-\delta$, every $\blambda \in \Lset$ satisfies
  \begin{align*}
    \Abs{ \E_X\Brackets{ \frac{1}{1-P_{\blambda}(X)} } -
      \widehat\E_X\Brackets{ \frac{1}{1-P_{\blambda}(X)} } } &
    \ \leq \ \sqrt{\frac{2\ln|\H|}{\minp^2\varepsilon^2\unlab}} +
    \sqrt{\frac{\Parens{\minp^2 + 1/\varepsilon}\ln(3/\delta)}{\unlab}}
  \end{align*}
  and for all $h \in \H$,
  \begin{align*}
    \Abs{ \E_X\Brackets{ \frac{\disind{h}(X)}{P_{\blambda}(X)} } -
      \widehat\E_X\Brackets{ \frac{\disind{h}(X)}{P_{\blambda}(X)} }
    } & \ \leq \ \sqrt{\frac{2\ln|\H|}{\minp^4\varepsilon^2\unlab}} +
    \sqrt{\frac{\ln(3|\H|/\delta)}{\minp^2\unlab}} +
    \sqrt{\frac{\ln(6|\H|/\delta)}{2\unlab}}
  \end{align*}
  and
  \begin{align*}
    \Abs{ \E_X\brackets{\disind{h}(X)} -
      \widehat\E_X\brackets{\disind{h}(X)} } & \ \leq
    \ \sqrt{\frac{\ln(6|\H|/\delta)}{2\unlab}} .
  \end{align*}
\end{lemma}

\begin{proof}
  Observe that $1/(1-P_{\blambda}(x)) = 1 + q_{\blambda}(x)$ for all
  $\blambda \in \Lset$ and $x \in \dataspace$.  Now we apply
  Lemma~\ref{lem:uniform} to the function $\phi_1(z,x) :=
  \sqrt{\minp^2+z}$, which is $(2\minp)^{-1}$-Lipschitz with respect
  to its first argument.  Since $q_{\blambda}(x) = f_1(\sum_{h \in \H}
  \lambda_h \disind{h}(x), x) \leq \sqrt{\minp^2 + 1/\varepsilon}$ for
  all $\blambda \in \Lset$ and $x \in \dataspace$, Lemma~\ref{lem:uniform}
  implies that, with probability at least $1-\delta/3$,
  \begin{equation}
    \label{eqn:rad1}
    \Abs{ \E_X\Brackets{ \frac{1}{1-P_{\blambda}(X)} } -
      \widehat\E_X\Brackets{ \frac{1}{1-P_{\blambda}(X)} } } \ \leq
    \ \frac{1}{\minp\varepsilon} \sqrt{\frac{2\ln|\H|}{\unlab}} +
    \sqrt{\frac{\Parens{\minp^2 +
          1/\varepsilon}\ln(3/\delta)}{\unlab}} , \quad \forall
    \blambda \in \Lset .
  \end{equation}
  Next, observe that for every $h \in \H$ and $x \in \dataspace$,
  \[
    \frac{\disind{h}(x)}{P_{\blambda}(x)}
    \ = \ \disind{h}(x) + \frac{\disind{h}(x)}{q_{\blambda}(x)}
    .
  \]
  By Hoeffding's inequality and a union bound, we have with
  probability at least $1-\delta/3$,
  \begin{equation}
    \label{eqn:rad2}
    \Abs{ \E_X\brackets{\disind{h}(X)} -
      \widehat\E_X\brackets{\disind{h}(X)} } \ \leq
    \ \sqrt{\frac{\ln(6|\H|/\delta)}{2\unlab}} , \quad \forall h \in \H .
  \end{equation}
  Now we apply Lemma~\ref{lem:uniform} to the functions $\phi_h(z,x)
  := \disind{h}(x)/\sqrt{\minp^2+z}$ for each $h \in \H$; each
  function $\phi_h$ is $(2\minp^2)^{-1}$-Lipschitz with respect to its
  first argument.  Furthermore, since $\phi_h(\sum_{h\in\H}\lambda_h
  \disind{h}(x), x) = \disind{h}(x) / q_{\blambda}(x) \leq 1/\minp$
  for all $\blambda \in \Lset$ and $x \in \dataspace$,
  Lemma~\ref{lem:uniform} and a union bound over all $h \in \H$
  implies that, with probability at least $1-\delta/3$
  \begin{equation}
    \label{eqn:rad3}
    \Abs{ \E_X\Brackets{ \frac{\disind{h}(X)}{q_{\blambda}(X)} } -
      \widehat\E_X\Brackets{ \frac{\disind{h}(X)}{q_{\blambda}(X)} }
    } \ \leq \ \sqrt{\frac{2\ln|\H|}{\minp^4\varepsilon^2\unlab}} +
    \sqrt{\frac{\ln(3|\H|/\delta)}{\minp^2\unlab}} , \quad \forall
    \blambda \in \Lset ,\, h \in \H .
  \end{equation}
  Finally, by a union bound, all of~\eqref{eqn:rad1},
  \eqref{eqn:rad2}, and \eqref{eqn:rad3} hold simultaneously with
  probability at least $1-\delta$.
\end{proof}

We can now prove Theorem~\ref{thm:optprob-unlabeled}. We first state a
slightly more explicit version of the theorem, which is then proved.

\begin{theorem}
  \label{thm:optprob-unlabeled-big}
  Let $S$ be an i.i.d.~sample of size $\unlab$ from the $\distx$.
  Suppose Algorithm~\ref{alg:coord} is run on the $\epoch$-th epoch
  for solving $(\optprob_{S,\samperr})$ up to slack $\samperr$ in the
  variance constraints. Then the following holds:
\begin{enumerate}
\item Algorithm \ref{alg:coord} halts in at most $\frac{\widehat{\prob}(D_m)}{8 P_{\min,\epoch}^3\varepsilon^2}$ iterations, where $\widehat{\prob}(D_\epoch) := \sum_{X \in S} \bbmone(X \in D_m) / \unlab$.
\item The solution $\hat{\blambda} \geq \bzero$ it outputs has bounded $\ell_1$ norm:
\begin{equation*}
\lone{\hat{\blambda}}\leq \widehat{\prob}(D_\epoch) / \varepsilon.
\end{equation*}
\item There exists an absolute constant $C>0$ such that the following
holds. 
  If
  \[
  \unlab \ \geq \ C \cdot \Parens{ \Parens{ \frac1{\pmin^4\samperr^2} +
      \alpha^4 } \cdot \frac{\log|\H|}{\samperr^2} + \Parens{
      \frac1{\pmin^2} + \frac1\samperr + \alpha^4 } \cdot
    \frac{\log(1/\delta)}{\samperr^2} } ,
  \]
  then with probability at least $1-\delta$, the query probability function 
$P_{\hat{\blambda}}(x)$ satisfies
\begin{itemize}
\item All constraints of $(\optprob)$ except with
  slack $2.5\samperr$ in constraints~\eqref{eq:queryp}, 
\item Approximate primal optimality:
  \begin{equation*}
    \E_X \Brackets{ \frac{1}{1 - P_{\hat{\blambda}}(X)} } \ \leq
    \ 
    f^* + 8 \pmin
    \prob(D_\epoch) + \parens{ 2 + 4 \pmin} \samperr, 
  \end{equation*}
where $f^*$ is the optimal value of $(\optprob)$ defined in \eqref{eq:op_value}. 
\end{itemize}
\end{enumerate}
\end{theorem}

Theorem~\ref{thm:optprob-unlabeled} is just a result of some
simplifications in the $\order(\cdot)$ notation in the above
result. We now prove the theorem.

\begin{proof-of-theorem}[\ref{thm:optprob-unlabeled-big}]
The first two statements, finite convergence and boundedness of the solution's $\ell_1$ norm, 
can be proved with the techniques in Appendix \ref{sec:opt-pop} that establish the same for Theorem \ref{thm:opt_converge}.
We thus focus on proving the third statement here.

Let $\widehat\E_X[\cdot]$ denote empirical expectation with
respect to $S$.
Hoeffding's inequality implies that with probability at least
  $1-\delta/2$,
  \begin{equation}
    \label{eqn:disagree-discrep}
    \widehat\E_X[\ind{X \in D_\epoch}] \ \leq \ \E_X[\ind{X \in
        D_\epoch}] + \samperr .
  \end{equation}
  Also, Lemma~\ref{lem:optprob-uniform} implies that with probability
  at least $1-\delta/2$,
  \begin{align}
    \Abs{ \widehat\E_X \Brackets{ \frac{1}{1 - P_{\blambda}(X)} } -
      \E_X \Brackets{ \frac{1}{1 - P_{\blambda}(X)} } } & \ \leq
    \ \samperr , \quad \forall \blambda \in \Lsetbig ;
    \label{eqn:obj-discrep} \\
    \Abs{ \E_X\brackets{\disind{h}(X)} -
      \widehat\E_X\brackets{\disind{h}(X)} } & \ \leq
    \ \samperr/(8\alpha^2) , \quad \forall h \in \H ;
    \label{eqn:disind-discrep} \\
    \Abs{ \E_X\Brackets{ \frac{\disind{h}(X)}{P_{\blambda}(X)} } -
      \widehat\E_X\Brackets{ \frac{\disind{h}(X)}{P_{\blambda}(X)} }
    } & \ \leq \ \samperr/4 , \quad \forall \blambda \in \Lsetbig ,\,
    h \in \H .
    \label{eqn:cons-discrep}
  \end{align}
  Therefore, by a union bound, there is an event of probability mass
  at least $1-\delta$ on which Eqs.~\eqref{eqn:disagree-discrep},
  \eqref{eqn:obj-discrep}, \eqref{eqn:disind-discrep},
  \eqref{eqn:cons-discrep} hold simultaneously.  We henceforth
  condition on this event.

  By Theorem~\ref{thm:opt_converge}, $\hat\blambda$ satisfies
  $\lone{\hat\blambda} \leq 1/\samperr$, the bound constraints
  in~\eqref{eq:minbndcons}, as well as
  \begin{equation}
    \label{eqn:conshat-lhat}
    \widehat\E_X\Brackets{ \frac{\disind{h}(X)}{P_{\hat\blambda}(X)}
    } \ \leq b_\epoch(h) + 2\samperr , \quad \forall h
    \in \H , 
  \end{equation}
  and
  \begin{equation}
    \label{eqn:objhat-lhat}
    \widehat\E_X \Brackets{ \frac{1}{1 - P_{\hat{\blambda}}(X)} }
    \ \leq \ \widehat\E_X \Brackets{ \frac{1}{1 - {\Psolnhatapx}(X)}
    } + 4 \pmin \widehat\E_X[\ind{X \in D_\epoch}]
  \end{equation}
  where $\Psolnhatapx$ is the optimal solution\footnote{Note that on a finite sample $S$, the primal optimization variables $P(x)$ are in a compact and convex subset of $\mathbb{R}^{|S|}$, and therefore an optimal solution can always be attained.} to
  $(\optprob_{S,\samperr})$.
  We use this to show that $P_{\hat\blambda}$ is a feasible
  solution for $(\optprob_{2.5\samperr})$, and compare its
  objective value to the optimal objective value for $(\optprob)$.

  Applying~\eqref{eqn:disind-discrep} and~\eqref{eqn:cons-discrep}
  to \eqref{eqn:conshat-lhat} gives
  \begin{equation}
    \E_X\Brackets{ \frac{\disind{h}(X)}{P_{\hat\blambda}(X)} }
    \ \leq \ b_\epoch(h) + 2.5\samperr , \quad \forall h \in \H .
    \notag
  \end{equation}
  Since $P_{\hat\blambda}$ also satisfies the bound constraints
  in~\eqref{eq:minbndcons}, it follows that $P_{\hat\blambda}$ is
  feasible for $(\optprob_{2.5\samperr})$.

  Now we turn to the objective value.
  Applying~\eqref{eqn:disagree-discrep} and~\eqref{eqn:obj-discrep} to
  \eqref{eqn:objhat-lhat} gives
  \begin{equation}
    \label{eqn:obj-lhat}
    \E_X \Brackets{ \frac{1}{1 - P_{\hat{\blambda}}(X)} } \ \leq
    \ \widehat\E_X \Brackets{ \frac{1}{1 - {\Psolnhatapx}(X)} } + 4
    \pmin \E_X[\ind{X \in D_\epoch}] + \parens{ 1 + 4 \pmin}
    \samperr .
  \end{equation}
  We need to relate the first term on the right-hand side to the
  optimal objective value for $(\optprob)$.

  Let $\blambda^*$ be the output of running Algorithm~\ref{alg:coord}
  for solving $(\optprob)$ up to slack $\samperr/2$.  By
  Theorem~\ref{thm:opt_converge}, $\blambda^*$ satisfies
  $\lone{\blambda^*} \leq 2/\samperr$, the bound constraints
  in~\eqref{eq:minbndcons}, as well as
  \begin{equation}
    \label{eqn:cons-lstar}
    \E_X\Brackets{ \frac{\disind{h}(X)}{P_{\blambda^*}(X)} } \ \leq
    \ b_\epoch(h) + \samperr/2 , \quad \forall h \in \H , \notag
  \end{equation}
  and
  \begin{equation}
    \label{eqn:obj-lstar}
    \E_X \Brackets{ \frac{1}{1 - P_{\blambda^*}(X)} } \ \leq
    \ 
      f^* + 4 \pmin
    \E_X[\ind{X \in D_\epoch}] .
  \end{equation}
  Applying~\eqref{eqn:obj-discrep} to~\eqref{eqn:obj-lstar}, we have
  \begin{equation}
    \label{eqn:objhat-lstar}
    \widehat\E_X \Brackets{ \frac{1}{1 - P_{\blambda^*}(X)} } \ \leq
    \ 
      f^* + 4 \pmin
    \E_X[\ind{X \in D_\epoch}] + \samperr .
  \end{equation}
  And applying~\eqref{eqn:disind-discrep} and~\eqref{eqn:cons-discrep}
  to~\eqref{eqn:cons-lstar} gives
  \begin{equation}
    \label{eqn:conshat-lstar}
    \widehat\E_X\Brackets{ \frac{\disind{h}(X)}{P_{\blambda^*}(X)} }
    \ \leq \ b_\epoch(h) + \samperr , \quad \forall h \in \H .
  \end{equation}
  This establishes that $\blambda^*$ is a feasible solution for
  $(\optprob_{S,\samperr})$.
  In particular,
  \begin{align*}
    \widehat\E_X \Brackets{ \frac{1}{1 - {\Psolnhatapx}(X)} } &
    \ \leq \ \widehat\E_X \Brackets{ \frac{1}{1 - P_{\blambda^*}(X)}
    } \\ & \ \leq \ 
    f^* + 4
    \pmin \E_X[\ind{X \in D_\epoch}] + \samperr
  \end{align*}
  where the second inequality follows from~\eqref{eqn:objhat-lstar}.
  We now combine this with~\eqref{eqn:obj-lhat} to obtain
  \begin{equation*}
    \E_X \Brackets{ \frac{1}{1 - P_{\hat{\blambda}}(X)} } \ \leq
    \ 
    f^* + 8 \pmin
    \E_X[\ind{X \in D_\epoch}] + \parens{ 2 + 4 \pmin} \samperr .
  \end{equation*}
\end{proof-of-theorem}

\section{Experimental Details}
\label{appendix:exp_details}
Here we provide more details about the experiments.
\begin{table}[t]
\centering
\caption{Binary classification datasets used in experiments}
\label{tbl:datasets}
\begin{tabular}{c||c|c|c|c}
	Dataset & $n$ & $s$ & $d$ & $r$
\\
\hline
\hline
titanic &  2201  &  3  & 8                            &0.323 \\
abalone &  4176  &  8  & 8                          &0.498 \\
mushroom &  8124  &  22  & 117                      &0.482 \\
eeg-eye-state &  14980  &  13.9901  & 14            &0.449 \\
20news &  18845  &  93.8854  & 101631               &0.479 \\
magic04 &  19020  &  9.98728  & 10                  &0.352 \\
letter &  20000  &  15.5807  & 16                   &0.233 \\
ijcnn1&  24995  &  13  & 22                         &0.099 \\ 
nomao &  34465  &  82.3306  & 174                   &0.286 \\ 
shuttle&  43500  &  7.04984  & 9                    &0.216 \\
bank &  45210  &  13.9519  & 44                     &0.117 \\
a9a &  48841  &  13.8676  & 123                     &0.239 \\
adult &  48842  &  11.9967  & 105                   &0.239 \\
w8a &  49749  &  11.6502  & 300                     &0.030 \\
bio &  145750  &  73.4184  & 74                     &0.009 \\
maptaskcoref &  158546  &  40.4558  & 5944          &0.438 \\
activity &  165632  &  18.5489  & 20                &0.306 \\
skin &  245057  &  2.948  & 3                       &0.208 \\
vehv2binary &  299254  &  48.5652  & 105            &0.438 \\
census &  299284  &  32.0072  & 401                 &0.062 \\
covtype &  581011  &  11.8789  & 54                 &0.488 \\ 
rcv1&  781265  &  75.7171  & 43001                  &0.474 
\end{tabular}
\end{table}
\subsection{Datasets}
\label{appendix:datasets}
Table \ref{tbl:datasets} gives details about the 22 binary classification datasets used in our experiments, 
where $n$ is the number of examples, $d$ is the number of 
features, $s$ is the average number of non-zero features per
example, and $r$ is the proportion of the minority class.
\subsection{Hyper-parameter Settings}
\label{appendix:hyper-parameters}
We start with the actual hyper-parameters used by \cover.  
Going back to Algorithm \ref{alg:main}, we note that the tuning parameters get used 
in mostly the following three quantities: $\gamma \Delta_{i-1}$ , $\alpha$ and $\beta$. 
We use this fact to reduce the number of input parameters. 
Let $c_0 := \gamma^2 c_1 32(\log(|\H|/\delta) + \log(i-1))$ (treating $\log(i-1)$ as a constant) and 
set $\regconst = 864, \gamma = \regconst / 4$ and $c_2 = \regconst c_1^2 / 4$ according to our theory.
Then we have
\begin{eqnarray*}
\gamma \Delta_{i-1} 
&=& \sqrt{\gamma^2 c_1 \epsilon_{i-1} \err(h_i,\ztil_{i-1})} + \gamma c_2 \epsilon_{i-1} \log(i-1)	\\
&=& \sqrt{\frac{c_0 \err(h_i,\ztil_{i-1})}{i-1}} + c_0 \frac{c_2}{\gamma c_1} \frac{\log(i-1)}{i-1}, 
\end{eqnarray*}
where
$\frac{c_2}{\gamma c_1}
= c_1 = \order(\alpha).$
Based on this, we use 
\begin{equation}
	\widehat{\Delta}_{i-1} := \sqrt{\frac{c_0 \err(h_i,\ztil_{i-1})}{i-1}} + \max(2\alpha,4)c_0 \frac{\log(i-1)}{i-1}
	\label{eq:cover_threshold}
\end{equation}
in Algorithm \ref{alg:cover} in place of $\gamma \Delta_{i-1}$.
Next we consider 
\begin{eqnarray*}
	\beta^2 &\leq& \frac{1}{216 n \epsilon_n \log n} \\
	      &\approx& \frac{\gamma^2 c_1}{216 c_0 \log n} \qquad \because n \epsilon_n \approx c_0 / (\gamma^2 c_1) \mbox{ by treating } \log n \mbox{ as a constant} \\
	      &=& \order\left(\frac{\alpha}{c_0}\right) \qquad \mbox{ by again treating } \log n \mbox{ as a constant and } c_1 = \order(\alpha).
\end{eqnarray*}
Based on the last expression, we set $\beta := \frac{\sqrt{\alpha / c_0}}{\beta_{scale}}$, where 
$\beta_{scale} > 0$ is a tuning parameter that controls the influence of the regret term in the variance constraints.
In sum, 
the actual input parameters boil down to the cover size $\coversize$, $\alpha \geq 1, c_0$ and $\beta_{scale}$, and we use them to set
\begin{eqnarray*}
	\gamma \Delta_{i-1} &:=& \sqrt{\frac{c_0 \err(h_i,\ztil_{i-1})}{i-1}} + \max(2 \alpha,4) c_0 \frac{\log(i-1)}{i-1}, \quad
	\beta = \frac{\sqrt{\alpha / c_0}}{\beta_{scale}}.
\end{eqnarray*}
Finally, we use the following setting for the minimum query probability:
\begin{equation*}
	P_{\min,i} = \min\left( \frac{1}{\sqrt{(i-1)\err(h_i, \ztil_{i-1})} + \log (i-1)} , \frac{1}{2}\right).
\end{equation*}

Next we describe hyper-parameter settings for different algorithms. 
A common hyper-parameter is the learning rate of the underlying online oracle, which is a reduction 
to importance-weighted logistic regression. For all active learning algorithm, we try the following 11 learning 
rates: $10^{-1} \cdot \{2^{-2},2^{-1},\ldots,2^8\}$. Active learning hyper-parameter settings are given in the following table:
\begin{center}
{
	\renewcommand{\arraystretch}{1.5}	
\begin{tabular}{r||c|c}
	algorithm & parameter settings & total number of settings \\ \hline \hline
	\cover& \multirow{2}{0.5\textwidth}{$(c_0, \coversize, \beta_{scale}, \alpha) \in \big\{0.1 \cdot \{2^{-10}, \ldots, \cdot 2^{-1}\}, 0.1, 0.3, \ldots, 0.9, 2^0, \ldots, 2^4\big\} \times \{3,6,12,24,48\} \times \{\sqrt{10}\} \times  \{1\} $}  & 100 \\ 
              & & \\ \hline
	\iwalzero& $C_0 \in \big\{0.1 \cdot \{2^{-17}, 2^{-16}, \ldots, 2^0\}, 2^0, 2^1, \ldots, 2^4 \big\}$ & 23 \\ \hline
	\iwalorazero & $C_0 \in \{2^{-17},\ldots,2^{5}\}$ & 23 \\ \hline
	\iwal& $C_0$ the same as \iwalzero& 23 \\ \hline
	\iwalora&  $C_0$ the same as \iwalorazero& 23
\end{tabular}
}
\end{center}
Good hyper-parameters of the algorithms often lie in the interior of these value ranges.
\subsection{More Experimental Results}
\label{appendix:results}
We provide detailed per-dataset results in Figures \ref{fig:err_vs_query_medium} to \ref{fig:err_vs_query_large-per-data}.
Figures \ref{fig:err_vs_query_medium} and \ref{fig:err_vs_query_large} show test error rates 
obtained by each algorithm using the best fixed hyper-parameter setting against number of label queries 
for small (fewer than $10^5$ examples) and large (more than $10^5$ examples) datasets. 
Figures \ref{fig:err_vs_query_medium-per-data} and \ref{fig:err_vs_query_large-per-data} 
show results obtained by each algorithm using the best hyper-parameter setting for each dataset.
%
%
%

\begin{figure}[t]
\begin{tabular}{@{}c@{}c@{}c@{}}
\includegraphics[width=0.33\textwidth,height=0.25\textwidth]{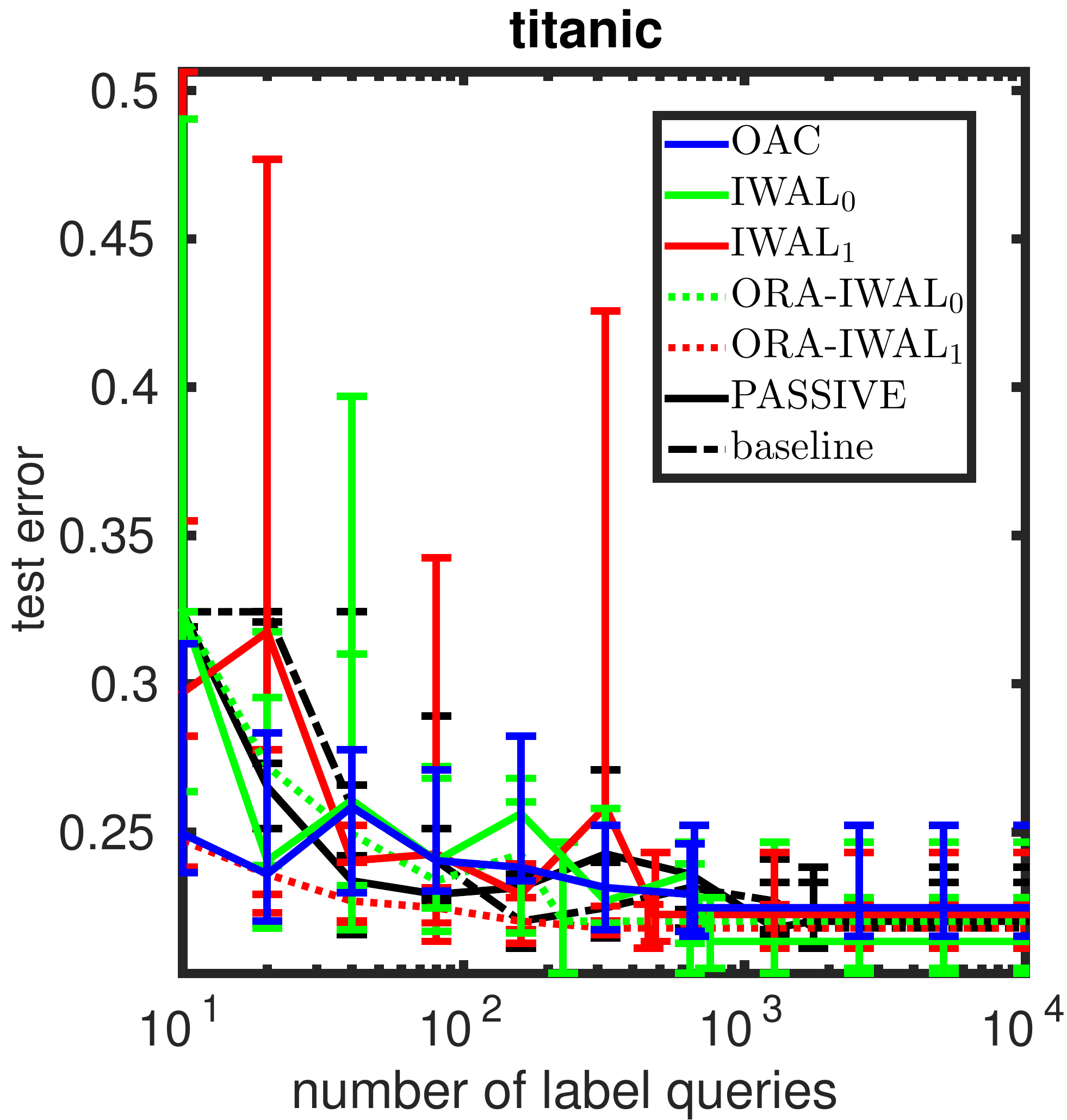}  & \includegraphics[width=0.33\textwidth,height=0.25\textwidth]{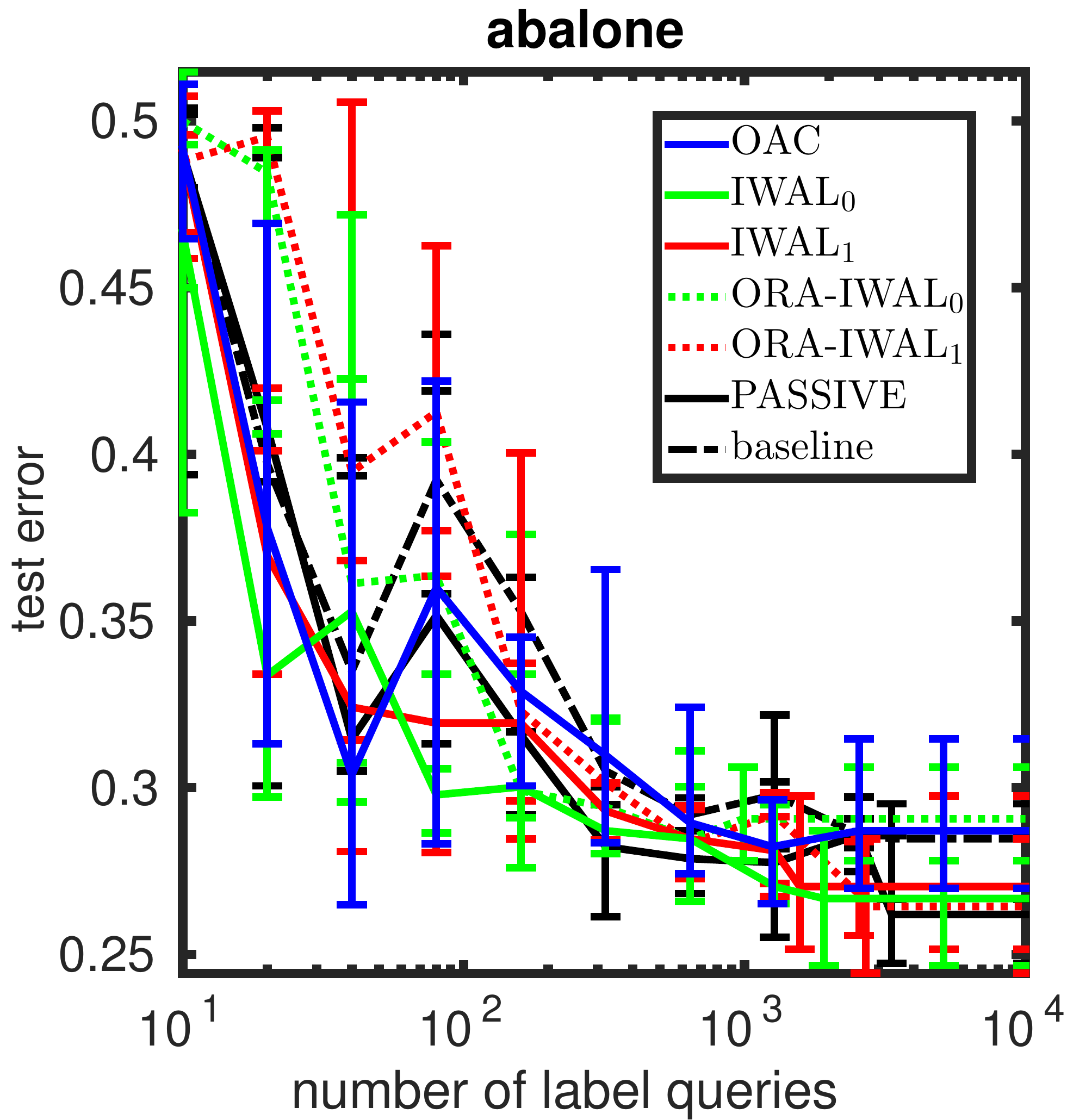}  & \includegraphics[width=0.33\textwidth,height=0.25\textwidth]{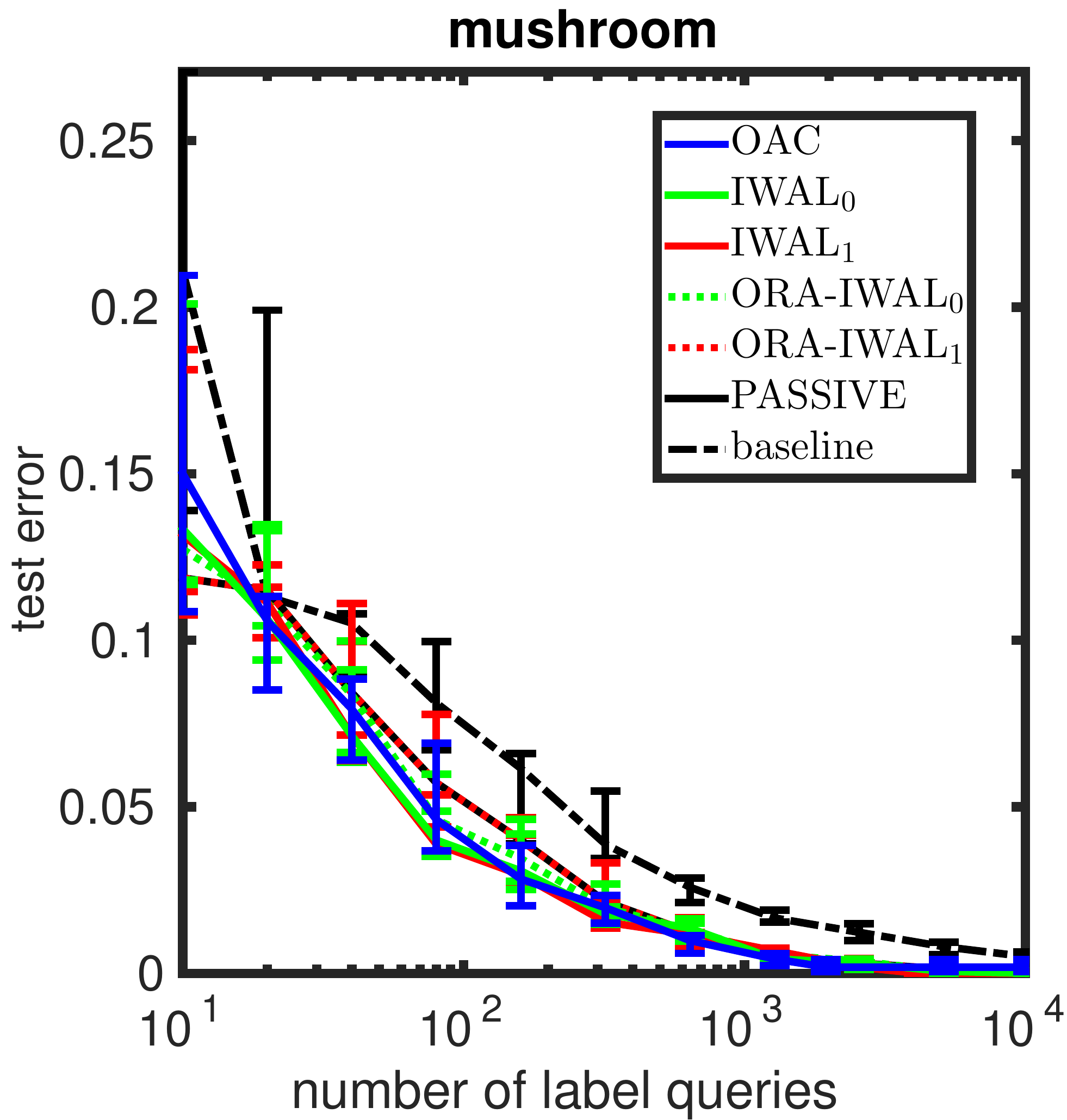}   \\
\includegraphics[width=0.33\textwidth,height=0.25\textwidth]{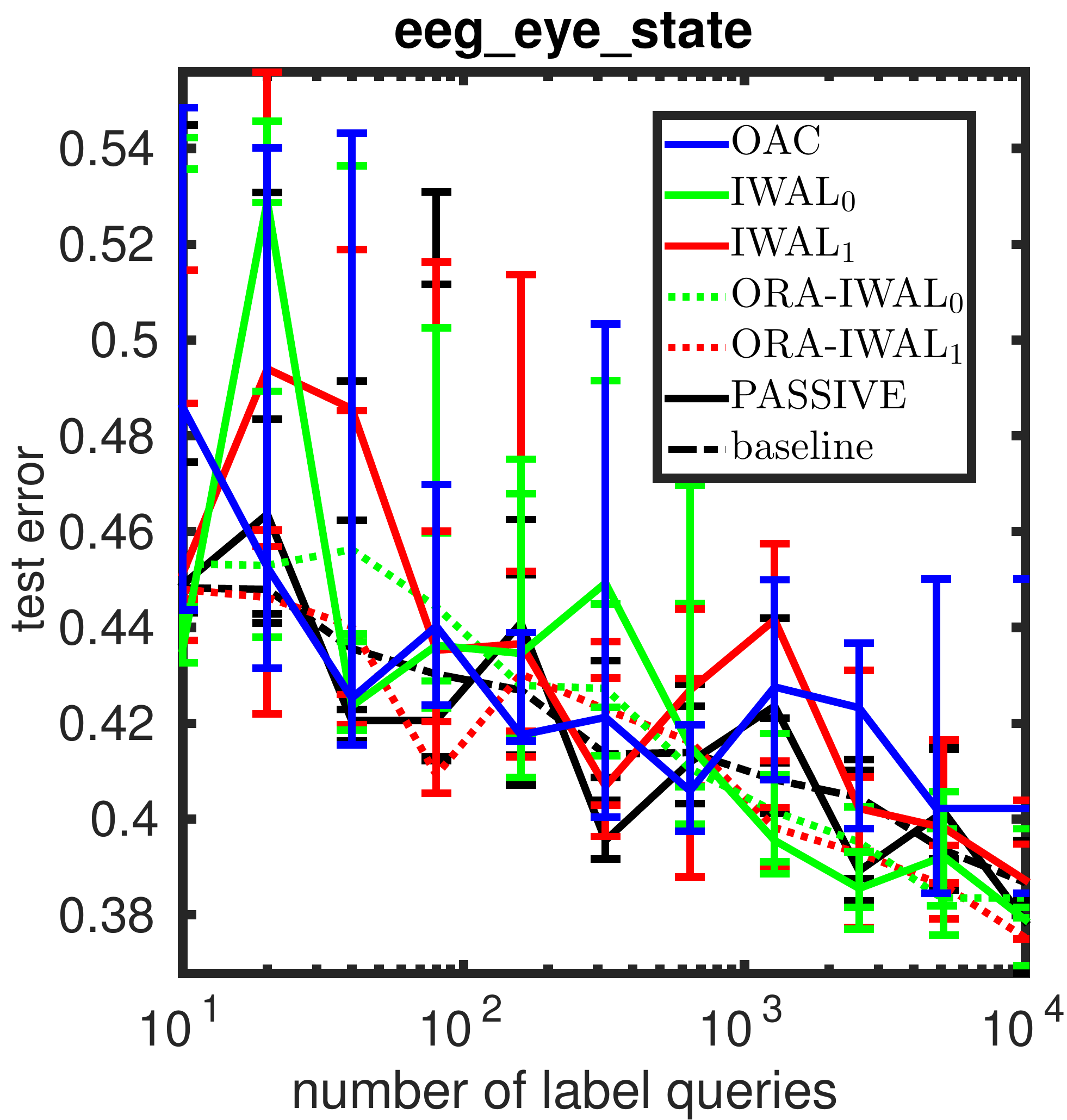}  & \includegraphics[width=0.33\textwidth,height=0.25\textwidth]{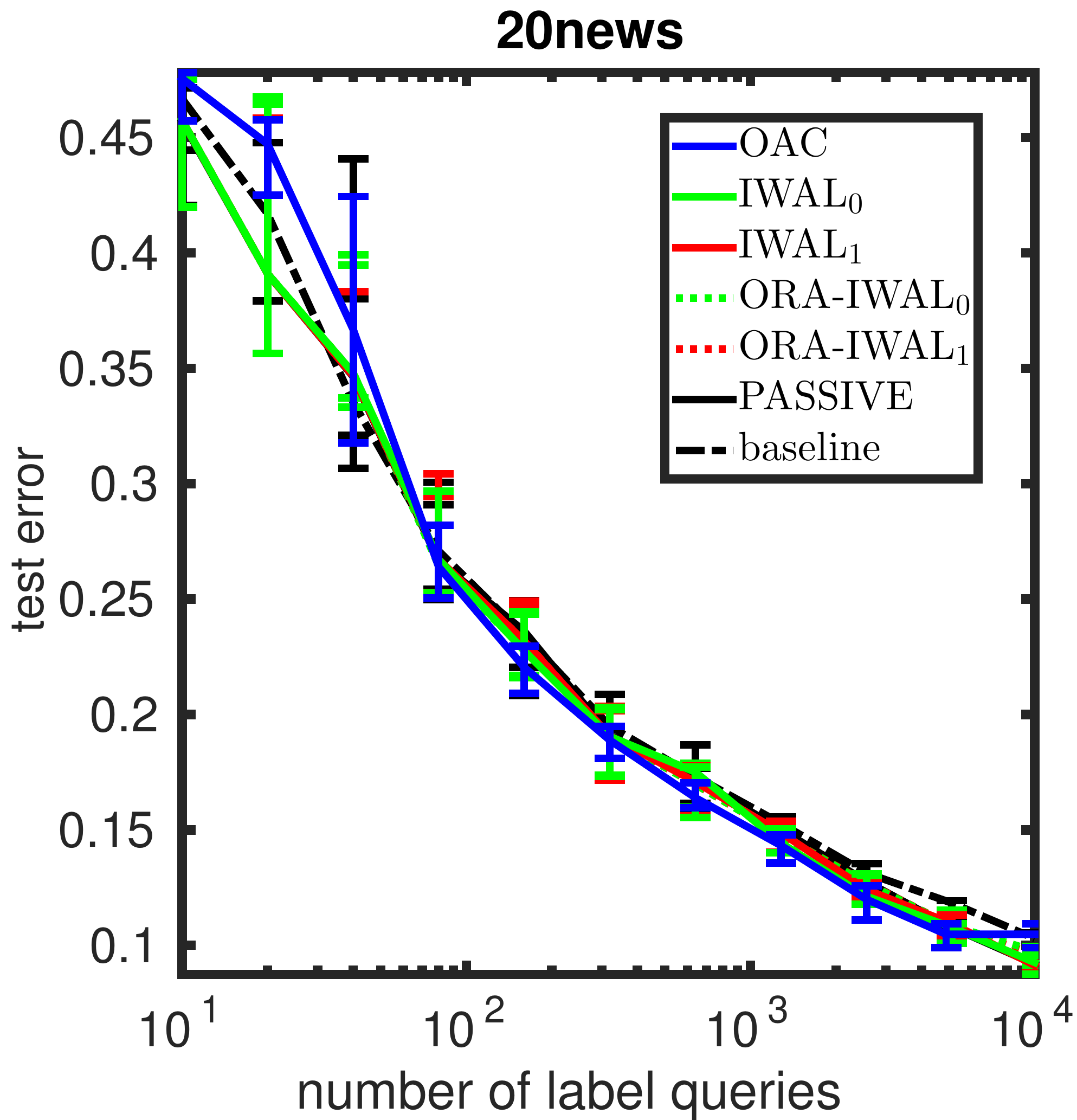}  & \includegraphics[width=0.33\textwidth,height=0.25\textwidth]{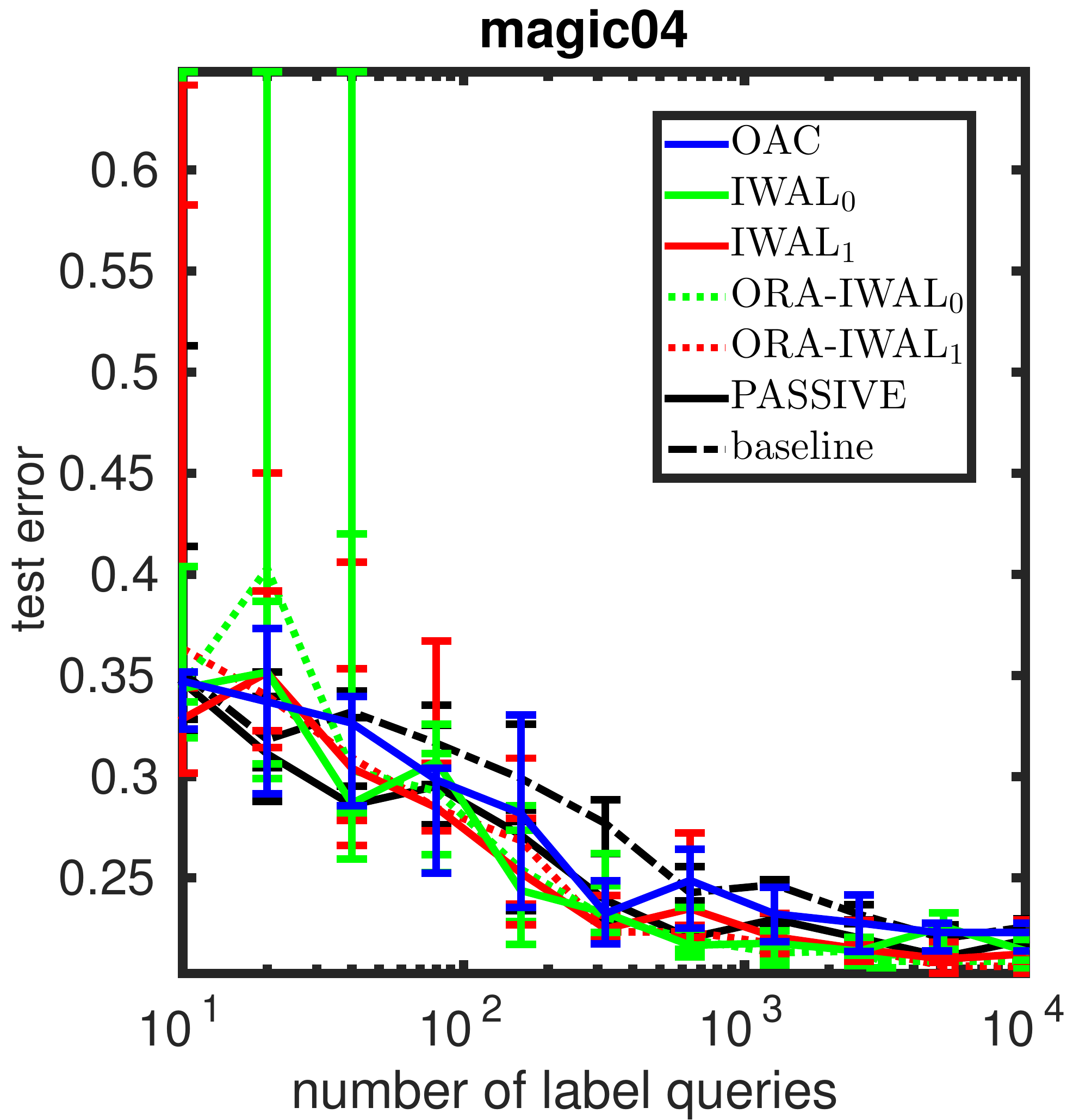}   \\
\includegraphics[width=0.33\textwidth,height=0.25\textwidth]{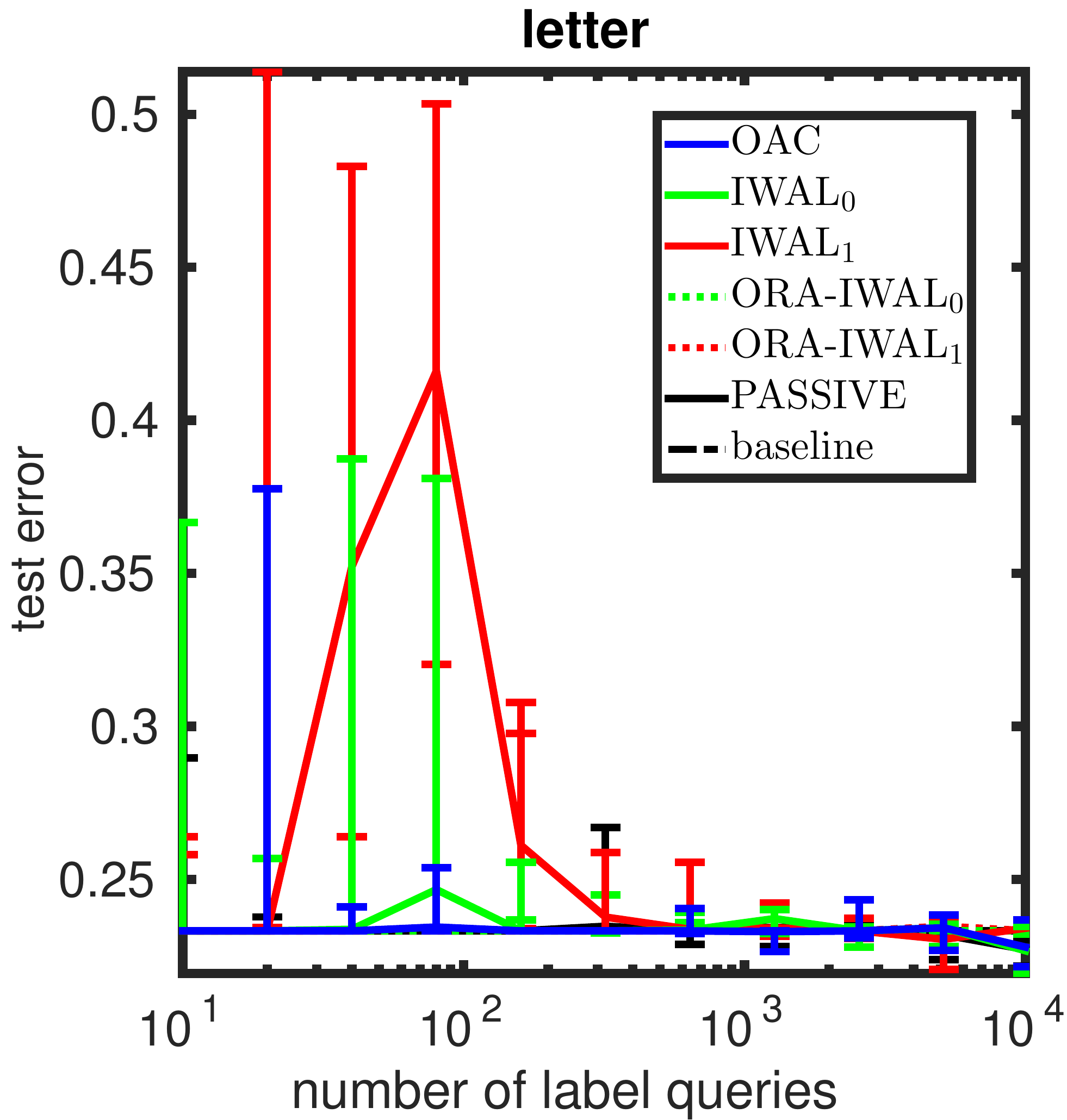}  & \includegraphics[width=0.33\textwidth,height=0.25\textwidth]{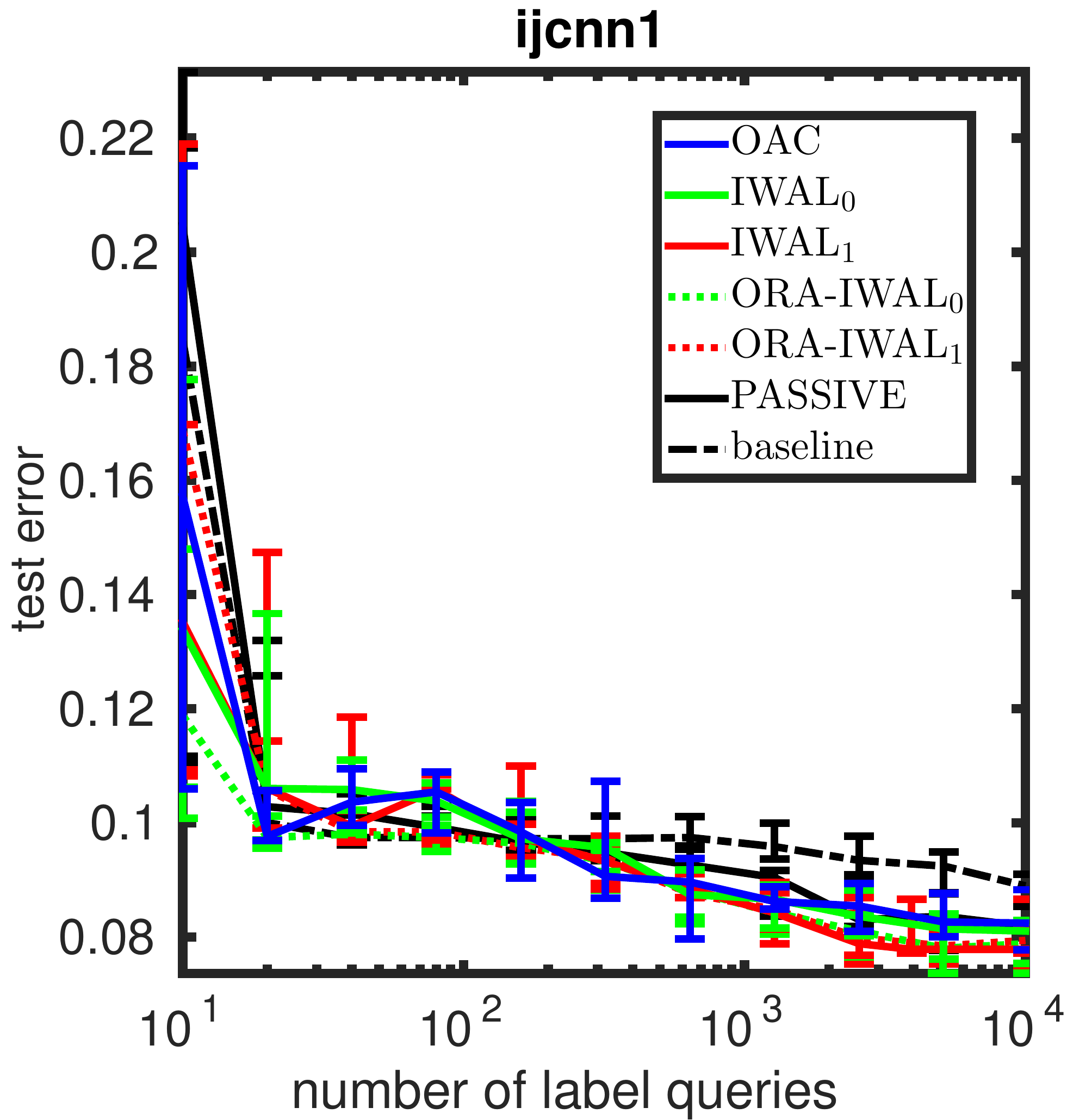}  & \includegraphics[width=0.33\textwidth,height=0.25\textwidth]{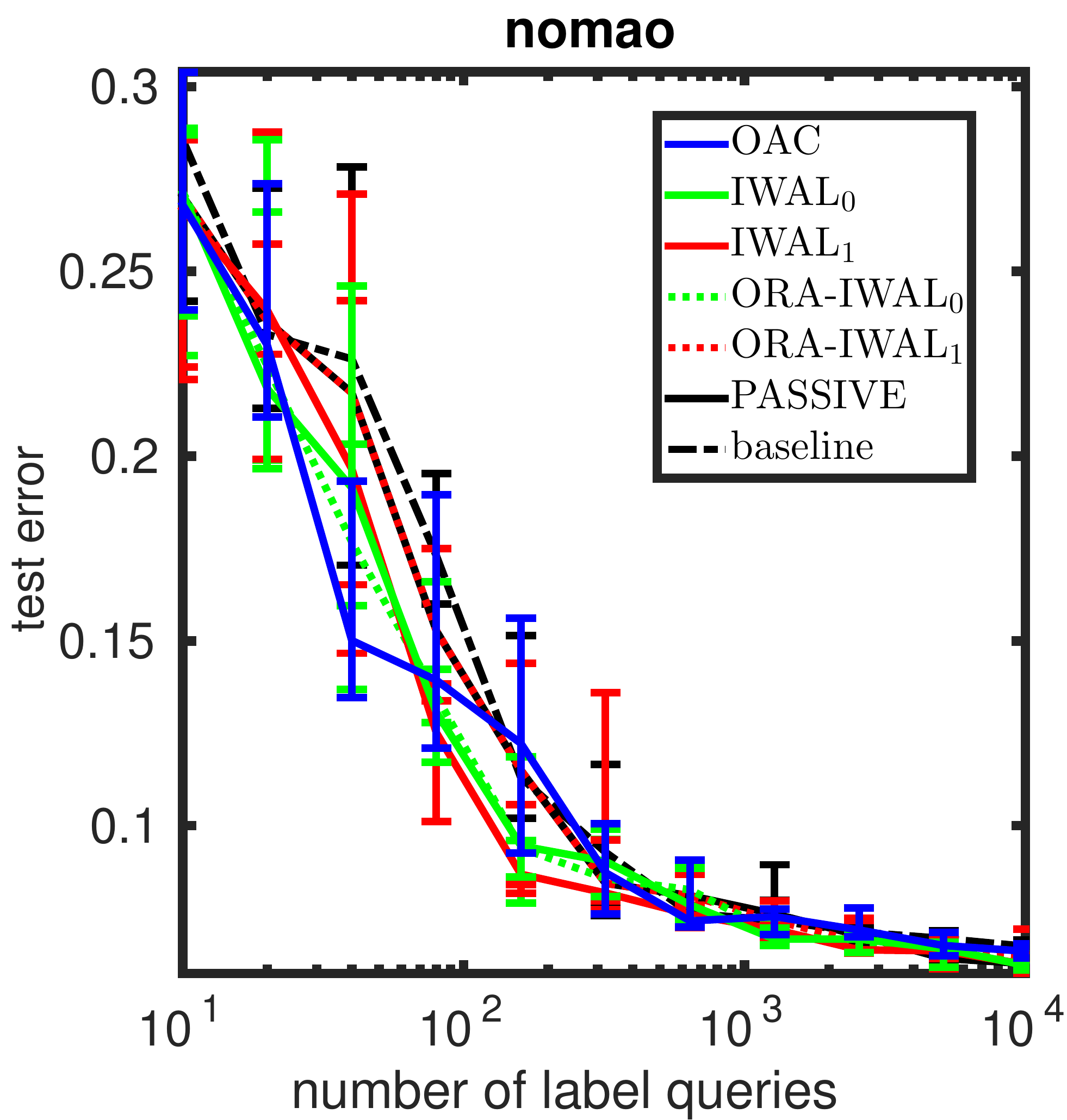}   \\
\includegraphics[width=0.33\textwidth,height=0.25\textwidth]{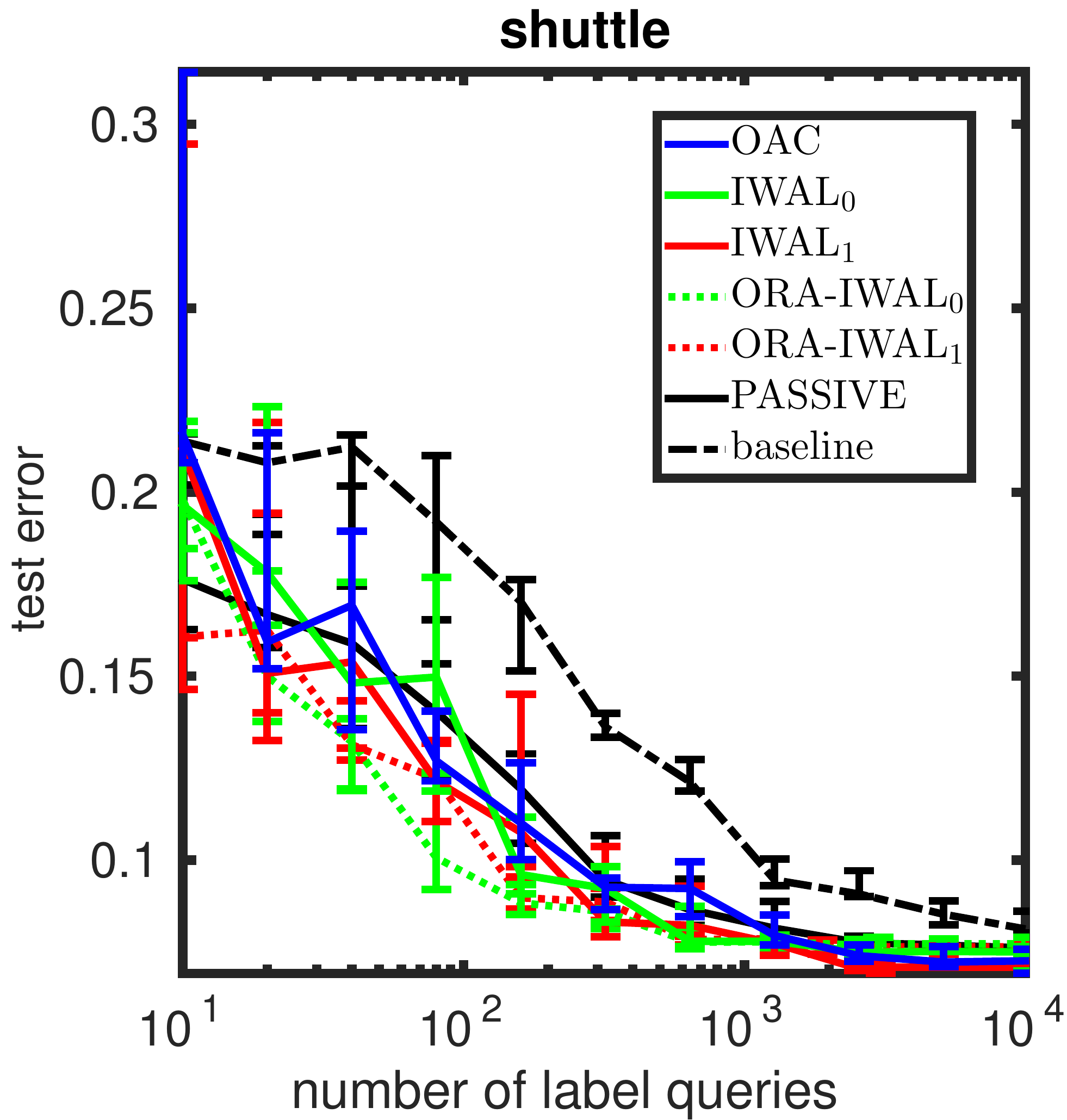}  & \includegraphics[width=0.33\textwidth,height=0.25\textwidth]{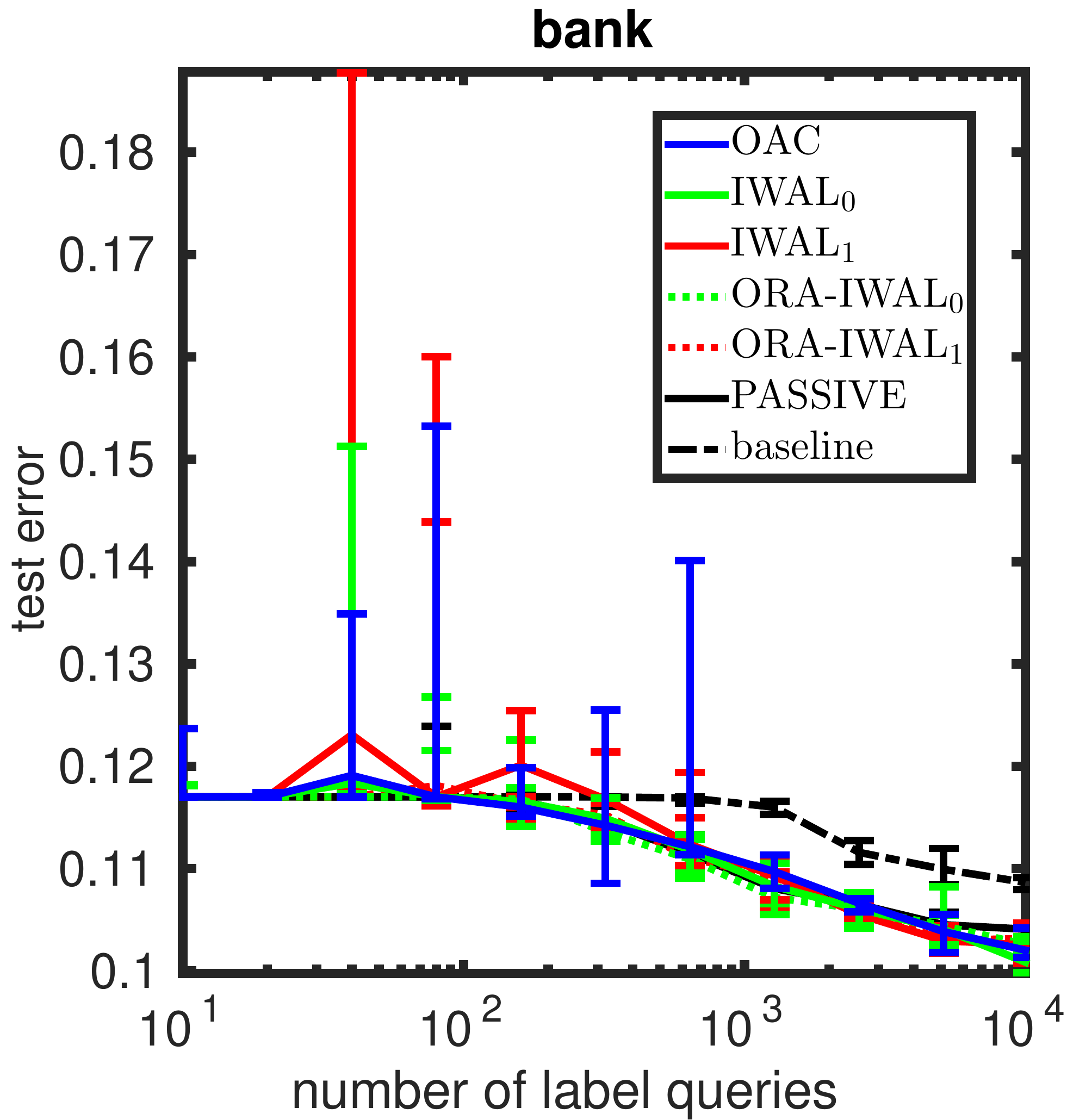}  & \includegraphics[width=0.33\textwidth,height=0.25\textwidth]{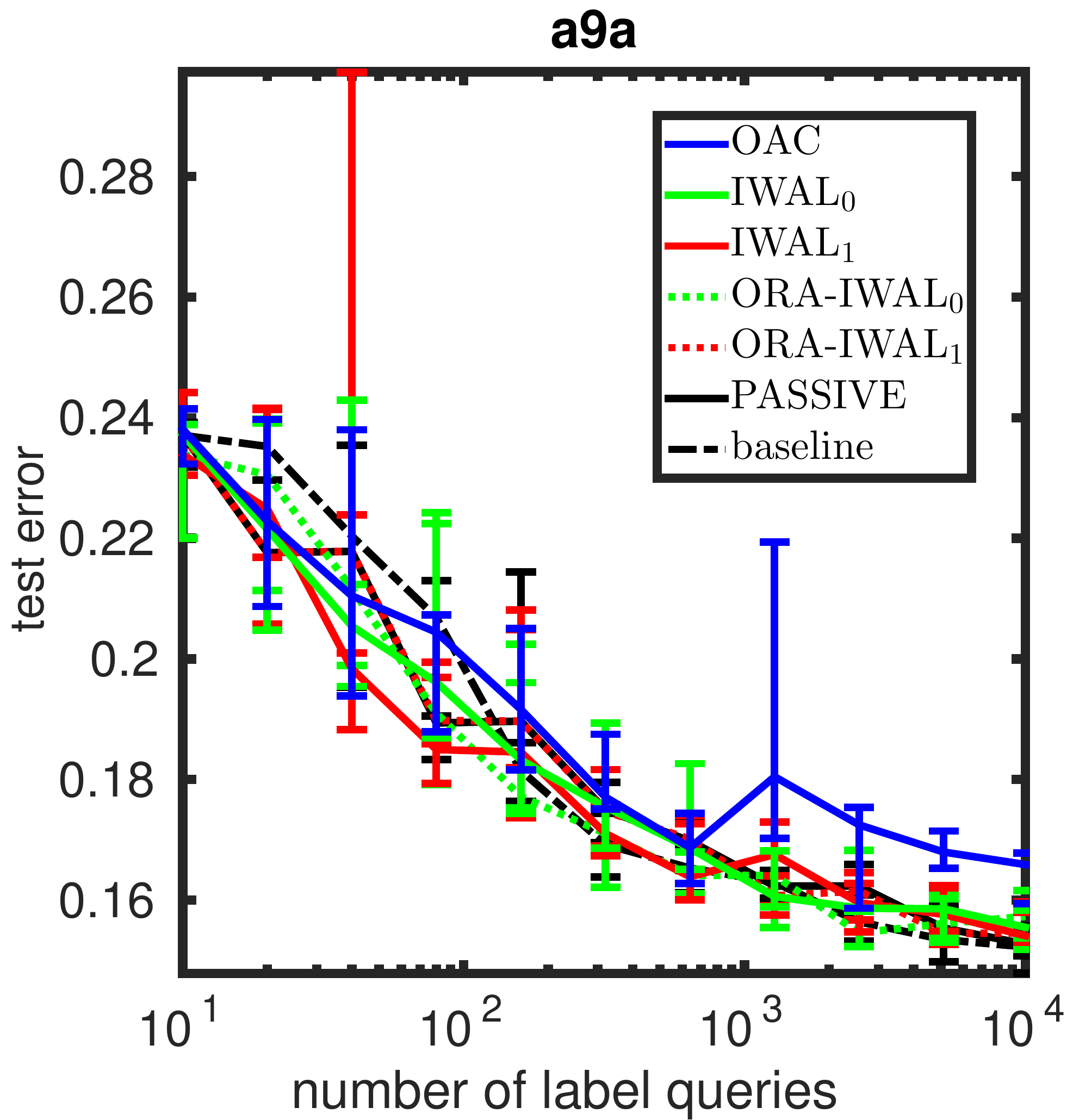} \\ 
\includegraphics[width=0.33\textwidth,height=0.25\textwidth]{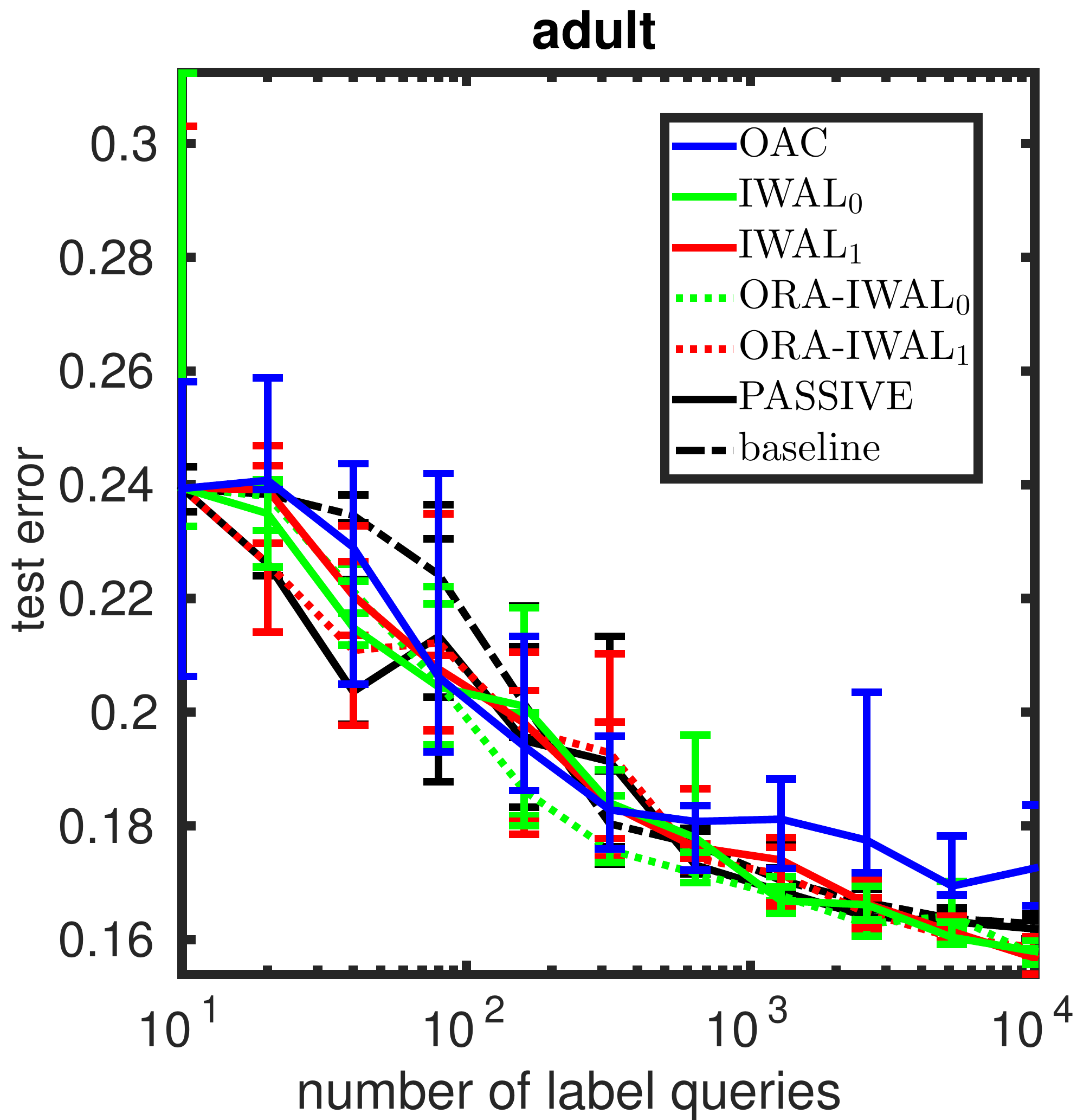}  & \includegraphics[width=0.33\textwidth,height=0.25\textwidth]{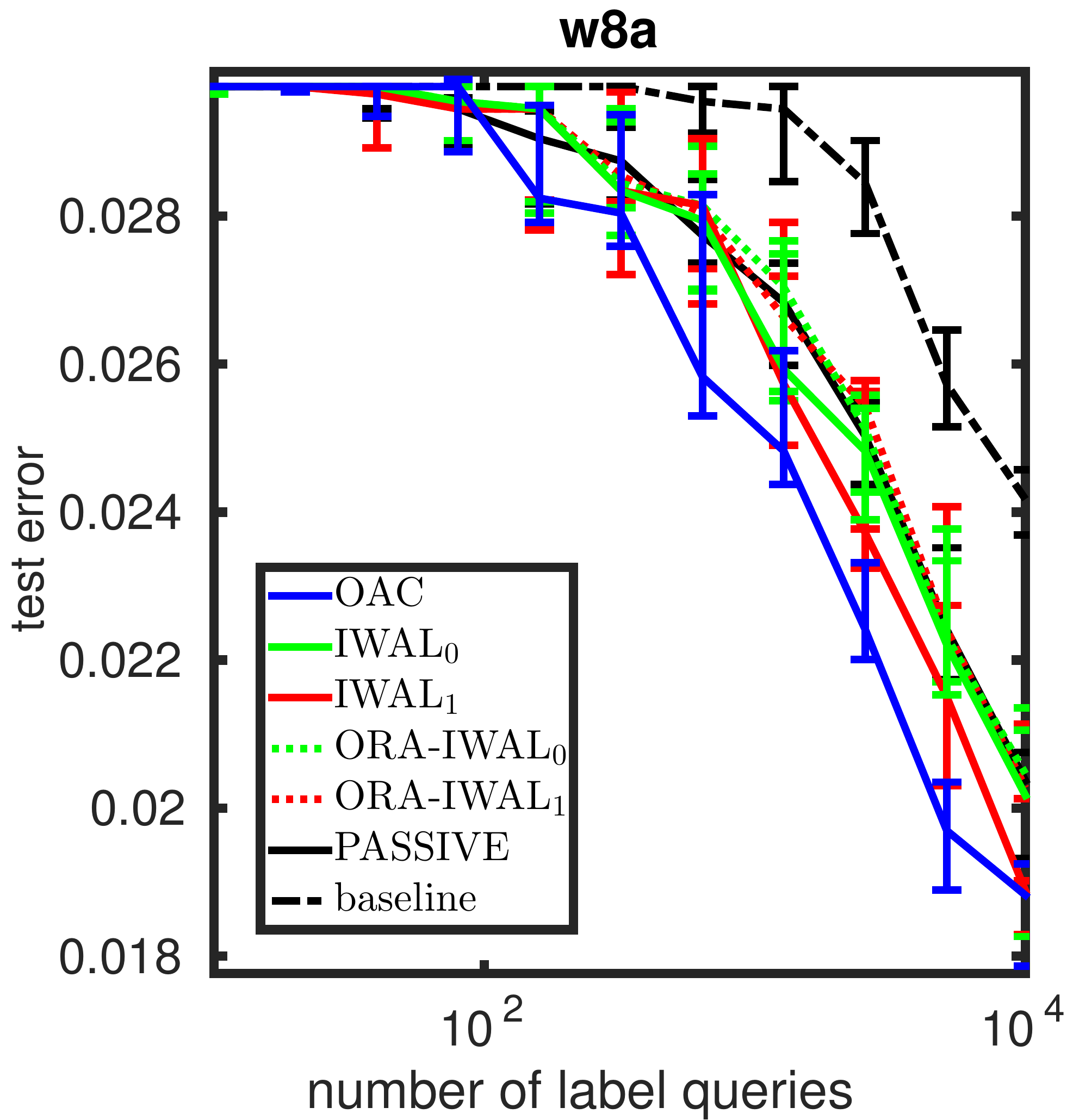}  &  
\end{tabular}
\caption{Test error under the best fixed hyper-parameter setting vs. number of label queries for datasets with fewer than $10^5$ examples}
\label{fig:err_vs_query_medium}
\end{figure}

\begin{figure}[t]
\begin{tabular}{@{}c@{}c@{}c@{}}
\includegraphics[width=0.33\textwidth,height=0.25\textwidth]{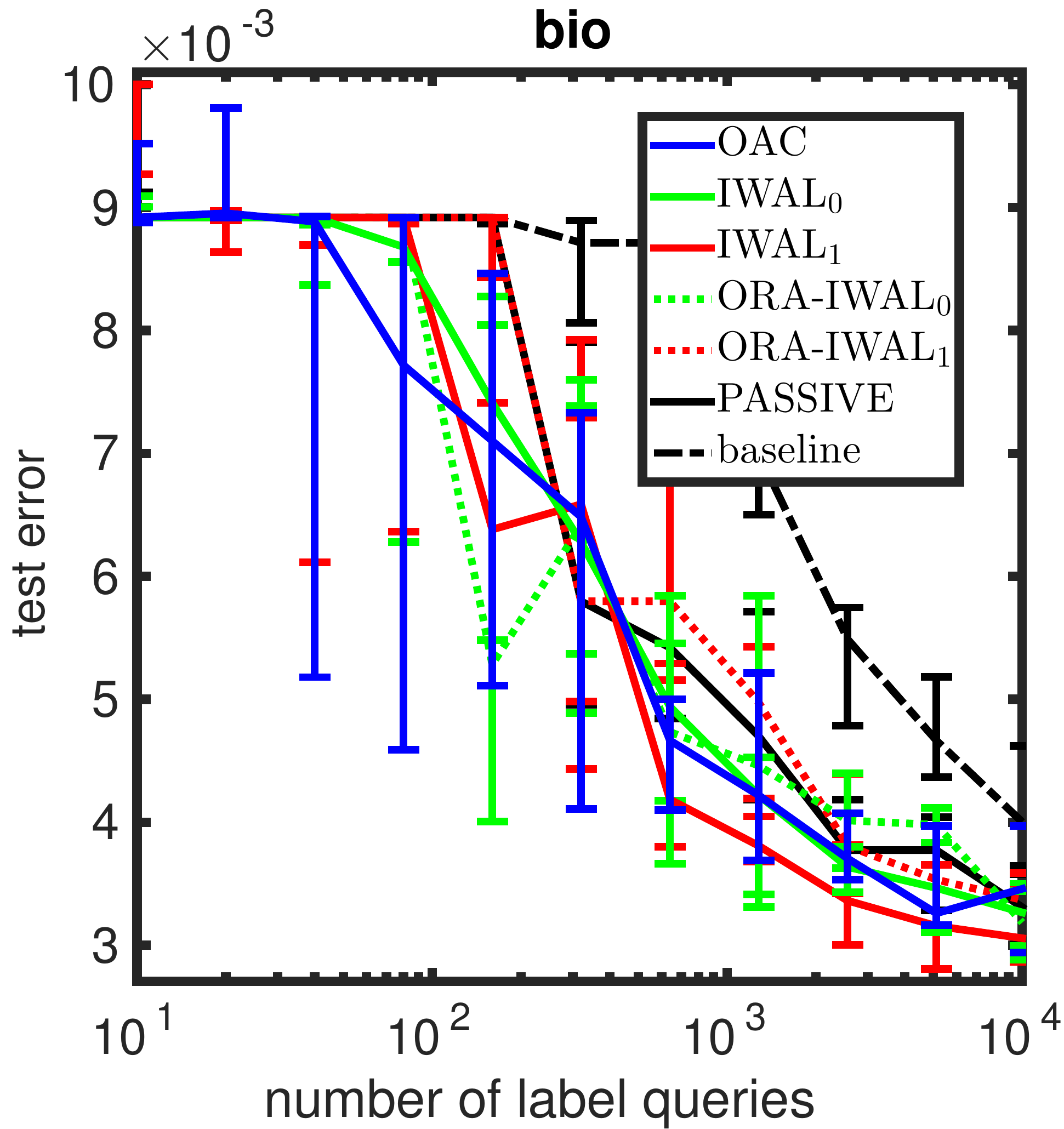} &  \includegraphics[width=0.33\textwidth,height=0.25\textwidth]{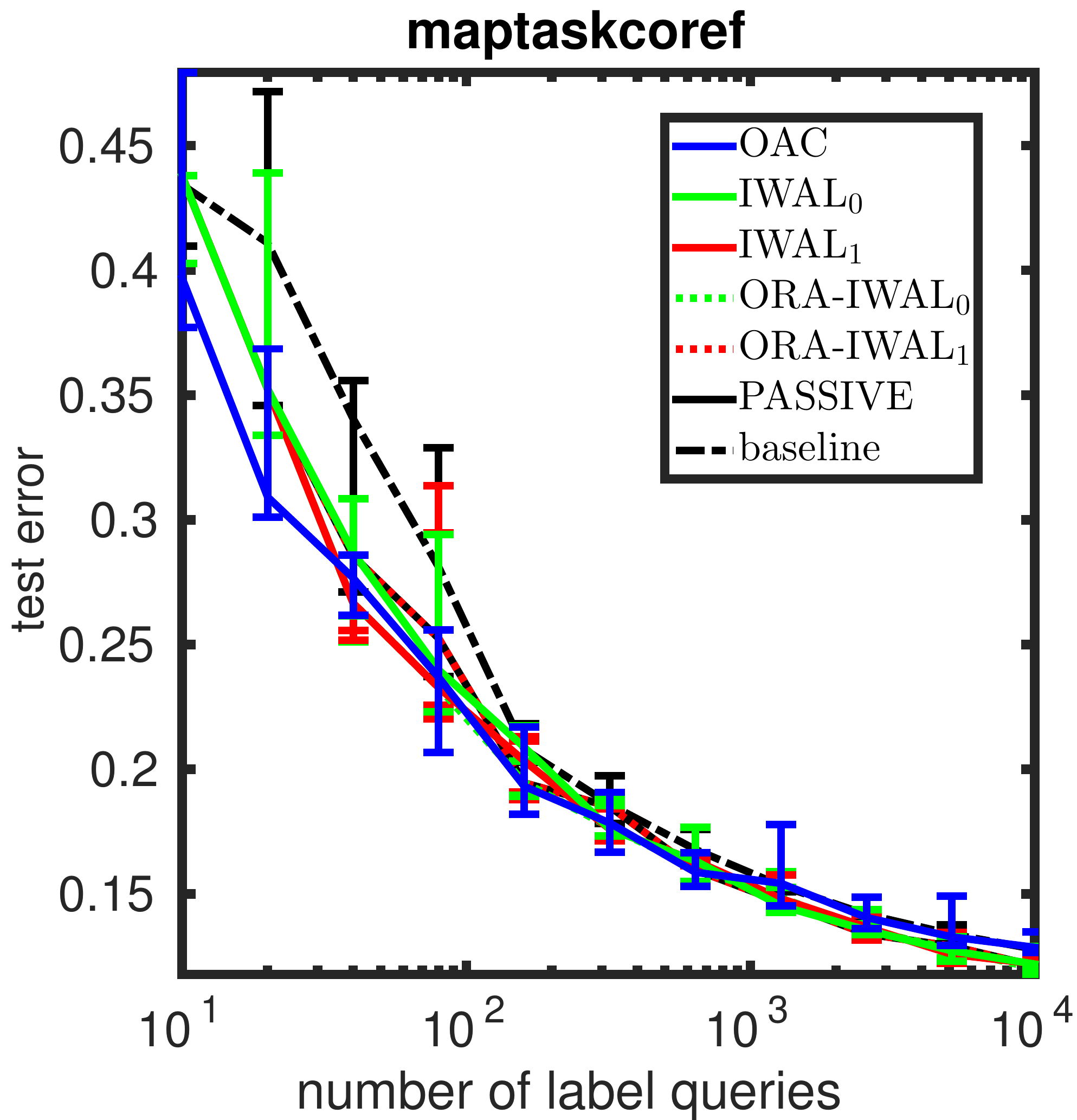}  &  \includegraphics[width=0.33\textwidth,height=0.25\textwidth]{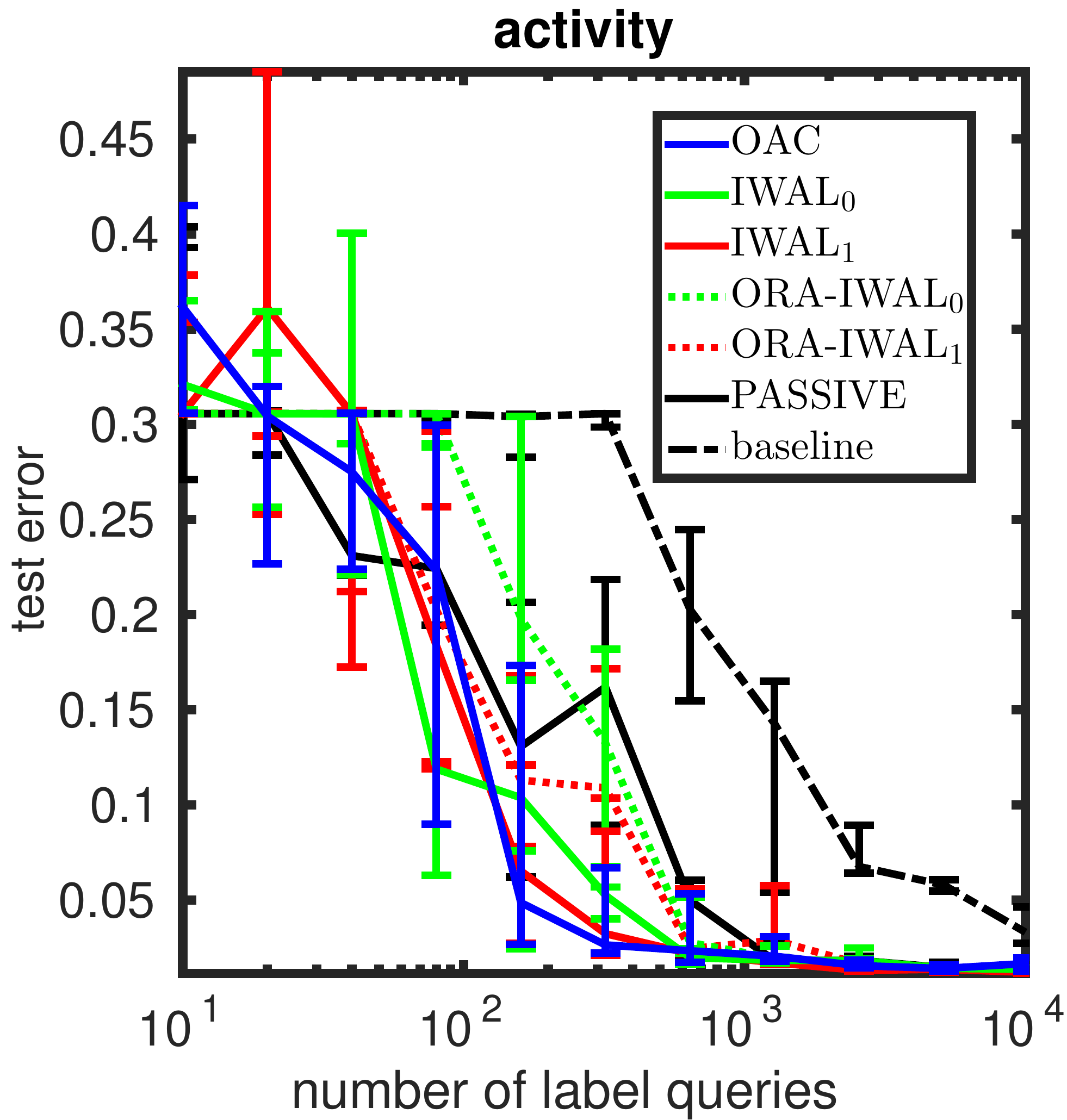} \\
\includegraphics[width=0.33\textwidth,height=0.25\textwidth]{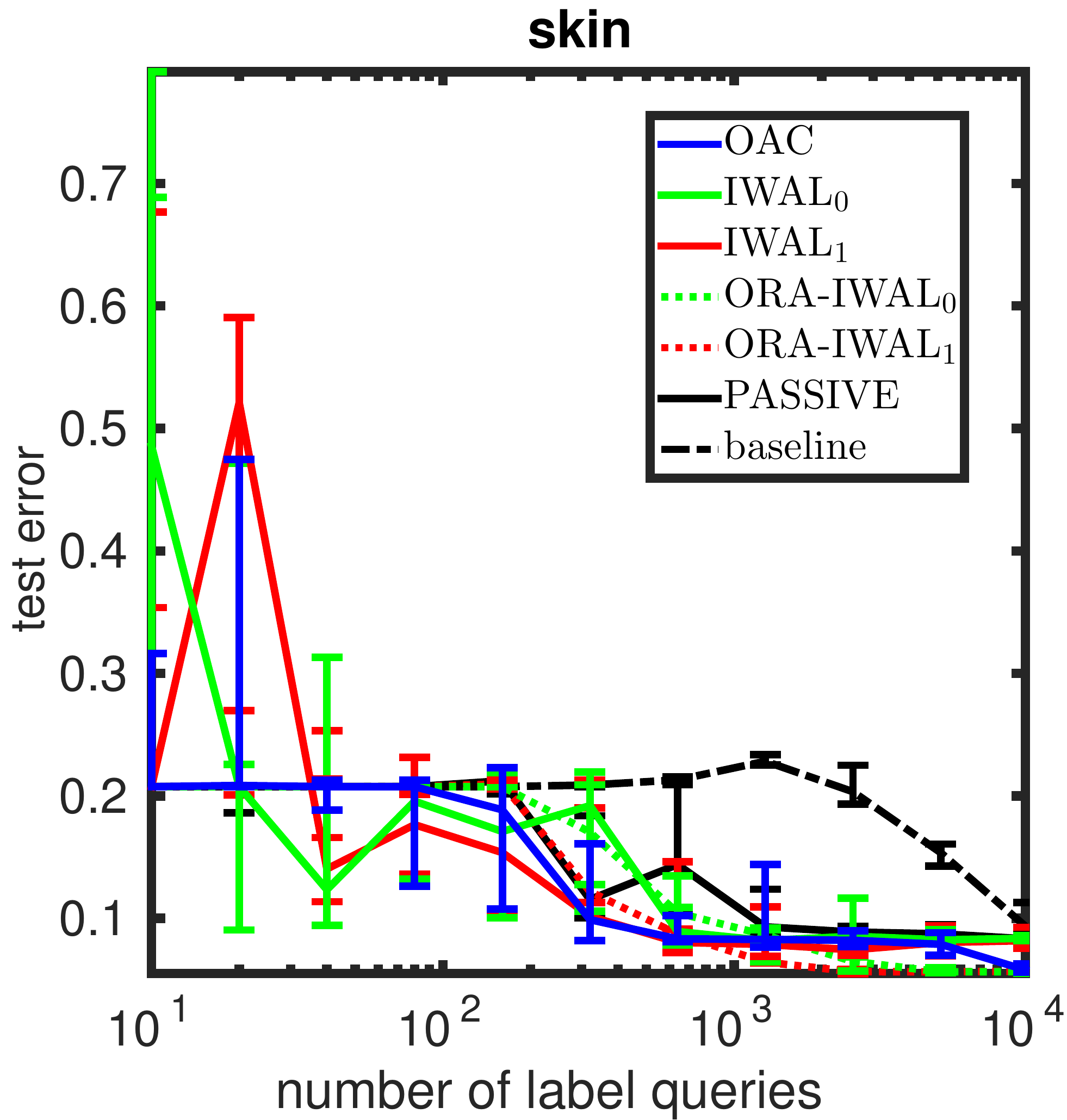} & \includegraphics[width=0.33\textwidth,height=0.25\textwidth]{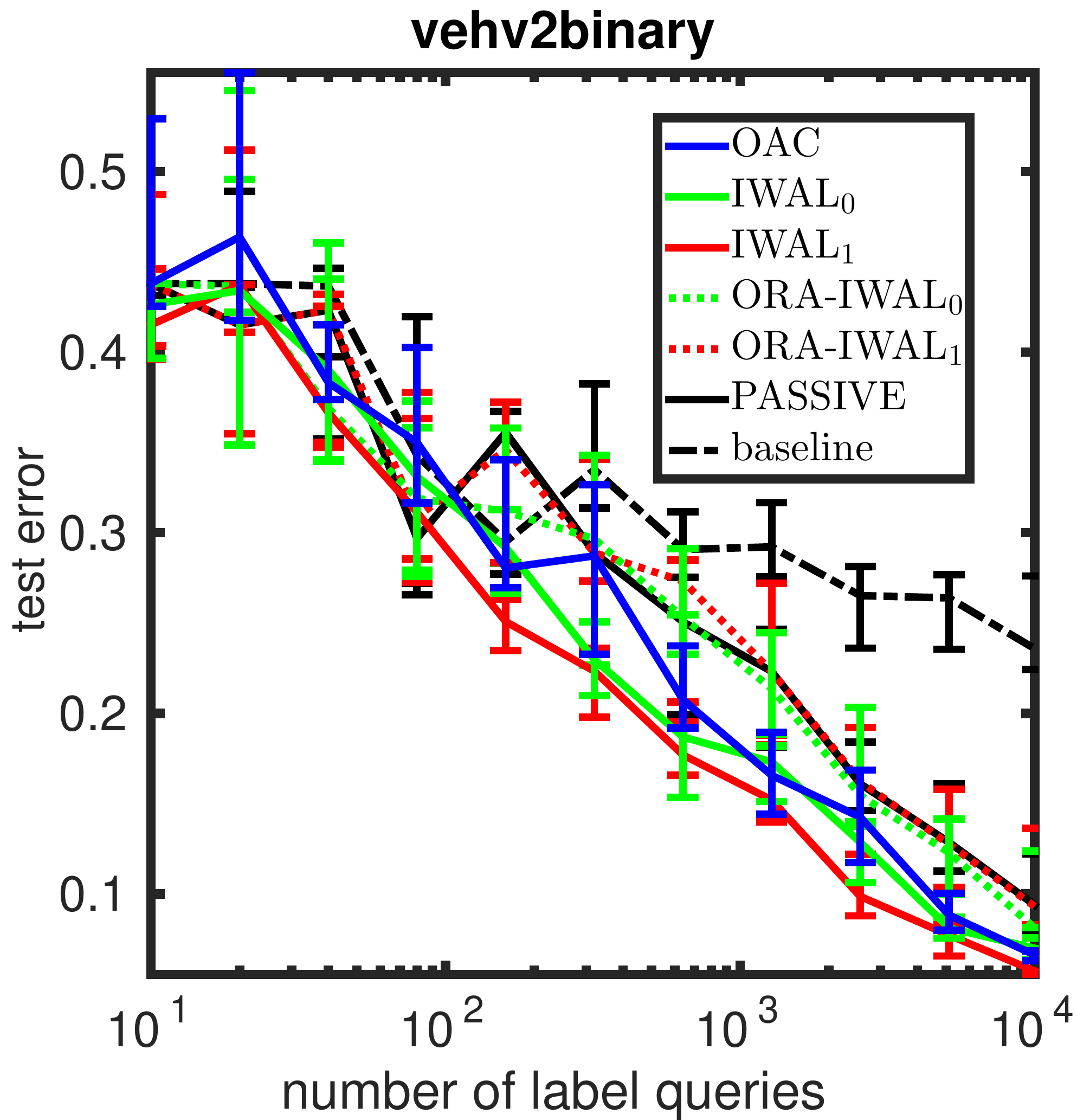}   &  \includegraphics[width=0.33\textwidth,height=0.25\textwidth]{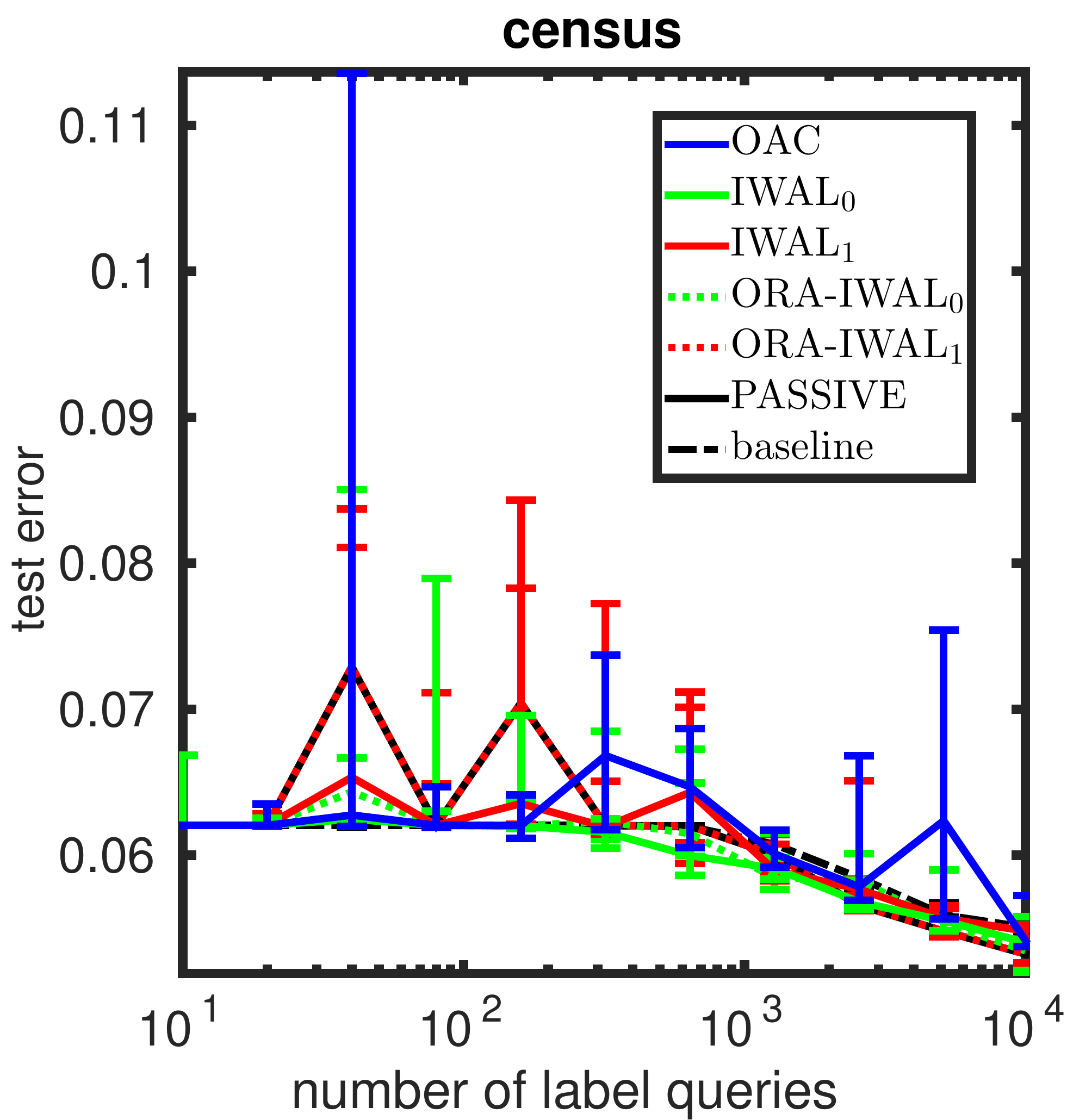}   \\ 
\includegraphics[width=0.33\textwidth,height=0.25\textwidth]{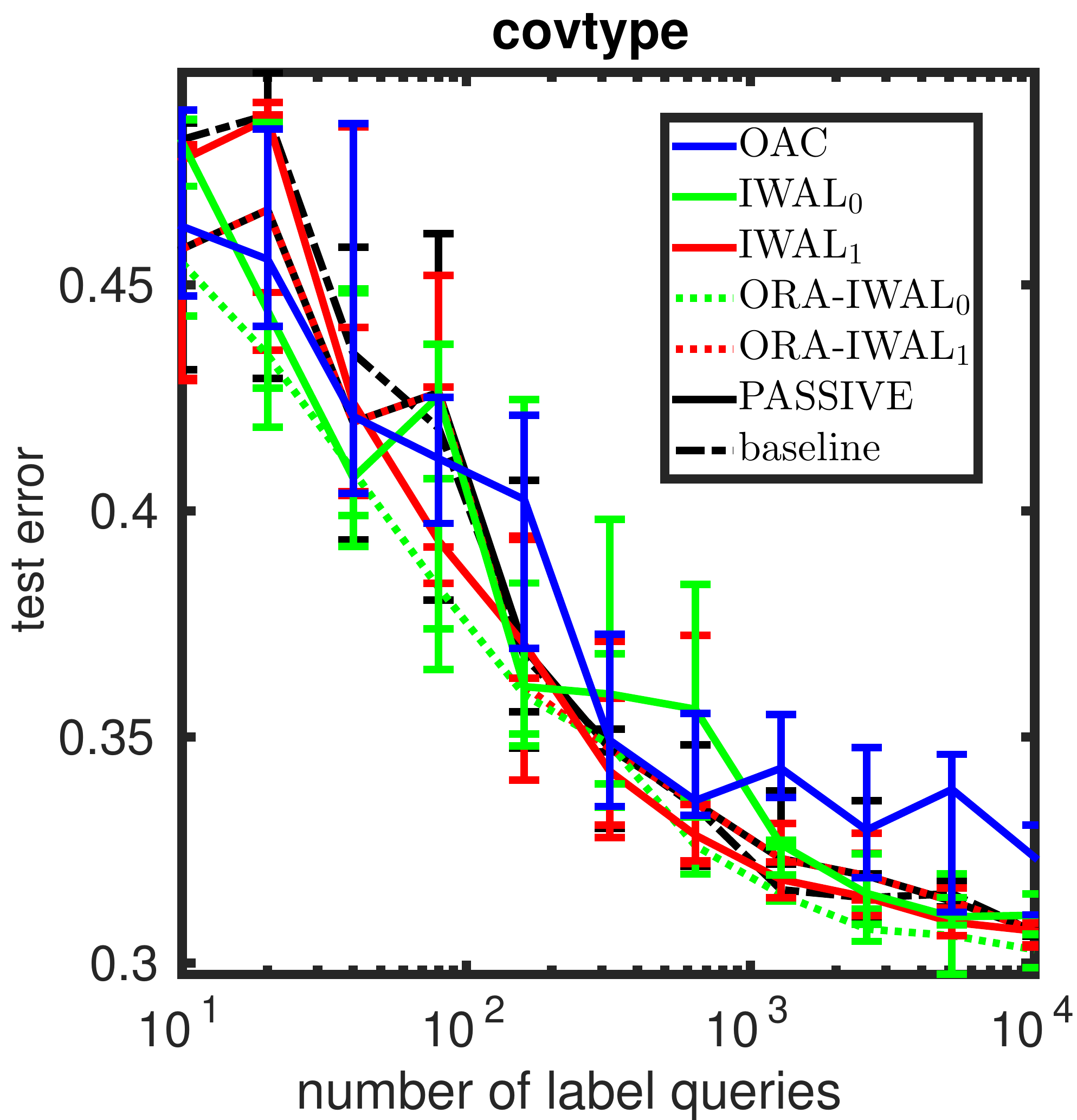}  &  \includegraphics[width=0.33\textwidth,height=0.25\textwidth]{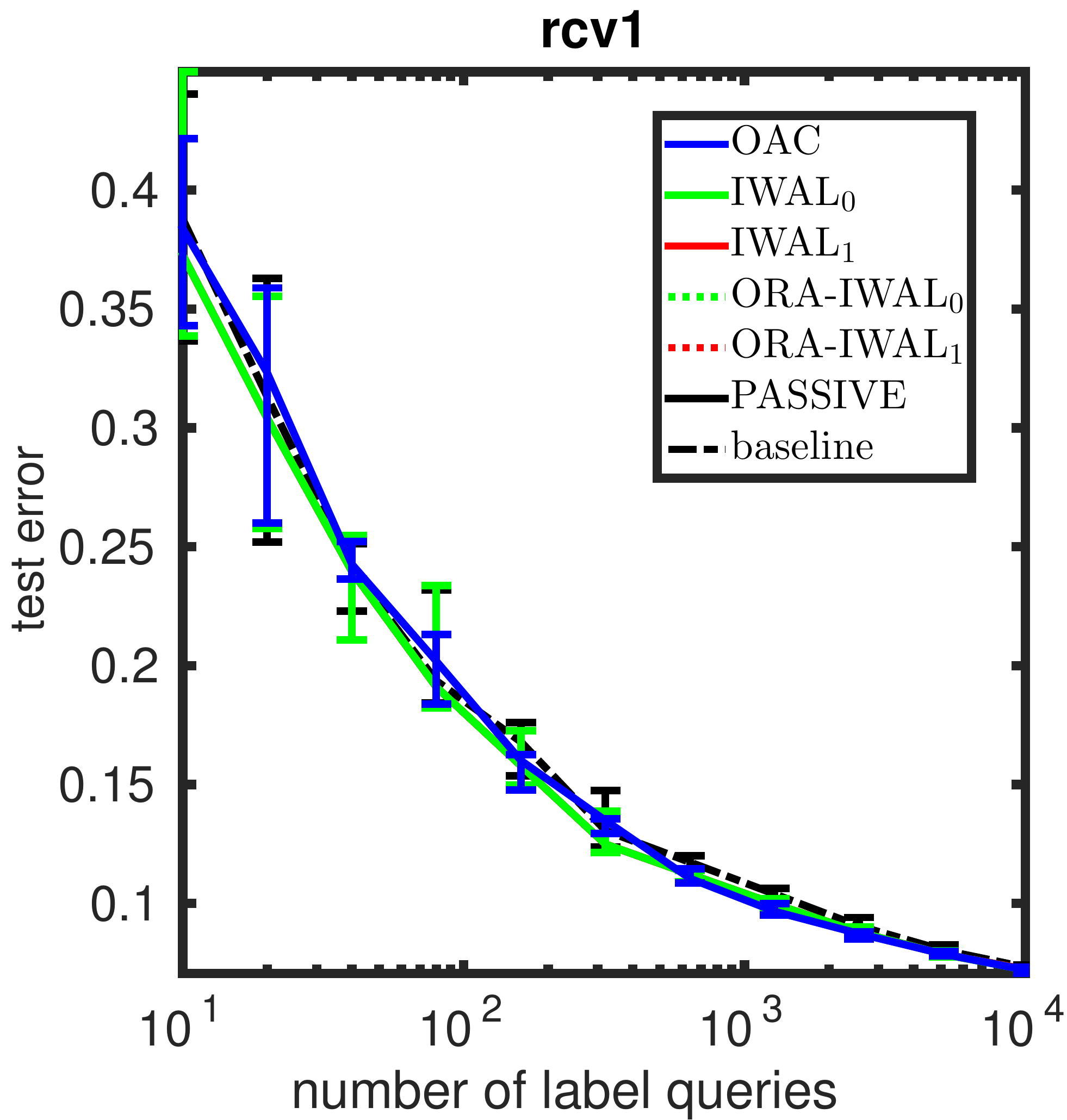} 
\end{tabular}
\caption{Test error under the best fixed hyper-parameter setting vs. number of label queries for datasets with more than $10^5$ examples}
\label{fig:err_vs_query_large}
\end{figure}

\begin{figure}[t]
\begin{tabular}{@{}c@{}c@{}c@{}}
\includegraphics[width=0.33\textwidth,height=0.25\textwidth]{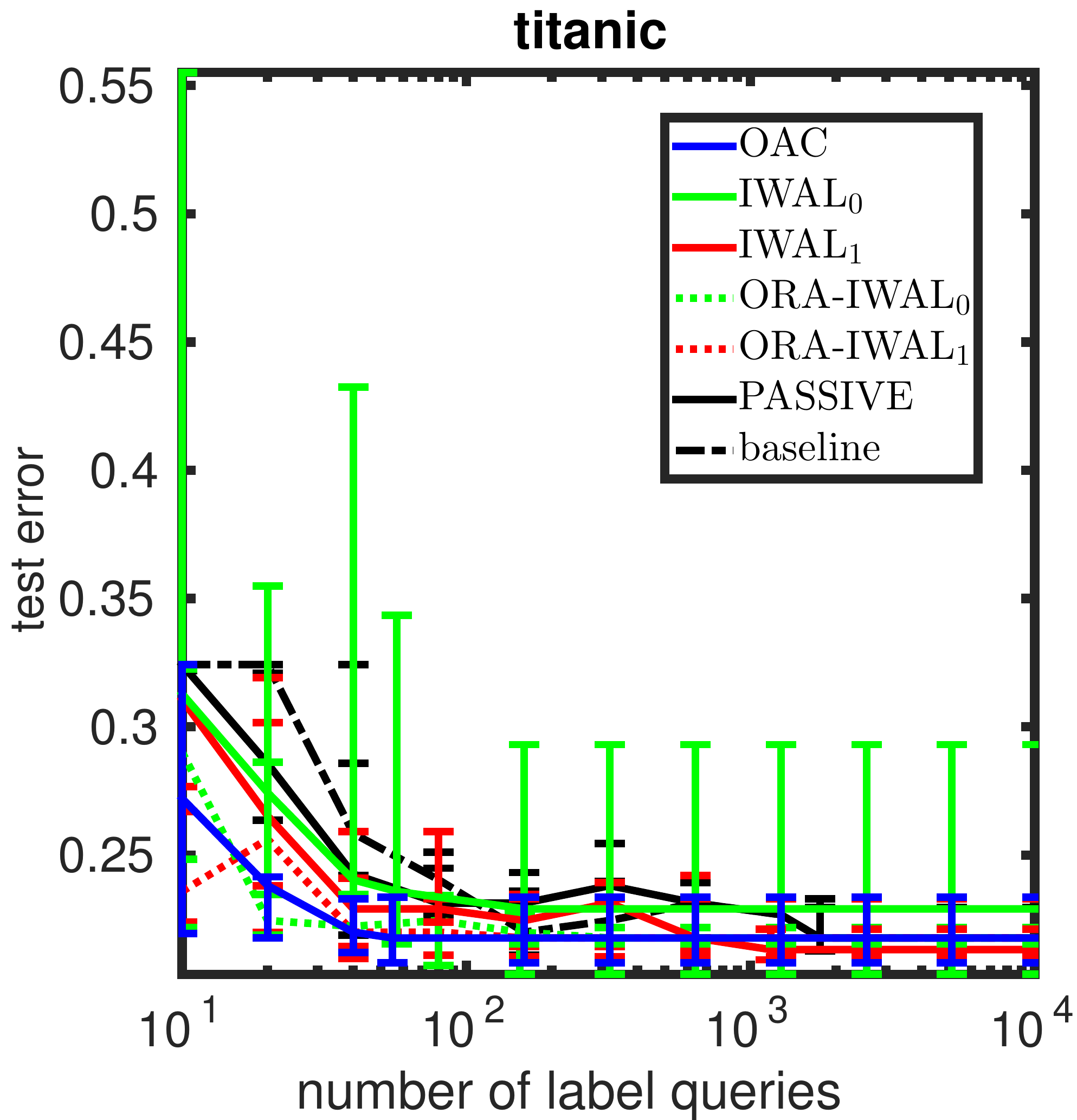}  & \includegraphics[width=0.33\textwidth,height=0.25\textwidth]{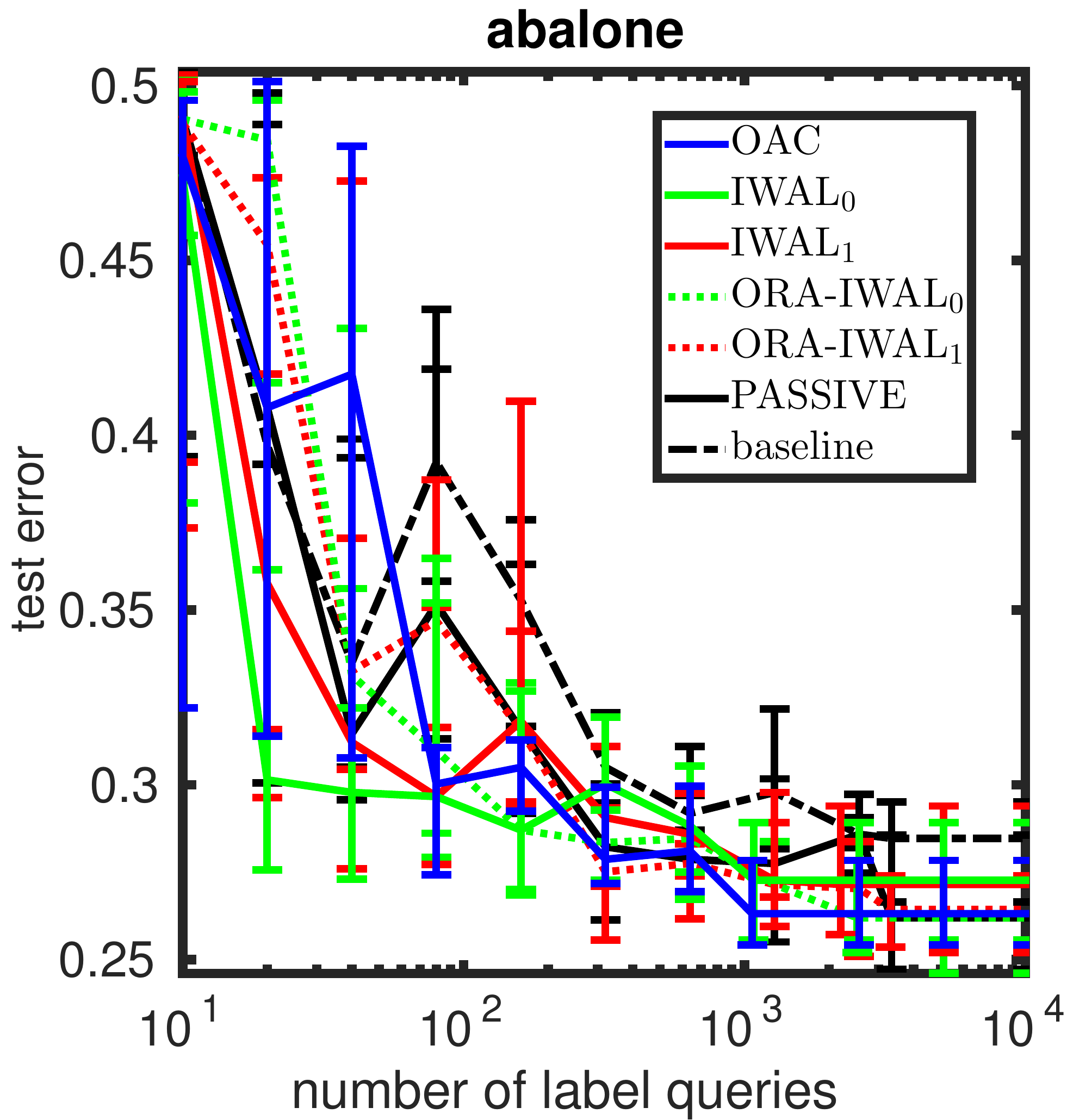}  & \includegraphics[width=0.33\textwidth,height=0.25\textwidth]{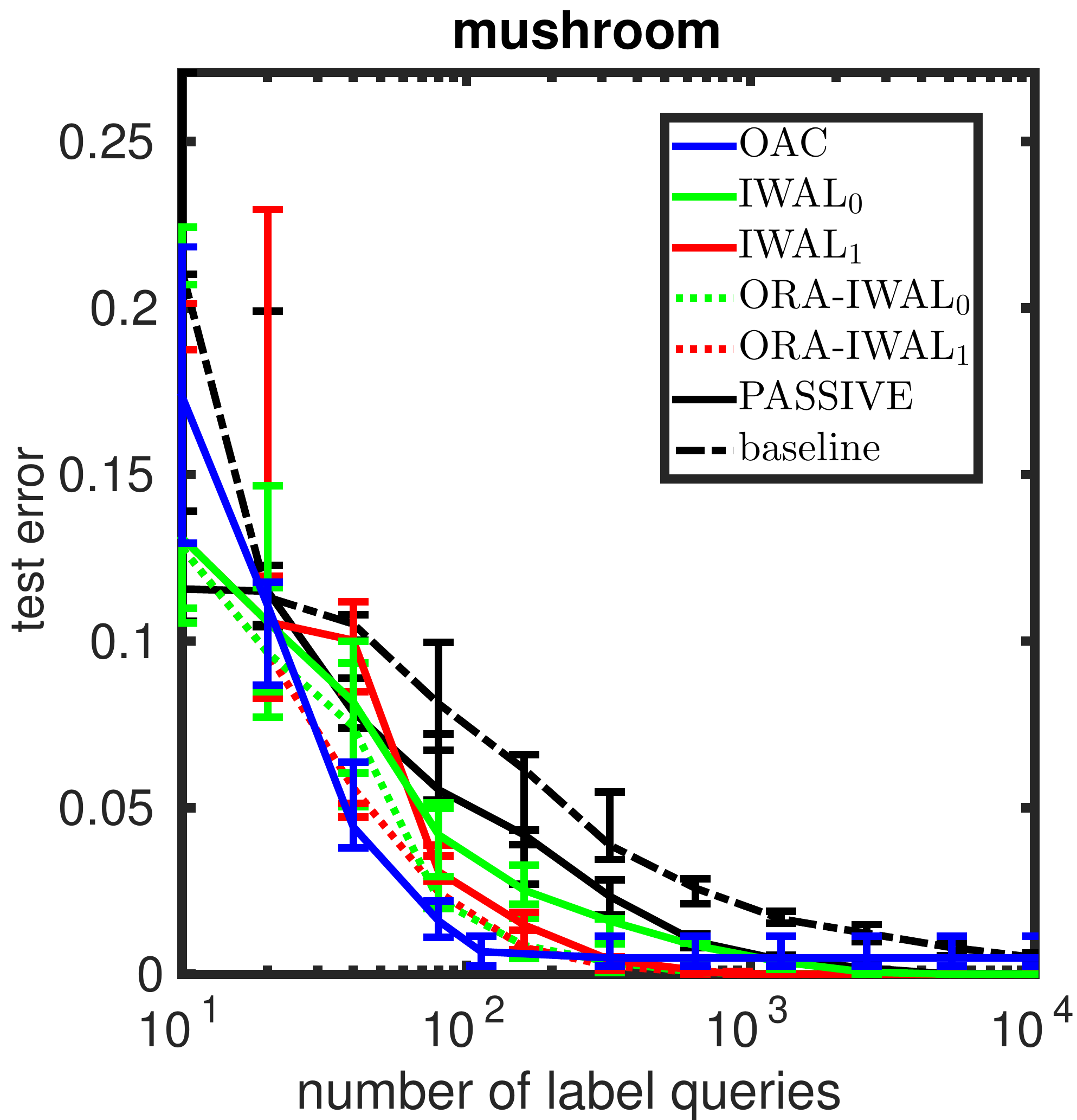}   \\
\includegraphics[width=0.33\textwidth,height=0.25\textwidth]{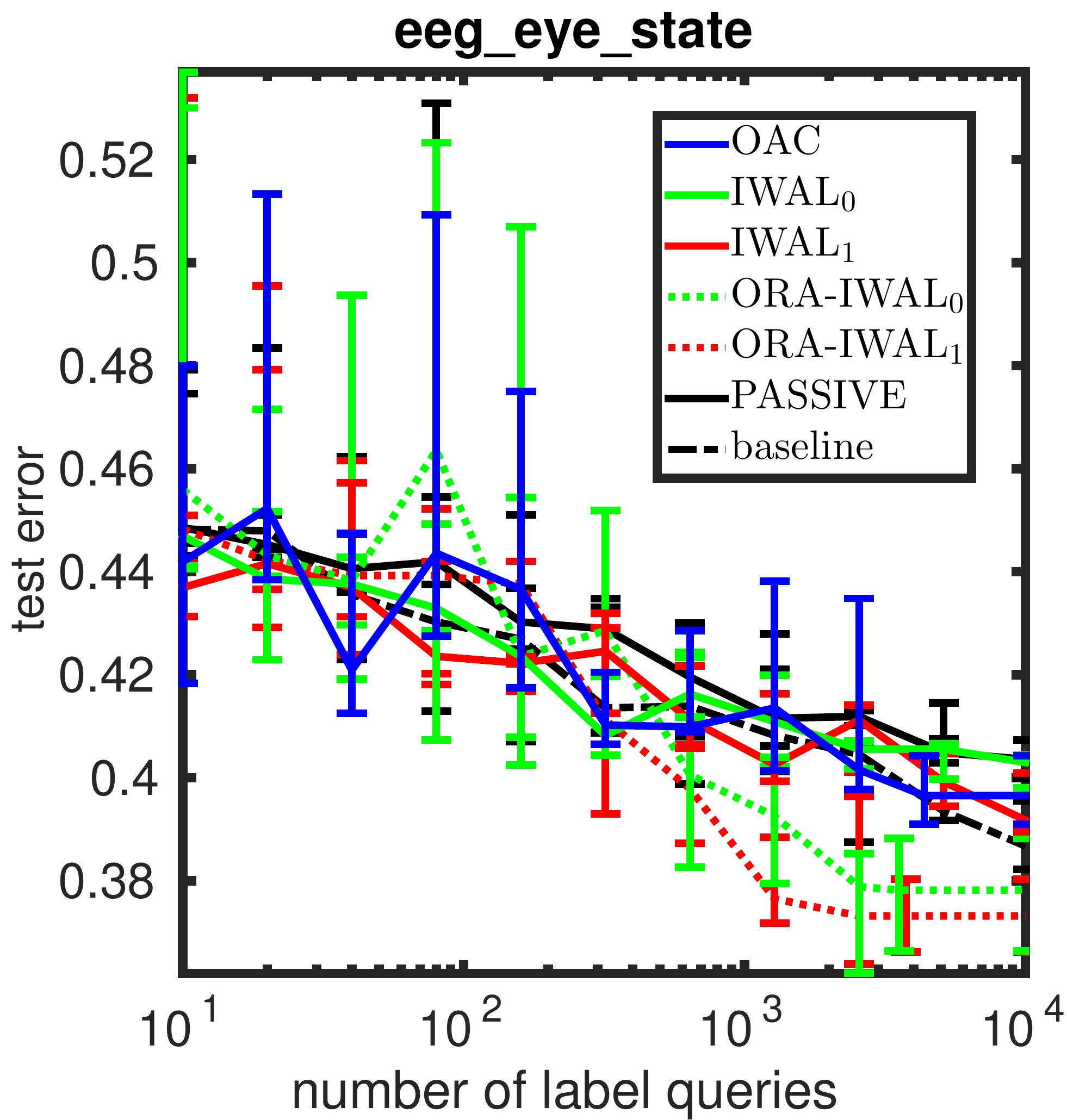}  & \includegraphics[width=0.33\textwidth,height=0.25\textwidth]{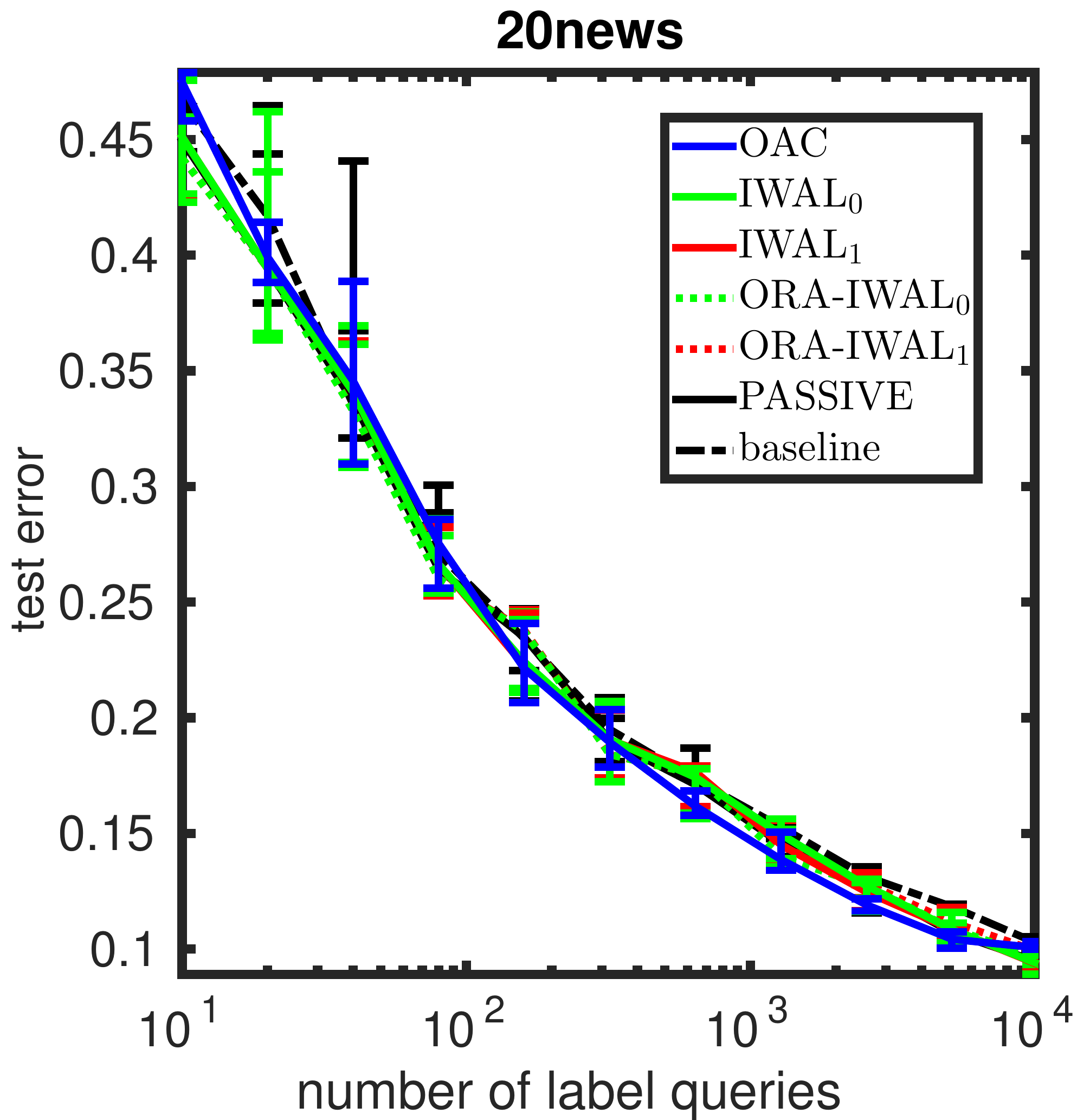}  & \includegraphics[width=0.33\textwidth,height=0.25\textwidth]{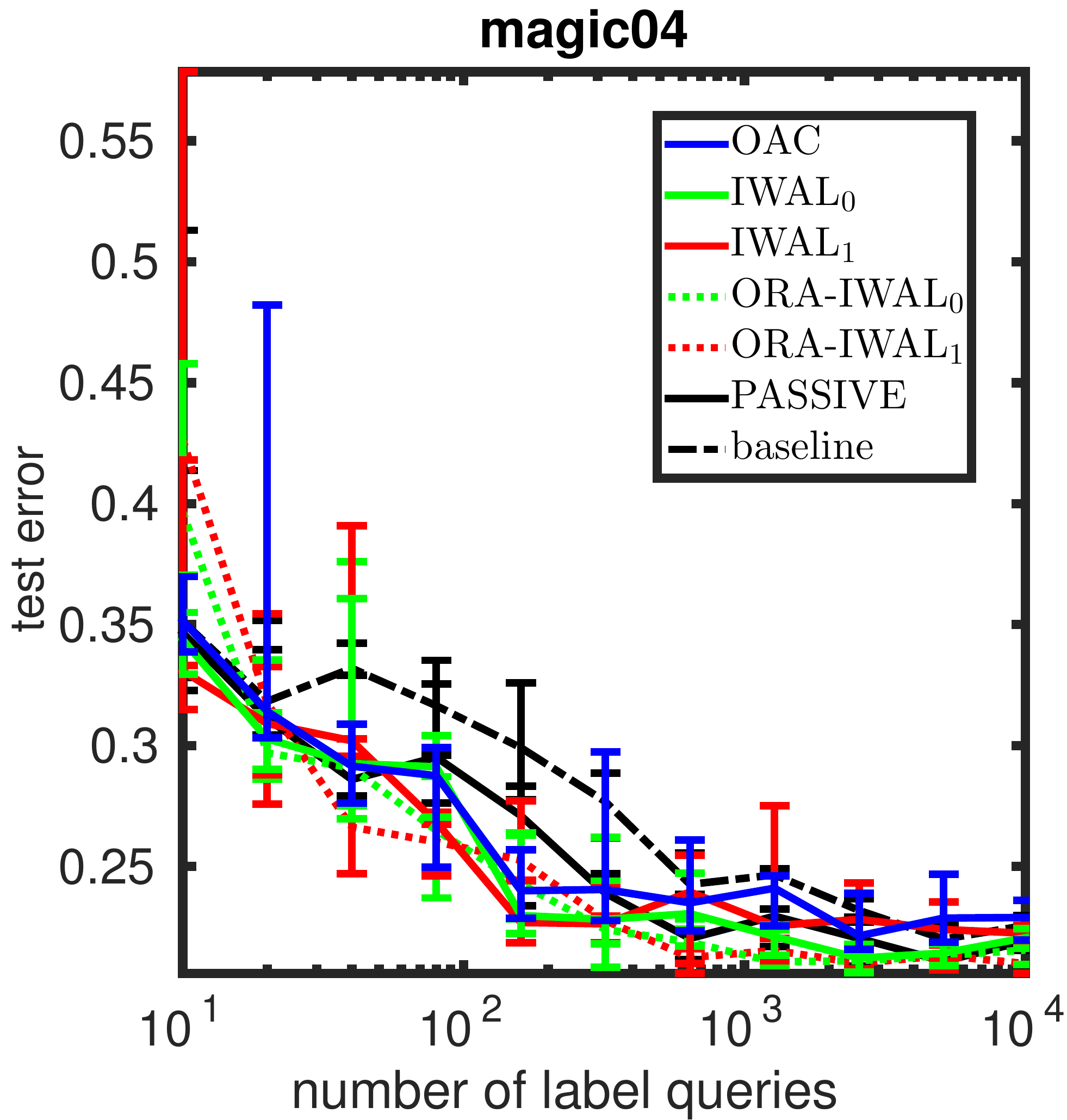}   \\
\includegraphics[width=0.33\textwidth,height=0.25\textwidth]{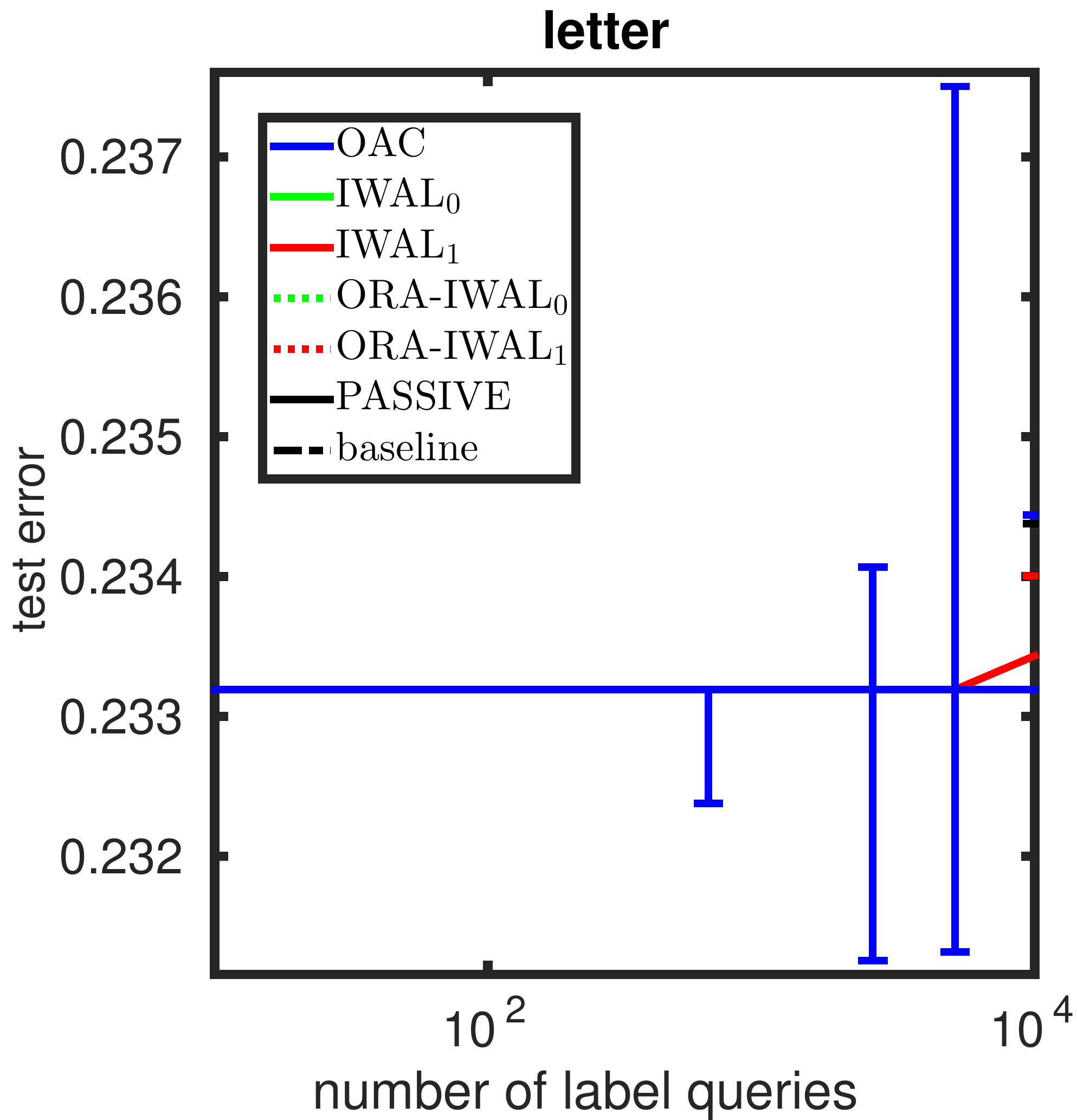}  & \includegraphics[width=0.33\textwidth,height=0.25\textwidth]{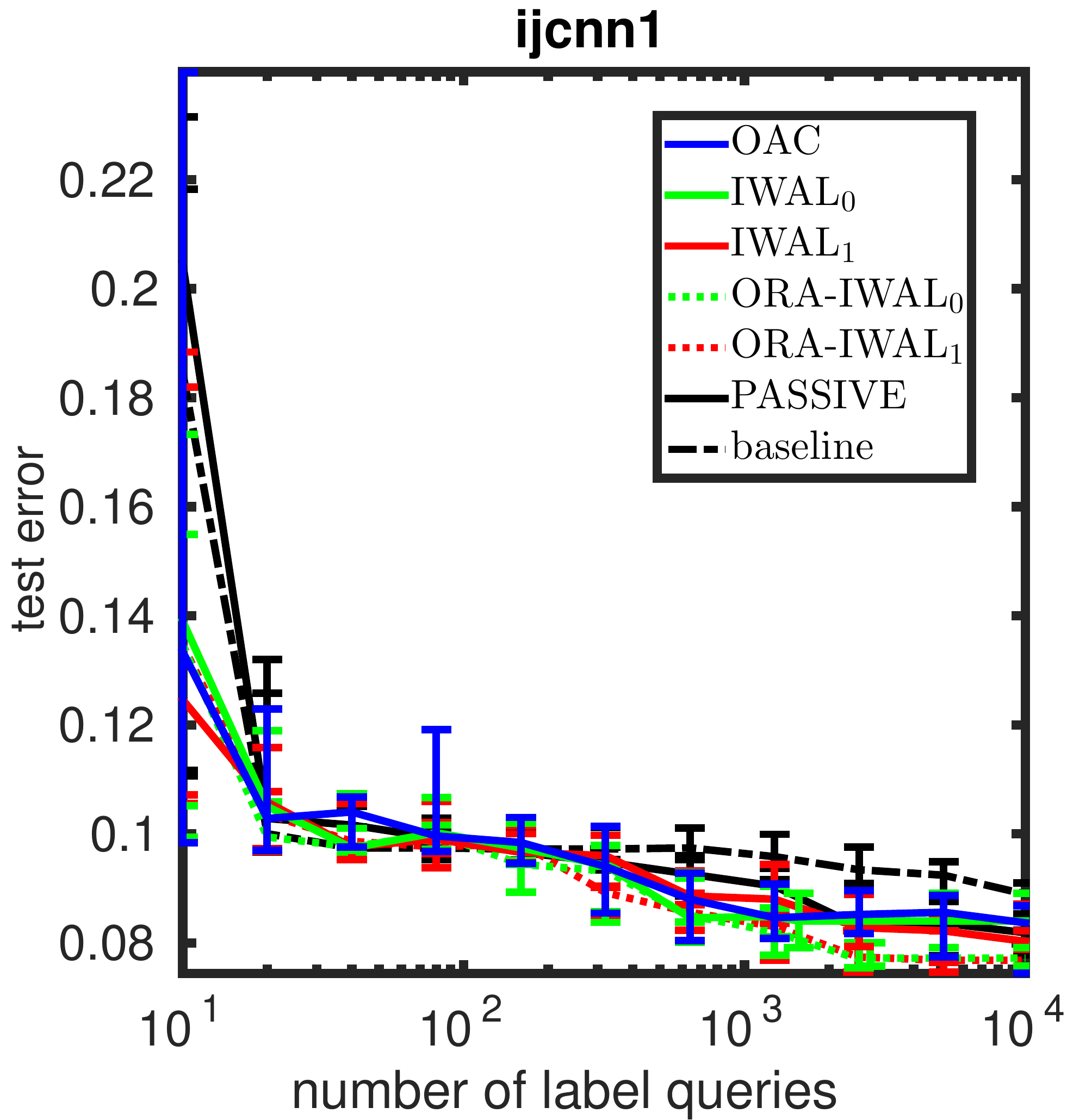}  & \includegraphics[width=0.33\textwidth,height=0.25\textwidth]{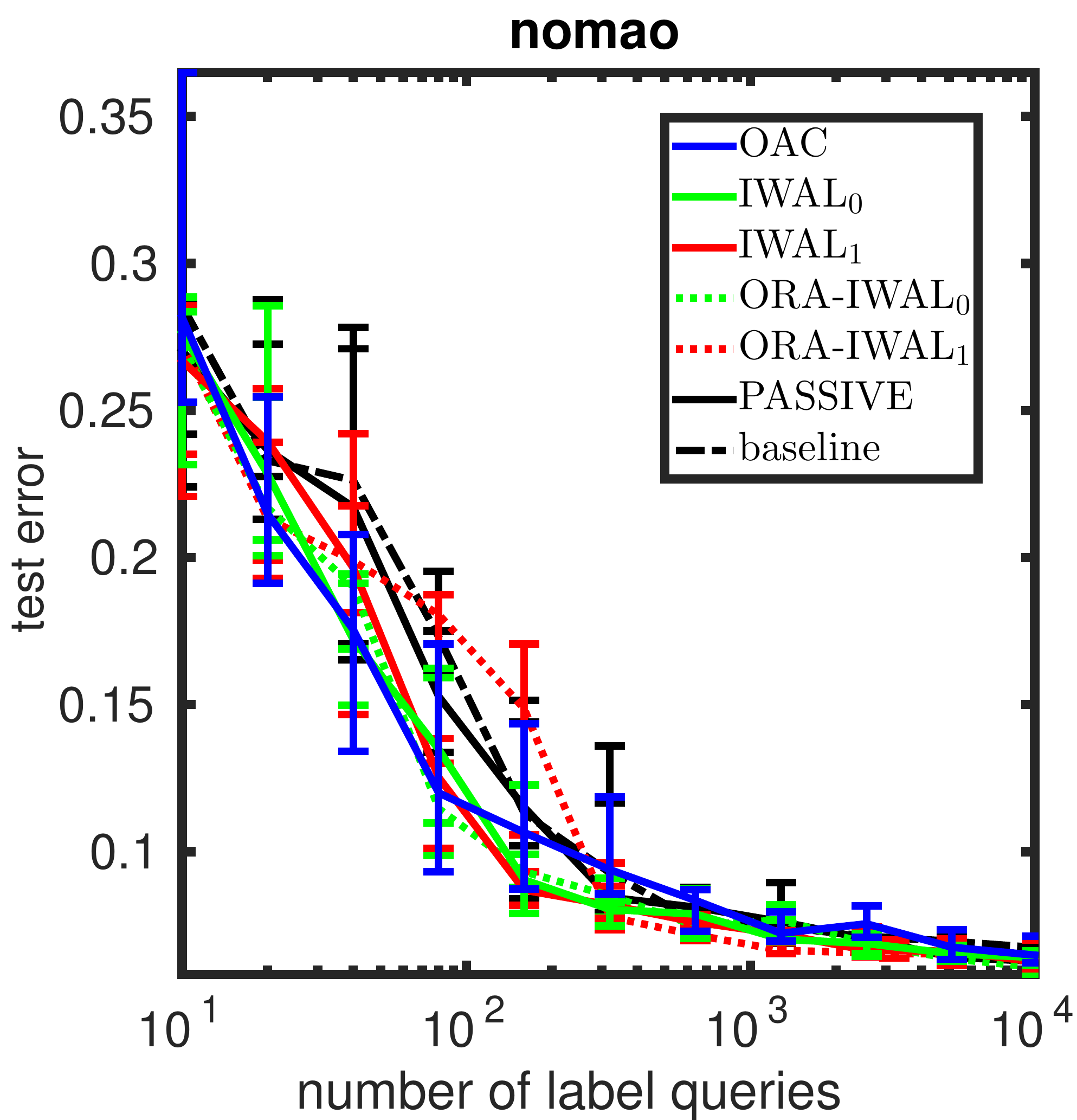}   \\
\includegraphics[width=0.33\textwidth,height=0.25\textwidth]{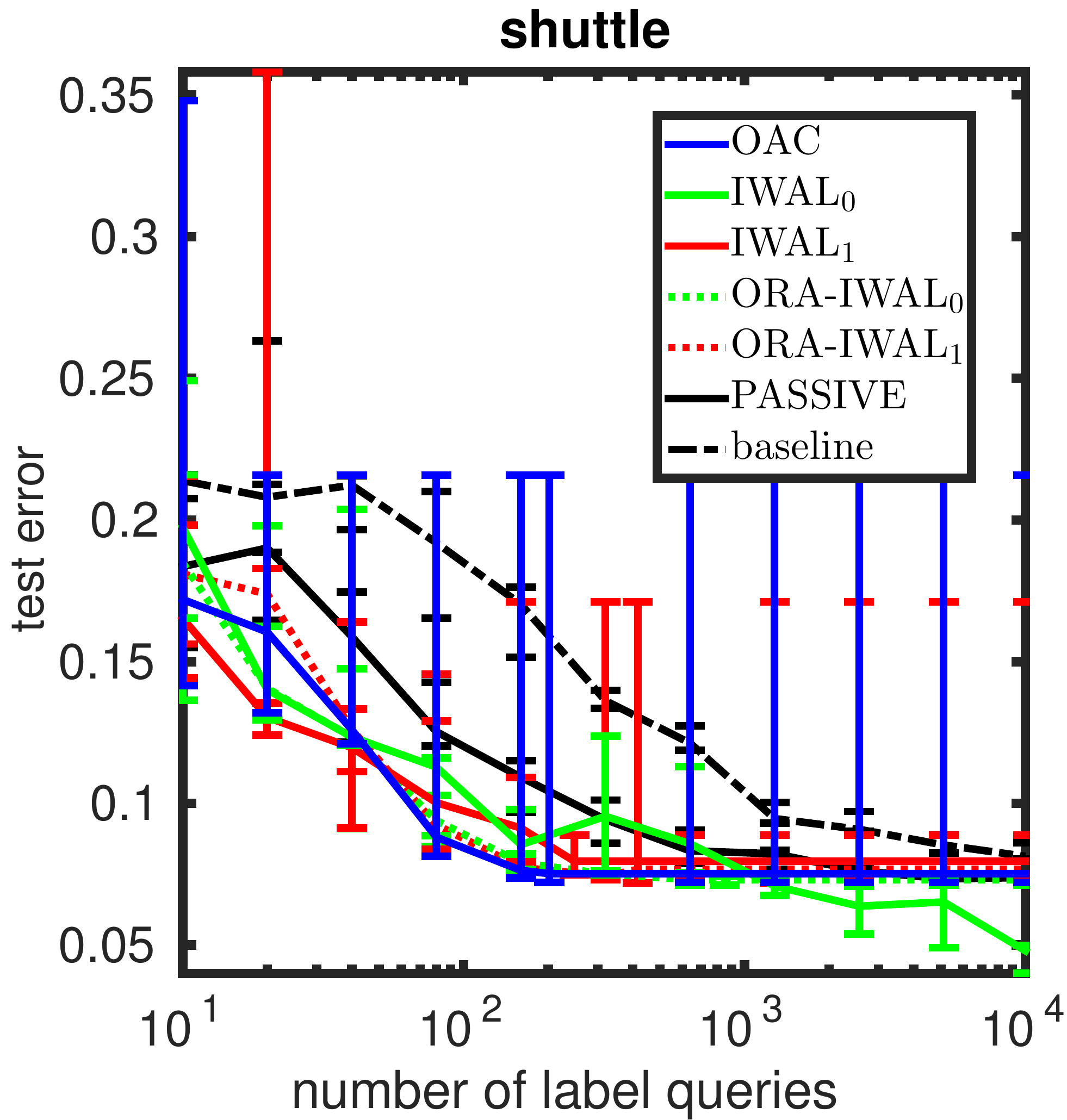}  & \includegraphics[width=0.33\textwidth,height=0.25\textwidth]{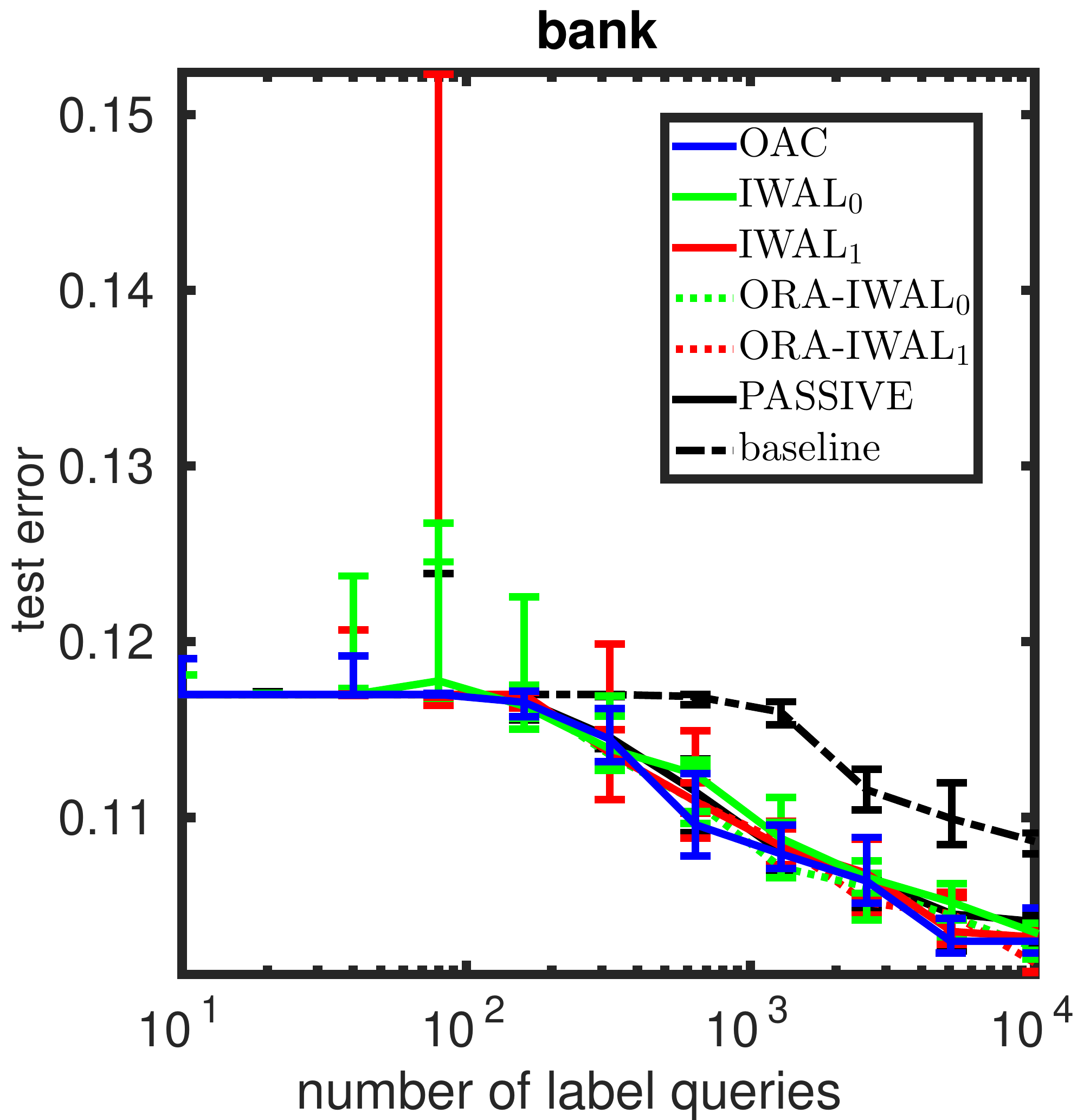}  & \includegraphics[width=0.33\textwidth,height=0.25\textwidth]{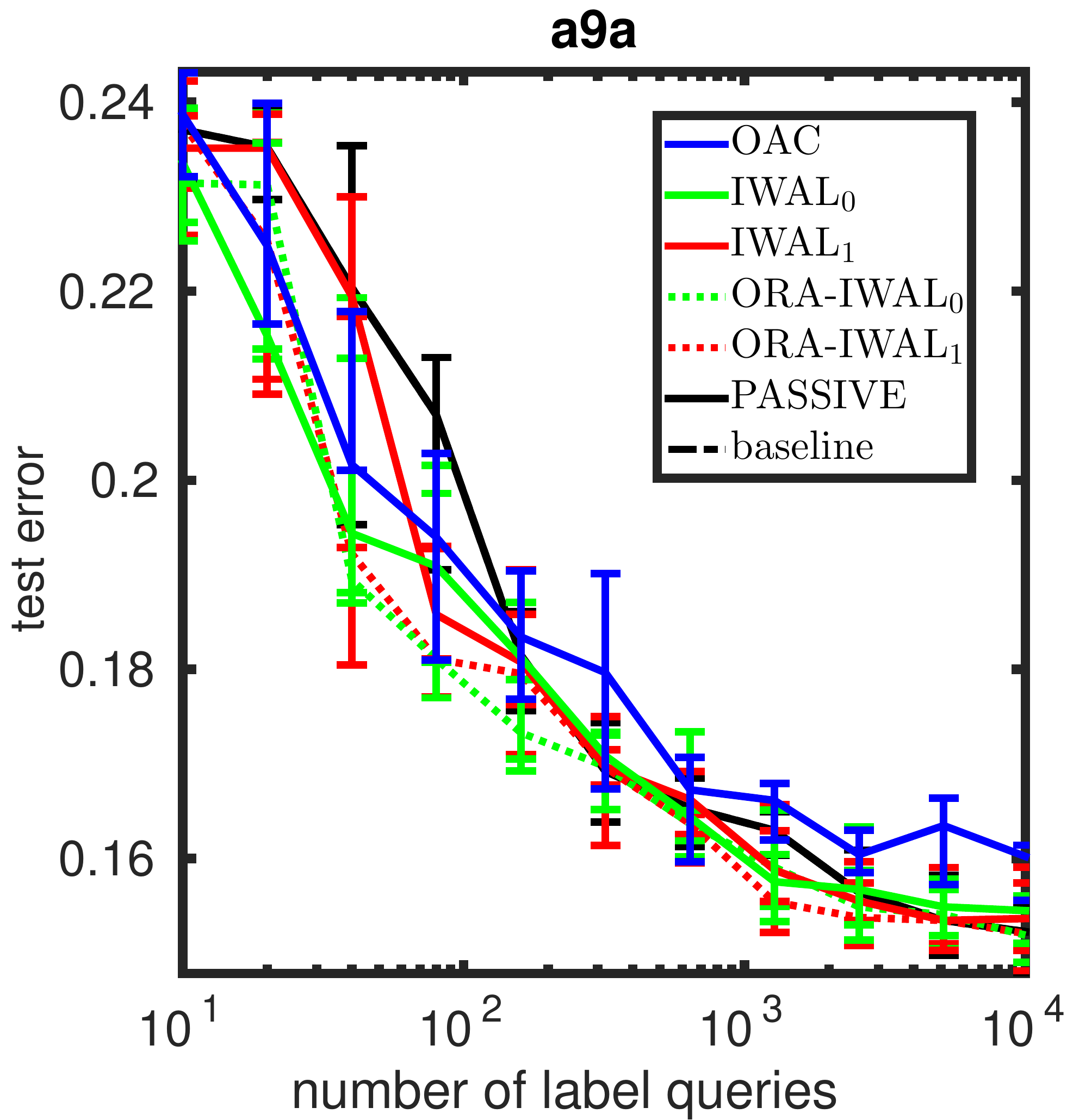} \\ 
\includegraphics[width=0.33\textwidth,height=0.25\textwidth]{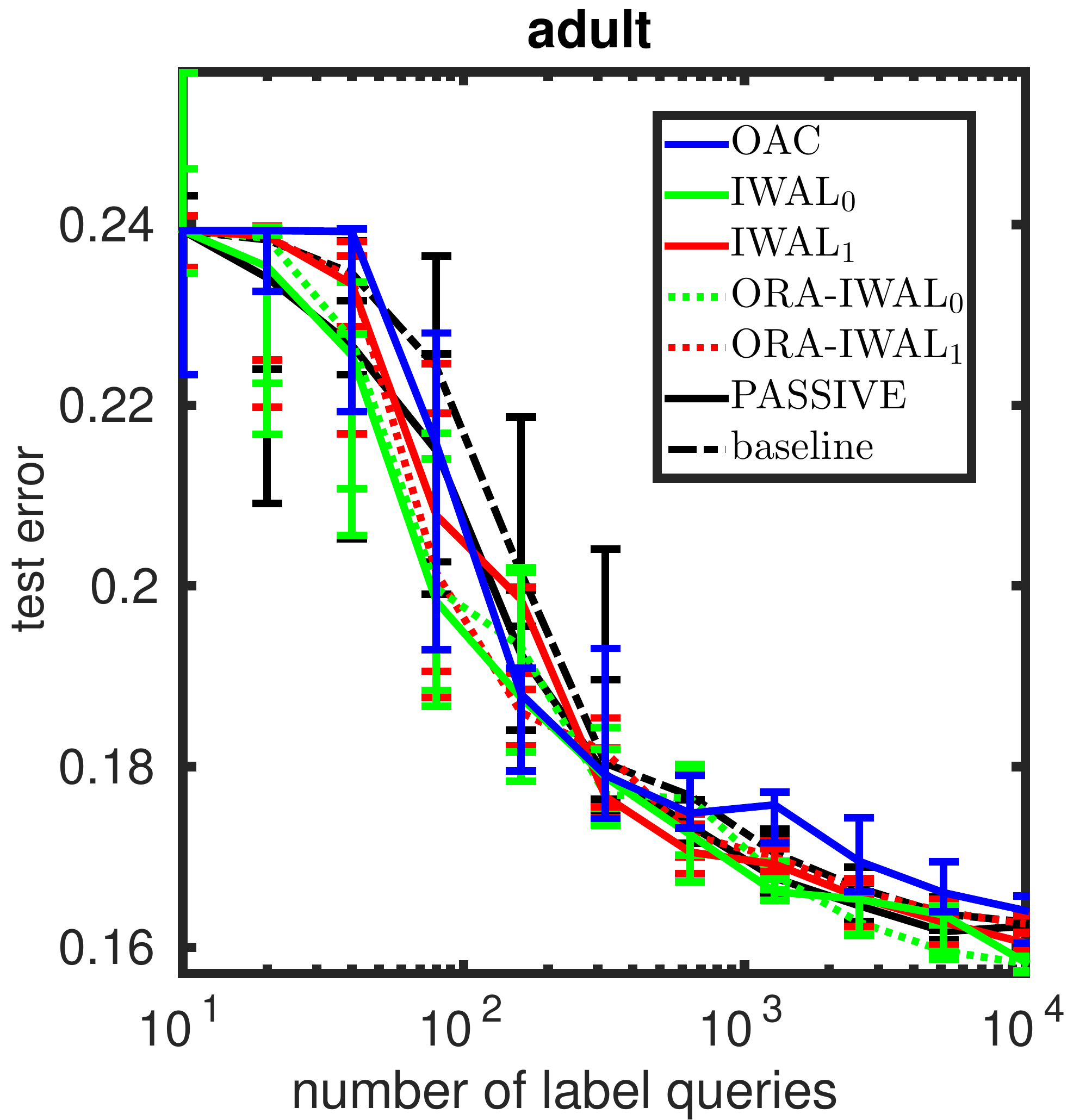}  & \includegraphics[width=0.33\textwidth,height=0.25\textwidth]{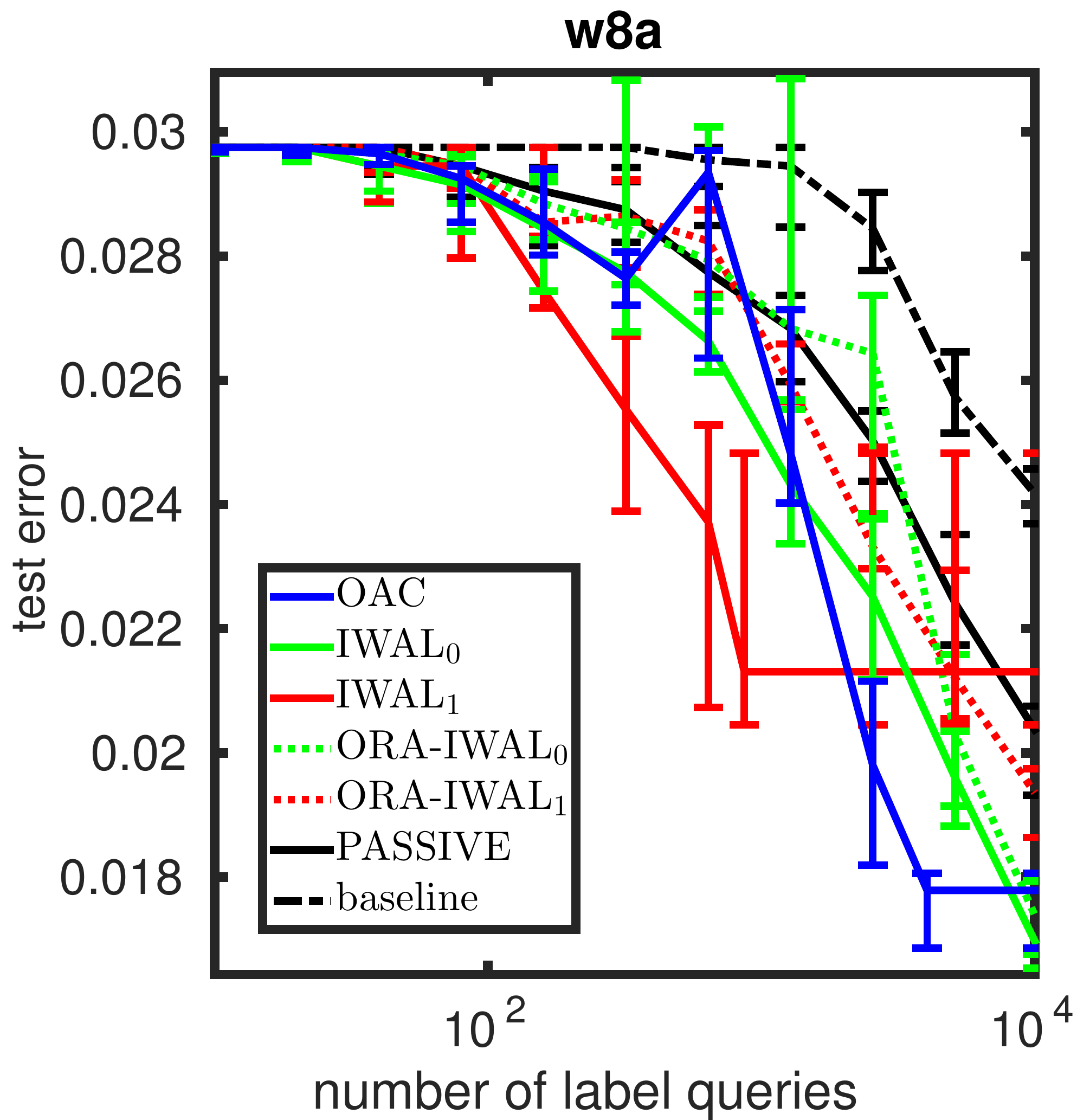}  &  
\end{tabular}
\caption{Test error under the best hyper-parameter setting for each dataset vs. number of label queries for datasets with fewer than $10^5$ examples}
\label{fig:err_vs_query_medium-per-data}
\end{figure}

\begin{figure}[t]
\begin{tabular}{@{}c@{}c@{}c@{}}
\includegraphics[width=0.33\textwidth,height=0.25\textwidth]{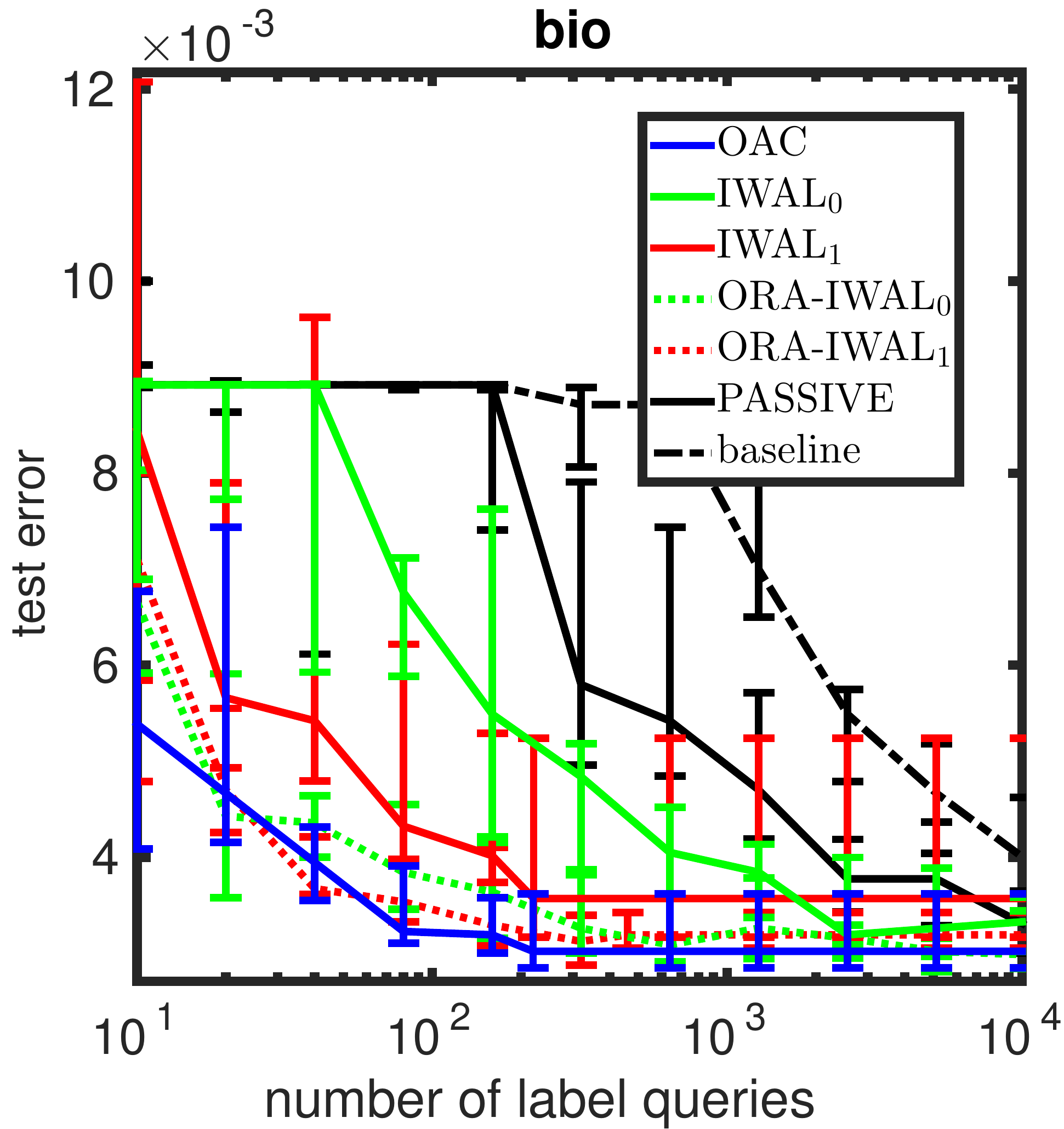} &  \includegraphics[width=0.33\textwidth,height=0.25\textwidth]{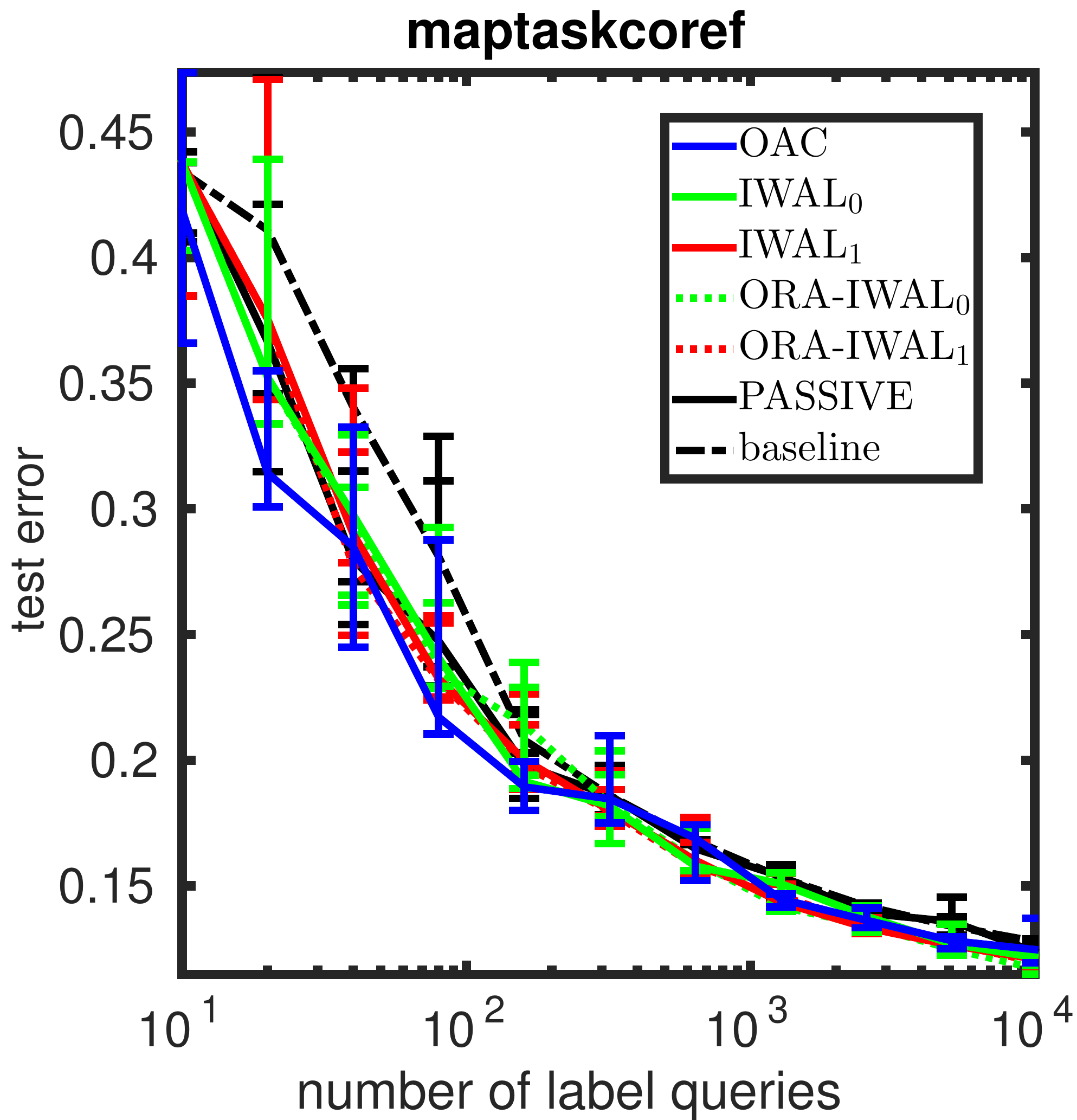}  &  \includegraphics[width=0.33\textwidth,height=0.25\textwidth]{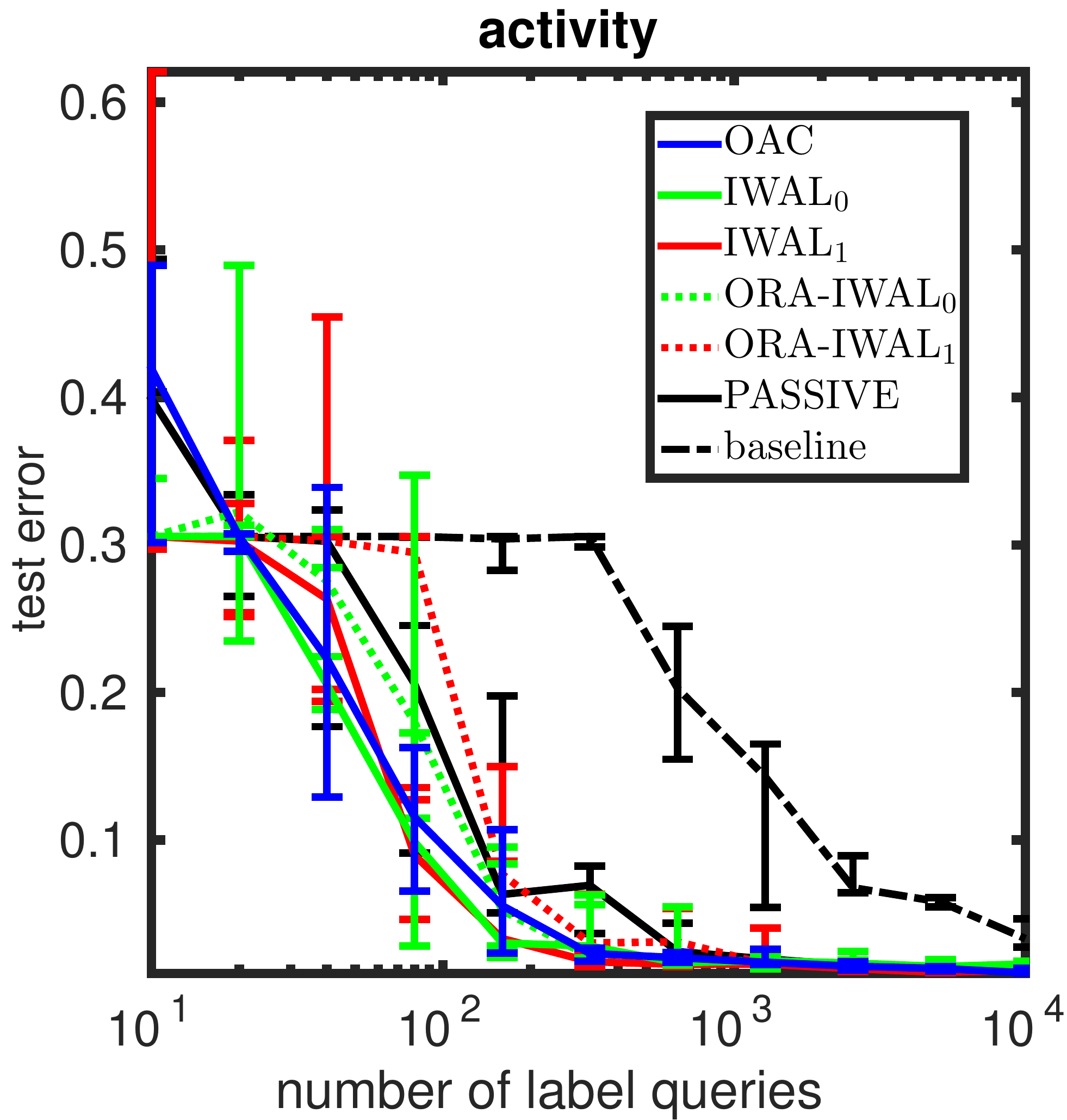} \\
\includegraphics[width=0.33\textwidth,height=0.25\textwidth]{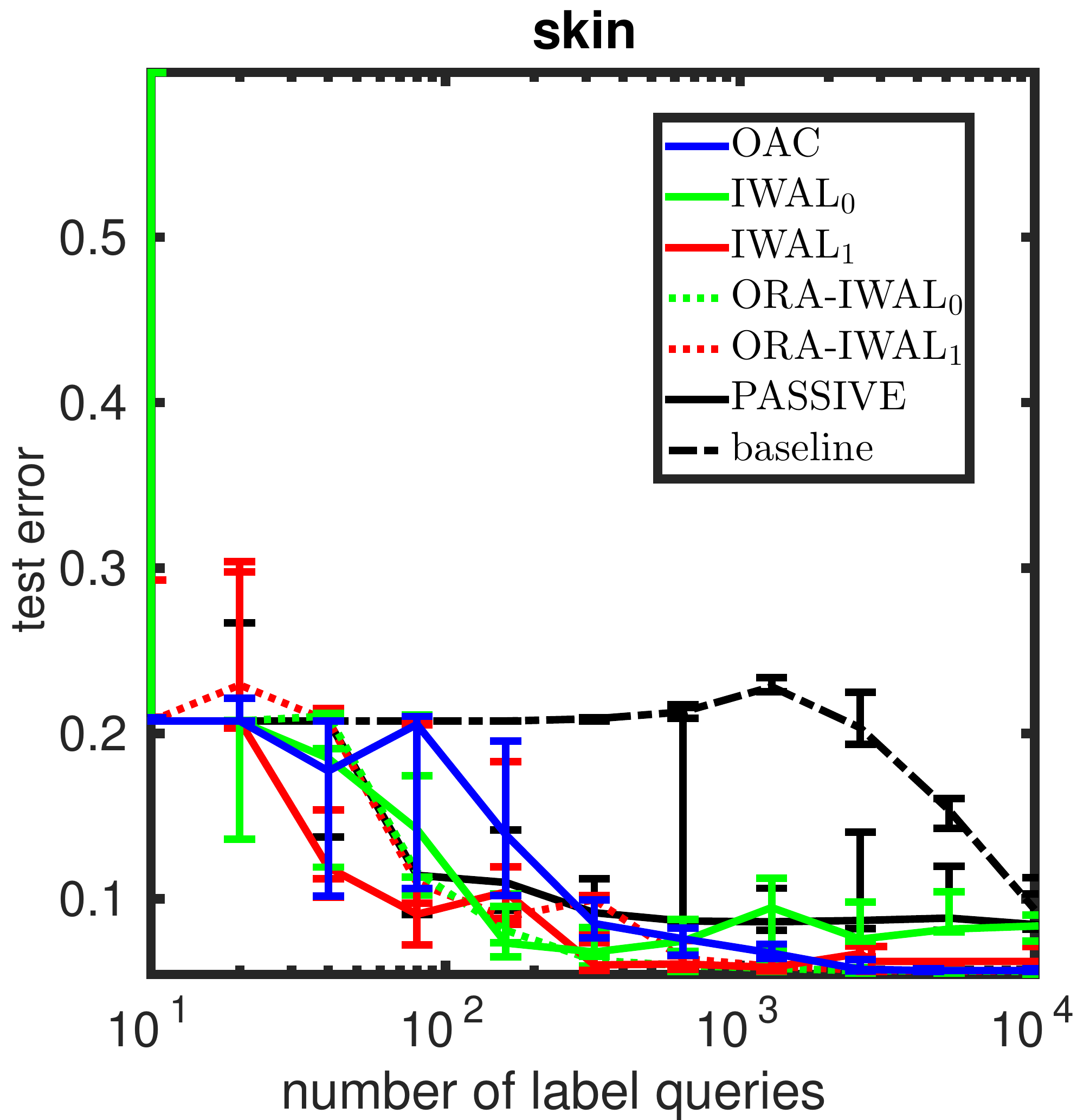} & \includegraphics[width=0.33\textwidth,height=0.25\textwidth]{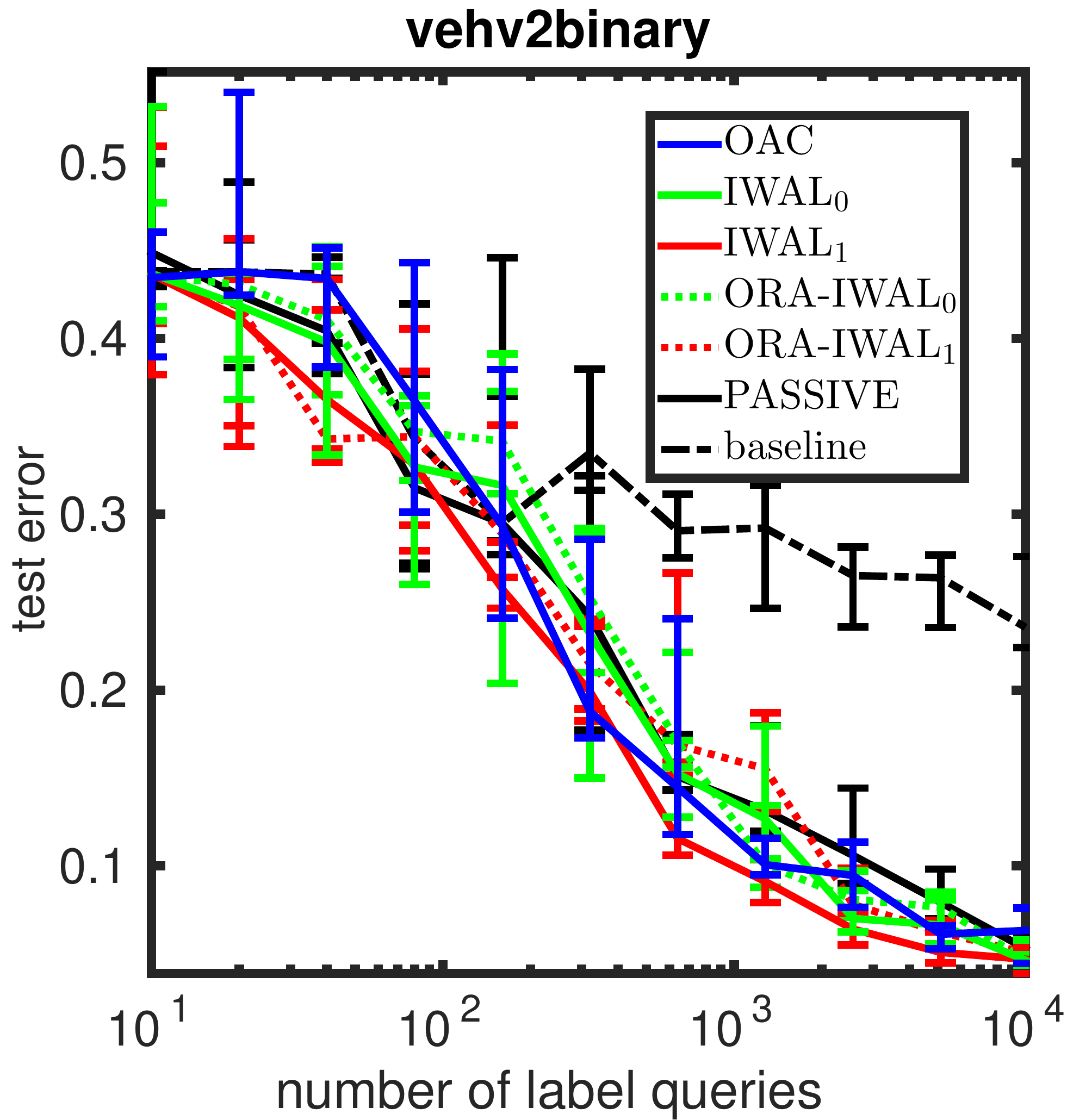}   &  \includegraphics[width=0.33\textwidth,height=0.25\textwidth]{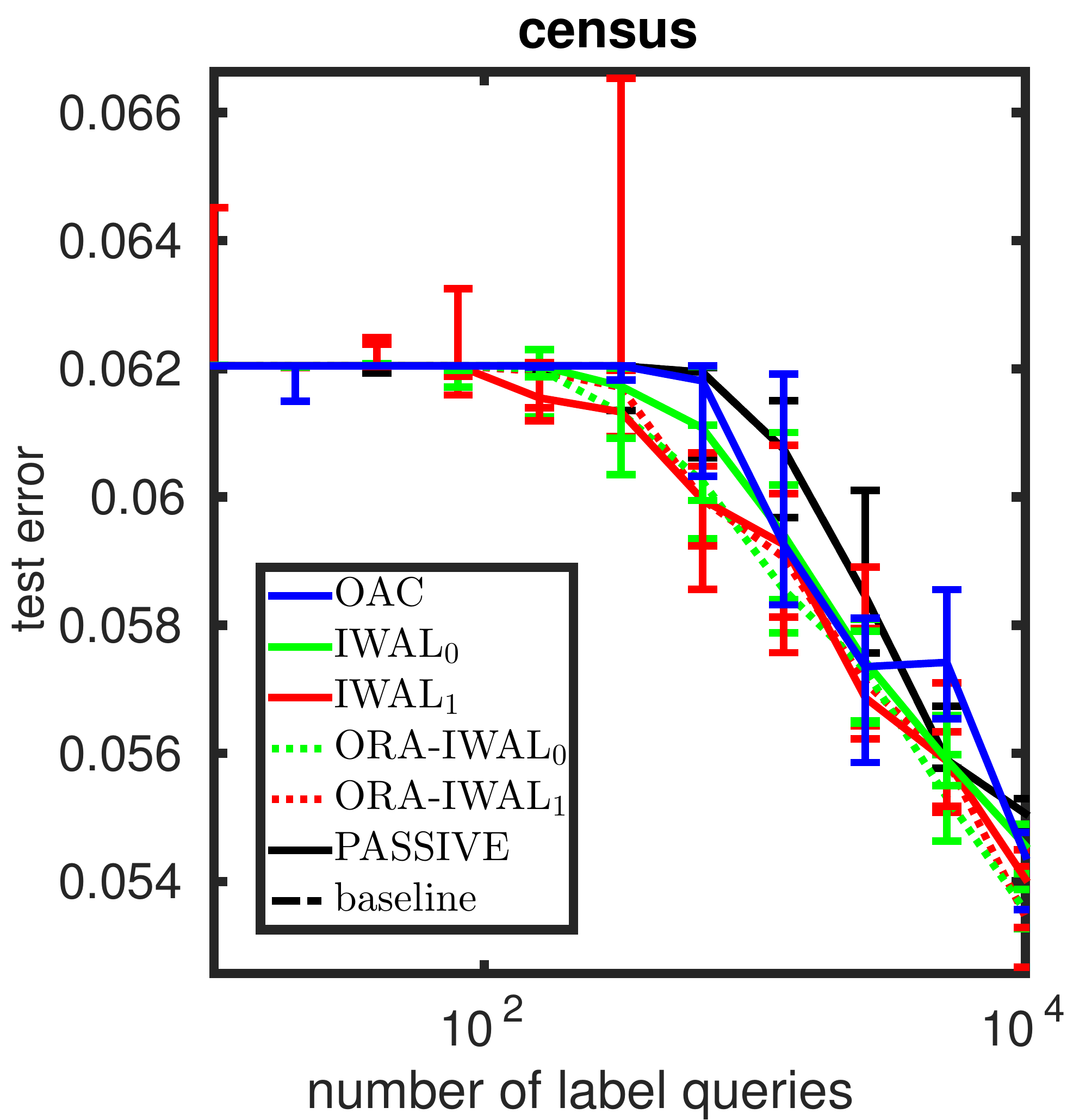}   \\ 
\includegraphics[width=0.33\textwidth,height=0.25\textwidth]{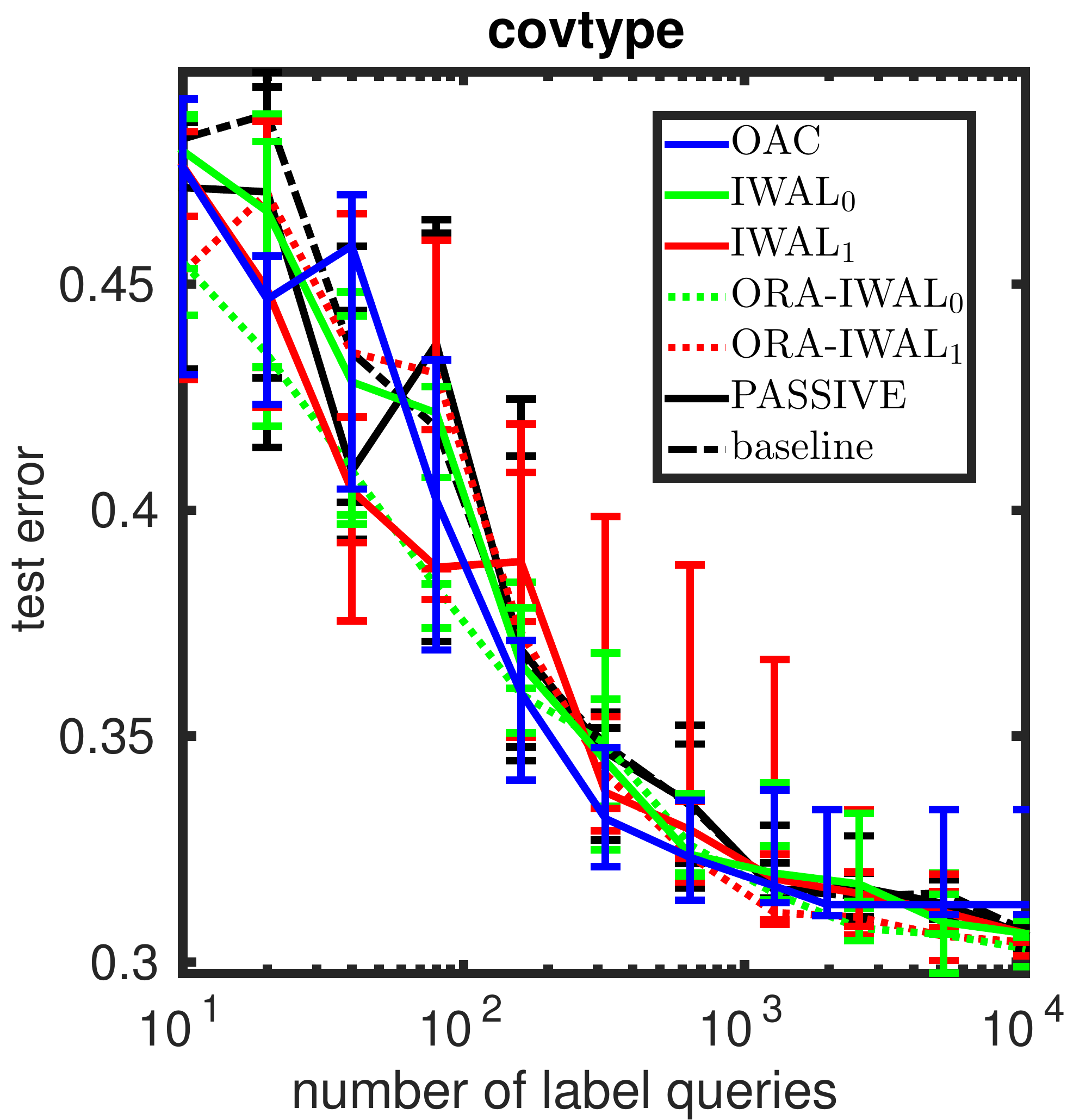}  &  \includegraphics[width=0.33\textwidth,height=0.25\textwidth]{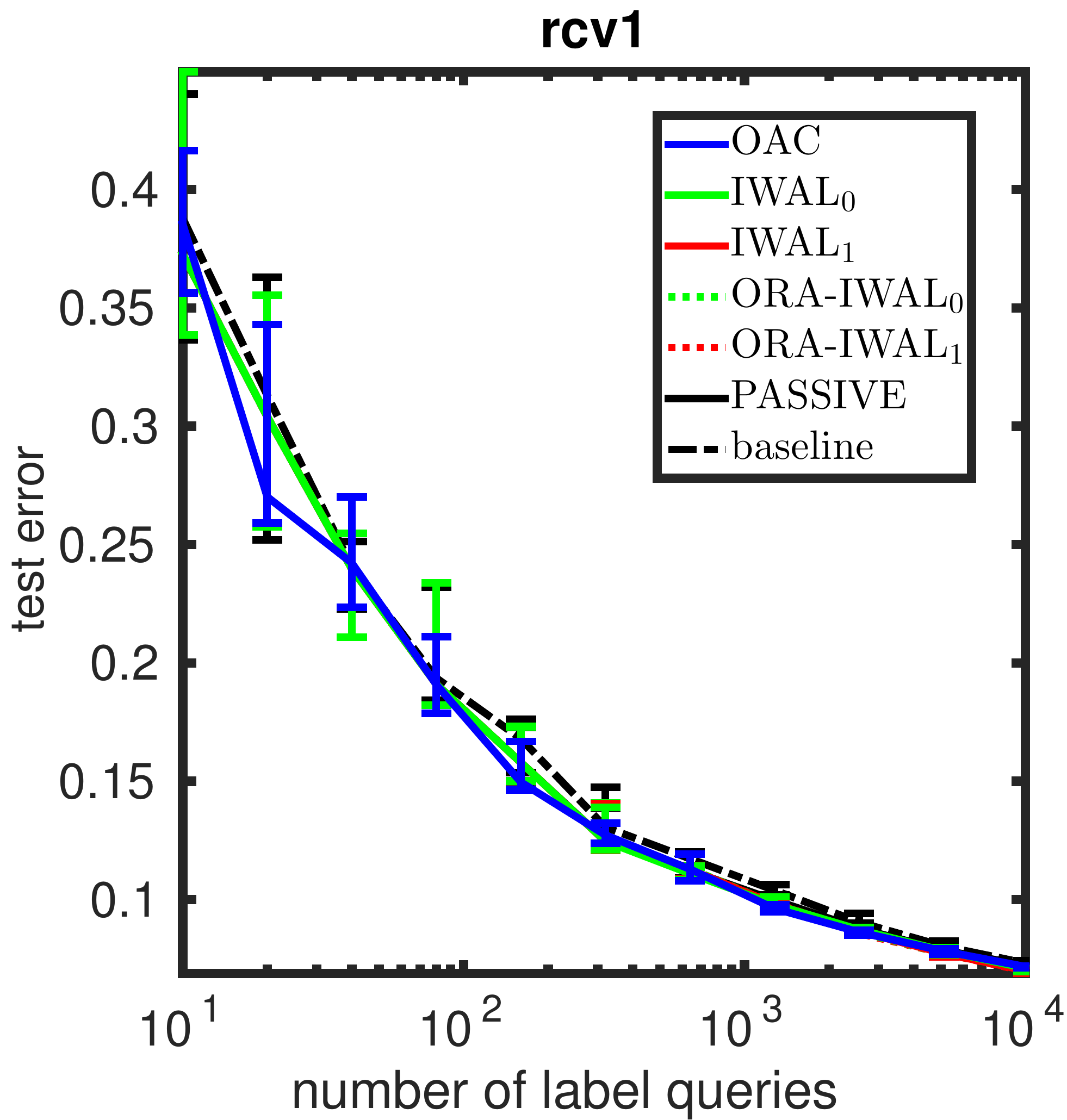} 
\end{tabular}
\caption{Test error under the best hyper-parameter setting for each dataset vs. number of label queries for datasets with more than $10^5$ examples}
\label{fig:err_vs_query_large-per-data}
\end{figure}

\end{document}